\documentclass{article}

\usepackage[final]{neurips_2023}

\usepackage[utf8]{inputenc} 
\usepackage[T1]{fontenc}    
\usepackage{hyperref}
\hypersetup{
    colorlinks,
    linkcolor={red!60!black},
    citecolor={blue!80!black},
    urlcolor={blue!80!black}
}
\usepackage{url}            
\usepackage{booktabs}       
\usepackage{amsfonts}       
\usepackage{nicefrac}       
\usepackage{microtype}      
\usepackage{xcolor}         

\usepackage[utf8]{inputenc}   
\usepackage[T1]{fontenc}      

\usepackage{comment}
\usepackage{caption}
\usepackage[tbtags]{amsmath}
\usepackage{amsthm}
\usepackage{bm}
\usepackage{amssymb,mathrsfs}
\usepackage{amsfonts}
\usepackage{upgreek}
\usepackage{graphicx}
\usepackage{float}
\usepackage{wrapfig}
\usepackage{booktabs}       
\usepackage{multirow}
\usepackage{color}
\usepackage{algorithm,algorithmic}

\usepackage{stmaryrd}

\usepackage{array}
\newcolumntype{?}{!{\vrule width 1.5pt}}

\usepackage{esvect}
\usepackage[inline]{enumitem}

\usepackage{graphicx} 
\usepackage{tikz}
\usepackage{pgfplots}
\usepackage{xcolor}
\usepackage{colortbl}
\colorlet{linkcolor}{blue!70!black}
\definecolor{pearDark}{HTML}{2980B9}
\usepackage{bbm}

\usepackage{ifthen}
\usepackage{xargs}

\usepackage[disable]{todonotes}
\usepackage{wrapfig}

\usepackage{aliascnt}
\usepackage{cleveref}
\usepackage{autonum}
\makeatletter
\newtheorem{theorem}{Theorem}
\crefname{theorem}{theorem}{Theorems}
\Crefname{Theorem}{Theorem}{Theorems}

\newtheorem*{lemma_nonumber*}{Lemma}

\newaliascnt{lemma}{theorem}
\newtheorem{lemma}[lemma]{Lemma}
\aliascntresetthe{lemma}
\crefname{lemma}{lemma}{lemmas}
\Crefname{Lemma}{Lemma}{Lemmas}

\newaliascnt{corollary}{theorem}
\newtheorem{corollary}[corollary]{Corollary}
\aliascntresetthe{corollary}
\crefname{corollary}{corollary}{corollaries}
\Crefname{Corollary}{Corollary}{Corollaries}

\newaliascnt{proposition}{theorem}
\newtheorem{proposition}[proposition]{Proposition}
\aliascntresetthe{proposition}
\crefname{proposition}{proposition}{propositions}
\Crefname{Proposition}{Proposition}{Propositions}

\newaliascnt{definition}{theorem}
\newtheorem{definition}[definition]{Definition}
\aliascntresetthe{definition}
\crefname{definition}{definition}{definitions}
\Crefname{Definition}{Definition}{Definitions}

\newaliascnt{remark}{theorem}

\aliascntresetthe{remark}
\crefname{remark}{remark}{remarks}
\Crefname{Remark}{Remark}{Remarks}

\crefname{figure}{figure}{figures}
\Crefname{Figure}{Figure}{Figures}

\newtheorem{assumption}{\textbf{A}\hspace{-3pt}} 
\Crefname{assumption}{\textbf{A}\hspace{-3pt}}{\textbf{A}\hspace{-3pt}}
\crefname{assumption}{\textbf{A}}{\textbf{A}}
\crefformat{assumption}{{\textbf{A}}#2#1#3}
\crefname{assumption}{assumption}{assumptions}
\Crefname{Assumption}{Assumption}{Assumptions}

\newtheorem{assumptionF}{\textbf{F}\hspace{-3pt}}
\crefformat{assumptionF}{{\textbf{F}}#2#1#3}

\Crefname{assumptionB}{\textbf{B}\hspace{-3pt}}{\textbf{B}\hspace{-3pt}}
\crefname{assumptionB}{\textbf{B}}{\textbf{B}}

\Crefname{assumptionC}{\textbf{C}\hspace{-3pt}}{\textbf{C}\hspace{-3pt}}
\crefname{assumptionC}{\textbf{C}}{\textbf{C}}

\Crefname{assumptionH}{\textbf{H}\hspace{-3pt}}{\textbf{H}\hspace{-3pt}}
\crefname{assumptionH}{\textbf{H}}{\textbf{H}}

\Crefname{assumptionT}{\textbf{T}\hspace{-3pt}}{\textbf{T}\hspace{-3pt}}
\crefname{assumptionT}{\textbf{T}}{\textbf{T}}

\Crefname{assumptionT}{\textbf{T}\hspace{-3pt}}{\textbf{T}\hspace{-3pt}}
\crefname{assumptionT}{\textbf{T}}{\textbf{T}}

\Crefname{assumptionL}{\textbf{L}\hspace{-3pt}}{\textbf{L}\hspace{-3pt}}
\crefname{assumptionL}{\textbf{L}}{\textbf{L}}

\Crefname{assumptionQ}{\textbf{Q}\hspace{-3pt}}{\textbf{Q}\hspace{-3pt}}
\crefname{assumptionQ}{\textbf{Q}}{\textbf{Q}}

\Crefname{assumptionAR}{\textbf{AR}\hspace{-3pt}}{\textbf{AR}\hspace{-3pt}}
\crefname{assumptionAR}{\textbf{AR}}{\textbf{AR}}

\makeatletter
\newcommand*{\centerfloat}{%
  \parindent \z@
  \leftskip \z@ \@plus 1fil \@minus \textwidth
  \rightskip\leftskip
  \parfillskip \z@skip}
\makeatother

\usetikzlibrary{spy}
\usetikzlibrary[patterns] 
\usetikzlibrary{calc,decorations.pathreplacing} 
\usetikzlibrary{shadows.blur} 
\usetikzlibrary{shapes.symbols}
\def\mod{\ \mathrm{mod}}

\def \bfX {\mathbf{X}}
\def \bfB {\mathbf{B}}
\def \bfY {\mathbf{Y}}

\def\Qbb {\mathbb{Q}}
\def\Pbb {\mathbb{P}}

\def\argmin{\mathrm{argmin}}

\def\msp{\mathsf{P}}
\def\borel{\mathcal{B}}
\def\Pmeasure{\mathscr{P}}

\def\msd{\mathsf{D}}

\def\complementary{\mathrm{c}}
\def\msi{\mathsf{I}}
\def\msa{\mathsf{A}}

\def\mss{\mathsf{S}}
\def\msn{\mathsf{N}}
\def\mst{\mathsf{T}}
\def\msl{\mathsf{L}}

\def\msb{\mathsf{B}} 
\def\msc{\mathsf{C}}
\def\mse{\mathsf{E}}

\def\msv{\mathsf{V}}

\def\msp{\mathsf{P}}

\def\msx{\mathsf{X}}
\def\msz{\mathsf{Z}}

\def\rset{\mathbb{R}}

\def\nset{\mathbb{N}}
\def\nsets{\mathbb{N}^*}

\def\rmd{\mathrm{d}}

\def\rme{\mathrm{e}}

\def\rmC{\mathrm{C}}

\def\KL{\mathrm{KL}}
\def\ent{\mathrm{H}}

\newcommand{\abs}[1]{\left\vert #1 \right\vert}
\newcommand{\absLigne}[1]{\vert #1 \vert}

\newcommandx{\psr}[3][3=]{\left\langle#1,#2 \right\rangle_{#3}}
\newcommandx{\normr}[2][2=]{ \left\Vert#1 \right\Vert_{#2}}
\newcommandx{\psrLigne}[3][3=]{\langle#1,#2 \rangle_{#3}}
\newcommandx{\normrLigne}[2][2=]{ \Vert#1 \Vert_{#2}}

\newcommandx{\norm}[2][1=]{\ifthenelse{\equal{#1}{}}{\left\Vert #2 \right\Vert}{\left\Vert #2 \right\Vert^{#1}}}
\newcommand{\normLigne}[1]{\| #1 \|}

\newcommand\probaMarkovTilde[2][2=]
{\ifthenelse{\equal{#2}{}}{{\widetilde{\mathbb{P}}_{#1}}}{\widetilde{\mathbb{P}}_{#1}\left[ #2\right]}}

\newcommand{\plusinfty}{+\infty}
\def\ie{\textit{i.e.}}

\def\eqsp{\;}

\newcommand{\coint}[1]{\left[#1\right)}
\newcommand{\ocint}[1]{\left(#1\right]}

\newcommand{\ccint}[1]{\left[#1\right]}

\newcommand{\ccintLigne}[1]{[#1]}

\newcommand{\cball}[2]{\overline{\operatorname{B}}(#1,#2)}

\newcommand\sequence[3][2=,3=]
{\ifthenelse{\equal{#3}{}}{\ensuremath{\{ #1_{#2}\}}}{\ensuremath{\{ #1_{#2}, \eqsp #2 \in #3 \}}}}
\newcommand\sequenceD[3][2=,3=]
{\ifthenelse{\equal{#3}{}}{\ensuremath{\{ #1_{#2}\}}}{\ensuremath{( #1)_{ #2 \in #3} }}}

\newcommand\sequenceDouble[4][3=,4=]
{\ifthenelse{\equal{#3}{}}{\ensuremath{\{ (#1_{#3},#2_{#3}) \}}}{\ensuremath{\{  (#1_{#3},#2_{#3}), \eqsp #3 \in #4 \}}}}

\def\Id{\mathrm{Id}}
\def\Idd{\mathrm{I}_d}

\newcommand{\ensembleLigne}[2]{\{#1\,:\eqsp #2\}}

\def\card{\operatorname{card}}
\def\path{\operatorname{path}}

\def\diag{\operatorname{diag}}

\def\vareps{\varepsilon}

\newcommand{\1}{\mathbbm{1}}

\DeclareMathOperator{\Var}{Var}
\DeclareMathOperator{\Cov}{Cov}
\DeclareMathOperator{\BWUVP}{BW_2^2-UVP}
\DeclareMathOperator{\BW}{BW_2^2}

\def\transpose{\top}

\newcommand{\beq}{\begin{equation}}
\newcommand{\eeq}{\end{equation}}

\def\Leb{\mathrm{Leb}}

\def\Pmeasurell{\mathscr{P}^{(\ell)}}

\title{Tree-Based Diffusion Schrödinger Bridge \\ with Applications to Wasserstein Barycenters}

\author{%
  Maxence Noble\thanks{Corresponding author. Contact at: maxence.noble-bourillot@polytechnique.edu.} \\
  CMAP, CNRS, École polytechnique, \\
  Institut Polytechnique de Paris, \\
91120 Palaiseau, France  
  \And
  Valentin De Bortoli\\
  Computer Science Department, \\
ENS, CNRS, PSL University
  \And
  Arnaud Doucet \\
Department of Statistics, \\ 
University of Oxford, UK  
  \And
  Alain Oliviero Durmus \\
  CMAP, CNRS, École polytechnique, \\
  Institut Polytechnique de Paris, \\
91120 Palaiseau, France    
}

\begin{document}

\maketitle

\begin{abstract}
  Multi-marginal Optimal Transport (mOT), a generalization of OT, aims at minimizing the integral of a cost function with respect to a distribution with some prescribed marginals. In this paper, we consider an entropic version of mOT
  with a tree-structured quadratic cost, i.e., a function that can be written as
  a sum of pairwise cost functions between the nodes of a tree. To address this
  problem, we develop Tree-based Diffusion Schr\"odinger Bridge (TreeDSB), an
  extension of the Diffusion Schr\"odinger Bridge (DSB) algorithm. TreeDSB
  corresponds to a dynamic and continuous state-space counterpart of the
  multi-marginal Sinkhorn algorithm. A notable use case of our methodology is to
  compute Wasserstein barycenters which can be recast as the solution of a mOT
  problem on a star-shaped tree. We demonstrate that our methodology can be applied in high-dimensional settings such as image interpolation and
  Bayesian fusion.
\end{abstract}

\section{Introduction}\label{sec:intro}

In the last decade, computational Optimal Transport (OT) has shown great success
with applications in various fields such as biology
\citep{schiebinger2019optimal,bunne2022proximal}, shape correspondence
\citep{su2015optimal,feydy2017optimal,eisenberger2020deep}, control theory
\citep{bayraktar2018martingale,acciaio2019extended} and computer vision
\citep{schmitz2018wasserstein,carion2020end}. While OT commonly seeks at computing the transport plan that minimizes the cost of moving between two distributions, it can naturally be extended to the multi-marginal setting (mOT) when considering several distributions. This extension of OT has notably been studied in quantum chemistry \citep{cotar2013density}, clustering \citep{cuturi2014fast} and statistical
inference \citep{srivastava2018scalable}. In particular, a popular application in unsupervised learning of mOT with Euclidean cost consists in computing the Wasserstein barycenter of a set of probability distributions \citep{agueh2011barycenters, benamou2015iterative, alvarez2016fixed, peyre2019computational}.

Interior point methods can be used to solve OT and mOT problems but they come
with computational challenges \citep{pele2009fast}. In order to mitigate these
limitations, one often considers an \emph{entropic regularization} of OT, known
as Entropic OT (EOT). This regularized formulation can be efficiently solved in discrete state-spaces using the celebrated \emph{Sinkhorn}
algorithm \citep{cuturi2013sinkhorn,knight2008sinkhorn,sinkhorn1967concerning}, which admits a continuous state-space counterpart referred to as the \emph{Iterative Proportional Fitting} (IPF) procedure \citep{fortet1940resolution,kullback1968probability,ruschendorf1995convergence}. In the case of a quadratic cost, EOT is equivalent to the \emph{static} formulation of the Schr\"odinger Bridge (SB) problem \citep{schrodinger1932theorie}. Given a reference diffusion with finite time horizon $T$ and two probability measures, solving SB amounts to finding the closest diffusion to the reference (in terms of Kullback–Leibler divergence on path spaces) with the given marginals at times $t=0$ and $t=T$. This framework naturally arises in stochastic control
\citep{dai1991stochastic} where one aims at controlling the marginal
distribution of a stochastic process at a fixed time. Recently,
\citet{debortoli2021diffusion} introduced Diffusion Schr\"odinger Bridge (DSB), an approximation of a \emph{dynamic} version of the IPF scheme on path spaces, see also \cite{vargas2021solving, chen2021likelihood}. This methodology leverages
advances in the field of denoising diffusion models
\citep{song2020score,ho2020denoising} in order to derive a
scalable and efficient scheme to solve SB, and thus EOT. 

Similarly to OT, mOT admits an entropic regularization (EmOT), which can be solved via a multi-marginal generalization of Sinkhorn/IPF algorithm \citep{benamou2015iterative, marino2020optimal}. Recently, \cite{haasler2021multimarginal} proposed an extension of the \emph{static} SB problem in \emph{discrete} state-space to any multi-marginal tree-based setting. They notably made the correspondence between this formulation and EmOT, when the cost function writes as the sum of interaction energies onto the given tree structure, and introduced an efficient version of Sinkhorn algorithm to solve it.

\paragraph{Motivations and contributions.} In this work, we investigate the \emph{continuous} and \emph{dynamic} counterpart of the tree-based framework from \cite{haasler2021multimarginal}. To be more specific, we present an extension of the static SB formulation in continuous state-space to any multi-marginal tree-based setting, referred to as TreeSB. Then, we establish the equivalence between TreeSB and a formulation of EmOT relying on a (quadratic) tree-structured cost function, analogously to \cite{haasler2021multimarginal}. Inspired by DSB, we develop TreeDSB, a dynamic counterpart of the multi-marginal IPF (mIPF) to solve it, by operating on path spaces and using score-based diffusion techniques. To bridge gaps in literature, we prove the convergence of mIPF iterations in a \emph{non-compact} setting under mild assumptions, by extending results on IPF convergence \citep{ruschendorf1995convergence}. Finally, we illustrate our approach on examples of Wasserstein barycenters from statistical inference and image processing.

Although our approach can be applied to any tree, we focus on \emph{star-shaped trees}. In this setting, we show that TreeSB reduces to a regularized Wasserstein barycenter problem. Our method comes with several benefits compared to existing works. First, it is out-of-sample, \ie, it does not require re-running the full procedure when given a new data point. Second, our formulation of the Wasserstein barycenter problem obtained from TreeSB allows us to avoid numerical issues of having to choose the regularization too small, see \Cref{sec:barycenter}. Finally, to the best of our knowledge, this is the first methodology to extend ideas from diffusion-based models to the computation of Wasserstein barycenters. In particular, we believe that the idea of iterative refinement, \ie, solving the \emph{dynamic} counterpart of a \emph{static} problem, plays a key role in the efficiency and scalability of the method. 

\paragraph{Notation.} For any measurable space $(\msx, \mathcal{X})$, we denote
by $\Pmeasure(\msx)$ the space of probability measures defined on
$(\msx, \mathcal{X})$. Unless specified, $\mathcal{X}$ is defined as the Borel sets on $\msx$. For any $\ell\in \nset$, let
$\Pmeasurell=\Pmeasure((\rset^d)^\ell)$; we denote $\Pmeasure^{(1)}$ by
$\Pmeasure$. Assume that $\msx=(\rset^d)^\ell$ for some $\ell\in \nset$. For any
$x \in \msx$ and any $m,n\in \{0, \hdots, \ell\}$ such that $m \leq n$, let
$x_{m:n}=(x_m, x_{m+1}, \hdots, x_n)$. Let $\Leb$ be the Lebesgue measure. For any non-negative function
$f : \msx \to \rset_+$, such that $\int_{\msx} f \rmd \Leb < \plusinfty$, define
$\ent(f) = -\int_{\msx}f\log f \rmd \Leb \in \ocint{-\infty,\plusinfty}$. For
any distribution $\mu \in \Pmeasure(\msx)$, we define the entropy of $\mu$ as
$\ent(\mu) = \ent(\rmd \mu/ \rmd \Leb)$ if $\mu \ll \Leb$ and
$\ent(\mu) = \plusinfty$ otherwise.  For any two arbitrary measures $\mu$ and
$\nu$ defined on $(\msx, \mathcal{X})$, define the Kullback--Leibler divergence
between $\mu$ and $\nu$ as
$\KL(\mu | \nu)= \int_{\msx} \log(\rmd \mu/\rmd \nu) \rmd \mu -\int_{\msx}\rmd
\mu + \int_{\msx}\rmd \nu$ if $\mu \ll \nu$ and $\KL(\mu \mid \nu)= \plusinfty$
otherwise. For any $T>0$, we denote by $\rmC(\ccintLigne{0,T}, \rset^d)$ the space of continuous functions
from $\ccintLigne{0,T}$ to $\rset^d$. For any path measure
$\Pbb \in \Pmeasure(\rmC(\ccintLigne{0,T}, \rset^d))$, we denote by
$\mathrm{Ext}(\Pbb) \in \Pmeasure^{(2)}$ the coupling between the \emph{extremal} distributions of $\Pbb$, \ie,
$\mathrm{Ext}(\Pbb) = \Pbb_{0,T}$. Note that, for a given coupling $\pi_{0,T}\in \Pmeasure^{(2)}$, there may exist several path measures $\Pbb$ verifying $\mathrm{Ext}(\Pbb) = \pi_{0,T}$. For any undirected tree $\mst=(\msv, \mse)$ with vertices $\msv$ and edges
$\mse$, we denote by $\{v,v'\}$ (or $\{v',v\}$) the undirected edge between $v\in \msv$ and
$v'\in \msv$, if it exists. Given $r\in \msv$, we denote by
$\mst_r=(\msv, \mse_r)$ the directed version of $\mst$ rooted in $r$, where the
directed edges $\mse_r$ are uniquely defined from the edges $\mse$, see
\Cref{sec:recap-trees} for further details. In this case, the edge linking
$v\in \msv$ to $v'\in \msv$ in $\mst_r$ is denoted by $(v,v')$. Finally, for any integers $(n,K) \in \nset\times \nset^*$, we define $n \mod(K)$ as the the remainder of the Euclidean division of $n$ by $K$.


\section{Background and setting}
\label{sec:backgr-optim-transp}

\paragraph{Multi-marginal optimal transport.}

Let $\ell \in \nset^*$. Given a cost function
$c: (\rset^d)^{\ell+1} \to \rset$, a subset $\mss\subset \{0, \hdots, \ell\}$ and a family of
probability measures $\{\mu_i\}_{i \in \mss}\in \Pmeasure^{\abs{S}}$, mOT consists in solving
\begin{align}\label{eq:mot}\tag{mOT}
  \textstyle 
  \pi^\star= \arg \min \left\{\int c(x_{0:\ell}) \rmd \pi (x_{0:\ell}): \pi \in \Pmeasure^{(\ell+1)}, \eqsp \pi_i = \mu_i\eqsp, \forall i \in \mss \right\} , 
\end{align}
where 
$\pi_i$ is the $i$-th marginal of $\pi$, \ie,
$\pi_i(\msa)=\pi(\text{proj}_i^{-1}(\msa))$ for any
$\msa \in \mathcal{B}(\rset^d)$, with $\text{proj}_i: x_{0:\ell} \mapsto
x_i$. Given some weights
$(w_i)_{i \in \{1, \hdots, \ell\}}\in (\rset_+)^{\ell}$, the Wasserstein barycenter
between the measures $\{\mu_i\}_{i \in \mss}$ is
given by $\pi^\star_{0}$ in \eqref{eq:mot}, in the case where
$\mss=\{1, \hdots, \ell\}$ and
$c(x_{0:\ell})=\sum_{i=1}^\ell w_{i} \| x_0 - x_i \|^2$ \citep{peyre2019computational}. In particular, when
$w_i=1/\ell$, the distribution $\pi^\star_0$ can be regarded as the Fr\'echet
mean \citep{karcher2014riemannian} of the measures $\{\mu_i\}_{i \in \mss}$ for
the Wasserstein distance of order 2. Similarly to OT, \eqref{eq:mot} can be relaxed using the following entropic regularization
\begin{align}\label{eq:mot_eps}\tag{EmOT}
  \textstyle 
    \textstyle{\pi^\star= \arg \min \left\{\int c(x_{0:\ell}) \rmd \pi (x_{0:\ell}) + \varepsilon \KL(\pi| \nu): \pi \in \Pmeasure^{(\ell+1)}, \eqsp \pi_i = \mu_i\eqsp, \forall i \in \mss \right\}} \eqsp ,
\end{align}
where $\vareps > 0$ is a hyperparameter
and $\nu$ is an arbitrary measure defined on $((\rset^d)^{\ell+1}, \mathcal{B}((\rset^d)^{\ell+1}))$.

\paragraph{Link with Schr\"odinger Bridge. }

We first recall the relationship between Schr\"odinger Bridge and
EOT. Given $T>0$, $\Qbb$ a (reference) path measure, \ie,
$\Qbb \in \Pmeasure(\rmC(\ccint{0,T}, \rset^d))$ and two measures
$\mu_0, \mu_1 \in \Pmeasure(\rset^d)$, solving the SB
problem amounts to finding the path measure $\Pbb^\star$ defined by
\begin{equation}
  \label{eq:dynamic_schrodinger_bridge}\tag{SB}
  \Pbb^\star = \argmin \ensembleLigne{\KL(\Pbb|\Qbb)}{\Pbb \in \Pmeasure(\rmC(\ccint{0,T}, \rset^d)), \ \Pbb_0 = \mu_0, \ \Pbb_T = \mu_1} \eqsp . 
\end{equation}
If $\Qbb$ is associated with a Stochastic Differential Equation
(SDE)\footnote{We refer to \Cref{sec:martingale-problems} for details on
  solutions of SDEs and associated measures.},
of the form $\rmd \bfX_t = -a \bfX_t \rmd t + \rmd \bfB_t$,  with
$a \geq 0$, then it can be shown, see \cite[Proposition 1]{leonard2014survey}
that $\Pbb^\star_{0,T}$ verifies
\begin{equation}
  \label{eq:static_schrodinger_bridge}\tag{static-SB}
  \textstyle \Pbb^\star_{0,T} = \argmin \ensembleLigne{\KL(\pi|\Qbb_{0,T})}{\pi \in \Pmeasure^{(2)}, \ \pi_0 = \mu_0, \ \pi_1=\mu_1} \eqsp . 
\end{equation}
This is called the \emph{static} formulation of SB.
It can be shown that solving \eqref{eq:static_schrodinger_bridge} is equivalent
to solving EOT with quadratic cost and regularization $\varepsilon=2 \sinh(aT)/a$ if $a>0$, $\varepsilon=2T$ if $a=0$. Moreover, since
$\Pbb^\star = \Pbb^\star_{0,T} \otimes \Qbb_{|0,T}$, where $\Qbb_{|0,T}$ is the
measure $\Qbb$ conditioned on initial and terminal conditions, solving the
\emph{dynamic} problem \eqref{eq:dynamic_schrodinger_bridge} is equivalent to
solving \eqref{eq:static_schrodinger_bridge}.

Similarly, \eqref{eq:mot_eps} can be easily rewritten in a \emph{static} multi-marginal SB
fashion
\begin{equation}
\tag{mSB-like}
  \label{eq:multimarginal_sb_static}
  \textstyle{\pi^\star  = \argmin \ensembleLigne{\KL(\pi|\pi^0)}{\pi \in \Pmeasure^{(\ell+1)}, \eqsp \pi_i = \mu_i\eqsp, \forall i \in \mss}} \eqsp ,
  \end{equation}
  with
  $\textstyle{(\rmd \pi^0 / \rmd \Leb)(x_{0:\ell}) \propto
    \exp[-c(x_{0:\ell})/\vareps]} (\rmd \nu / \rmd
  \Leb)(x_{0:\ell})$, 
  where $\pi^0$ is the \emph{reference} measure.

\paragraph{Diffusion Schr\"odinger Bridge.} Recently,
\citet{debortoli2021diffusion} introduced Diffusion Schr\"odinger Bridge (DSB),
a numerical scheme to solve \eqref{eq:dynamic_schrodinger_bridge}. It approximates the iterates of a \emph{dynamic} version of the \emph{Iterative Proportional Fitting} (IPF) scheme
\citep{sinkhorn1967concerning,knight2008sinkhorn,peyre2019computational,cuturi2014fast}, which can be described as follows: consider a sequence of path measures
$(\Pbb^n)_{n \in \nset}$ such that $\Pbb^0 = \Qbb$ and for any $n \in \nset$
\begin{equation}
  \Pbb^{2n+1} = \argmin \ensembleLigne{\KL(\Pbb|\Pbb^{2n})}{\Pbb_T = \mu_1} , \qquad \Pbb^{2n+2} = \argmin \ensembleLigne{\KL(\Pbb|\Pbb^{2n+1})}{\Pbb_0 = \mu_0} \eqsp .
\end{equation}
This procedure alternatively projects between the measures with fixed initial
distribution and the ones with fixed terminal distribution. For the first
iteration, we get that $\Pbb^1 = \mu_1 \otimes \Qbb_{|T}$. Assuming that $\Qbb$
is given by $\rmd \bfX_t = f_t(\bfX_t) \rmd t + \rmd \bfB_t$, with
$f: \ \ccint{0,T} \times \rset^d \to \rset^d$, then $\Pbb^1$ is associated with
the \emph{time-reversal} of this SDE initialized at $\mu_1$.  The time-reversal
of an SDE has been derived under mild assumptions on the drift and diffusion
coefficients \citep{haussmann1986time,cattiaux2021time}. In this case, we have
$(\bfY_{T-t})_{t \in \ccint{0,T}} \sim \Pbb^1$
, with $\bfY_0 \sim \mu_1$ and
\begin{equation}
  \rmd \bfY_t = \{-f_{T-t}(\bfY_t) + \nabla \log p_{T-t}(\bfY_t) \} \rmd t + \rmd \bfB_t ,
\end{equation}
where $p_t$ is the density of $\Pbb^0_t$ w.r.t. the Lebesgue measure. The score
$\nabla \log p_t$ is estimated using score matching techniques
\citep{hyvarinen2005estimation,vincent2011connection}. The first
iterate of DSB, $\Pbb^1$, corresponds to a \emph{denoising diffusion model}
\citep{ho2020denoising,song2020score}.
DSB
iterates further and not only parameterizes the backward process but also the
forward process. It can therefore be seen as a refinement of diffusion models
drawing a bridge between generative modeling and optimal transport.

\paragraph{Tree-based framework.} Consider an undirected tree $\mst = (\msv, \mse)$, with vertices $\msv$ and edges $\mse$, such that $\msv$ is identified with $\{0, \hdots, \ell\}$. Inspired by \cite{haasler2021multimarginal}, we restrict our study of \eqref{eq:mot_eps}, to the
case where the cost function $c$ is the tree-structured \emph{quadratic} cost derived from $\mst$
\begin{equation}\label{eq:tree_cost}
    \textstyle c(x_{0:\ell})=\sum_{\{v,v'\}\in \mse}w_{v,v'} \|x_{v}- x_{v'}\|_2^2 \eqsp, 
  \end{equation}
where $w_{v, v'}$ is a weight on the edge $\{v,v'\}$,
  which links $v$ to $v'$ (and $v'$ to $v$). Furthermore, as in
\cite{haasler2021multimarginal}, we choose $\mss$, \ie, the set of vertices of $\mst$ with constrained
marginals, to coincide with the \emph{leaves}
of $\mst$. This framework recovers important applications, from Wasserstein
barycenters to Wasserstein propagation, see \cite{solomon2014wasserstein,solomon2015convolutional}. We emphasize that it differs from an OT problem defined on the space of graphs \citep{chen2016robust}. Here, each node represents a probability
  measure (observed or to be inferred) and each edge represents a coupling
  between two distributions.

We consider  an arbitrary vertex $r\in\msv$ and choose $\nu$ in \eqref{eq:mot_eps} such that $(\rmd\nu/\rmd \Leb)(x_{0:\ell})=\varphi_r(x_r)$, where $\varphi_r$ is a density defined on $\rset^d$. Due to the form of $\nu$ and $c$, the reference measure $\pi^0$ in \eqref{eq:multimarginal_sb_static} is therefore a \emph{probability} distribution which
factorizes along $\mst_r=(\msv, \mse_r)$, the directed version of $\mst$
rooted in $r$. We refer to \Cref{sec:recap-trees} for more details on the notion of directed trees. In this setting, \eqref{eq:mot_eps} is equivalent to the tree-based problem
\begin{align}\tag{TreeSB}
  \label{eq:tree_sb_static}
  \textstyle{\pi^\star  = \argmin \ensembleLigne{\KL(\pi|\pi^0)}{\pi \in \Pmeasure^{(\abs{\msv})}, \eqsp \pi_i = \mu_i\eqsp, \forall i \in \mss}} \eqsp ,\\
\label{eq:pi_0}
  \textstyle 
    \text{with} \quad \pi^0=\pi^0_r \bigotimes_{(v,v')\in \mse_r}\pi^0_{v'|v} \eqsp ,
\end{align}
where $\pi^0_{v'|v}(\cdot \mid x_v)=\mathrm{N}(x_v, \vareps/(2w_{v,v'})\Idd)$
and $\pi^0_r \ll \Leb$ with density $\varphi_r$. In a manner akin to \cite{haasler2021multimarginal}, we thus establish, in \emph{continuous} state-space, the correspondence between \eqref{eq:tree_sb_static}, a \emph{static} tree-based version of SB, and a version of EmOT with tree-structured cost \eqref{eq:tree_cost}. In our work, we make the following assumption on the constrained marginals $\{\mu_i\}_{i \in \mss}$.
\begin{assumption}\label{ass:mu_i} For any $i\in \mss$, $\mu_i\ll \Leb$ and $\ent(\mu_i)<\infty$.
\end{assumption}
\vspace{-0.2cm}
In what follows, we define $K$ as the number of leaves of $\mst$
, denoting $\mss = \{i_0, \dots, i_{K-1}\}$, and define the horizon times
$T_{v,v'}=\vareps/(2w_{v,v'})$ for any $\{v,v'\}\in \mse$. For any $i_k\in \mss$, we will denote by $\mst_{k}=(\msv, \mse_{k})$ the directed version of $\mst$ rooted in the leaf $i_k$.  In the next 
section, we present our \emph{dynamic} method to solve \eqref{eq:tree_sb_static}, called \emph{Tree-based Diffusion Schr\"odinger Bridge}.


\vspace{-0.2cm}
\section{Tree-based Diffusion Schr\"odinger Bridge}
\label{sec:tree-based-diffusion}
\vspace{-0.2cm}
In this section, we present a method to solve \eqref{eq:tree_sb_static}
in the case where $r\in \mss$, \ie, $r$ is a leaf of $\mst$. We refer to
\Cref{sec:addit-deta-tree} for the extension to the case where
$r\in \msv \backslash \mss$. Without loss of generality, see \Cref{sec:addit-deta-tree}, we assume that $r=i_{K-1}$
and choose $\varphi_{r} = \rmd \mu_{i_{K-1}} / \rmd \Leb$, such that
$\pi^0_{i_{K-1}} = \mu_{i_{K-1}}$.

\vspace{-0.2cm}
\paragraph{Dynamic approach to mIPF.}In order to approximate solutions of \eqref{eq:tree_sb_static}, we
consider the \emph{multi-marginal} extension of the IPF algorithm, denoted by mIPF. Namely, we define a sequence of probability distributions $(\pi^n)_{n \in \nset}$ such that for any
$n \in \nset$
\begin{equation} \tag{mIPF}
  \label{eq:ipf_multimarginal}
  \hspace{-.48cm}\pi^{n+1} = \argmin \ensembleLigne{\KL(\pi|\pi^{n})}{\pi \in \Pmeasure^{(\abs{\msv})}, \eqsp \pi_{i_{k_n+1}} = \mu_{i_{k_n+1}}} \eqsp ,
\end{equation}
where $k_n=(n-1)\mod(K)$ and $(k_n+1)$ is identified with $n \mod(K)$. We define a \emph{mIPF cycle} as a sequence of $K$ consecutive mIPF updates. In particular, each marginal constraint is considered exactly once during one mIPF cycle. In a practical setting, our main aim is to sample from the \eqref{eq:ipf_multimarginal} iterates at the lowest cost. Although these updates can be made explicit, see \cite{marino2020optimal} for instance, direct sampling is unfeasible in practice when $d$ is large. To overcome this limitation, we suggest to compute these iterates in a \emph{dynamic} fashion with equivalent path measures. 

Since $\pi^0$ factorizes along $\mst$, see \eqref{eq:pi_0}, one can show that the iterates of \eqref{eq:ipf_multimarginal} also factorize along $\mst$, see \Cref{sec:theor-prop-tree}. Since these iterates all have a constrained marginal, we obtain the following decomposition for any $n\in \nset$: $\pi^n=\mu_{i_{k_n}} \otimes_{(v,v')\in \mse_{k_n}}\pi^n_{v'|v}$ where $\mse_{k_n}$ denotes the set of edges of the directed tree $\mst_{k_n}$. Then, our approach consists in computing \emph{dynamic} iterates, \ie, path measures, along the edges of $\mst$ that coincide on their extremal times with the \emph{static} iterates $(\pi^n)_{n \in \nset}$. Namely, for any $n\in \nset$, for any edge $(v,v')\in\mse_{k_n}$, we define a path measure $\Pbb^{n}_{(v,v')}\in \Pmeasure(\rmC(\ccintLigne{0,T_{v,v'}}, \rset^d))$ such that $\mathrm{Ext}(\Pbb^{n}_{(v,v')})=\pi^n_{v,v'}$, where $\mathrm{Ext}(\Pbb^{n}_{(v,v')})$ stands for the joint distribution of $\Pbb^{n}_{(v,v')}$ at times 0 and $T_{v,v'}$. In particular, it comes that $\pi^n_{v'|v}=\Pbb^{n}_{(v,v'),T_{v,v'}|0}$. Using the tree-based form of the \eqref{eq:ipf_multimarginal} iterates, we can thus sample from $\pi^n$ by (i) following the directed edges of $\mst_{k_n}$, (ii) diffusing along them the corresponding path measures $(\Pbb^{n}_{(v,v')})_{(v,v')\in \mse_{k_n}}$ and (iii) picking the samples on the vertices. When $\mst$ is a \emph{bridge-shaped} tree (2 vertices, 1 edge), it simply reduces to the dynamic reformulation of the IPF scheme. In what follows, we explain how to obtain our \emph{dynamic} sequence.


\paragraph{Definition of the dynamic iterates.} We first compute the iterate $\Pbb^0$, corresponding to the dynamic version of $\pi^0$ defined \eqref{eq:pi_0}, in \Cref{prop:init_tree_dsb}. Then, we build the following iterates by recursion on $n\in \nset$ and prove their well-posedness in \Cref{prop:sinkhorn_continuous}.

\begin{proposition}
  \label{prop:init_tree_dsb}
  Let $\mst_{K-1} = (\msv, \mse_{K-1})$, the directed tree associated with
  $\mst = (\msv, \mse)$ and root $i_{K-1}$. Then, for any $(v,v') \in \mse_{K-1}$, there
  exists
  $\Pbb^{0}_{(v,v')} \in \Pmeasure(\rmC(\ccintLigne{0,T_{v,v'}}, \rset^d))$ with
  $\mathrm{Ext}(\Pbb^{0}_{(v,v')}) = \pi^{0}_{(v,v')}$ and such that
  $\Pbb^{0}_{(v,v')|0}$ is the distribution of
  $(\bfB_t)_{t \in \ccintLigne{0,T_{v,v'}}}$, recalling that
  $T_{v,v'}=\vareps/(2w_{v,v'})$.
\end{proposition}


Before deriving the dynamic counterpart of the \eqref{eq:ipf_multimarginal} iterates, we introduce several definitions. For any path measure $\Pbb$,
we denote by $\Pbb^R$ the \emph{time-reversal} of $\Pbb$. For any directed tree and any vertex $v$ of this tree, $p(v)$ refers to the (unique) \emph{parent} of $v$, and $c(v)$ to the unique \emph{child} of $v$ when it exists, see \Cref{sec:recap-trees} for more details.

\begin{wrapfigure}{r}{0.45\textwidth}
  \vspace{-0.5cm}
\begin{center}
  \includegraphics[width=.35\textwidth]{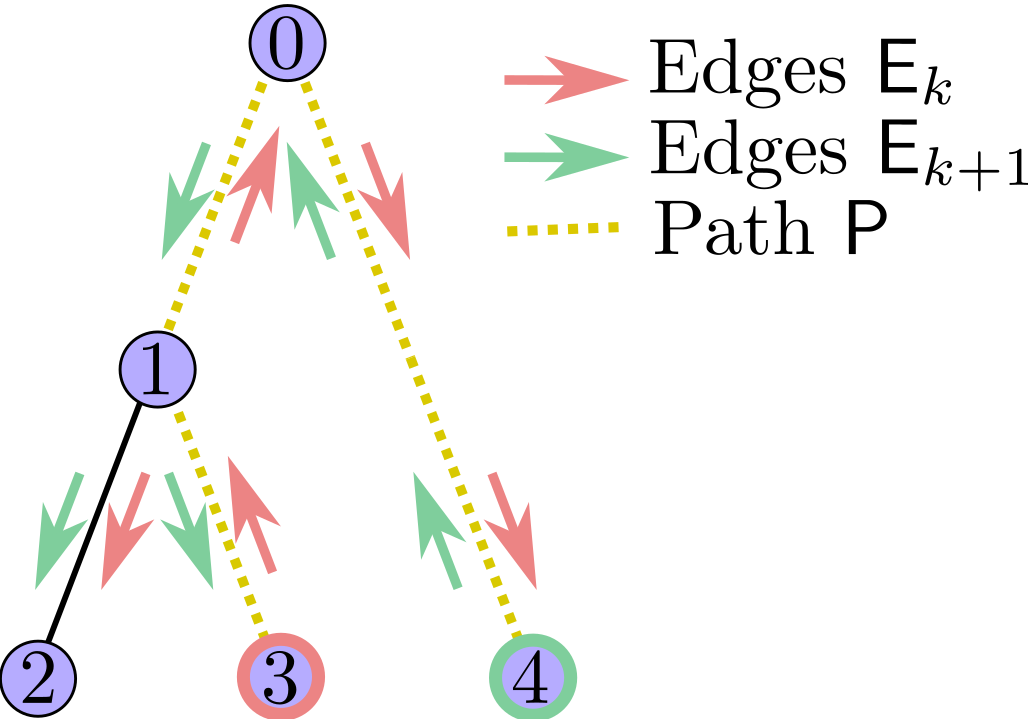}
  \end{center}
  \vspace{-0.3cm}
  \caption{Illustration of the change of root in a toy tree with 5 vertices.}
  \vspace{-0.5cm}
  \label{fig:truc}
\end{wrapfigure}

Let $n \in \nset$. Assume that we have defined the sequence of our dynamic iterates $(\Pbb^{m}_{(v,v')})_{(v,v')\in \mse_{k_m}, m\leq n}$ up to stage $n$.

Consider the path
  $\msp_n = \{(v_j,v_{j+1})\}_{j=1}^J$ in the directed tree $\mst_{k_n}$ such that $v_1 = i_{k_n}$ and $v_{J+1} = i_{k_n+1}$. In particular, for
  any $(v,v') \in \mse_{k_n+1}$, either $(v',v) \in \msp_n$ or
  $(v,v') \in \mse_{k_n} \backslash \msp_n$. This is illustrated in \Cref{fig:truc} when $\msv=\{0,1,2,3,4\}$,
$S=\{2,3,4\}$, $i_k=3$ and $i_{k+1}=4$: in this case, $\msp=\{(3,1), (1,0),(0,4)\}$ and $(1,2)$ is the only edge common to $\mse_{k}$ and $\mse_{k+1}$.  
  
  Consider now the directed tree $\mst_{k_{n+1}}$. We define the $(n+1)$-th iterate of our dynamic sequence by recursion on the edges of this tree, following the breadth-first order. In this order, $(i_{k_n +1},c(i_{k_n +1}))=(v_{J+1}, v_J)$ is the first edge considered.  
  
  First, we define $\Pbb^{n+1}_{(v_{J+1}, v_J)}= \mu_{i_{k_n+1}} \otimes (\Pbb^{n}_{(v_J, v_{J+1})})^R_{|0}$. In the case of a \emph{bridge-shaped} tree, this is exactly the $(n+1)$-th update described in DSB. Then, for any
  $(v,v') \in \mse_{k_n+1}\backslash \{(v_{J+1}, v_{J})\}$, 
  \begin{enumerate}[wide, labelwidth=!, labelindent=0pt, label=(\alph*)]
  \vspace{-0.2cm}
  \item either $(v,v') \in \mse_{k_n} \backslash \msp_n$, and
    we define $\Pbb^{n+1}_{(v,v')} = \Pbb^{n+1}_{(p(v),v),T_{p(v),v}} \otimes
    \Pbb^{n}_{(v,v')|0}$, \label{item:a_sinkhorn}
    \vspace{-0.1cm}
  \item or $(v',v) \in \msp_n$, and we define $\Pbb^{n+1}_{(v,v')} = \Pbb^{n+1}_{(p(v),v),T_{p(v),v}} \otimes (\Pbb^{n}_{(v',v)})^R_{|0}$. \label{item:b_sinkhorn}
  \end{enumerate}

\begin{proposition}
  \label{prop:sinkhorn_continuous} Consider the sequence of dynamic iterates defined by \ref{item:a_sinkhorn} and \ref{item:b_sinkhorn}. Then, for any $n \in \nset$ and any $(v,v') \in \mse_{k_n}$, $\Pbb^{n}_{(v,v')}\in \Pmeasure(\rmC(\ccintLigne{0,T_{v,v'}}, \rset^d))$ and we have
$\mathrm{Ext}(\Pbb^{n}_{(v,v')}) = \pi^{n}_{(v,v')}$.
\end{proposition}

\Cref{prop:sinkhorn_continuous} highlights the equivalence between the \eqref{eq:ipf_multimarginal} iterates and our dynamic iterates. These path measures are
defined iteratively, by following the updates \ref{item:a_sinkhorn} and \ref{item:b_sinkhorn} along the edges of $\mst$. The key observation here is that the computation of each dynamic iterate reduces to a sequence of updates \ref{item:b_sinkhorn} on a \emph{path} linking two leaves of $\mst$. We emphasize that our iterates could be similarly obtained by directly considering a dynamic formulation of \eqref{eq:tree_sb_static} and introducing the formalism of deterministic time branching processes. We leave the study of this problem for future work. We now get into the details of our practical implementation, which relies on score-based methods.

\paragraph{Approximation of the dynamic iterates.} The time-reversal operated in the update \ref{item:b_sinkhorn} can be computed
explicitly, see \cite{haussmann1986time} for instance. Indeed, assuming that
$\Pbb^{n}_{(v',v)}$ is associated with
$\rmd \bfX_t = f_{t, v', v}(\bfX_t) \rmd t + \rmd \bfB_t$ with
$\bfX_0 \sim \pi_{v'}^{n}$, then, under mild conditions, its time-reversal
$(\Pbb^{n}_{(v',v)})^R$ is associated with
$\rmd \bfY_t = \{-f_{T-t, v', v} + \nabla \log p_{v', v, T-t}\}(\bfY_t) \rmd t +
\rmd \bfB_t$ with $\bfY_0 \sim \pi_{v}^{n+1}$, where $p_{v', v,t}$ is the
density of $\Pbb^n_{(v',v),t}$ w.r.t. the Lebesgue measure. The score
$\nabla \log p_{v', v, T-t}$ can then be approximated using score-matching
techniques \citep{hyvarinen2005estimation,vincent2011connection} which are now
ubiquitous in diffusion models \citep{song2020score} and used in DSB
\cite{debortoli2021diffusion}.
Therefore, at iteration $(n+1)$, the update \ref{item:b_sinkhorn} is similar
to the one of DSB \emph{for each edge} on the path joining $i_{k_n}$ and $i_{k_n+1}$. In practice, we parameterize the drifts $f_{t, v, v'}$ for
any $\{v,v'\} \in \mse$ with neural networks $f_{t, \theta_{v,v'}}$ and use the \emph{mean-matching} loss introduced
by \cite{debortoli2021diffusion}. Note that
doing so, we obtain $2 |\mse|$ neural networks. The whole procedure consisting in computing our dynamic iterates using the DSB framework is called \emph{Tree-based
  Diffusion Schr\"odinger Bridge} (TreeDSB) and is summarized in \Cref{algo:treedsb}.
\newpage
\begin{minipage}{.6\textwidth}
\vspace{-0.5cm}
\begin{algorithm}[H]
    \caption{TreeDSB (Training)}
    \label{algo:treedsb}
    \begin{algorithmic}[1]
      \STATE{\textbf{Input: } $\mst=(\msv, \mse)$, $\{\mu_i\}_{i\in \mss}$,
        $\{\theta_{v,v'}\}_{ \{v,v'\} \in \mse}$, $N \in \nset$}
      \FOR{$n=0,\dots,N$} \STATE{Let $k_n=(n-1) \mod(K)$} \STATE{Get path between
        $i_{k_n}$ and $i_{k_n+1}$, $\msp_n= \{v_j, v_{j+1}\}_{j=1}^{J}$} \WHILE{not
        converged} \FOR{$j=1, \dots, J$} \STATE{Sample from
        $\Pbb^n_{v_j, v_{j+1}}$ (Euler-Maruyama)} \STATE{Compute \emph{mean
          matching} loss $\ell(\theta_{v_{j+1},v_{j}})$}
      \STATE{$\theta_{v_{j+1},v_{j}} \leftarrow \textrm{Gradient
          Step}(\ell(\theta_{v_{j+1},v_j}))$} \STATE{Update
        $f_{t, \theta_{v_{j+1}, v_j}}$} \ENDFOR \ENDWHILE \ENDFOR
      \STATE{\textbf{Output:}  $\{\theta_{v,v'}\}_{ \{v,v'\} \in \mse}$}
    \end{algorithmic}
  \end{algorithm}
\end{minipage}
\hfill 
\begin{minipage}{.37\textwidth}
  The algorithm is initialized with $f_{t, \theta_{v,v'}} = 0$ for all
  $\{v,v'\} \in \mse$. This corresponds to Brownian motion dynamics when sampling
  at the first iteration of TreeDSB, see \Cref{prop:init_tree_dsb}. Note that in
  \Cref{algo:treedsb}, when we sample from $\Pbb^n_{(v_j,v_{j+1})}$, we update
  $f_{t, \theta_{v_{j+1}, v_j}}$ which will be used to sample from
  $\Pbb^{n+1}_{(v_{j+1},v_{j})}$ in the next iterations.  In order to sample from
  the dynamics $\Pbb^n_{(v_j,v_{j+1})}$, we consider its Euler--Maruyama
  discretization, see \Cref{sec:details-algo} for more
  details. We describe the different steps of the algorithm in the case of a toy
  example below, see \Cref{fig:illustration_wasserstein}
  for an illustration.
\end{minipage}
\vspace{-0.3cm}
\paragraph{TreeDSB on a toy tree.} We consider a
star-shaped tree with three leaves denoted $\{1,2,3\}$ and its central node
$\{0\}$. Following \eqref{eq:pi_0}, we define $\pi^0$
 with $r=3$ and $\varphi_r=(\rmd \mu_3 / \rmd \Leb)$. During the first iteration of TreeDSB, $\mst$ is rooted at vertex $3$ and we compute samples from
the \emph{forward} path $\msp_0 =\{(3, 0), (0,1)\}$ with Brownian motions, see \Cref{prop:init_tree_dsb}, in order to learn the \emph{backward} path $\{(1, 0), (0,3)\}$. In the next iteration, we re-root the tree $\mst$ at vertex $1$ and consider the \emph{forward} path $\msp_1 =\{(1, 0), (0,2)\}$, where the edges $(1, 0)$ and $(0,2)$ are respectively given by the first iteration and the initialisation. This highlights that \emph{TreeDSB does not require to update the whole tree}. The following iterations are done similarly. At each iteration $n\in \nset$, we sample from $\pi^n$ by first sampling from $\mu_{k_n}$ at leaf $i_{k_n}$ and then following the parameterized SDEs on  the directed edges of $\mst_{k_n}$. 
\vspace{-0.2cm}
\begin{figure}[h!]
  \centering
  \includegraphics[width=.8\linewidth]{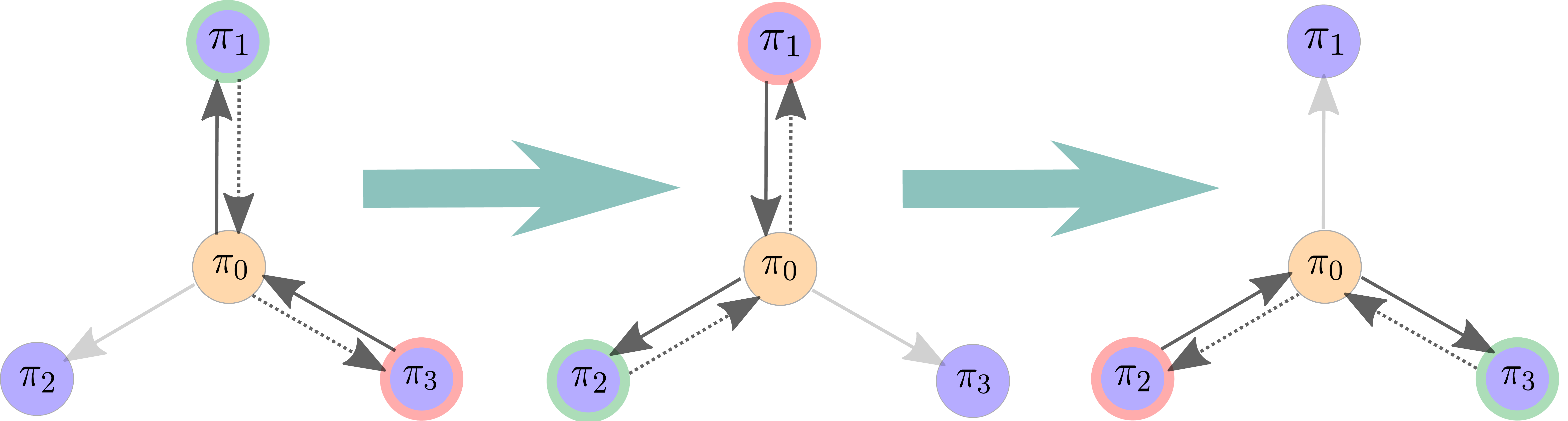}
  \caption{Illustration of one mIPF cycle solved by TreeDSB for a toy star-shaped tree. At each iteration, our method learns the \emph{backward} stochastic process (dotted arrows) that goes from the target leaf (green-circled), corresponding to the constrained marginal, to the current root of the tree (red-circled) by using samples from the \emph{forward} stochastic process (solid arrows).}
  \label{fig:illustration_wasserstein}
\end{figure}


\vspace{-0.3cm}
\section{Theoretical properties of mIPF}
\label{sec:theor-prop-tree}
\vspace{-0.2cm}

In this section, we study some of the theoretical properties of the \emph{static} iterates $(\pi^n)_{n \in \nset}$, that are equivalent to our \emph{dynamic} iterates according to \Cref{prop:sinkhorn_continuous}. In the case where the cost function $c$ is bounded in \eqref{eq:mot_eps}, results of convergence of
\eqref{eq:ipf_multimarginal} exist \citep{marino2020optimal,
  carlier2022linear}. However, our setting does not satisfy their
assumptions, since our transport cost is quadratic and the measures are defined
on $\rset^d$. In what follows, we provide the first non-quantitative convergence
results for \eqref{eq:ipf_multimarginal} in a \emph{non-compact} setting.

For the rest of the section, we consider a static formulation of the multi-marginal
Schr\"odinger bridge problem which is more general than \eqref{eq:tree_sb_static}, defined as
\begin{align} \tag{static-mSB}
  \label{eq:multimarginal_sb_static_nu}
     \textstyle{\pi^\star  = \argmin \ensembleLigne{\KL(\pi|\pi^0)}{\pi \in \Pmeasure^{(\ell+1)}, \eqsp \pi_i = \mu_i\eqsp, \forall i \in \mss}} \eqsp ,
\end{align}
where $\mss \subset \{0, \hdots, \ell\}$, $\pi^0\in \Pmeasure$,
$\{\mu_i\}_{i \in \mss} \in \Pmeasure^{\abs{\mss}}$.  We consider the following
set of assumptions.
\begin{assumption}\label{ass:measures}There exists a
family of measures $\{\nu_i\}_{i\in \{0, \hdots, \ell\}}$ defined on
$(\rset^d, \mathcal{B}(\rset^d))$ such that
$\pi^0 \ll \bigotimes_{i=0}^\ell \nu_i$ with density
$h=\rmd \pi^0 / (\rmd \bigotimes_{i=0}^\ell \nu_i)$ and $\mu_i \ll \nu_i$ with
density $r_i=\rmd \mu_i / \rmd \nu_i$ for any $i \in \mss$.
\end{assumption} 

\begin{assumption} \label{ass:kl_bounded} $\{\pi \in \Pmeasure^{(\ell+1)}: \KL(\pi \mid \pi^0) < \infty, \eqsp \pi_i=\mu_i, \eqsp \forall i \in \mss\}\neq \emptyset$ .
\end{assumption}

\begin{assumption} \label{ass:equivalence} There exists a family of probability measures $\{\Tilde{\mu}_j\}_{j \in \{0, \hdots, \ell\}\backslash \mss}$ such that $\pi^0 \sim \Tilde{\pi}^0$, where $\Tilde{\pi}^0=\bigotimes_{i\in \mss} \mu_i \bigotimes_{j \in \{0, \hdots, \ell\}\backslash \mss}\Tilde{\mu}_j$.
\end{assumption}
In particular, \eqref{eq:multimarginal_sb_static_nu} recovers
\eqref{eq:tree_sb_static} by considering $\nu_i=\Leb$ for any
$i\in \{0, \hdots, \ell\}$ and
$h(x_{0:\ell})=\varphi_r(x_r) \exp[-c(x_{0:\ell})/\vareps]$ in
\Cref{ass:measures}. We detail in \Cref{sec:proofs-main} how
\Cref{ass:kl_bounded} and \Cref{ass:equivalence} can be met in
\eqref{eq:tree_sb_static}. Under these assumptions, the multi-marginal
Schr\"odinger Bridge exists.

\begin{proposition} \label{prop:existence} Assume \textup{\Cref{ass:measures}}
  and \textup{\Cref{ass:kl_bounded}}. Then, there exists a unique solution
  $\pi^\star$ to \eqref{eq:multimarginal_sb_static_nu}. In addition, assume
  \textup{\Cref{ass:equivalence}}. Then, there exists a family
  $\{\psi_i^\star\}_{i \in \mss}$ of measurable functions
  $\psi_i^\star: \rset^d \to \rset$ such that
\begin{align}
    \textstyle{(\rmd \pi^\star / \rmd \pi^0)=\exp[\bigoplus_{i \in \mss}\psi_i^\star]} \quad \text{ $\pi^0$-a.s.}
\end{align}
\end{proposition}

In order to establish the existence and uniqueness result of
\Cref{prop:existence}, we extend results from \cite{nutz2021introduction} to the
multi-marginal setting. A consequence of \Cref{prop:existence} is that the
iterates of \eqref{eq:ipf_multimarginal} can be described using potentials. 

\begin{corollary}\label{corollary:existence} Assume \textup{\Cref{ass:measures}},
  \textup{\Cref{ass:kl_bounded}} and \textup{\Cref{ass:equivalence}}. Let
  $(\pi^n)_{n \in \nset}$ be the sequence given by
  \eqref{eq:ipf_multimarginal}. Then, for any $n \in \nsets$ with $k_n = (n-1)\mod(K)$ and $q_n \in \nset$ such that $n = q_n K + k_n +1$, there exists a family of measurable functions
  $\{\psi^{q_n+1}_{i_0}, \hdots, \psi^{q_n+1}_{i_{k_n}}, \psi^{q_n}_{i_{k_n+1}}, \hdots,  \psi^{q_n}_{i_{K-1}}\}$ such that 
\begin{align}
    \textstyle{(\rmd \pi^n / \rmd \pi^0)(x_{0:\ell})=\exp[\bigoplus_{j=0}^{k_n} \psi^{q_n+1}_{i_j}(x_{i_j}) \bigoplus_{j=k_n+1}^{K-1} \psi^{q_n}_{i_j}(x_{i_j})]} \quad \text{ $\pi^0$-a.s.}
\end{align}
\end{corollary}

In the tree-based setting, \Cref{corollary:existence} explains why the \eqref{eq:ipf_multimarginal} iterations preserve the tree-based Markovian nature of $\pi^0$. We now prove that the marginal $\pi^n_i$ converges to $\mu_i$ for any $i\in \mss$, as $n$ goes to infinity, $\ie$, we have marginal convergence on the leaves of $\mst$.

\begin{proposition}\label{prop:marginal} Assume \textup{\Cref{ass:measures}} and \textup{\Cref{ass:kl_bounded}}. Let $(\pi^n)_{n \in \nset}$ be the sequence given by
  \eqref{eq:ipf_multimarginal}. Then, we have
  $\lim_{n \to \infty}\| \pi^n_i - \mu_i\|_{\mathrm{TV}}= 0$ for any
  $i \in \mss$.
\end{proposition}

The previous result does not ensure the convergence of $(\pi^n)_{n \in \nset}$ to the solution to \eqref{eq:multimarginal_sb_static_nu}. In particular, \Cref{prop:marginal} does not provide the convergence of the marginals on the nodes $v \in \msv \backslash \mss$, which is key to compute regularized Wasserstein barycenters with TreeDSB. Relying on additional assumptions, we now derive the convergence of \eqref{eq:ipf_multimarginal}.

\begin{assumption}\label{ass:closeness} $\bigoplus_{i \in \mss}\mathrm{L}^1(\mu_i) \subset \mathrm{L}^1(\pi^\star)$ is closed.
\end{assumption}

\begin{assumption}\label{ass:ratio_bounded}
  There exist $\bar{c}\in (0, \infty)$ such that $\exp(\psi^n_{i_k}-\psi^{n+1}_{i_k})\leq \bar{c}$, for any $n\in \nset$, any
  $k\in \{0, \hdots, K-2\}$.
\end{assumption}

These assumptions can be seen as multi-marginal extensions of the ones of
\cite{ruschendorf1995convergence}, see \Cref{sec:proofs-main} for a discussion
and examples.

\begin{proposition} \label{prop:conv_mipf}Assume \textup{\Cref{ass:measures}},
  \textup{\Cref{ass:kl_bounded}}, \textup{\Cref{ass:equivalence}},
  \textup{\Cref{ass:closeness}} and \textup{\Cref{ass:ratio_bounded}}. Let
  $(\pi^n)_{n \in \nset}$ be the sequence given by
  \eqref{eq:ipf_multimarginal}. Then, we have
  $\lim_{n \to \infty} \| \pi^n -\pi^\star\|_{\mathrm{TV}}= 0$, where
  $\pi^\star$ is given in \Cref{prop:existence}.
\end{proposition}

To the best of our knowledge, \Cref{prop:conv_mipf} is the first convergence
result of \eqref{eq:ipf_multimarginal} without assuming that the space is
compact or that the cost is bounded. We highlight that traditional techniques to
prove the convergence of IPF cannot be easily extended to the multi-marginal
setting as pointed by \citet{carlier2022linear}. In the case of bounded
cost, quantitative results exist \citep{marino2020optimal,
  carlier2022linear}. We leave the study of such results  in the
\emph{unbounded} cost setting for future work.


\section{Application to Wasserstein barycenters} \label{sec:barycenter}

Although \Cref{algo:treedsb} can be applied to trees $\mst$ with fixed marginals on the leaves, one case of particular interest is star-shaped trees, \ie,
trees with a central node, denoted by index $0$, and such that
$\mss=\{1, \hdots,\ell\}$ (see \Cref{fig:illustration_wasserstein} for an
illustration with $\ell=3$). In this section, we draw a link between
\eqref{eq:tree_sb_static} and regularized Wasserstein barycenters.  We
 recall the definition of the Wasserstein distance of order 2 with
$\vareps$-entropic regularization between  $\mu$ and $\nu$
\citep[Chapter 4]{peyre2019computational}
\begin{align}  
    \textstyle{ W_{2, \vareps}^2(\mu, \nu) = \inf \{\int\|x_1-x_0\|^2\rmd \pi(x_0,x_1) -\vareps\ent(\pi): \pi \in \Pmeasure^{(2)}, \pi_0=\mu, \pi_1=\nu\} \eqsp .}\label{eq:true_reg_wasserstein}
\end{align}
In this work, we consider the $(\ell \vareps,(\ell-1)\vareps)$-doubly-regularized Wasserstein-2 barycenter problem \citep{chizat2023doubly} defined as follows
\begin{align} \tag{regWB} \label{eq:wasserstein_barycenter_pb}
    \textstyle{\mu^\star_\vareps=\arg \min \{ \sum_{i=1}^\ell w_{i}W^2_{2, \vareps/w_i}(\mu, \mu_i) + (\ell-1)\vareps \ent(\mu) : \mu \in \Pmeasure\} \eqsp, }
\end{align}
where $(w_i)_{i \in \{1, \hdots,\ell\}}\in (0, +\infty)^\ell$. The following proposition shows the
 equivalence between the barycenter problem \eqref{eq:wasserstein_barycenter_pb}
 and the multi-marginal Schr\"odinger bridge problem \eqref{eq:tree_sb_static} over $\mst$.  In particular, it allows us to use TreeDSB to estimate the solution $\mu^\star_\varepsilon$ of \eqref{eq:wasserstein_barycenter_pb}.
\begin{proposition}\label{prop:wasserstein_barycenter_equi}Let $\vareps>0$. Assume \Cref{ass:mu_i}. Also assume that $\mst$ is a star-shaped tree with central node indexed by $0$, and that the reference measure of \eqref{eq:tree_sb_static} defined in \eqref{eq:pi_0} verifies $r=i_{K-1}$ and $\varphi_r=\rmd \mu_{i_{K-1}}/\rmd \Leb > 0$. Under
  \textup{\Cref{ass:kl_bounded}},
  \eqref{eq:wasserstein_barycenter_pb} has a unique solution 
  $\pi^\star_0$, where $\pi^\star$ solves
  \eqref{eq:tree_sb_static}.  
\end{proposition}
The proof of this result is postponed to \Cref{sec:proofs-main}. More generally, we
show in \Cref{sec:proofs-main} that, for any tree $\mst$,
\eqref{eq:tree_sb_static} is equivalent to a regularized version of the
Wasserstein propagation problem \citep{solomon2014wasserstein,
  solomon2015convolutional}. Moreover, we present in \Cref{sec:addit-deta-tree} an extension of \Cref{prop:wasserstein_barycenter_equi} in the case where the chosen root $r$ is not a leaf of $\mst$. We finally emphasize that the formulation of \eqref{eq:wasserstein_barycenter_pb} leads to a \emph{minimization} of the entropy of the barycenter. In particular, this allows us to choose $\varepsilon$ reasonably large in TreeDSB, which is a stability advantage compared to other regularized methods which do not consider this further regularization.


\vspace{-0.2cm}
\section{Related work}
\label{sec:related-work}
\vspace{-0.1cm}
\paragraph{Diffusion Schr\"odinger Bridge.}
Schr\"odinger Bridges \citep{schrodinger1932theorie} have been extensively
studied using tools from stochastic control and probability theory
\citep{leonard2014survey,dai1991stochastic,chen2020optimal}. 
More recently, algorithms were proposed to efficiently approximate
such bridges in the context of machine learning.  In particular,
\cite{debortoli2021diffusion} proposed DSB while
\cite{vargas2021solving,chen2021likelihood} developed related
algorithms.
In \cite{chen2023deepmomentum}, the authors study a multi-marginal version of
DSB in a linear tree-based setting, where the
set of observed nodes is the whole set of vertices.
However, contrary to our setting, \citet{chen2023deepmomentum} introduced a
momentum variable. This allows for smoother trajectories which are desirable for
single-cell trajectories applications and correspond to some spline
interpolation in the space of probability measures \citep{chen2018measure}. A
general framework for tree-based static Schr\"odinger Bridges on discrete state-spaces was given in
\cite{haasler2021multimarginal}. In this work, we extend their formulation to a dynamic and continuous setting, see \Cref{sec:proofs-main} for more a thorough comparison.
\vspace{-0.2cm}

\paragraph{Wasserstein barycenters.} The notion of Wasserstein barycenter was
first introduced in \cite{rabin2012wasserstein} and then later studied in
\cite{agueh2011barycenters}. The algorithms to solve this problem can be split
into two families: the in-sample based approaches and the parametric
ones. In-sample approaches require access to all the measures $\mu_i$ which are
assumed to be empirical measures 
\citep{cuturi2014fast,benamou2015iterative,solomon2015convolutional}. Related to
this class of algorithms is the semi-discrete approach, which aims at computing
a Wasserstein barycenter between continuous distribution but rely on a
discretization of the barycenter
\citep{claici2018stochastic,staib2017parallel,mi2020variational}. Most recent approaches do not rely on a discrete
representation of the barycenter, but instead parameterize it using neural
networks. These approaches can be further split into two categories. First,
\emph{measure-based optimization} approaches parameterize the measures using a
neural network. This is the case of \cite{cohen2020estimating}, where the
barycenter is given by a generative model, which is then
optimized
. \citet{fan2020scalable} introduce an optimization procedure which relies on a
\emph{min-max-min} problem using the framework of
\cite{makkuva2020optimal}. More recently, \cite{korotin2022wasserstein}
considered a fixed point-based algorithm introduced in \cite{alvarez2016fixed}
to update a generative model parametrizing the barycenter. On the one hand,
\emph{potential-based methods} rely on a dual formulation of the barycenter.
\cite{korotin2021continuous} parameterized the dual potentials using Input
Convex Neural Network and considered regularizing losses imposing conjugacy and
congruency. On the other hand, \citet{li2020continuous} consider a dual version
of the \emph{regularized} Wasserstein barycenter problem contrary to other
works.
Our approach applied to start-shaped trees also approximates a \emph{regularized}
Wasserstein barycenter. However, contrary to \cite{li2020continuous}, we do not
consider a parameterization of the potentials in the \emph{static} setting but
instead, parameterize the drift of an associated \emph{dynamic} formulation
using Schr\"odinger bridges. To the best of our knowledge TreeDSB is the first
approach leveraging DSB-like algorithms to compute
Wasserstein barycenters.


\vspace{-0.2cm}
\section{Experiments}
\label{sec:experiments}
\vspace{-0.2cm}
In our experiments\footnote{Code available at \url{https://github.com/maxencenoble/tree-diffusion-schrodinger-bridge}.}, we illustrate the performance of TreeDSB to
compute entropic regularized Wasserstein barycenters for various
tasks
. We choose to compare our method with state-of-the-art regularized algorithms:
fast free-support Wasserstein barycenter (fsWB) \citep{cuturi2014fast}
, and continuous regularized Wasserstein barycenter (crWB)
\citep{li2020continuous}.
In all of our settings, we consider a star-shaped tree with $K$ leaves and edge
weights that are equal to $1/K$, resulting in a sequential training procedure
over $2K$ neural networks. The initial diffusion is always a Brownian motion parameterized as explained in \Cref{prop:init_tree_dsb}. Hence, the time horizon on each edge is defined by $T=K\varepsilon/2$. The order of the
leaves is randomly shuffled between the mIPF cycles. We consider 50 steps for the time discretization on $[0,T]$. We refer to
\Cref{sec:details-exps} for details on the choice of the schedule, the 
architecture of the neural networks and the settings of our experiments.
\bigskip
\vspace{-0.1cm}
\begin{minipage}[h]{.5\linewidth} 
  \textbf{Synthetic two dimensional datasets.}  We first illustrate TreeDSB in a synthetic two dimensional setting. We consider three different
  datasets \emph{Swiss-roll} (vertex 0, starting node $r$), \emph{Circle} (vertex 2) and \emph{Moons} (vertex 3) and
  compute their Wasserstein barycenter (vertex 1) by running TreeDSB for 50 mIPF cycles with $\varepsilon=0.1$. In
  \Cref{fig:wasserstein_barycenter_2d}, we show the estimated densities of the datasets on the leaves of the tree (we emphasize that the
  distributions plotted on each leaf are generated from the central barycenter
  measure). In \Cref{fig:wasserstein_coherence}, we observe the consistency
  between the barycenters generated from the different leaves. In \Cref{sec:details-exps}, we present additional results for this setting.
\end{minipage} 
\hfill 
\begin{minipage}[h]{.47\linewidth}
\vspace{-0.5cm}
  \centering
  \includegraphics[width=.48\linewidth]{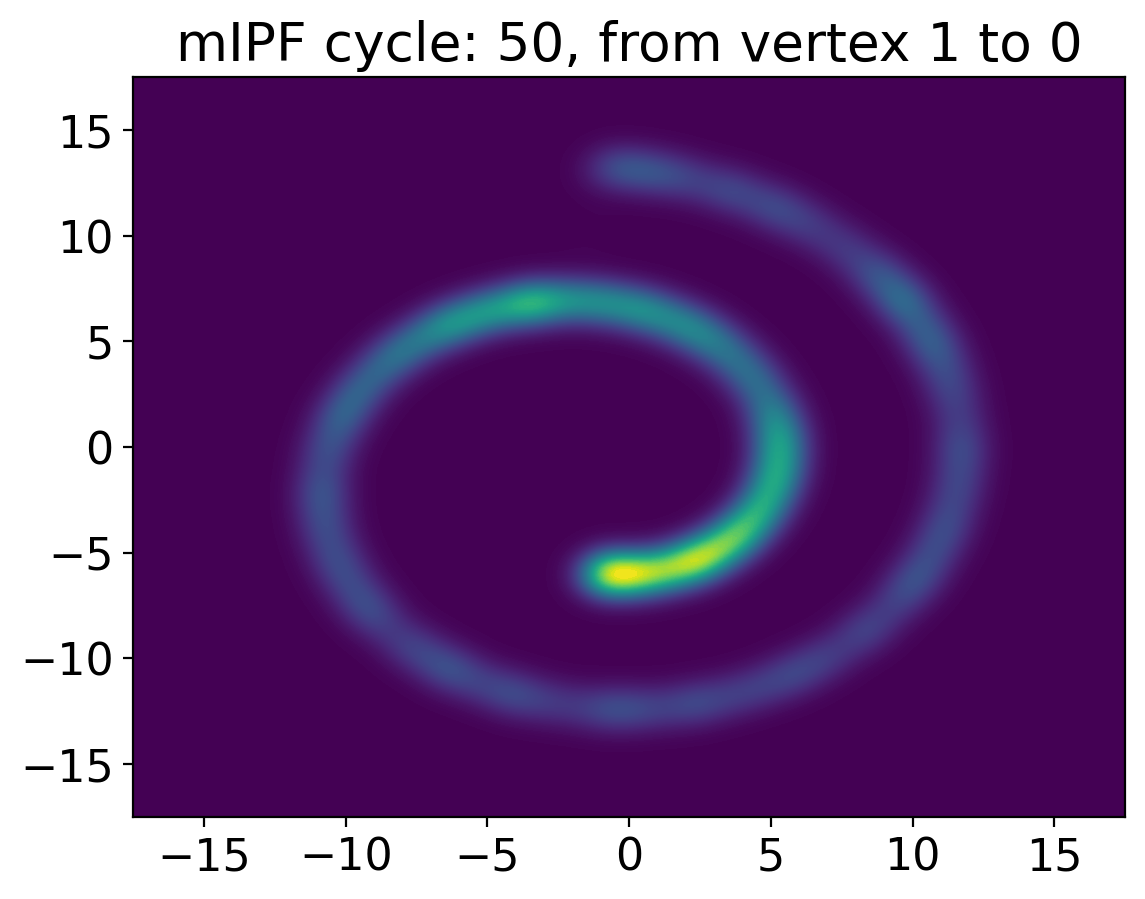}
  \\
  \includegraphics[width=.48\linewidth]{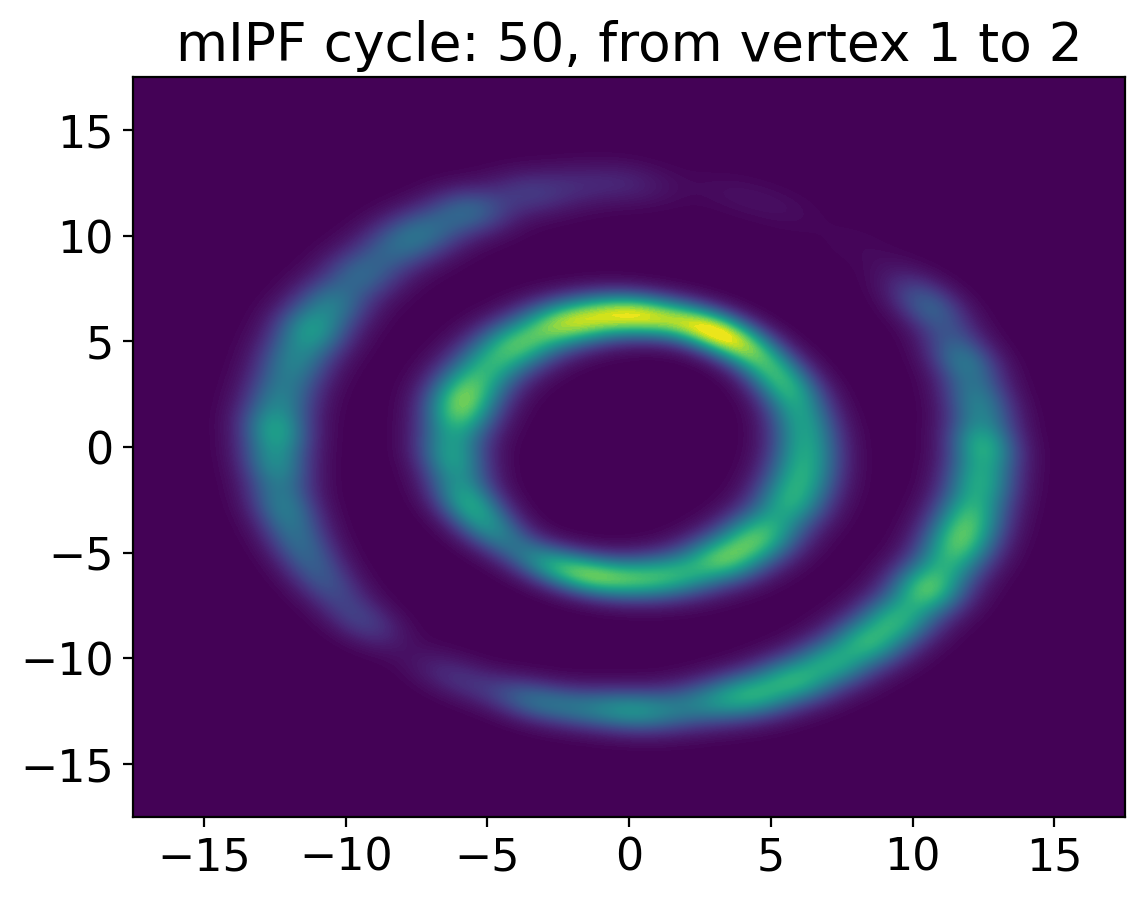}
  \includegraphics[width=.48\linewidth]{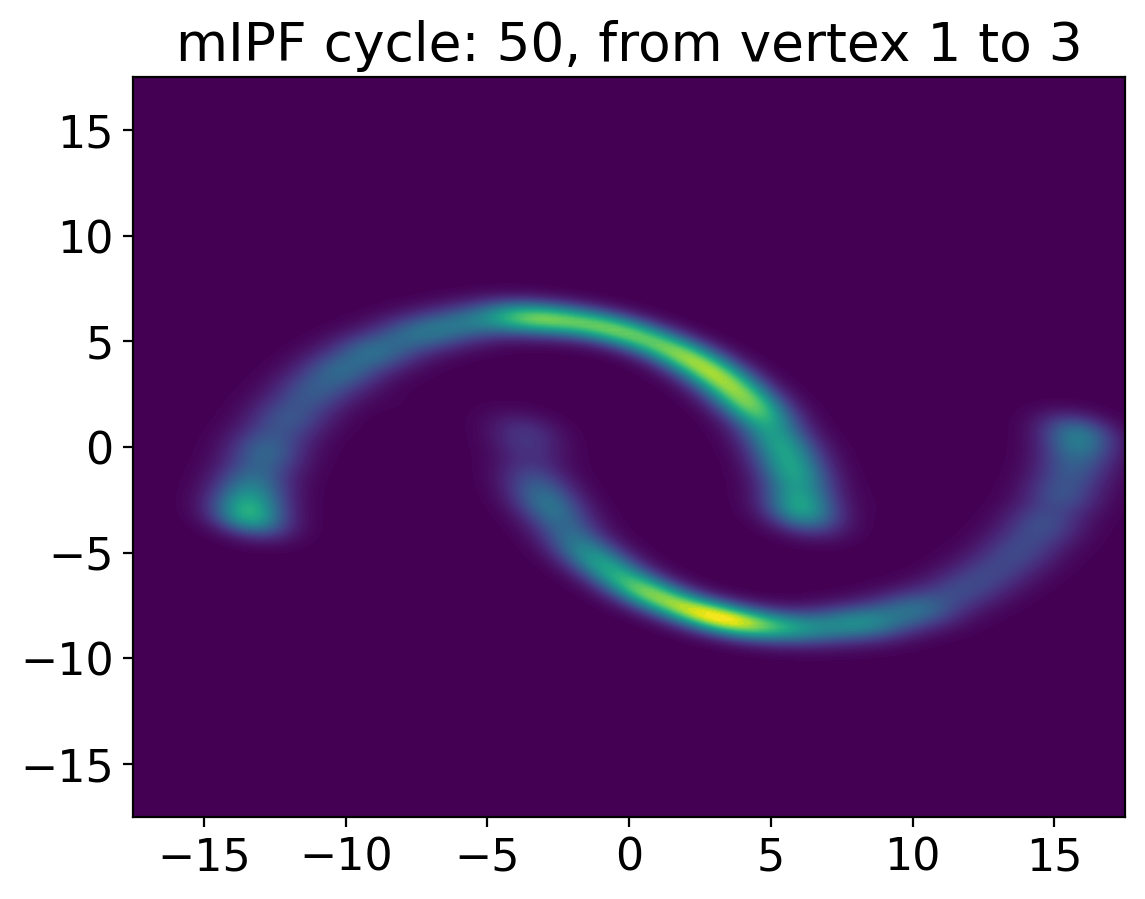}
  \vspace{-0.2cm}
  \captionof{figure}{Estimated densities on the leaves.}
\label{fig:wasserstein_barycenter_2d}
 \end{minipage}

\vspace{-0.1cm}
\begin{figure}[h]
  \centering
  \includegraphics[width=.32\linewidth]{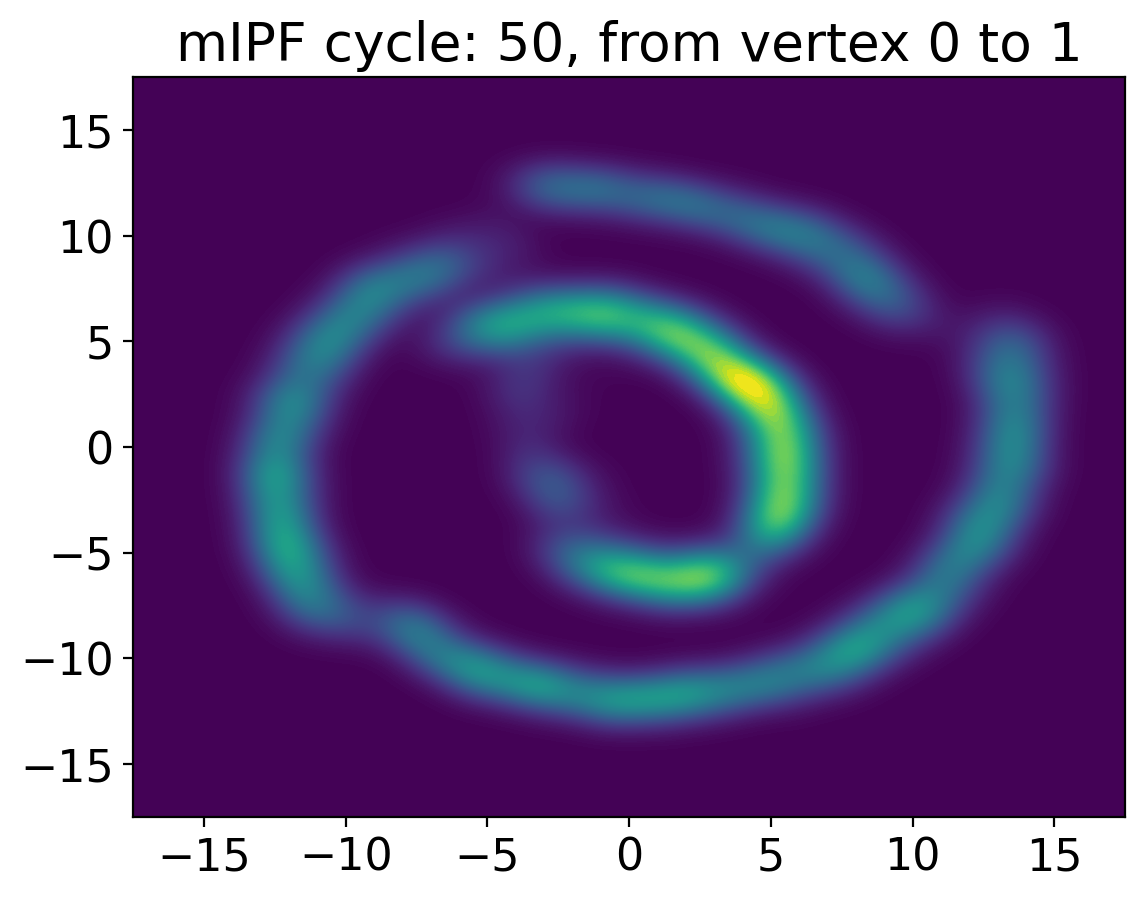} \hfill
  \includegraphics[width=.32\linewidth]{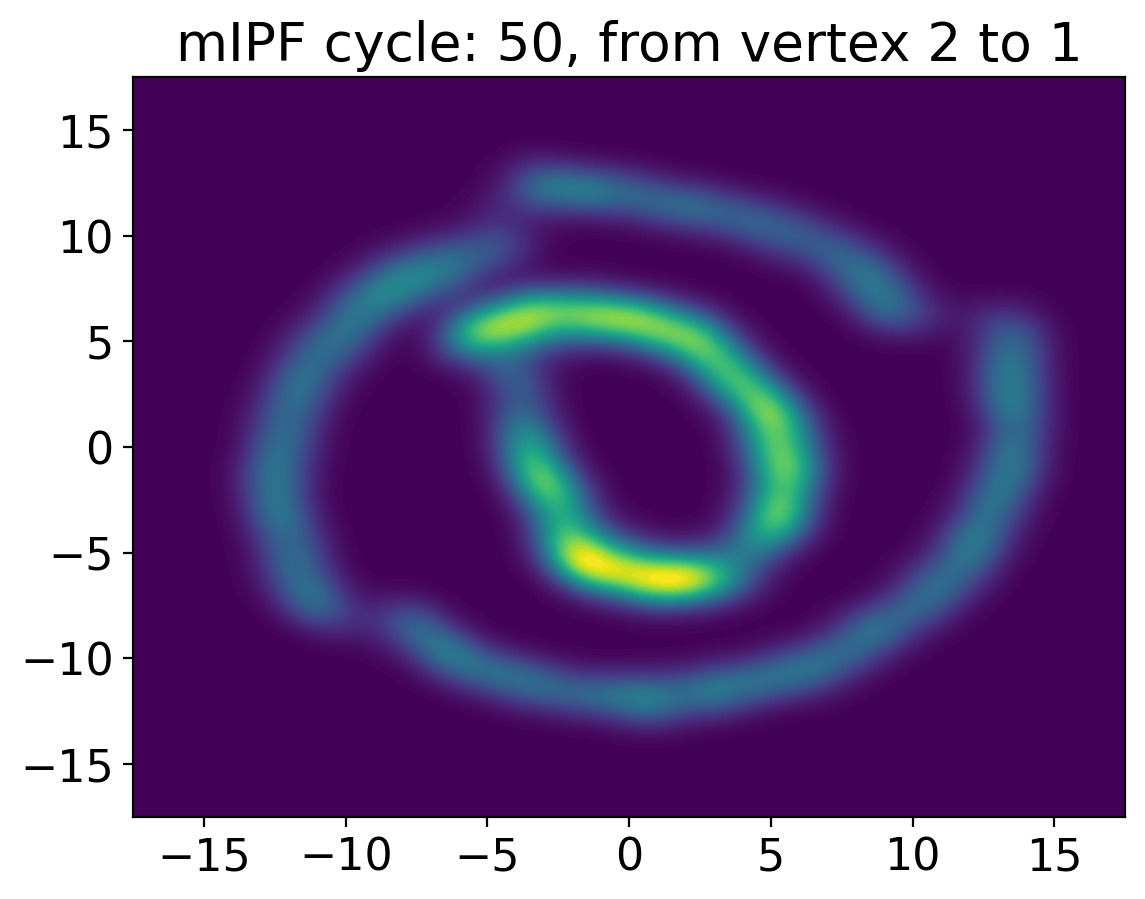} \hfill
  \includegraphics[width=.32\linewidth]{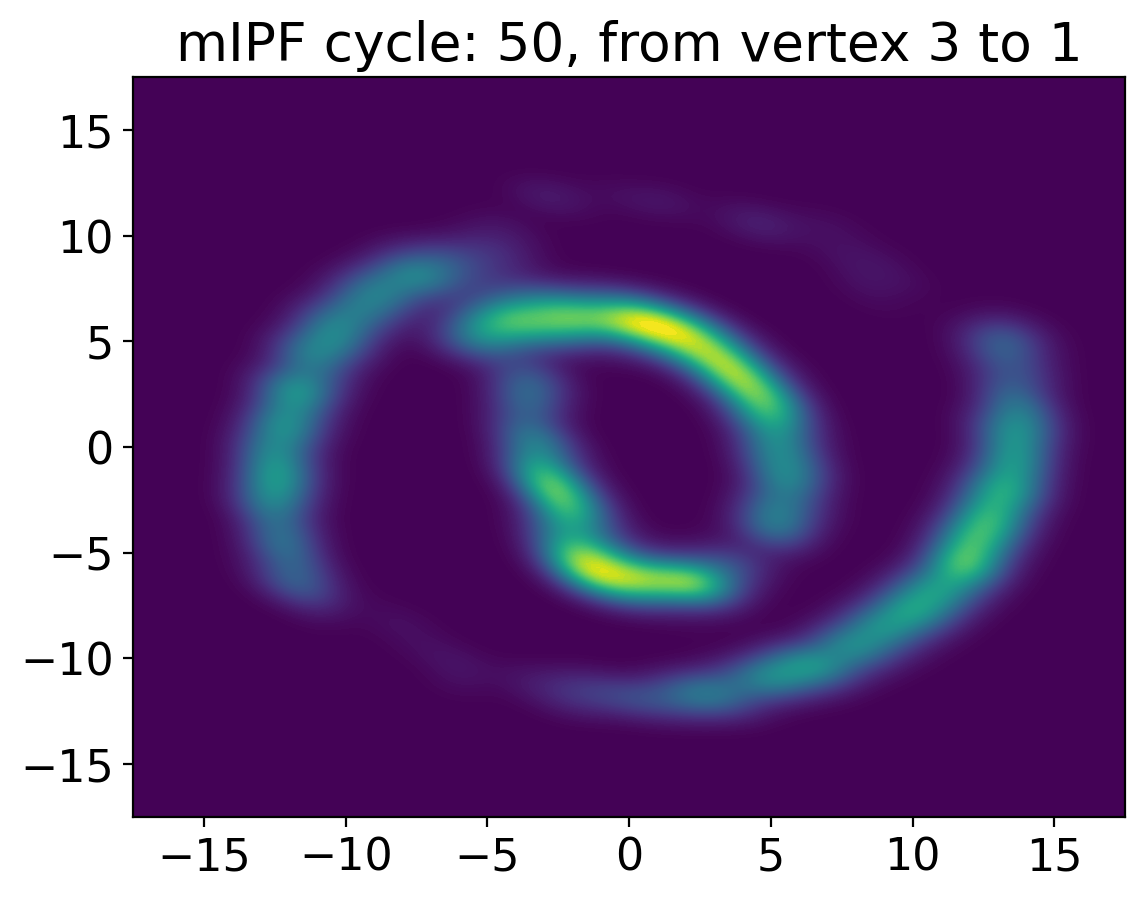}
  \vspace{-0.1cm}
  \caption{From left to right: barycenter estimated from the leaves \textit{Swiss-roll}, \textit{Circle} and \textit{Moons}.}
  \label{fig:wasserstein_coherence}
\end{figure} 

\vspace{-0.1cm}
\paragraph{Synthetic Gaussian datasets.} Next, we consider three independent
Gaussian distributions with zero mean and random non-diagonal covariance
matrices whose conditional number is less than 10, following
\cite{fan2020scalable}. In this case, the non-regularized barycenter can be
exactly computed. To evaluate the
performance of the algorithms, we use the Bures-Wasserstein Unexplained Variance
Percentage (UVP), following \cite[Section 5]{korotin2021continuous}. Given a
target distribution $\mu^\star\in \Pmeasure$ and some approximation
$\mu\in \Pmeasure$, we define
\begin{align}
    \BWUVP(\mu, \mu^\star)=  100\cdot 2\BW(\mu, \mu^\star)/\Var(\mu^\star)\% \eqsp ,
\end{align}
where $\BW(\mu, \mu^\star)=W_2^2(\mathrm{N}(\mathbb{E}[\mu], \Cov(\mu)),  \mathrm{N}(\mathbb{E}[\mu^\star], \Cov(\mu^\star))$.

\begin{table}[h]
    \setlength\tabcolsep{5pt}
    \centering
    \small
    \begin{tabular}{lccccccccc}
    \toprule
      Method
       & $d=2$ & $d=16$ & $d=64$ & $d=128$ &
                                                  $d=256$
      \\
    \midrule
fsWB \citep{cuturi2014fast} & $0.06_{\pm 0.01}$ & ${2.86_{\pm 0.06}}$ & ${11.12_{\pm 0.06}}$ &  ${14.47_{\pm 0.07}}$ & $17.41_{\pm 0.05}$
      \\  
crWB \citep{li2020continuous} &  $\bm{0.02_{\pm 0.01}}$ &  $\bm{1.52_{\pm 0.11}}$ & ${11.41_{\pm 0.73}}$ &  ${5.75_{\pm 0.02}}$ & ${18.27_{\pm 0.54}}$ 
      \\
Tree DSB & \cellcolor{pearDark!20} $0.63_{\pm 0.26}$ & \cellcolor{pearDark!20} $\bm{1.07_{\pm 0.58}}$ & \cellcolor{pearDark!20} $\bm{1.39_{\pm 0.07}}$ & \cellcolor{pearDark!20} $\bm{1.92_{\pm 0.02}}$ & \cellcolor{pearDark!20} $\bm{2.62_{\pm 0.07}}$ \\
    \bottomrule
    \end{tabular}
    \vspace{0.2cm}
     \caption{ Gaussian setting: comparison with the regularized methods crWB
       and fsWB.}
    \label{tab:comparison_gaussian}
  \end{table}

\vspace{-0.5cm}
In this setting, we choose $\mu^\star$ to be the non-regularized barycenter and assess the dependency w.r.t. the dimension of the algorithms using the $\BWUVP$
  metric. In \Cref{tab:comparison_gaussian}, we compare ourselves with the two regularized methods \cite{li2020continuous} ($\mathrm{L}_2$-reg. equal to $10^{-4}$)
  and \cite{cuturi2014fast}. We run TreeDSB for 10 mIPF cycles with $\varepsilon=0.1$. Bold
  numbers represent the best values up to statistical significance. While \cite{li2020continuous} and \cite{cuturi2014fast} enjoy better performance in low dimensions
  ($d=2$), TreeDSB outperforms these methods as the dimension increases. 

\vspace{-0.2cm}
\paragraph{MNIST Wasserstein barycenter.} We then turn to an image experiment using MNIST dataset \citep{lecun1998mnist}. Here, an image is not considered as a 2D-dimensional distribution as in \cite{cuturi2014fast} and \cite{li2020continuous}, but as a sample from a high-dimensional probability measure $(d=784)$. We aim at computing a Wasserstein
  barycenter between the digits $2$,$4$ and $6$. To do so, we run TreeDSB for 10 mIPF cycles with $r$ that corresponds to the digit 6 and $\varepsilon=0.5$. In \Cref{fig:wasserstein_barycenter_mnist}, we display samples from the estimated marginals on the leaves, to assess the reconstruction of the digits $2$, $4$ and $6$, and samples from the barycenter, obtained by diffusing from the leaf corresponding to the digit $6$. Our results prove the scalability of TreeDSB to the high-dimensional setting, compared to state-of-the-art regularized methods.
   Additional results on MNIST dataset are given in \Cref{sec:details-exps}.

\begin{figure}[h]
  \centering
  \includegraphics[width=.2\linewidth]{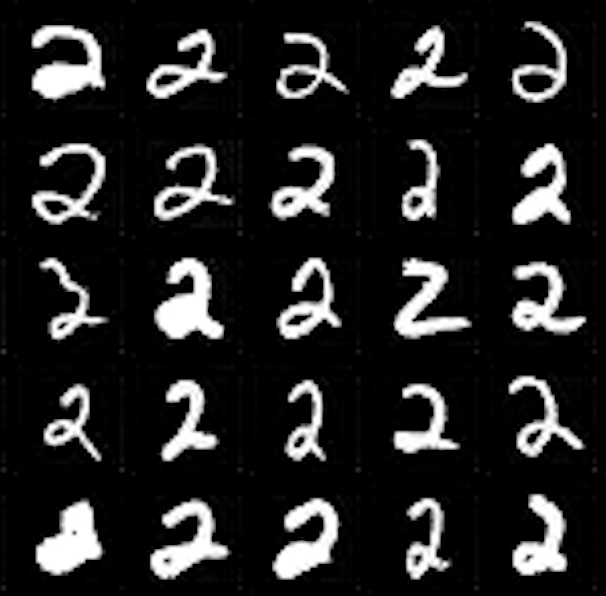} \hfill
  \includegraphics[width=.2\linewidth]{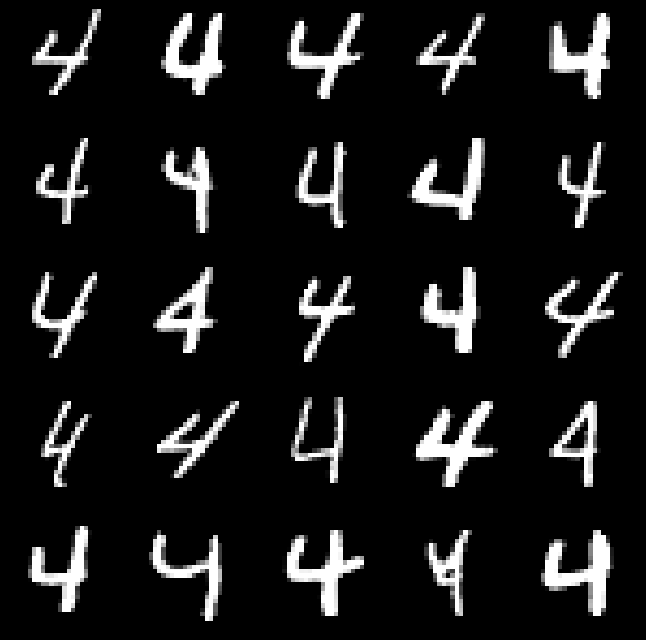} \hfill
  \includegraphics[width=.2\linewidth]{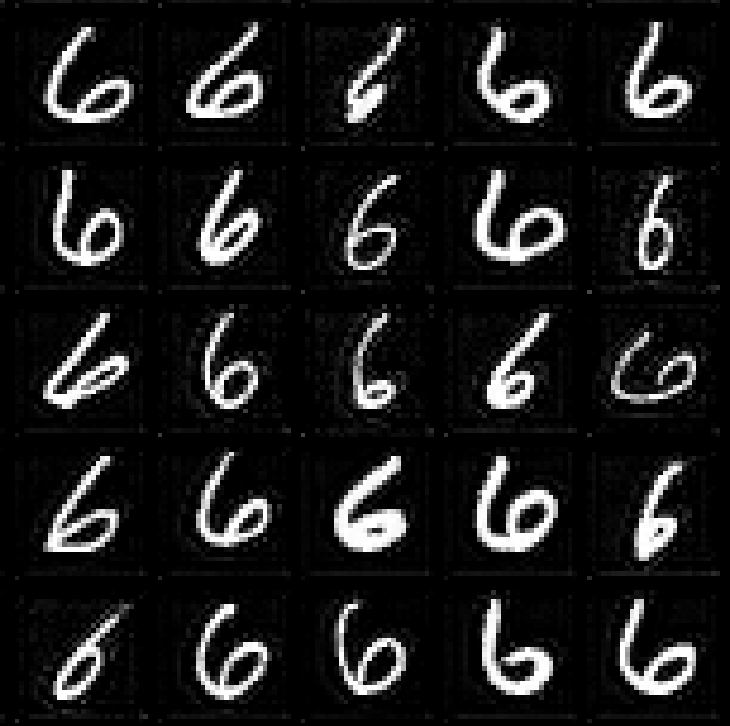} \hfill  
  \includegraphics[width=.2\linewidth]{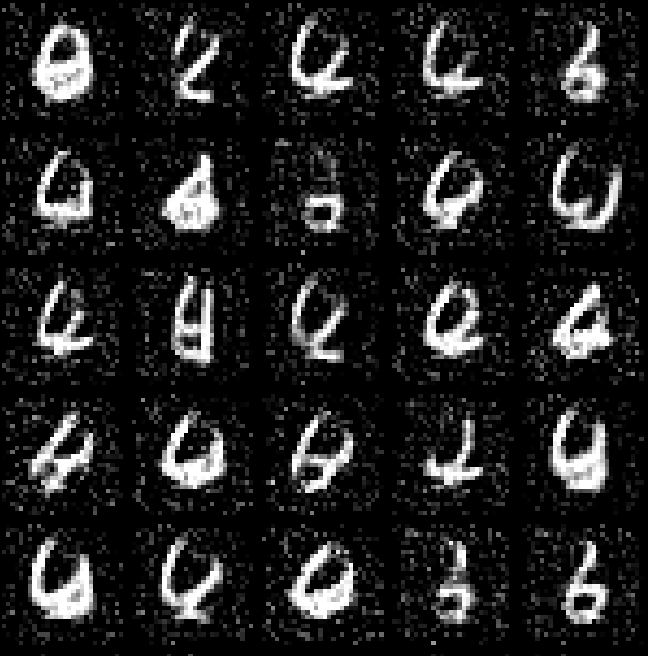}
  \caption{Samples from the estimated MNIST 2-4-6 marginals and from their Wasserstein barycenter.}
  \label{fig:wasserstein_barycenter_mnist}
\end{figure}

\begin{wraptable}{r}{8cm}
    \vspace{-0.1cm}
    \setlength\tabcolsep{5pt}
    \centering
    \small
    \begin{tabular}{lccccccccc}
    \toprule
      Method
       & Without het. & With het.
      \\
    \midrule
      fsWB \citep{cuturi2014fast} &  $12.95_{\pm 0.35}$ &  $14.43_{\pm 0.51}$  \\ 
 crWB \citep{li2020continuous} & $20.66_{\pm 0.71}$ & ${23.06_{\pm 0.12}}$
      \\      
Tree DSB & \cellcolor{pearDark!20} $\bm{8.69_{\pm 0.12}}$ & \cellcolor{pearDark!20} $\bm{8.90_{\pm 0.68}}$ \\
    \bottomrule
    \end{tabular}
     \caption{Bayesian fusion setting: comparison with the regularized methods crWB
       and fsWB.}
    \label{tab:comparison_bayesian}
\vspace{-0.2cm}
  \end{wraptable}

\paragraph{Subset posterior aggregation.} Finally, we evaluate TreeDSB in
the context of \emph{Bayesian fusion} \citep{srivastava2018scalable}, also called posterior aggregation. Given a Bayesian model and a dataset partitioned into several shards, this task aims at recovering the full data posterior distribution from the posterior distributions computed on each shard. 

In particular, it has been proved that the barycenter of the subdataset posteriors is close to the
full data posterior under mild assumptions \citep{srivastava2018scalable}. Here, we consider a logistic regression model applied to the
\texttt{wine} dataset\footnote{\url{https://archive.ics.uci.edu/ml/datasets/wine}} ($d=42$) and proceed as follows. We first split this dataset into 3 subsets, with or without heterogeneity, and estimate the posterior parameters on each shard. Then, we draw samples from the
obtained logistic distributions to define $\mu_1, \mu_2, \mu_3$. Then, we compute the Wasserstein barycenter of these measures, and compare it to the posterior computed on the full dataset. As in the synthetic Gaussian
experiment, we run TreeDSB for 10 mIPF cycles $\vareps=0.1$ and we compare ourselves with \cite{li2020continuous} ($\mathrm{L}_2$-reg. equal to $10^{-4}$) and \cite{cuturi2014fast}. We evaluate the methods using the $\BWUVP$ metric, where $\mu^\star$ is the estimated full data posterior, and report the results in \Cref{tab:comparison_bayesian}. In both settings, we observe that our method outperforms existing regularized methods to compute Wasserstein barycenters.

\paragraph{Limitations.} 
One of the main limitation of entropic regularized OT approach is that their
behavior is usually badly conditioned as $\vareps \to 0$. In our setting, we
observe that if $\vareps$, or equivalently $T$, is too low then the algorithm
becomes less stable as the training of the models slows down. In the future, we
plan to mitigate this issue by incorporating fixed point techniques like the one
used in \cite{korotin2022wasserstein}. Finally, since our algorithm is based on
DSB \citep{debortoli2021diffusion}, it suffers from the same limitations. In
particular, training different neural networks iteratively incurs some bias in
the SDE which is harmful for large number
of mIPF iterations.


\section{Discussion}
\label{sec:discussion}

In this paper, we introduced Tree-based Diffusion Schr\"odinger Bridge (TreeDSB) a
scalable scheme to approximate solutions of entropic-regularized multi-marginal
Optimal Transport (mOT) problems. Our methodology leverages tools from the diffusion
model literature and extends Diffusion Schr\"odinger Bridge
\citep{debortoli2021diffusion}. In particular, it approximates the iterates of the
multi-marginal Iterative Proportional Fitting (mIPF) algorithm, for which we prove its convergence under mild assumptions. We illustrate the efficiency of
TreeDSB for image processing and Bayesian fusion, using the link
between mOT and Wasserstein barycenters. In future work, we would like to study quantitative convergence bounds for
mIPF in the \emph{unbounded} cost setting. Another line
of work would be to scale TreeDSB to higher dimensional problems building on
recent developments in the diffusion model and flow matching community
\citep{lipman2022flow,peluchetti2023diffusion,shi2023diffusion}.


\newpage

\section*{Acknowledgments}

We thank James Thornton for the DSB codebase\footnote{\url{https://github.com/JTT94/diffusion_schrodinger_bridge}} and useful discussions. AD acknowledges support from the Lagrange
Mathematics and Computing Research Center. AD and MN would like
to thank the Isaac Newton Institute for Mathematical Sciences for support and hospitality during the programme
\emph{The mathematical and statistical foundation of future data-driven engineering} when work on this paper was undertaken. MN acknowledges funding from the grant SCAI (ANR-19-CHIA-0002).

\bibliography{bibliography}
\bibliographystyle{icml2023}

\newpage
\appendix


\section*{Appendix organization}
\label{sec:pres-append}

First, additional notation is introduced in \Cref{sec:additional-notation}. Then, we briefly recall some notions on undirected and directed trees in
\Cref{sec:recap-trees}. Similarly, martingale problems are introduced in
\Cref{sec:martingale-problems}. The proofs of the main manuscript and additional theoretical results on Tree Schr\"odinger Bridges are given in
\Cref{sec:proofs-main}. Additional details on our consideration of the tree-based static SB problem are
described in \Cref{sec:addit-deta-tree}. Details on the implementation of TreeDSB are given in \Cref{sec:details-algo}
and the experiments are investigated
in \Cref{sec:details-exps}. 


\section{Additional notation}
\label{sec:additional-notation}

\def \msj {\mathsf{J}}

For any finite set $\mse$, we equivalently refer to the cardinal of $\mse$ as $\card(\mse)$ or $\abs{\mse}$. Let $(\msx, \mathcal{X})$ be a measurable space. For any
$x \in (\rset^d)^{\ell+1}$ and any $m\in \{0, \hdots, \ell\}$, let
$x_{-m}=(x_0,\hdots, x_{m-1}, x_{m+1}, \hdots, x_\ell)$. For any family of
measures $\{\nu_j\}_{j \in \{0,\hdots, \ell\}}$ defined on $(\msx, \mathcal{X})$ and
any $i\in \{0,\hdots, \ell\}$, let
$\nu_{-i}=\bigotimes_{j \in \{0, \hdots, \ell\}\backslash\{i\}}\nu_j$. Let
$I=\{i_1, \hdots, i_q\} \subset \{1, \hdots, \ell\}$ and $\mu \in \Pmeasurell$
such that $\mu \ll \Leb$. We define
$I^\complementary=\{1, \hdots, \ell\}\backslash I$ and denote it by
$\{i^\complementary_1, \hdots, i^\complementary_{\bar{q}}\}$ where
$\bar{q}=\ell -q$. We denote the marginal of $\mu$ along $I$ by $\mu_{I}$, \ie,
$\mu_{I} \in \Pmeasure^{(q)}$ and we have for any
$\msa \in \borel((\rset^d)^q)$,
$\mu_{I}(\msa)=\int_\msx \mu(x) \prod_{j=1}^q\delta_{x_{i_j}}(\msa_j) \rmd
x$. In addition, note that $\mu_I \ll \Leb$. We denote the conditional
distribution of $\mu$ given $I$ by $\mu_{|I}(\cdot|\cdot)$, \ie,
$\mu_{|I}(\cdot|\cdot) \in \Pmeasure^{(\bar{q})} \times (\rset^d)^q$ and we have
for any $y \in (\rset^d)^q$ and any $\msa \in \borel((\rset^d)^{\bar{q}} )$,
$\mu_{| I}(\msa | y)=\int_\msx \mu(x)/\mu_I(y) \prod_{j=1}^q \delta(x_{i_j}-y_j)
\prod_{j'=1} ^{\bar{q}} \delta_{x_{i^\complementary_{j'}}}(\msa_{j'}) \rmd
x$. Remark that for any $y \in (\rset^d)^q$, $\mu_{|I}(\cdot|y) \ll \Leb$. For
any subset $\msj \subset \msi^\complementary$ with $\card(\msj)=q_\msj$, we also
define
$\mu_{\msj |\msi}(\cdot|\cdot) \in \Pmeasure^{(q_\msj)} \times (\rset^d)^q$ such
that for any $y \in (\rset^d)^q$,
$\mu_{\msj|\msi}(\cdot|y)=\{\mu_{|\msi}(\cdot|y)\}_\msj$. For a collection of
functions $\{f_i\}_{i \in \msi}$, with $\msi \subset \{1, \dots, n\}$ and
$n \in \nset$ such that $f_i: \ \rset^d \to \rset$, we define
$\oplus_{i \in \msi} f_i: (\rset^d)^n \to \rset$ such that for any
$x = (x_1, \dots, x_n) \in (\rset^d)^n$,
$\oplus_{i \in \msi}f(x) = \sum_{i \in \msi} f_i(x_i)$.

\section{Introduction to trees}
\label{sec:recap-trees}

\paragraph{Undirected tree.} An undirected graph $\mst=(\msv, \mse)$, with vertices $\msv$ and edges $\mse$, is said to be an \emph{undirected tree} if it is  \emph{acyclic} and \emph{connected} \citep[Definition 1.19.]{valiente2002algorithms}. In particular, we have $\card(\mse)=\card(\msv)-1$. The undirected edge between two nodes $v_1$ and $v_2$ is similarly denoted by $\{v_1,v_2\}$ or $\{v_2,v_1\}$. We say that $\mst'=(\msv',\mse')$ is a \emph{sub-tree} of $\mst$ if $\mst'$ is an undirected tree with vertices $\msv' \subset\msv$ and edges $\mse' \subset \mse$. For any vertex $v \in \msv$, we define the set of its \emph{neighbours} $\msn_v$ as the set of vertices $v'\in \msv$ such that $\{v,v'\}\in \mse$. The integer $\card(\msn_v)$ is referred to as the degree of $v$.  
The vertices with degree 1 are called \emph{leaves}, and we denote the set of leaves by $\msv_\msl \subset \msv$. The (unique) \emph{path} in $\mst$ between two vertices $v$ and $v'$ is the sequence of two-by-two distinct edges $\{\{v_i, v_{i+1}\}\}_{i=1}^{n}$ (with $n\geq 1$) such that $v_k=v_{k+1}$ for any $k\in \{1, \hdots, n\}$ such that $k=0 \mod (2)$, $v_1=v$ and $v_{n+1}=v'$. This path can be seen as a linear sub-tree of $\mst$, and we define $n$ as the \emph{length} of this path. We say that $\mst$ is \emph{weighted} if there exists a map $w:\mse \mapsto \rset_+$; in this case, $w(\{v_1,v_2\})$, or equivalently $w(\{v_2,v_1\})$ (also denoted by $w_{v_1,v_2}$ or $w_{v_2,v_1}$ ) is called the weight of the edge $\{v_1,v_2\}$. The tree $\mst$ is said to be \emph{rooted} in $r\in \msv$ if $r$ defines a partial ordering $\leq_{\mst, r} \subset \msv \times \msv$ such that for any $v_1,v_2 \in\msv$,
$v_1 \leq_{\mst, r} v_2$ if the node $v_1$ lies on the unique path between $r$ and $v_2$.

\paragraph{Directed tree.} Consider a directed graph $\mst_r=(\msv, \mse_r)$ rooted in $r\in \msv$. Any directed edge $e \in \mse_r$ from $v_1\in \msv$ to $v_2\in\msv$ is denoted by $(v_1,v_2)$. $\mst_r$ is a said to be a \emph{directed tree} rooted in $r$ if (i) the underlying undirected graph $\mst=(\msv, \mse)$ is an undirected tree rooted in $r$ and (ii) any $(v_1,v_2)\in \mse_r$ is directed according to the partial ordering $\leq_{\mst,r}$, \ie, $\{v_1,v_2\}\in \mse$ and $v_1\leq_{\mst,r} v_2$. For any vertices $(v,v')\in \msv \times \msv$ such that $v\leq_{\mst,r} v'$, the (unique) \emph{path} in $\mst_r$ from $v$ to $v'$, denoted by $\path_{\mst_r}(v,v')$, is defined as the directed version of the path in $\mst$ between $v$ and $v'$ (viewed as a sub-tree of $\mst$), which is rooted in $v$. We say that $\mst_r$ is \emph{weighted}, if $\mst$ is weighted and the edges of $\mst_r$ have the same weights as the corresponding undirected edges of $\mst$. For any $(v_1,v_2)\in \mse_r$, we denote this weight by $w_{v_1,v_2}$. We say that $\mst_r$ is the (unique) \emph{directed version} of $\mst$ rooted in $r$. It is endowed with a canonical vertex numbering $\zeta:\msv \to \{0,\hdots, \card(\msv)-1\}$, corresponding to a depth-first traversal of its nodes, starting from the root $r$ \citep[Definition 3.1.]{valiente2002algorithms}. This numbering is consistent with the partial ordering on $\mst$, \ie, if $v_1 \leq_{\mst,r} v_2$, $\zeta(v_1) \leq \zeta(v_2)$, and satisfies $\zeta(r)=0$. In the rest of the paper, we will write in an equivalent manner $v$ or $\zeta(v)$.

For any vertices $(v_1,v_2)\in \mse\times \mse$ such that $v_1\leq_{\mst,r}v_2$, $\operatorname{path}_{\mst_r}(v_1,v_2)$ corresponds to the ordered set of edges in $\mse_r$ which define the ordered path between two vertices $v_1$ and $v_2$. For any vertex $v \in \msv$, we define:
\begin{enumerate}[wide, labelwidth=!, labelindent=0pt, label=(\alph*)]
    \item the set of its \emph{children} $\msc_v$ as the set of vertices $v'\in \msv$ such that $(v,v')\in \mse_r$. In particular, for any $v \in \msv_L$, the set of leaves, one has $\msc_v=\emptyset$.
    \item its \emph{parent} as the unique vertex $p(v)$ such that $(p(v), v)\in \mse_r$, if $v\neq r$ (the parent of the root is not defined). 
\end{enumerate}
Note that $\msn_r=\msc_r$ and, for any vertex $v\in \msv\backslash\{r\}$, $\msn_v=\{p(v)\} \cup \msc_v$.

\begin{definition}[Tree-structured directed Probabilistic Graphical Model (PGM)] \label{def:directed_pgm} Consider a directed tree $\mst_r=(\msv, \mse_r)$. The directed PGM induced by $\mst_r$ \citep[Definition 3.4.]{koller2009probabilistic}, denoted by $\Pmeasure_{\mst_r}$, is the family of distributions $\pi\in \Pmeasure^{(\abs{\msv})}$ which have a Markovian factorization along $\mst_r$, \ie, 
\begin{align}
    \textstyle \Pmeasure_{\mst_r}=\{\pi\in \Pmeasure^{(\abs{\msv})}: \pi=\pi_r \bigotimes_{(v,v')\in \mse_r}\pi_{v'|v}\} \eqsp .
\end{align}
\end{definition} 

\begin{lemma}\label{lemma:markov-tree} Consider an undirected tree $\mst=(\msv, \mse)$. Let $(r,r')\in \msv\times \msv$. Let $\mst'$ be a sub-tree of $\mst$ with vertices $\msv'$ such that $r'\in \msv'$. Denote by $\mst'_{r'}$ the directed version of $\mst'$ rooted in $r'$. Then, for any $\pi \in \Pmeasure_{\mst_r}$, we have $\pi_{\msv'}\in \Pmeasure_{\mst'_{r'}}$.
\end{lemma}
\begin{proof} Let $(r,r')\in \msv\times \msv$. We denote by $\mst_{r}=(\msv, \mse_{r})$, respectively $\mst_{r'}=(\msv, \mse_{r'})$, the directed version of $\mst$ rooted in $r$, respectively $r'$. We define the paths $\msp_{r,r'}=\path_{\mst_{r}}(r,r')\subset \mse_{r}$ and $\msp_{r',r}=\path_{\mst_{r'}}(r',r)\subset \mse_{r'}$. It is easy to see that 
\begin{enumerate}[wide, labelwidth=!, labelindent=0pt, label=(\alph*)]
    \item $\mse_{r}\backslash \msp_{r,r'} = \mse_{r'}\backslash \msp_{r',r}$,
    \item $\msp_{r,r'}=\{(v_2,v_1): (v_1,v_2) \in \msp_{r',r} \}$,
    \item $\msp_{r',r}=\{(v_2,v_1): (v_1,v_2) \in \msp_{r,r'} \}$.
\end{enumerate}

Let $\pi \in \Pmeasure_{\mst_r}$. First note that for any $(v_1,v_2)\in \mse_{r}$, we have by Bayes decomposition $\pi_{v_1}\pi_{v_2|v_1}=\pi_{v_2}\pi_{v_1|v_2}=\pi_{v_1,v_2}$. Then it comes
\begin{align}
    \pi& = \textstyle{\pi_{r} \bigotimes_{(v_1,v_2)\in \mse_{r}}\pi_{v_2|v_1}}\\
    & = \textstyle{\pi_{r} \bigotimes_{(v_1,v_2)\in \msp_{r,r'}}\pi_{v_2|v_1} \bigotimes_{(v_1,v_2)\in \mse_{r}\backslash \msp_{r,r'}}\pi_{v_2|v_1}}\\
    & = \textstyle{\pi_{r} \bigotimes_{(v_2,v_1)\in \msp_{r',r}}\pi_{v_2|v_1} \bigotimes_{(v_1,v_2)\in \mse_{r'}\backslash \msp_{r',r}}\pi_{v_2|v_1}}\\
    & =\textstyle{\pi_{r'} \bigotimes_{(v_1,v_2)\in \msp_{r',r}}\pi_{v_2|v_1}\bigotimes_{(v_1,v_2)\in \mse_{r'}\backslash \msp_{r',r}}\pi_{v_2|v_1}}\\
    & = \textstyle{\pi_{r'} \bigotimes_{(v_1,v_2)\in \mse_{r'}}\pi_{v_2|v_1}} \eqsp, 
\end{align}
and therefore, we have $\pi \in \Pmeasure_{\mst_{r'}}$.

Let $\mst'$ be a sub-tree of $\mst$ with vertices $\msv'$ such that $r'\in \msv'$. First note that $\mse'_{r'}\subset \mse_{r'}$. Using the previous computation, we have for any $\msa \in \borel((\rset^d)^{\abs{\msv'}})$,
\begin{align}
    \pi_{\msv'}(\msa)& =\textstyle{\int_{(\rset^d)^{\abs{\msv}}} \pi_{r'}(x_{r'}) \bigotimes_{(v_1,v_2)\in \mse_{r'}}\pi_{v_2|v_1}(x_{v_2}|x_{v_1}) \prod_{v'\in \msv'}\delta_{x_{v'}}(\msa_{v'}) \rmd x} \\
    & = \textstyle{\int_{(\rset^d)^{\abs{\msv}-\abs{\msv'}}} \{\pi_{r'}(\msa_{r'})\bigotimes_{(v_1,v_2)\in \mse'_{r'}}\pi_{v_2|v_1}(\msa_{v_2}|x_{v_1}) \} \bigotimes_{(v_1,v_2)\in \mse_{r'}\backslash \mse'_{r'}}\pi_{v_2|v_1}(x_{v_2}|x_{v_1}) \rmd x_{\msv \backslash \msv'}}\\
    & = \textstyle{\{\pi_{r'} \bigotimes_{(v_1,v_2)\in \mse'_{r'}}\pi_{v_2|v_1}\}(\msa)} \eqsp,
\end{align}
which proves that $\pi_{\msv'} \in \Pmeasure_{\mst'_{r'}}$.
\end{proof}

\paragraph{Discretized undirected tree.} Let $N\geq 1$. Consider an undirected tree $\mst=(\msv, \mse)$ with weights $w$. We say that $\mst^{(N)}=(\msv^{(N)}, \mse^{(N)})$ is a $N$-discretized version of $\mst$ if it is an undirected tree with weights $w^{(N)}$ such that
\begin{enumerate}[wide, labelwidth=!, labelindent=0pt, label=(\alph*)]
    \item $\msv^{(N)} = \msv \bigsqcup \cup _{\substack{e \in \mse, \\ k \in \{1,\ldots,N-1\}}} \{v_e^{k}\}$,
    \item $\mse^{(N)} = \cup_{e\in\mse} \cup_{k=0, \ldots,N-1} \left\{\{v_e^{k}, v_e^{k+1}\}\right\}$ with the convention that the vertices $v_e^{N}$ and $v_e^{N}$ are defined such that $\{v_e^{0}, v_e^{N}\}=e$,
    \item $\sum_{e\in \operatorname{path}_{\mst}(v,v')}1/w^{(N)}_{e}=1/w_{v,v'}$, if $\{v,v'\} \in \mse$.
\end{enumerate}
 Remark that the leaves of $\mst^{(N)}$ are exactly the original leaves of $\mst$ and that $\mst^{(1)}=\mst$. The non-uniqueness of $\mst^{(N)}$ comes from the freedom of choice on the weights of its edges.

\paragraph{Discretized directed tree.} Let $N\geq 1$. Consider a directed tree $\mst_r=(\msv, \mse_r)$ rooted in $r\in \msv$ with weights $w$. We say that $\mst_r^{(N)}=(\msv^{(N)}, \mse_r^{(N)})$ is a $N$-discretized version of $\mst_r$ if it is the directed version of $\mst^{(N)}$ rooted in $r$, where $\mst^{(N)}$ is a $N$-discretized version of the underlying undirected tree of $\mst_r$.


\section{Background on martingale problems}
\label{sec:martingale-problems}

In this section, we introduce the background on Stochastic Differential
Equations (SDEs) and weak solutions of SDEs following the framework of
\cite[Section 10.1, page 249]{stroock1997multidimensional}. We recall that
$\rmC_0^\infty(\rset^d)$ is the space of infinitely differentiable real-valued
functions which vanish at infinity. In addition, we have that $\mathcal{S}_+^d$
is the space of $d \times d$, symmetric, non-negative matrices.

\begin{definition}
  Let $T > 0$ or $T=+\infty$,
  $\sigma: \ \coint{0,T} \times \rset^d \to \mathcal{S}_+^d$ and
  $b: \ \coint{0,T} \times \rset^d \to \rset^d$, locally bounded
  measurable functions. We define the \emph{infinitesimal generator},
  $\mathcal{A}$, given for any $f \in \rmC_0^\infty(\rset^d)$,
  $t \in \coint{0,T}$ and $x \in \rset^d$ by
  \begin{equation}
    \label{eq:infinitesimal_generator_def}
    \textstyle \mathcal{A}_t(f)(x) = \langle b_t(x), \nabla f(x) \rangle + \tfrac{1}{2} \langle \sigma_t(x) \sigma_t(x)^\top, \nabla^2 f(x) \rangle .   
  \end{equation}
  We say that a probability measure $\Pbb$ \emph{satisfies the martingale
    problem for $\mathcal{A}$} if for any $t \in \coint{0,T}$ and
  $f \in \rmC_0^\infty(\rset^d)$, we have that
  $(f(\bfX_t) - \int_0^t \mathcal{A}_s(f)(\bfX_s) \rmd s)_{s \in \ccint{0,t}}$
  is a $\Pbb$-martingale.
\end{definition}

In the main document, see \Cref{sec:backgr-optim-transp}, we say that ``a path
measure $\Pbb$ is associated with
$\rmd \bfX_t = b(t, \bfX_t) \rmd t + \sigma(t, \bfX_t) \rmd \bfB_t$ with
$(\bfB_t)_{t \geq 0}$ a $d$-dimensional Brownian motion'' if $\Pbb$ solves the
martingale problem associated with $\mathcal{A}$ given by
\eqref{eq:infinitesimal_generator_def}. Unless specified, we always assume that
such a path measure exists and is unique. Below, we recall the following
theorem, see \cite[Theorem 10.2.2]{stroock1997multidimensional}, which gives
sufficient conditions for the existence and uniqueness of solutions to the
martingale problem.

\begin{theorem}
  Assume that for any $x \in \rset^d$ we have
  \begin{align}
    &\textstyle \inf \ensembleLigne{\langle \theta, \sigma \sigma^\top(s, x) \theta \rangle }{\theta \in \rset^d, \ \normLigne{\theta}=1, \ s \in \ccint{0,T}} > 0 , \\
    &\lim_{y \to x} \sup \ensembleLigne{\normLigne{\sigma(s,x) - \sigma(s,y)}}{s \in \ccint{0,T}} = 0 . 
  \end{align}
  In addition, assume that there exists $C > 0$ such that for any $x \in \rset^d$
  \begin{equation}
    \textstyle \sup \ensembleLigne{\normLigne{\sigma \sigma^\top(t,x)}}{s \in \ccint{0,T}} + \sup \ensembleLigne{\langle x, b(t,x) \rangle}{s \in \ccint{0,T}} \leq C ( 1+ \normLigne{x}^2) .
  \end{equation}
  Then, there exists a unique solution to the martingale problem with
  initialization $x_0 \in \rset^d$.
\end{theorem}


\section{Theoretical results on Tree Schr\"odinger Bridges}
\label{sec:proofs-main}

We respectively provide in \Cref{ref:subsec:proofs-treeDSB}, \Cref{subsec:proofs-results} and \Cref{subsec:proofs-barycenter} the proofs of the results of the main manuscript presented in \Cref{sec:tree-based-diffusion}, \Cref{sec:theor-prop-tree} and \Cref{sec:barycenter}. Finally, we make a detailed comparison between our setting and the framework of \cite{haasler2021multimarginal} in \Cref{subsec:proofs-haasler}. In the rest of this section, we consider an undirected tree $\mst=(\msv, \mse)$, where $\abs{\msv}=\ell+1$, and some subset $\mss\subset \msv$ which we denote by $\mss=\{i_0, \hdots, i_{K-1}\}$. We define $\mss^\complementary=\msv \backslash \mss$.

\subsection{Proofs of \Cref{sec:tree-based-diffusion}}\label{ref:subsec:proofs-treeDSB}

\Cref{prop:init_tree_dsb} is straightforward to obtain by combining the definition of the Brownian motion with the definition of $\pi^0$ given in \eqref{eq:pi_0}. The following lemma details the recursion relation between the \eqref{eq:ipf_multimarginal} iterates, which is key to prove \Cref{prop:sinkhorn_continuous}.

\begin{lemma}\label{lemma:mipf-markov}Let $(\pi^n)_{n \in \nset}$ be the sequence given by
  \eqref{eq:ipf_multimarginal}. Let $n\in \nset$, $k_n=(n-1) \mod(K)$, $k_n+1=n \mod(K)$. Denote by $\mst_{k_n}$, respectively $\mst_{k_n+1}$ with edges $\mse_{k_n+1}$, the directed  version of $\mst$ rooted in $i_{k_n}$, respectively in $i_{k_n+1}$. We have:
  \begin{enumerate}[label=(\roman*)]
      \item \label{item:a_markov}$\pi^n \in \Pmeasure_{\mst_{k_n}}$,
          \item \label{item:b_markov}
          $\pi^{n+1}=\mu_{i_{k_n+1}} \bigotimes_{(v,v')\in \mse_{k_n+1}}\pi^n_{v'|v}$. In particular,
          for any $(v,v')\in \mse_{k_n+1}$, $\pi^{n+1}_{v'|v} = \pi^{n}_{v'|v}$.
  \end{enumerate}
\end{lemma}
\begin{proof} We show the result \ref{item:a_markov} by recursion on $n\in \nset$, and will deduce \ref{item:b_markov} from the proof. Using \eqref{eq:pi_0}, we first have $\pi^0 \in \Pmeasure_{\mst_r}$, where $r$ is chosen as $i_{K-1}$, see \Cref{sec:tree-based-diffusion}. Thus, we obtain the result \ref{item:a_markov} at step $n=0$. Assume now that $\pi^n \in \Pmeasure_{\mst_{k_n}}$ for some $n\in \nset$.

Consider the paths $\msp_{n} = \path_{\mst_{k_n}}(i_{k_n}, i_{k_n+1})$ and $\msp_{n+1} = \path_{\mst_{k_n+1}}(i_{k_n+1}, i_{k_n})$. Note that these two paths have the same length, denoted by $J$, and contain the same vertices, denoted by $\msv_{n}$. Let $\pi \in \Pmeasure^{(\ell +1)}$ such that $\KL(\pi | \pi^{n}) < +\infty$. We have
  the following decomposition
  \begin{equation}
    \textstyle 
    \KL(\pi | \pi^{n}) = \KL(\pi_{\msv_{n}} | \pi^{n}_{\msv_{n}}) + \int_{(\rset^{d})^{J+1}} \KL(\pi_{|\msv_{n}} | \pi^{n}_{|\msv_{n}}) \rmd \pi_{\msv_{n}}(x_{\msv_{n}}) \eqsp .
  \end{equation}
Hence, the $(n+1)$-th iterate of \eqref{eq:ipf_multimarginal} is given by $\pi^{n+1} = \pi^{n+1}_{\msv_{n}} \otimes \pi^{n}_{|\msv_{n}}$,
  with
  \begin{equation}
     \pi_{\msv_{n}}^{n+1} = \argmin \ensembleLigne{\KL(\pi | \pi^{n}_{\msv_{n}})}{\pi \in \Pmeasure^{(J+1)}, \ \pi_{i_{k_n+1}} = \mu_{i_{k_n+1}}} \eqsp . 
   \end{equation}
   
   Since $\pi^{n}\in \Pmeasure_{\mst_{k_n}}$, we have (i) $\pi^{n}_{|\msv_{n}}= \bigotimes_{(v,v')\in \mse_{k_n} \backslash \msp_{n}}\pi^n_{v'|v}$ and (ii) $\pi^{n}_{\msv_{n}}\in \Pmeasure_{\msp_{n+1}}$ by \Cref{lemma:markov-tree}, where $\msp_{n+1}$ is viewed as a directed tree rooted in $i_{k_n+1}$. Defining $\msv_{n+1}=\msv_{n} \backslash \{i_{k_n+1}\}$, we thus have $\pi^{n}_{\msv_{n}}= \pi^n_{i_{k_n+1}}\otimes\pi^{n}_{\msv_{n+1} | i_{k_n+1}}$ where $\pi^{n}_{\msv_{n+1} | i_{k_n+1}}= \bigotimes_{(v,v')\in \msp_{n+1} }\pi^n_{v'|v}$.

   Let $\pi \in \Pmeasure^{(J+1)}$ such that $\pi_{i_{k_n+1}} = \mu_{i_{k_n+1}}$ and $\KL(\pi | \pi^{n}_{\msv_{n}}) < +\infty$. Similarly to the previous computation, we have the following decomposition
   \begin{align}
    \KL(\pi | \pi^{n}_{\msv_{n}}) &= \textstyle{\KL(\pi_{i_{k_n+1}} | \pi^{n}_{i_{k_n+1}}) + \int_{\rset^{d}} \KL(\pi_{|i_{k_n+1}} | \pi^{n}_{\msv_{n+1} | i_{k+1}}) \rmd \pi_{i_{k_n+1}}(x_{i_{k_n+1}})}\\
    & = \textstyle{ \KL(\mu_{i_{k_n+1}} | \pi^{n}_{i_{k_n+1}}) + \int_{\rset^{d}} \KL(\pi_{|i_{k_n+1}} | \pi^{n}_{\msv_{n+1} | i_{k_n+1}}) \rmd \mu_{i_{k_n+1}}(x_{i_{k_n+1}})} \eqsp .
  \end{align}
  Therefore, we obtain 
  \begin{align}
      \textstyle\pi^{n+1}_{\msv_{n}} = \mu_{i_{k_n+1}} \otimes \pi^{n}_{\msv_{n+1}| i_{k+1}}= \mu_{i_{k+1}}\bigotimes_{(v,v')\in \msp_{n+1} }\pi^n_{v'|v} \eqsp .
  \end{align}
  Noting that $\mse_{k_n} \backslash \msp_{n}= \mse_{k_n+1} \backslash \msp_{n+1}$ and recalling that $\pi^{n+1}=\pi^{n+1}_{\msv_{n}} \otimes \pi^{n}_{|\msv_{n}}$, it finally comes
  \begin{align} \label{eq:markov_it}
      \pi^{n+1} = \mu_{i_{k_n+1}}\bigotimes_{(v,v')\in \msp_{n+1} }\pi^n_{v'|v} \bigotimes_{(v,v')\in \mse_{k_n+1} \backslash \msp_{n+1}}\pi^n_{v'|v} = \mu_{i_{k_n+1}} \bigotimes_{(v,v')\in \mse_{k_n+1}}\pi^n_{v'|v} \eqsp.
  \end{align}
Therefore, $\pi^{n+1}\in \Pmeasure_{\mst_{k_n+1}}$, which achieves the recursion for \ref{item:a_markov}, and we obtain \ref{item:b_markov} by \eqref{eq:markov_it}.
\end{proof}

Hence, \Cref{lemma:mipf-markov} shows that the \eqref{eq:ipf_multimarginal} iterates admit a Markovian factorization on $\mst$, and can be defined recursively using the edges of $\mst$. We now provide the proof of \Cref{prop:sinkhorn_continuous}.

\begin{proof}[Proof of \Cref{prop:sinkhorn_continuous}] We will prove this result by recursion on $n\in \nset$. Observe that the initialisation is directly given by \Cref{prop:init_tree_dsb}. Assume now that the result of \Cref{prop:sinkhorn_continuous} stands for some $n\in \nset$. Let $k_n=(n-1)\mod(K)$, $k_n+1=n \mod(K)$. Denote by $\mst_{k_n}$ with edges $\mse_{k_n}$, respectively $\mst_{k_n+1}$ with edges $\mse_{k_n+1}$, the directed version of $\mst$ rooted in $i_{k_n}$, respectively in $i_{k_n+1}$. For any vertex $v$ of $\mst_{k_n+1}$, we define $p(v)$ as the (unique) parent of $v$ and $c(v)$ as the unique child of $v$ when it exists. Consider the $(n+1)$-th dynamic iterate defined by \ref{item:a_sinkhorn} and \ref{item:b_sinkhorn}, \ie, $(\Pbb^{n+1}_{(v,v')})_{(v,v')\in \mse_{k_n+1}}$. To prove that this iterate has the properties stated in \Cref{prop:sinkhorn_continuous}, we proceed by recursion on the edges of $\mst_{k_n+1}$, following the bread-first order in $\mst_{k_n+1}$. In this order, the edge $(i_{k_n+1}, c(i_{k_n+1}))$ is the first to be considered. Remark that $c(i_{k_n+1})$ is well defined since $i_{k_n+1}$ is a leaf of $\mst$.

Here, we denote $T_{c(i_{k_n+1}),i_{k_n+1}}$ by $T$. By construction, we have $\Pbb^{n+1}_{(i_{k_n+1}, c(i_{k_n+1}))}= \mu_{i_{k_n+1}} \otimes (\Pbb^{n}_{(c(i_{k_n+1}), i_{k_n+1})})^R_{|0}$. By recursion assumption, $\Pbb^{n}_{(c(i_{k_n+1}), i_{k_n+1})}\in \Pmeasure(\rmC(\ccintLigne{0,T}, \rset^d))$ since $(c(i_{k_n+1}),i_{k_n+1})\in \mse_{k_n}$. Then, $\Pbb^{n+1}_{(i_{k_n+1}, c(i_{k_n+1}))}$ is a well defined path measure on  $\ccintLigne{0,T}$. By definition of the \eqref{eq:ipf_multimarginal} sequence, we have $\mu_{i_{k_n+1}}=\pi^{n+1}_{i_{k_n+1}}$. By recursion assumption, we also have that $\mathrm{Ext}(\Pbb^{n}_{(c(i_{k_n+1}), i_{k_n+1})})=\pi^n_{c(i_{k_n+1}), i_{k_n+1}}$. Hence, it comes that $(\Pbb^{n}_{(c(i_{k_n+1}), i_{k_n+1})})^R_{T|0}=\pi^n_{c(i_{k_n+1}) | i_{k_n+1}}=\pi^{n+1}_{c(i_{k_n+1}) | i_{k_n+1}}$, where the last equality comes from \Cref{lemma:mipf-markov}. Finally, we obtain that $\mathrm{Ext}(\Pbb^{n+1}_{(i_{k_n+1}, c(i_{k_n+1}))})=\pi^{n+1}_{i_{k_n+1}, c(i_{k_n+1})}$, which proves the initialisation.

Assume now that $\Pbb^{n+1}$ is well defined and has the right properties, up to some edge in $\mst_{k_n+1}$. Consider the following edge, denoted by $(v,v')\in \mse_{k_n+1}$, in the breadth-first order. By edge recursion, we have that $\mathrm{Ext}(\Pbb^{n+1}_{(p(v),v)})=\pi^{n+1}_{p(v),v}$, and thus $\Pbb^{n+1}_{(p(v),v),T_{p(v),v}}=\pi^{n+1}_{v}$. Define the path $\msp_n=\operatorname{path}_{\mst_{k_n}}(i_{k_n}, i_{k_n+1})$. Then, we face two cases.

(i) Either $(v,v') \in \mse_{k_n} \backslash \msp_n$. Then, we have by \ref{item:a_sinkhorn} that 
\begin{align}
    \Pbb^{n+1}_{(v,v')} = \Pbb^{n+1}_{(p(v),v),T_{p(v),v}} \otimes \Pbb^{n}_{(v,v')|0}=\pi^{n+1}_{v} \otimes \Pbb^{n}_{(v,v')|0}
\end{align}
 In particular, $\Pbb^{n+1}_{(v,v')}$ is a well defined path measure on $\ccintLigne{0,T_{v,v'}}$. Since $(v,v')\in \mse_{k_n}$, $\mathrm{Ext}(\Pbb^{n}_{(v,v')})=\pi^n_{v,v'}$ by recursion assumption. In particular, $\Pbb^{n}_{(v,v'),T_{v,v'}|0}=\pi^n_{v'|v}=\pi^{n+1}_{v'|v}$ where the last equality comes from \Cref{lemma:mipf-markov}. We thus have $\mathrm{Ext}(\Pbb^{n+1}_{(v,v')})=\pi^{n+1}_{v,v'}$.

(ii) Or $(v',v) \in \msp_n$. Then, we have by \ref{item:b_sinkhorn} that
\begin{align}
    \Pbb^{n+1}_{(v,v')} = \Pbb^{n+1}_{(p(v),v),T_{v,v'}} \otimes (\Pbb^{n}_{(v',v)})^R_{|0}= \pi^{n+1}_{v}\otimes (\Pbb^{n}_{(v',v)})^R_{|0}
\end{align}
In particular, $\Pbb^{n+1}_{(v,v')}$ is a well defined path measure on $\ccintLigne{0,T_{v,v'}}$. Here, $(v',v) \in \mse_{k_n}$ and thus, $\mathrm{Ext}(\Pbb^{n}_{(v',v)})=\pi^n_{v',v}$ by recursion assumption. In particular, $(\Pbb^{n}_{(v',v)})^R_{T_{v',v}|0}=\pi^n_{v'|v}=\pi^{n+1}_{v'|v}$ where the last equality comes from \Cref{lemma:mipf-markov}. We thus have $\mathrm{Ext}(\Pbb^{n+1}_{(v,v')})=\pi^{n+1}_{v,v'}$.

This achieves the recursion.
\end{proof}

\subsection{Proofs of \Cref{sec:theor-prop-tree}}\label{subsec:proofs-results}

\paragraph{Remark on assumption \Cref{ass:measures}.} Although \emph{\Cref{ass:measures} is not needed to establish the result of \Cref{prop:existence}, \Cref{corollary:existence} and \Cref{prop:marginal}}, it is however crucial in the proof of convergence of \eqref{eq:ipf_multimarginal} stated in \Cref{prop:conv_mipf}. Nevertheless, we choose to keep \Cref{ass:measures} as an assumption in the statement of every theoretical result presented in \Cref{sec:theor-prop-tree} for sake of clarity.

\paragraph{Additional definitions.} We define the set $\Pmeasure_\mss=\cap_{i \in \mss} \Pmeasure_i$, where $\Pmeasure_i= \{\pi \in \Pmeasure^{(\ell +1)}: \eqsp \pi_i=\mu_i\}$, \ie, $\Pmeasure_\mss$ is the set of all probability measures $\pi \in \Pmeasure^{(\ell +1)}$ which verify
\begin{align}
    \textstyle\int_{(\rset^d)^{\ell+1}} f_i(x_i) \rmd \pi(x_{0:\ell}) = \int_{\rset^d} f_i(x_i) \rmd \mu_i(x_i) \eqsp,
\end{align}
for any family of bounded measurable functions $\{f_i\}_{i \in \mss}\in \mathrm{C}(\rset^d, \rset)^K$. Since $\rset^d$ is separable, there exists a dense family of functions $\{f^k_i\}_{k \in \nset^*, i \in \mss}$, with $f_i^k \in \mathrm{L}^\infty(\mu_i)$ for any $k\in \nset^*$ and any $i\in \mss$, such that $\pi \in \Pmeasure_\mss$ if and only if
 \begin{align}
     \textstyle\int_{(\rset^d)^{\ell+1}} f_i^k(x_i) \rmd \pi(x_{0:\ell})= \int_{\rset^d} f_i^k(x_i) \rmd \mu_i(x_i)
 \end{align}
 or equivalently, upon centering $f_i^k$, 
 \begin{align}
    \textstyle \int_{(\rset^d)^{\ell+1}} f_i^k(x_i) \rmd \pi(x_{0:\ell})= 0 \eqsp .
 \end{align}
 In the rest of the section, we consider such family $\{f^k_i\}_{k \in \nset^*, i \in \mss}$.

For any $n\in \nset^*$, we also define $\Pmeasure^n_\mss=\cap_{i \in \mss} \Pmeasure_i^n$, where $\Pmeasure_i^n=\{\pi \in \Pmeasure^{(\ell +1)}: \int_{(\rset^d)^{\ell+1}} f_i^k(x_i) \rmd \pi(x_{0:\ell})= 0, \eqsp \forall k \in \{1, \hdots, n\} \}$. In particular, we have
\begin{align} \label{eq:pmeasure_mss}
    \Pmeasure_\mss=\cap_{n \in \nset^*} \Pmeasure^n_\mss \eqsp .
\end{align}
Finally, \eqref{eq:multimarginal_sb_static_nu} can be rewritten as 
\begin{align} \label{eq:msb_P-s}
    \pi^\star=\argmin \{\KL(\pi \mid \pi^0): \pi \in \Pmeasure_\mss\} \eqsp .
\end{align}

\paragraph{Proof of \Cref{prop:existence} and \Cref{corollary:existence}.} In this part of the section, we present an extension of the theoretical results from \cite{nutz2021introduction} to the multi-marginal setting. We first present two technical results, \Cref{lemma:nutz_1} and \Cref{lemma:nutz_2}, which are respectively adapted from \cite[Lemma 2.10.]{nutz2021introduction} and \cite[Lemma 2.11.]{nutz2021introduction}.

\begin{lemma}\label{lemma:nutz_1} Let $\{\tilde{\mu}_j\}_{j \in \mss^\complementary}$ be a family of probability measures defined on $(\rset^d, \mathcal{B}(\rset^d))$. We define $\tilde{\pi}^0=\bigotimes_{i\in \mss}\mu_i \bigotimes_{j \in \mss^\complementary}\tilde{\mu}_j$. Let $\msa \in \bigotimes_{m=0}^\ell \mathcal{B}(\rset^d)$ such that $\tilde{\pi}^0(\msa)=1$. Then, for $\tilde{\pi}^0$-almost any $x^\star \in \msa$, there exists a family of sets $\{\msx_m^0\}_{m=0}^\ell \subset (\rset^d)^{\ell+1}$ such that
\begin{enumerate}[wide, labelwidth=!, labelindent=0pt, label=(\alph*)]
    \item \label{item:nutz_1-a} $\mu_i(\msx^0_i)=1$ for any $i \in \mss$, and $\tilde{\mu}_j(\msx^0_j)=1$ for any $j\in \mss^\complementary$,
    \item \label{item:nutz_1-b}$\msa^0=\msa \cap (\prod_{m=0}^\ell \msx^0_m)$ satisfies $x^\star \in \msa^0$ and 
    \begin{align}
         (x_0^\star, \hdots, x^\star_{m-1}, x_m,x^\star_{m+1}, \hdots, x_\ell^\star)\in \msa^0, \forall x \in \msa^0, \forall m \in \{0, \hdots, \ell\} \eqsp .
    \end{align}
\end{enumerate}
\end{lemma}

\begin{proof} Consider such set $\msa$. We define for any $m \in \{0, \hdots, \ell\}$ the set
\begin{equation}
     \msx_m = \{u \in \rset^d: \tilde{\pi}^0_{-m}(\msa_m^u)=1\}\eqsp ,
   \end{equation}
   where
   $\msa_m^u=\{y\in (\rset^d)^\ell: (y_0, \hdots, y_{m-1}, u, y_m, \hdots,
   y_{\ell-1}) \in \msa\}$.

   Take $i\in \mss$. Assume that $\mu_i(\msx_i)<1$. We recall that
   $\tilde{\pi}^0=\tilde{\pi}^0_{-i} \otimes \mu_i$. Using Fubini's theorem and
   that $\int_{\msa_i^{x_i}} \rmd \tilde{\pi}^0_{-i}(x_{-i})<1$ for any
 $x_i \not \in \msx_i$, we have 
\begin{align}
  1 = \tilde{\pi}^0(\msa) & = \textstyle{\int_{\msa} \rmd\tilde{\pi}^0_{-i}(x_{-i}) \otimes \rmd \mu_i(x_i)} \\
                          &= \textstyle{\int_{\rset^d}  \{\int_{\msa_i^{x_i}} \rmd \tilde{\pi}^0_{-i}(x_{-i})\}\rmd \mu_i(x_i)} \\
                          & = \textstyle{\int_{\msx_i}} \{\textstyle{\int_{\msa_i^{x_i}}} \rmd \tilde{\pi}^0_{-i}(x_{-i})\} \rmd \mu_i(x_i)  + \textstyle{\int_{\msx_i^\complementary}} \{\textstyle{\int_{\msa_i^{x_i}}} \rmd \tilde{\pi}^0_{-i}(x_{-i})\} \rmd \mu_i(x_i) \\
                          &< \mu_i(\msx_i) + \mu_i(\msx_i^\complementary) = 1 \eqsp ,
\end{align}
which is absurd. Therefore, we obtain $\mu_i(\msx_i)=1$, and similarly, we have
$\tilde{\mu}_j(\msx_j)=1$ for any $j\in \mss^\complementary$. For any
$y\in (\rset^d)^\ell$, any $m\in\{0, \hdots, \ell\}$, we define the set
\begin{align}
    \bar{\msa}_m^y=\{u\in \rset^d: (y_0, \hdots, y_{m-1}, u, y_{m}, \hdots, y_{\ell-1}) \in \msa\} \eqsp .
\end{align}

Let $i\in \mss$. We have by Fubini's theorem
\begin{align}
    1 = \tilde{\pi}^0(\msa)& = \textstyle{\int_{\msa} \rmd \mu_i(x_i) \otimes \rmd \tilde{\pi}^0_{-i}(x_{-i})} \\
    &= \textstyle{\int_{(\rset^d)^{\ell}} \{\int_{\bar{\msa}_i^{x_{-i}}} \rmd \mu_i(x_i)\} \rmd \tilde{\pi}^0_{-i}(x_{-i})} \\
    &= \textstyle{\int_{\prod_{\substack{m=0\\ m\neq i}}^{\ell} \msx_i} \{\int_{\bar{\msa}_i^{x_{-i}}} \rmd \mu_i(x_i)\} \rmd \tilde{\pi}^0_{-i}(x_{-i})} \eqsp ,
\end{align}
where the last equality comes from the fact that $\mu_i(\msx_i)=1$ for any
$i \in \mss$, $\tilde{\mu}_j(\msx_j)=1$ for any $j \in \mss^{\complementary}$
and that
$\tilde{\pi}^0=\bigotimes_{i\in \mss}\mu_i \bigotimes_{j \in
  \mss^\complementary}\tilde{\mu}_j$. Consequently, there exists a measurable
set $\msa_{-i}\subset \prod_{\substack{m=0\\ m\neq i}}^{\ell} \msx_i$ such that
the following properties hold:
\begin{enumerate*}[label=(\alph*)]
    \item $\mu_i(\bar{\msa}_i^{y})=1$ for any $y\in \msa_{-i}$, 
    \item $\tilde{\pi}^0_{-i}(\msa_{-i})=1$.
\end{enumerate*}
Similarly, this result holds for any $j\in \mss^\complementary$, \ie, there exists a measurable set $\msa_{-j}\subset \prod_{\substack{m=0\\ m\neq j}}^{\ell} \msx_i$ such that the following properties hold:
\begin{enumerate*}[label=(\alph*)]
    \item $\tilde{\mu}_j(\bar{\msa}_j^{y})=1$ for any $y\in \msa_{-j}$, 
    \item $\tilde{\pi}^0_{-j}(\msa_{-j})=1$.
\end{enumerate*}
We consider such sets $\{\msa_{-m}\}_{m=0}^{\ell}$ for the rest of the proof and finally define the set
\begin{align}
    \tilde{\msa}&=\cap_{m=0}^\ell \tilde{\msa}_m \eqsp ,
\end{align}
where
$\tilde{\msa}_m= \msa_{-m} \times \{u \in \bar{\msa}_m^{y}: y\in \msa_{-m}\}$.
By definition, we have $\tilde{\msa}\subset \msa \cap \prod_{m=0}^\ell \msx_m$,
using the fact that $\tilde{\msa}_m \subset \msa$ for any
$m\in \{0, \hdots, \ell\}$. In addition, for any $i\in \mss$, we get by
Fubini's theorem
\begin{align}
    \textstyle\tilde{\pi}^0(\tilde{\msa}_i)=\int_{\tilde{\msa}_i} \rmd \mu_i(x_i) \otimes \rmd \tilde{\pi}^0_{-i}(x_{-i})= \int_{\msa_{-i}} \{\textstyle{\int_{\bar{\msa}_i^{x_{-i}}}} \rmd \mu_i(x_i)\} \rmd \tilde{\pi}^0_{-i}(x_{-i})= \tilde{\pi}^0_{-i}(\msa_{-i})=1 \eqsp,
\end{align}
and similarly, we get $\tilde{\pi}^0(\tilde{\msa}_j)=1$ for any $j\in \mss^\complementary$. We can deduce that $\tilde{\pi}^0(\tilde{\msa})=1$ since $\tilde{\pi}^0(\tilde{\msa}^\complementary) \leq \sum_{m=0}^\ell \tilde{\pi}^0(\tilde{\msa}_m^\complementary) =0$. 

Let $x^\star \in \tilde{\msa}$. In particular, $x^\star \in \msa$. We define the set $\msa^0=\msa \cap (\prod_{m=0}^\ell \msx^0_m)$, where $\msx^0_m= \msx_m\cap \bar{\msa}^{x^\star_{-m}}_m$ for any $m \in \{0, \hdots, \ell\}$. We now establish the result of \Cref{lemma:nutz_1}.

 We first prove \ref{item:nutz_1-a}. Let $i \in \mss$. Since $x^\star\in \tilde{\msa}$, we have $x^\star\in \tilde{\msa}_i$ and therefore $x^\star_{-i}\in \msa_{-i}$. By definition of $\msa_{-i}$, we obtain that $\mu_i(\bar{\msa}_i^{x^\star_{-i}})=1$ and thus,
 \begin{align}
     \mu_i(\{\msx^0_i\}^\complementary) \leq \mu_i(\msx_i^\complementary) + \mu_i(\{\bar{\msa}_i^{x^\star_{-i}}\}^\complementary)=0,
 \end{align}
 which gives $\mu_i(\msx^0_i)=1$, and similarly, we have $\tilde{\mu}_j(\msx^0_j)=1$ for any $j\in \mss^\complementary$.

 We now prove \ref{item:nutz_1-b}. Let $m\in \{0, \hdots, \ell\}$.  Since
 $x^\star\in \tilde{\msa}\subset \msa$, we get
 $x^\star_m \in \bar{\msa}^{x^\star_{-m}}_m$. Using that
 $\tilde{\msa}\subset \msa \cap_{m=0}^\ell \msx_m$, we get $x^\star \in
 \msa^0$. Let $x\in \msa^0$. We denote
 $x^{m}=(x_0^\star, \hdots, x^\star_{m-1}, x_m,x^\star_{m+1}, \hdots,
 x_\ell^\star)$. We need to show that $x^m \in \msa$ and
 $x^m \in \prod_{j=1}^\ell \msx_j^0 = \prod_{j=1}^\ell (\msx_j \cap
 \bar{\msa}_j^{x^\star_{-m}})$.  First, since $x^m_j = x_j$ or $x_j^\star$ for
 any $j \in \{0, \dots, \ell\}$, and $x \in \msa^0$ and $x^\star \in \msa^0$, we
 get that for any $j \in \{0, \dots, \ell\}$, $x^m_j \in \msx_j$. Similarly, for
 any $j \in \{0, \dots, \ell-1\}$, $x_j^m \in
 \bar{\msa}_j^{x^\star_m}$. Therefore, we get that
 $x^m \in \prod_{j=1}^\ell (\msx_j \cap \bar{\msa}_j^{x^\star_{-m}})$. Since
 $x_m \in \msa_m^{x^\star_{-m}}$ (because
 $x \in \prod_{j=1}^\ell (\msx_j \cap \bar{\msa}_j^{x^\star_{-m}})$), we get
 that $x \in \msa$, which concludes the proof. 
\end{proof}

\begin{lemma} \label{lemma:nutz_2}Let $\msa^0 \subset (\rset^d)^{\ell+1}$. For any $m\in \{0, \hdots, \ell\}$, we denote $\msx^0_m=\operatorname{proj}_m(\msa^0)$. We make the following assumptions.
\begin{enumerate}[wide, labelwidth=!, labelindent=0pt, label=(\alph*)]
    \item \label{item:ass_lemma_1}Assume there exists $x^\star\in \msa^0$ such that for any $x\in \msa^0$, for any $m\in \{0, \hdots, \ell\}$, we have $(x_0^\star, \hdots, x^\star_{m-1}, x_m,x^\star_{m+1}, \hdots, x_\ell^\star)\in \msa^0$.
    \item \label{item:ass_lemma_2}Assume there exists a family of functions $\{\varphi^n_{i_k}\}_{n \in \nset^*, k\in\{0,\hdots, K-1\}}$ with $\varphi^n_{i_k}:\msx^0_{i_k} \to [-\infty,+\infty)$ such that for any $n\in \nset^*$ and any $k \in \{0, \hdots, K-2\}$, we have $\varphi^n_{i_k}(x_{i_k}^\star)=0$.
    \item \label{item:ass_lemma_3} Denote $F^n(x)=\sum_{k=0}^{K-1}\varphi^n_{i_k}(x_{i_k})$ for any $x\in \msa^0$. Assume that for any $x \in \msa^0$, $F(x)=\lim_{n\to \infty}F^n(x)$ exists and is such that $F(x)\in [-\infty, +\infty)$ with $F(x^\star)\in \rset$.
\end{enumerate} 
Then, for any $i \in \mss$, for any $x_i \in \msx_i^0$, $\varphi_i(x_i)=\lim_{n \to \infty}\varphi^n_i(x_i)$ exists and is such that $\varphi_i(x_i) \in [-\infty, +\infty)$.    
\end{lemma}

\begin{proof} Consider $\msa^0\subset (\rset^d)^{\ell+1}$ such that assumptions \ref{item:ass_lemma_1}, \ref{item:ass_lemma_2} and \ref{item:ass_lemma_3} hold. Remark that we have $F^n(x^\star)=\varphi^n_{i_{K-1}}(x_{i_{K-1}}^\star)$.

Let $x\in \msa^0$. We denote $x^m=(x_1^\star, \hdots, x_{m-1}^\star, x_{m}, x_{m+1}^\star, \hdots, x_{l}^\star)$ for any $m\in\{0,\hdots, \ell\}$. In particular, we have $x^m\in \msa^0$ by assumption \ref{item:ass_lemma_1}. Let us define
\begin{align}
    \varphi_{i_k}(x_{i_k})& =F(x^{i_k}) -F(x^\star),\quad \forall  k \in \{0, \hdots, K-2\} \eqsp ,\\
    \varphi_{i_{K-1}}(x_{i_{K-1}})&=F(x^{i_{K-1}}) \eqsp .
\end{align}

Using assumption \ref{item:ass_lemma_3}, we have $\varphi_{i}(x_{i})\in [-\infty, +\infty)$ for any $i\in \mss$. Let $k\in\{0, \hdots, K-2\}$. We have by definition of $F^n$,
\begin{align}
    \textstyle\varphi^n_{i_k}(x_{i_k}) = F^n(x^{i_k}) - \sum_{\substack{m=0\\m\neq k}}^{K-1}\varphi^n_{i_m}(x_{i_m}^\star)= F^n(x^{i_k}) -F^n(x^\star) \eqsp ,
\end{align}
where we used assumption \ref{item:ass_lemma_2} in the last equality. Since $x^{i_k}\in \msa^0$ and $x^\star\in \msa^0$, we have by assumption \ref{item:ass_lemma_3},
\begin{align}
    \lim_{n\to \infty} \varphi^n_{i_k}(x_{i_k})= F(x^{i_k}) -F(x^\star)= \varphi_{i_k}(x_{i_k}) \eqsp .
\end{align}

Furthermore, by combining the definition of $F^n$ with assumption \ref{item:ass_lemma_2}, we have
\begin{align}
    \lim_{n\to \infty} \varphi^n_{i_{K-1}}(x_{i_{K-1}})= F(x^{i_{K-1}})= \varphi_{i_{K-1}}(x_{i_{K-1}}),
\end{align}
which concludes the proof.
\end{proof}

In what follows, before proving \Cref{prop:existence}, we respectively show in \Cref{prop:ass_kl} and \Cref{prop:ass_equiv} how
\Cref{ass:kl_bounded} and \Cref{ass:equivalence} can be satisfied in the case where $\pi^0\in \Pmeasure_{\mst_r}$, as in \eqref{eq:pi_0}, that is
\begin{equation} \label{eq:pi_0_appendix}
  \textstyle 
    \pi^0=\pi^0_r \bigotimes_{(v,v')\in \mse_r}\pi^0_{v'|v} .
\end{equation}

\begin{proposition}\label{prop:ass_kl} Let $\pi^0\in \Pmeasure_{\mst_r}$.
  Assume that $\pi_r^0 = \mu_r$ if $r \in \mss$ or $\pi_r^0 = \mathrm{N}(m_r, \sigma_r \Id)$, with $m_r \in \rset^d$
  and $\sigma_r > 0$ if $r \in \mss^\complementary$. In addition, assume
  that for any $(v,v') \in \mse_r$,
  $\pi^0_{v'|v}(\cdot|x_v) =\mathrm{N}(x_v, \sigma_{v,v'} \Id)$ with
 $\sigma_{v,v'} > 0$. Finally, assume that for any
  $i \in \mss$, $\int_{\rset^d} \normLigne{x}^2 \rmd \mu_i(x) < +\infty$ and
  $\ent(\mu_i)<+\infty$. Then \textup{\Cref{ass:kl_bounded}} is satisfied.
\end{proposition}

\begin{proof}
  Let $\pi = \otimes_{i \in \mss} \mu_i \otimes_{i \in \mss^\complementary} \nu_i$ with
  $\nu_i$ any Gaussian measure with positive definite covariance matrix. First,
  we have that
  \begin{equation}
    \textstyle \KL(\pi \mid \pi^0) = \KL(\pi_r \mid \pi_r^0) + \sum_{(v,v') \in \mse_r} \int_{\rset^d} \KL(\pi_{v'|v}|\pi_{v'|v}^0) \rmd \pi_{v} \eqsp . 
  \end{equation}
  For any $(v,v') \in \mse_r$, there exists $C_{v,v'} \geq 0$ such that 
  \begin{align}
    &\textstyle\int_{\rset^d} \KL(\pi_{v'|v}|\pi_{v'|v}^0) \rmd \pi_{v} \textstyle\leq C_{v,v'} - \mathrm{H}(\pi_{v'}) + \int_{\rset^d \times \rset^d} \normLigne{x_v - x_{v'}}^2/(2\sigma_{v,v'}^2) \rmd \pi_v \otimes \pi_{v'}(x_v, x_{v'}) \\
    &\qquad \textstyle \leq C_{v,v'} - \mathrm{H}(\pi_{v'}) + (1/\sigma_{v,v'}^2) \int_{\rset^d} \normLigne{x_v}^2 \rmd \pi_v(x_v) + (1/\sigma_{v,v'}^2) \int_{\rset^d} \normLigne{x_{v'}}^2 \rmd \pi_{v'}(x_{v'}) < +\infty .
  \end{align}
  We conclude the proof upon remarking that $\KL(\pi_r \mid \pi_r^0)<+\infty$.
\end{proof}

\begin{proposition}\label{prop:ass_equiv} Let $\pi^0\in \Pmeasure_{\mst_r}$.
  Assume that $\pi_r^0 = \mu_r$ if $r \in \mss$ or $\pi_r^0 = \mathrm{N}(m_r, \sigma_r \Id)$, with $m_r \in \rset^d$
  and $\sigma_r > 0$ if $r \in \mss^\complementary$. In addition, assume
  that for any $(v,v') \in \mse_r$,
  $\pi^0_{v'|v}(\cdot|x_v) =\mathrm{N}(x_v, \sigma_{v,v'} \Id)$ with
  $\sigma_{v,v'} > 0$. Finally, assume that for any $i \in \mss$, $\mu_i$ admits
  a positive density w.r.t.~the Lebesgue measure. Then \textup{\Cref{ass:equivalence}} is satisfied.
\end{proposition}

\begin{proof}
  We have that $\pi^0$ admits a positive density w.r.t~the Lebesgue
  measure. Letting
  $\Tilde{\pi}^0=\bigotimes_{i\in \mss} \mu_i \bigotimes_{j \in \mss^\complementary}\Tilde{\mu}_j$ where $\Tilde{\mu}_j$ which admits a
  positive density w.r.t.~the Lebesgue measure for any
  $j \in \mss^\complementary$, we get that $\Tilde{\pi}^0$ admits
  a positive density w.r.t.~the Lebesgue measure and therefore
  $\pi^0 \sim \Tilde{\pi}^0$, which concludes the proof.
\end{proof}

Using the preliminary results presented above, we are now ready to prove \Cref{prop:existence}.

\begin{proof}[Proof of \Cref{prop:existence}] Assume \Cref{ass:measures} and \Cref{ass:kl_bounded}. Since $\Pmeasure_\mss$ is convex and closed in total-variation norm, there exists a probability distribution $\pi^\star$ solution to \eqref{eq:msb_P-s}, or equivalently to \eqref{eq:multimarginal_sb_static_nu}, by using \Cref{ass:kl_bounded} with \cite[Theorem 2.1.]{csiszar1975divergence}. Moreover, this solution is unique by strict convexity of $\KL(\cdot\mid \pi^0)$. 

We now turn to the proof of existence of potentials defining $(\rmd \pi^\star/\rmd \pi^0)$, by adapting the arguments of \cite[Section 2.3.]{nutz2021introduction}. Define $\nu^n=\argmin\{\KL(\pi \mid \pi^0): \pi \in \Pmeasure_\mss^n\}$ for any $n\in \nset^*$. Since $\{\Pmeasure_\mss^n\}_{n\in \nset^*}\subset \Pmeasure^{(\ell +1)}$ is a decreasing sequence of sets that are convex and closed in total-variation norm such that \eqref{eq:pmeasure_mss} holds, we get from \cite[Proposition 1.17.]{nutz2021introduction} with \Cref{ass:kl_bounded} that
\begin{align}
    \lim_{n\to \infty}\|\nu^n -\pi^\star\|_{\mathrm{TV}}=0 \eqsp ,
\end{align}
or equivalently
\begin{align}\label{eq:nu^n}
   \lim_{n\to \infty}\| (\rmd \nu^n / \rmd \pi^0) - (\rmd \pi^\star / \rmd \pi^0)\|_{\mathrm{L}^1(\pi^0)}=0  \eqsp .
\end{align}

Following \cite[Example 1.18]{nutz2021introduction}, there exists a family of bounded measurable functions $\{\varphi^n_i\}_{n \in \nset^*,i \in \mss}$ with $\varphi^n_i:\rset^d \to \rset$ such that for any $n\in \nset^*$
\begin{align}\label{eq:nu^n_expr}
    \textstyle(\rmd \nu^n / \rmd \pi^0) = \exp[\bigoplus_{i \in \mss}\varphi^n_i] \eqsp .
\end{align}

We consider such family $\{\varphi^n_i\}_{n \in \nset^*,i \in \mss}$ for the rest of the proof. By combining \eqref{eq:nu^n} and \eqref{eq:nu^n_expr}, we obtain, up to extraction, 
\begin{align}\label{eq:pi^star_lim}
    \textstyle(\rmd \pi^\star/\rmd \pi^0)= \lim_{n \to \infty} \exp[\bigoplus_{i \in \mss}\varphi^n_i] \quad \text{$\pi^0$-a.s.} \eqsp .
\end{align}

We now define the following sets
\begin{align}
    \msa^\star & = \textstyle{\{x \in (\rset^d)^{\ell+1}: \lim_{n\to \infty}\bigoplus_{i \in \mss}\varphi^n_i(x_i)\in [-\infty, +\infty)\}} \eqsp, \label{eq:msa^star}\\
    \msb^\star & = \textstyle{\{x \in (\rset^d)^{\ell+1}: \lim_{n\to \infty}\bigoplus_{i \in \mss}\varphi^n_i(x_i)>-\infty\} \subset \msa^\star}
\end{align}

Using \eqref{eq:pi^star_lim}, we have $\pi^0(\msa^\star)=1$. Using  \Cref{ass:equivalence}, it comes $\tilde{\pi}^0(\msa^\star)=1$. Moreover, we also get that $\pi^\star(\msb^\star)=1$ by \eqref{eq:pi^star_lim}. Thus, it comes $\pi^0(\msb^\star)>0$, and $\tilde{\pi}^0(\msb^\star)>0$ using \Cref{ass:equivalence}.

We then apply \Cref{lemma:nutz_1} to $\tilde{\pi}^0$ and $\msa=\msa^\star$. Since $\tilde{\pi}^0(\msb^\star)>0$, it implies that there exists $x^\star \in \msb^\star$ and a measurable set $\msa^0\subset \msb^\star$ verifying the properties \ref{item:nutz_1-a} and \ref{item:nutz_1-b}. Following \cite[Corollary 2.12]{nutz2021introduction}, we may assume without loss of generality in the statement of \Cref{lemma:nutz_1} that the sets $\msx_m^0$ are measurable with $\prod_{m=0}^\ell \msx_m^0 \subset \msa$. In this case, we obtain that $\mu_i(\operatorname{proj}_i(\msa^0))=1$ for any $i\in \mss$.

We now aim at applying \Cref{lemma:nutz_2} to the set $\msa^0$. Remark that $\msa^0$ directly satisfies assumption \ref{item:ass_lemma_1}. For any $n\in \nset^*$, consider the following transformation of the functions $\{\varphi^n_{i}\}_{i \in \mss}$
\begin{align}
    \varphi^n_{i_k} &\xleftarrow{} \varphi^n_{i_k} -\varphi^n_{i_k}(x_{i_k}^\star),\quad \forall k \in \{0, \hdots, K-2\} \eqsp,\\
    \varphi^n_{i_{K-1}} &\xleftarrow{} \textstyle{\varphi^n_{i_{K-1}} +\sum_{k=0}^{K-2}\varphi^n_{i_k}(x_{i_k}^\star)} \eqsp .
\end{align}
For any $i\in \mss$, we restrict $\varphi^n_{i_k}$ to $\msx_{i_k}^0$, so that the family $\{\varphi^n_{i}\}_{n \in \nset^*, i \in \mss}$ now verifies assumption \ref{item:ass_lemma_2}. Finally, since $\msa^0 \subset \msa^\star$ and $x^\star\in \msb^\star$, we directly obtain 
assumption \ref{item:ass_lemma_3}.

Therefore, \Cref{lemma:nutz_2} may be applied. It provides us with the family of functions $\{\varphi_i\}_{i \in \mss}$ defined by $\varphi_i:\msx_i^0 \to [-\infty, +\infty)$ with $\varphi_i=\lim_{n\to \infty}\varphi^n_i$ $\mu_i$-a.s. for any $i\in \mss$. Since $\mu_i(\operatorname{proj}_i(\msa^0))=1$ for any $i \in \mss$, we may extend the functions $\varphi_i$ to $\rset^d$. In particular, we can find a family of functions $\{\psi_i^\star\}_{i \in \mss}$ with $\psi_i^\star:\rset^d\to [-\infty, +\infty)$ such that $\psi_i^\star=\varphi_i$ $\mu_i$-a.s.
Note that these functions are measurable as limits of measurable functions.

Since $\pi^0 \sim \tilde{\pi}^0$ by \Cref{ass:equivalence}, \eqref{eq:pi^star_lim} turns into
\begin{align}\label{eq:pi^star_lim_new}
    \textstyle(\rmd \pi^\star/\rmd \pi^0)= \exp[\bigoplus_{i \in \mss}\psi_i^\star] \quad \text{$\pi^0$-a.s.} \eqsp .
\end{align}

Finally, we show that the functions $\psi_i^\star$ are $\mu_i$-a.s. finite. Let $i\in \mss$. Let us define $\msa_i=\{x_i \in \rset^d: \psi^\star_i(x_i)=-\infty\}$. Using \eqref{eq:pi^star_lim_new}, we obtain 
$\textstyle(\rmd \pi^\star/\rmd \pi^0)(\msa_i \times (\rset^d)^\ell)=0$. Since $\pi^\star_i=\mu_i$, we have 
\begin{align}
    \textstyle\mu_i(\msa_i)=\pi^\star(\msa_i \times (\rset^d)^\ell)=\int_{\msa_i \times (\rset^d)^\ell}(\rmd \pi^\star /\rmd \pi^0) \rmd \pi^0=0 \eqsp ,
\end{align}
which gives the result.
\end{proof}

We now turn to the proof of \Cref{corollary:existence}, which states that the iterates of \eqref{eq:ipf_multimarginal} can be expressed via potentials, in the same manner as the solution $\pi^\star$ to \eqref{eq:multimarginal_sb_static_nu}.

\begin{proof}[Proof of \Cref{corollary:existence}] Assume \textup{\Cref{ass:measures}},
  \textup{\Cref{ass:kl_bounded}} and \textup{\Cref{ass:equivalence}}. We prove the result of this corollary by recursion on $n\in \nset^*$. First take $n=1$. In this case, the first iteration of \eqref{eq:ipf_multimarginal} is a multi-marginal SB problem of the form \eqref{eq:multimarginal_sb_static_nu} where $S=\{i_0\}$ with reference measure $\pi^0$. Therefore, using \textup{\Cref{ass:kl_bounded}} and \textup{\Cref{ass:equivalence}}, we can apply \Cref{prop:existence} and obtain existence of $\psi^1_{i_0}:\rset^d \to \rset$ such that
\begin{align}
    (\rmd \pi^1/\rmd \pi^0)=\exp[\psi_{i_0}^1] \quad \text{$\pi^0$-a.s} \eqsp .
\end{align}
By taking $\psi_{i_k}^0=0$ for $k\in \{1, \hdots, K-1\}$, we thus obtain the result at step $n=1$.

Now assume that the result is verified for some $n\in \nset^*$, with
$k_n=(n-1)\mod(K)$. We define $k_n+1=n \mod (K)$ and $q_n\in\nset$ as the quotient of
the Euclidean division of $n$ by $K$. In this case, the $(n+1)$-th iteration of
\eqref{eq:ipf_multimarginal} is a multi-marginal SB problem of the form
\eqref{eq:multimarginal_sb_static_nu} where $\mss=\{i_{k_n+1}\}$ with reference
measure $\pi^n$. Using \eqref{eq:bonne_def}, we have that
\textup{\Cref{ass:kl_bounded}} is satisfied for this new
\eqref{eq:multimarginal_sb_static_nu} problem. \textup{\Cref{ass:measures}} and
\textup{\Cref{ass:equivalence}} are satisfied for this problem, given the form
of $\pi^n$. Therefore, we can apply \Cref{prop:existence} and obtain existence
of $\psi^{q_n+1}_{i_{k_n+1}}:\rset^d \to \rset$ such that
\begin{align}\label{eq:pi^n+1_pi^n}
    (\rmd \pi^{n+1}/\rmd \pi^n)=\exp[\psi_{i_{k_n+1}}^{q_n+1}] \quad \text{$\pi^n$-a.s} \eqsp .
\end{align}
By assumption, we have that $\pi^n \ll \pi^0$. Hence, we obtain $\pi^{n+1}\ll \pi^0$ and thus,
\begin{align}
    (\rmd \pi^{n+1}/\rmd \pi^0)=(\rmd \pi^{n+1}/\rmd \pi^n)(\rmd \pi^{n}/\rmd \pi^0) \quad \text{$\pi^0$-a.s} \eqsp .
\end{align}
By combining \eqref{eq:pi^n+1_pi^n} with the result of the recursion at step $n$, we directly obtain the result at step $n+1$, which achieves the proof.
\end{proof}

\paragraph{Proofs of \Cref{prop:marginal} and \Cref{prop:conv_mipf}.} In this part of the section, we establish the proofs of results related to the convergence of \eqref{eq:ipf_multimarginal}, respectively \Cref{prop:marginal} and \Cref{prop:conv_mipf}, which can be seen as a natural extension of \cite[Proposition 2.1.]{ruschendorf1995convergence} and \cite[Theorem 3.1.]{ruschendorf1995convergence}.

\begin{proof}[Proof of \Cref{prop:marginal}] Under \Cref{ass:measures} and
  \Cref{ass:kl_bounded}, we obtain by \Cref{prop:existence} existence and
  uniqueness of a solution to \eqref{eq:multimarginal_sb_static_nu}, which we
  denote by $\pi^\star$. Since $\pi^\star \in \Pmeasure_\mss$, using recursively
  \cite[Theorem 3.12.]{csiszar1975divergence}, the fact that
  $\ensembleLigne{\pi_{i_k} = \mu_{i_k}}{\pi\in\Pmeasure^{(\abs{\msv})}}$ is
  convex for any $k \in \{0, \dots, K-1\}$ and \eqref{eq:ipf_multimarginal}, we
  obtain
  \begin{align}
    \textstyle
    \KL(\pi^\star \mid \pi^0)= \KL(\pi^\star \mid \pi^n) + \sum_{i=1}^n \KL(\pi^i \mid \pi^{i-1}) \eqsp . \label{eq:bonne_def}
\end{align}
Therefore, we have $\sum_{i=1}^\infty \KL(\pi^i \mid \pi^{i-1}) \leq \KL(\pi^\star \mid \pi^0) < \infty$ and thus,
\begin{align}\label{eq:KL_to0}
    \textstyle{\lim_{i \to + \infty} \KL(\pi^i \mid \pi^{i-1})=0} \eqsp .
\end{align}

Let $n\in \nset^*$ with $n > 2K$, $k \in \{0, \hdots,K-1\}$ and let $q_n\in \nset$ be
the quotient of the Euclidean division of $n-1$ by $K$. We define
$n_k=q_n K +k +1$ with $(n_k-1)=k\mod (K)$ if $n_k\leq n$. Otherwise, we set
$n_k=(q_n-1) K +k +1$ with $(n_k-1)=k\mod (K)$. Note that we always have
$\abs{n-n_k}\leq 2K$.  In particular, we have $\pi^{n_k}_{i_k}=\mu_{i_k}$ by
definition of \eqref{eq:ipf_multimarginal}. Therefore, we obtain
\begin{align}
    \|\pi^n_{i_k} - \mu_{i_k}\|_{\mathrm{TV}} & \leq \|\pi^n - \pi^{n_k}\|_{\mathrm{TV}} && \\
    & \leq \|\pi^n - \pi^{n-1}\|_{\mathrm{TV}} + \hdots + \|\pi^{n_k+1} - \pi^{n_k}\|_{\mathrm{TV}} && \text{(triangle inequality)}\\
    & \leq (2\KL(\pi^n \mid \pi^{n-1}))^{1/2} + \hdots + (2\KL(\pi^{n_k+1} \mid \pi^{n_k}))^{1/2} \eqsp , && \text{(Pinsker's inequality)}
\end{align}
where each term goes to 0 as $n \to + \infty$ in the last inequality by \eqref{eq:KL_to0}, which achieves the proof.
\end{proof}

We now turn to the proof \Cref{prop:conv_mipf}, which requires several preliminary technical results. For the rest of this section, we define, for any $n\in \nset$, $q_n$ as the quotient of the Euclidean division of $n-1$ by $K$ (in particular, $q_0=-1$).

\paragraph{Schr\"odinger equations.} Under \Cref{ass:measures}, \Cref{ass:kl_bounded} and \Cref{ass:equivalence}, we know from \Cref{prop:existence} that the unique solution $\pi^\star$ to \eqref{eq:multimarginal_sb_static_nu} can be $\pi^0$-a.s. written as $(\rmd \pi^\star / \rmd \pi^0)=\exp[\bigoplus_{i \in \mss}\psi_i^\star]$, where $\{\psi_i^\star\}_{i \in \mss}$ are measurable potentials, referred to as \emph{Schr\"odinger potentials}. These functions are determined by the fixed-point \emph{Schr\"odinger equations}
\begin{align}
    \textstyle \psi_i^\star(x_i)=\log[r_i(x_i)/\int_{(\rset^d)^\ell} \exp[\sum_{j \in \mss\backslash \{i\}}\psi_j^\star(x_j)] h(x_{0:\ell}) \rmd \nu_{-i} (x_{-i}) ] \quad \text{$\mu_i$-a.s.}, \quad \forall i \in \mss \eqsp ,
\end{align}
which are obtained by marginalising $\pi^\star$ along its constrained marginals. This family of potentials is not unique. Indeed, for any family of real numbers $\{\lambda_{i_k}\}_{k\in \{0, \hdots, K-2\}}$, we have
\begin{align}
    \textstyle(\rmd \pi^\star/\rmd \pi^0) = \exp[\bigoplus_{i \in \mss}\tilde{\psi}_i] \eqsp,
\end{align}
where $\tilde{\psi}_{i_k}=\psi_{i_k}^\star + \tilde{\lambda}_{i_k}$ for any $k \in \{0, \hdots, K-1\}$ with $\Tilde{\lambda}_{i_k}=\lambda_{i_k}$ if $k\in \{0, \hdots, K-2\}$ and $\Tilde{\lambda}_{i_{K-1}}=-\sum_{i=0}^{K-2}\lambda_{i_k}$.

\paragraph{Remark on the initialisation of \eqref{eq:ipf_multimarginal}.} Consider a probability measure $\bar{\pi}^0\in \Pmeasure^{(\ell+1)}$ of the form
\begin{align}\label{eq:bar_pi_0}
    \textstyle(\rmd \bar{\pi}^0/\rmd\pi^0)=\exp[\bigoplus_{i \in \mss}\psi_i^0] \eqsp ,
\end{align}
where $\{\psi_i^0\}_{i \in \mss}$ is a family of measurable potentials with $\psi^0_i:\rset^d \to \rset$ such that $\abs{\int_{\rset^d} \psi_i^0\rmd \mu_i}<\infty$ for any $i\in \mss$. Then, for any $\pi\in \Pmeasure_\mss$, we have
\begin{align}
    \textstyle\KL(\pi \mid \pi^0)= \KL(\pi \mid \bar{\pi}^0) + \int_{(\rset^d)^K}\bigoplus_{i \in \mss}\psi^0_i \rmd \pi= \KL(\pi \mid \bar{\pi}^0) + \sum_{i \in \mss} \int_{\rset^d}\psi^0_i \rmd \mu_i \eqsp .
\end{align}
Hence, \eqref{eq:multimarginal_sb_static_nu} is equivalent to the multi-marginal SB problem
\begin{align}
     \textstyle{\argmin \ensembleLigne{\KL(\pi|\bar{\pi}^0)}{\pi \in \Pmeasure^{(\ell+1)}, \eqsp \pi_i = \mu_i\eqsp, \forall i \in \mss}} \eqsp .
\end{align}
We refer to \citep[Proposition 4.2]{peyre2019computational} for the EOT
counterpart of this result. This means that the solutions of the multi-marginal
Schr\"odinger Bridge problem are invariant by multiplication of the reference measure by potentials on the \emph{fixed} marginals.  Consequently, the
initialisation of the \eqref{eq:ipf_multimarginal} sequence may be chosen as $\bar{\pi}^0$
instead of $\pi^0$.

\textbf{For sake of clarity, we now refer to the reference
  probability measure of \eqref{eq:multimarginal_sb_static_nu} as $\bar{\pi}$ or $\pi^{-1}$
  and to the initialisation of the \eqref{eq:ipf_multimarginal} iterates as $\pi^0$.}

\paragraph{Solving \eqref{eq:ipf_multimarginal} with potentials.} To prove the convergence of the \eqref{eq:ipf_multimarginal} iterates to the solution $\pi^\star$ given by \Cref{prop:existence}, we first rewrite these iterates with potentials, following the form of $\pi^\star$.

To do so, we recursively define the sequence of potentials $\{\psi^n_i\}_{n\in \nset, i\in \mss}$ by
\begin{align} \label{eq:def_psi_0}
    & \psi^0_{i_0}=\hdots=\psi^0_{i_{K-2}}=0 \eqsp ,\\
    &\textstyle{\psi^0_{i_{K-1}}(x_{i_{K-1}})= \log(r_{i_{K-1}}(x_{i_{K-1}})/\int_{(\rset^d)^\ell} h(x_{0:\ell}) \rmd \nu_{-{i_{K-1}}} (x_{-{i_{K-1}}}) )} \eqsp,
\end{align}
and for any $n \in \nset^*$ and  $k\in \{0, \hdots, K-1\}$
\begin{align}
  \psi^{q_n+1}_{i_k}(x_{i_k})&=\textstyle{\log[r_{i_k}(x_{i_k})/\int_{(\rset^d)^\ell} \exp[\bigoplus_{\ell=0}^{k}\psi^{q_n+1}_{i_\ell}(x_{i_\ell}) \bigoplus_{m=k+1}^{K-1}\psi^{q_n}_{i_m}(x_{i_m})]}
  \\
                             & \qquad \times h(x_{0:\ell}) \rmd \nu_{-{i_k}} (x_{-{i_k}}) ] \eqsp , \label{eq:def_psi_n}
\end{align}
recalling that $q_n$ is the quotient of the Euclidean division of $n-1$ by $K$.

We now define the sequence of probability measures $\{\pi^n\}_{n\in \nset}$ by
\begin{align}\label{eq:def_pi^n}
    \textstyle\rmd \pi^n/\rmd \bar{\pi}= \exp[\bigoplus_{\ell=0}^{k_n}\psi^{q_n+1}_{i_\ell} \bigoplus_{m=k_n+1}^{K-1}\psi^{q_n}_{i_m}], \ k_n=(n-1)\mod (K) , \ n = q_n K + k_n + 1 .
\end{align}
In particular, we have $(\rmd \pi^0/\rmd \bar{\pi})=\exp[\oplus_{\ell=0}^{K-1} \psi^0_{i_\ell}] = \exp[
\psi^0_{i_{K-1}}]$, and thus $\int_{\rset^d}\psi^0_{i_{K-1}}\rmd \mu_{i_{K-1}}=\KL(\mu_{i_{K-1}}\mid \bar{\pi}_{i_{K-1}})$. Consequently, \underline{$\pi^0$ can be chosen as the initialisation of \eqref{eq:ipf_multimarginal}}, following the previous remark, if we assume that $\KL(\mu_{i_{K-1}}\mid \bar{\pi}_{i_{K-1}}) <\infty$. In \eqref{eq:tree_sb_static} with $r=i_{K-1}$, the latter assumption is directly verified since we choose $\bar{\pi}_{i_{K-1}}=\mu_{i_{K-1}}$.

Let $n\in \nset$, with $k_n=(n-1)\mod(K)$, $k_n+1=n
\mod(K)$. Using \eqref{eq:def_psi_0} and \eqref{eq:def_psi_n}, we get that
$\pi^n_{i_{k_n}}=\mu_{i_{k_n}}$. Moreover, we have 
\begin{align} \label{eq:ratio_pi^n}
    \rmd \pi^n/\rmd \pi^{n-1} = \exp[\psi^{q_n+1}_{i_{k_n}}-\psi^{q_n}_{i_{k_n}}]\eqsp,
\end{align}
with the convention that $\psi^{-1}_{i_{K-1}}=0$. In particular, we obtain that
$\pi^{n+1}_{|i_{k_n+1}}=\pi^{n}_{|i_{k_n+1}}$.

In conclusion, the sequence $\{\pi^n\}_{n\in \nset}$ defined in \eqref{eq:def_pi^n} verifies $\pi^{n+1}=\mu_{i_{k_n+1}}\pi^{n}_{|i_{k_n+1}}$ for any $n\in \nset$. By decomposition property of the Kullback-Leibler divergence, this sequence solves \eqref{eq:ipf_multimarginal} with initialisation $\pi^0$. \emph{We consider such iterates in the following.}

Since $\pi^n_{i_{k_n}}=\mu_{i_{k_n}}$, we have that
\begin{align}\label{eq:KL_decreasing}
    \textstyle \KL(\pi^n \mid \pi^{n-1})= \int_{\rset^d} (\psi^{q_n+1}_{i_{k_n}}-\psi^{q_n}_{i_{k_n}})\rmd \mu_{i_{k_n}} \eqsp .
\end{align}

Before proving a multi-marginal counterpart to \cite[Lemma
4.1]{ruschendorf1995convergence}, we state and prove the following result.

\begin{proposition}
  \label{sec:KL_integrand_bound}
  Let $\pi_0, \pi_1$ two probability measures on $\rset^d$ such that
  $\pi_0 \ll \pi_1$. Then, denoting $f = \rmd \pi_0 / \rmd \pi_1$, the following
  assertions are equivalent:
\begin{enumerate}[label=(\alph*)]
\item \label{item:kl} $\KL(\pi_0 \mid \pi_1) < +\infty$
\item \label{item:abs} $\int_{\rset^d} \absLigne{\log(f)(x)} \rmd \pi_0(x) < +\infty$
\item \label{item:pos} $\int_{\rset^d} \log(f)(x) \1_{f(x) > 1} \rmd \pi_0(x) < +\infty$
\end{enumerate}
If one of these conditions is satisfied then 
$\int_{\rset^d} \abs{\log(f)(x)} \rmd \pi_0 \leq \KL(\pi_0 \mid \pi_1) + 2/\rme$.
\end{proposition}

\begin{proof}
  First, note that
\begin{equation}
  \label{eq:bound}
  \textstyle{\int_{\rset^d}\absLigne{ \log(f)(x) } \1_{f < 1}\rmd \pi_0(x)  \leq \int_{\rset^d} \absLigne{\log(f)(x)f(x)} \1_{f < 1}\rmd \pi_1(x)} \leq 1/\rme \eqsp ,
\end{equation}
where we have used that for any $u \in \ccint{0,1}$,
$\abs{u \log(u)} \leq 1 / \rme$.  We have that \ref{item:abs} implies
\ref{item:pos}. Using the previous result we have that \ref{item:pos} implies
\ref{item:abs}. Hence \ref{item:pos} and \ref{item:abs} are equivalent. In
addition, it is clear that \ref{item:abs} implies \ref{item:kl}. Finally (this
is more of a convention), we have that
$\KL(\pi_0 \mid \pi_1) = \int_{\rset^d} \log(f)(x) \1_{f(x) > 1} \rmd \pi_0(x) +
\int_{\rset^d} \log(f)(x) \1_{f(x) < 1} \rmd \pi_0(x)< +\infty$. Using
\eqref{eq:bound} this implies \ref{item:pos}.  Finally, we have
  \begin{align}
    \textstyle 
    \int_{\rset^d} \abs{\log(f)(x)} \rmd \pi_0(x) &\textstyle = \int_{\rset^d} \log(f)(x) \rmd \pi_0(x) - 2 \int_{\rset^d} \log(f)(x) \1_{f(x) < 1} \rmd \pi_0(x) \\
    &\leq \KL(\pi_0 \mid \pi_1) + 2/\rme ,
  \end{align}
  which concludes the proof.
\end{proof}

We begin with the following lemma which controls the integral of the potentials
uniformly w.r.t. $n \in \nset$. It can be seen as the \emph{multi-marginal}
counterpart of \cite[Lemma 4.1]{ruschendorf1995convergence}.

\begin{lemma}\label{lemma:ruschendorf_1}
  Assume \textup{\Cref{ass:closeness}}. There exist
  $\{c_i\}_{i\in \mss}\in (0, +\infty)^K$ such that for any function
  $f:(\rset^d)^{\ell+1} \to \rset$ of the form $f=\bigoplus_{i \in \mss}f_i$, we
  have
\begin{align} \label{eq:c_i}
    c_i\|f\|_{\mathrm{L}^1(\pi^\star)}\geq \|f_i\|_{\mathrm{L}^1(\mu_i)}, \quad \forall i \in \mss \eqsp .
\end{align}

  For any $n\in \nset^\star$, we have
    \begin{enumerate}[wide, labelwidth=!, labelindent=0pt, label=(\alph*)]
        \item \label{lemma1-item1} $\sum_{i \in \mss} \int_{\rset^d} \psi^n_i \rmd \mu_i \leq \KL(\pi^\star \mid \bar{\pi}) < \infty$, 
        \item \label{lemma1-item2} $\int_{(\rset^d)^{\ell+1}} (\bigoplus_{i \in \mss} \psi_i^\star - \bigoplus_{i \in \mss} \psi^n_i)\rmd \pi^\star \leq \KL(\pi^\star \mid \bar{\pi}) < \infty$, 
        \item \label{lemma1-item3} $\sup_{n\in \nset} \int_{\rset^d} \abs{\psi^n_i}\rmd \mu_i <\infty, \eqsp \forall i \in \mss$.
    \end{enumerate}
\end{lemma}

\begin{proof}
  First, we have that \eqref{eq:c_i} is a direct consequence of \cite[Theorem
  1]{kober1940theorem} and \Cref{ass:closeness}.  Let us now prove
  \ref{lemma1-item1}. Using \eqref{eq:KL_decreasing}, we have
    \begin{align}
        \textstyle{\sum_{m=0}^{Kn} \KL(\pi^m \mid \pi^{m-1})} &=\textstyle{\sum_{\ell=0}^{n-1} \sum_{k=0}^{K-1} \KL(\pi^{\ell K +k +1 } \mid \pi^{ \ell K + k}) + \KL(\pi^0 \mid \pi^{-1})}\\
        & =\textstyle{\sum_{\ell=0}^{n-1} \sum_{i \in \mss} \int_{\rset^d} (\psi^{\ell+1}_{i}-\psi^{\ell}_{i})\rmd \mu_{i} + \int_{\rset^d} (\psi^{0}_{i_{K-1}}-\psi^{-1}_{i_{K-1}})\rmd \mu_{i_{K-1}} } \\
        & =\textstyle{\sum_{i \in \mss} \sum_{\ell=0}^{n-1} \int_{\rset^d} (\psi^{\ell+1}_{i}-\psi^{\ell}_{i})\rmd \mu_{i} + \int_{\rset^d} (\psi^{0}_{i_{K-1}}-\psi^{-1}_{i_{K-1}})\rmd \mu_{i_{K-1}}}\\
        & = \textstyle{\sum_{i \in \mss} \int_{\rset^d} (\psi^{n}_{i}-\psi^{0}_{i})\rmd \mu_{i} + \int_{\rset^d} (\psi^{0}_{i_{K-1}}-\psi^{-1}_{i_{K-1}})\rmd \mu_{i_{K-1}}}\\
        & = \textstyle{\sum_{i \in \mss} \int_{\rset^d} \psi^{n}_{i} \rmd \mu_{i}} \leq \textstyle{\KL(\pi^\star\mid \bar \pi)} .  \eqsp ,
    \end{align}
where the last inequality follows the proof of \Cref{prop:marginal}.

Since the first term in the inequality of \ref{lemma1-item2} is equal to
$\KL(\pi^\star \mid \pi^{nK})$, we obtain \ref{lemma1-item2} using that
$\KL(\pi^\star \mid \pi^{nK}) \leq \KL(\pi^\star \mid \bar{\pi})$ following the
proof of \Cref{prop:marginal}.

Let us now prove \ref{lemma1-item3}. Since
$\KL(\pi^\star\mid \bar{\pi})<\infty$, using \Cref{sec:KL_integrand_bound}, we
have that $\bigoplus_{i \in \mss}\psi_i^\star \in \mathrm{L}^1(\pi^\star)$. From
\ref{lemma1-item2} and \Cref{sec:KL_integrand_bound}, we also get that
$\bigoplus_{i \in \mss}(\psi_i^\star-\psi_i^n) \in \mathrm{L}^1(\pi^\star)$, and
thus
$\int_{(\rset^d)^{\ell+1}}\abs{\bigoplus_{i \in
    \mss}(\psi_i^\star-\psi_i^n)}\rmd \pi^\star \leq C_0$ with $C_0 >
0$. Therefore, we have
\begin{align}
    \textstyle\int_{(\rset^d)^{\ell+1}} \abs{\bigoplus_{i \in \mss} \psi_i^n}\rmd \pi^\star \leq \int_{(\rset^d)^{\ell+1}} \abs{\bigoplus_{i \in \mss} \psi_i^\star}\rmd \pi^\star + \int_{(\rset^d)^{\ell+1}} \abs{\bigoplus_{i \in \mss} (\psi_i^\star-\psi_i^n)}\rmd \pi^\star \leq 2 C_0\eqsp .
\end{align}
Using \eqref{eq:c_i}, we conclude with \Cref{ass:closeness} that for any $i \in \mss$, we have
\begin{align}
    \textstyle\int_{\rset^d}\abs{\psi_i^n}\rmd \mu_i \leq 2 c_i C_0 \eqsp ,
\end{align}
which concludes the proof of \ref{lemma1-item3}.
\end{proof}

The next lemma gives an explicit expression for $\KL(\pi^n \mid \bar \pi)$. It
can be seen as the \emph{multi-marginal} counterpart of \cite[Lemma
4.2]{ruschendorf1995convergence}.

\begin{lemma}\label{lemma:ruschendorf_2} For any $n\in \nset$, with $k_n=(n-1)\mod(K)$, we have
\begin{align}
    \textstyle{\KL(\pi^n \mid \bar \pi)=\int_{\rset^d}\psi^{q_n+1}_{i_{k_n}}\rmd \mu_{i_{k_n}}} & \textstyle{+ \sum_{\ell=0}^{k_n-1}\int_{\rset^d} \psi^{q_n+1}_{i_\ell} \exp[\psi^{q_n+1}_{i_\ell} - \psi^{q_n+2}_{i_\ell}] \rmd \mu_{i_\ell}}\\
    & \textstyle{+ \sum_{m=k_n+1}^{K-1}\int_{\rset^d} \psi^{q_n}_{i_m} \exp[\psi^{q_n}_{i_m} - \psi^{q_n+1}_{i_m}] \rmd \mu_{i_m}} \eqsp . 
\end{align}

\end{lemma}

\begin{proof} Let $n\in \nset$, with $k_n=(n-1)\mod(K)$. Using \eqref{eq:def_pi^n}, we have
\begin{align}\label{eq:lemma_1}
    \textstyle\KL(\pi^n \mid \bar \pi)=\int_{\rset^d}\psi^{q_n+1}_{i_{k_n}}\rmd \mu_{i_{k_n}} + \sum_{\ell=0}^{k_n-1}\int_{\rset^d} \psi^{q_n+1}_{i_\ell} \rmd \pi^n_{i_\ell} + \sum_{m=k_n+1}^{K-1}\int_{\rset^d} \psi^{q_n}_{i_m} \rmd \pi^n_{i_m} \eqsp .
\end{align}

Consider $m\in \{k_n+1, \hdots, K-1\}$. Let $m_n$ be the closest integer to $n$ such that $m_n >n$ and $m=(m_n-1) \mod(K)$. By \eqref{eq:ratio_pi^n}, we have
\begin{align} 
    \textstyle \rmd \pi^n= \exp[\bigoplus_{j=k_n+1}^m \psi^{q_n}_{i_j}-\psi^{q_n+1}_{i_j}]\rmd \pi^{m_n} .
\end{align}
Using \eqref{eq:ratio_pi^n} recursively, we obtain
\begin{align}\label{eq:lemma_2}
    \rmd \pi^n_{i_m}= \exp[\psi^{q_n}_{i_m}-\psi^{q_n+1}_{i_m}]\rmd \pi^{m_n}_{i_m},
\end{align}
where we recall that $\pi^{m_n}_{i_m}=\mu_{i_m}$.

Consider now $\ell \in \{0, \hdots, k_n-1\}$. Let $\ell_n$ be the closest integer to $n$ such that $\ell_n >n$ and $\ell=(\ell_n-1)\mod(K)$. By \eqref{eq:ratio_pi^n}, we have
\begin{align}
    \textstyle \rmd \pi^n = \exp[\bigoplus_{j=k_n+1}^{K-1} \{\psi^{q_n}_{i_j}-\psi^{q_n+1}_{i_j}\}\bigoplus_{j'=0}^\ell \{\psi^{q_n+1}_{i_{j'}}-\psi^{q_n+2}_{i_{j'}}\}]\rmd \pi^{\ell_n} ,
\end{align}
and using \eqref{eq:ratio_pi^n} recursively, we obtain
\begin{align}\label{eq:lemma_3}
    \rmd \pi^n_{i_\ell}= \exp[\psi^{q_n+1}_{i_\ell}-\psi^{q_n+2}_{i_\ell}]\rmd \pi^{\ell_n}_{i_\ell},
\end{align}
where we recall that $\pi^{\ell_n}_{i_\ell}=\mu_{i_\ell}$. We conclude the proof
upon combining \eqref{eq:lemma_1}, \eqref{eq:lemma_2} and \eqref{eq:lemma_3}.
\end{proof}

We are now ready to prove a \emph{uniform integrability} result which is the
multi-marginal counterpart of \cite[Lemma 4.4]{ruschendorf1995convergence}. Before stating
\Cref{lemma:ruschendorf_4}, we prove the following well-known lemma. We recall
that a sequence $(\Psi_n)_{n \in \nset}$ such that for any $n \in \nset$,
$\Psi_n \in \mathrm{L}^1(\mu)$, is \emph{uniformly integrable} w.r.t. $\mu$ if
\begin{enumerate*}[label=(\roman*)]
\item $\sup_{n \in \nset} \int_{\rset^d} \abs{\Psi_n} \rmd \mu < +\infty$ and
\item for any $\vareps > 0$, there exists $K > 0$ such that for any
  $n \in \nset$, $\int_{\cball{0}{K}^\complementary} \abs{\Psi_n} \rmd \mu \leq \vareps$.
\end{enumerate*}

\begin{lemma}
  \label{lemma:uniform_integrability}
  Let $f: \ \rset \to \rset$, convex and non-decreasing on $\coint{A,+\infty}$ with $A >0$
  and $\lim_{x \to +\infty} f(x)/x = +\infty$. Assume that
  $\sup_{n \in \nset} \int_{\rset^d} f(\abs{\Psi_n}) \rmd \mu <+\infty$. Then,
  $(\Psi_n)_{n \in \nset}$ is uniformly integrable w.r.t. $\mu$.
\end{lemma}

\begin{proof}
  Since $f$ is convex, using Jensen's inequality, we get that
  $\sup_{n \in \nset} f(\int_{\rset^d} \abs{\Psi_n} \rmd \mu) <+\infty$ and
  since $\lim_{x \to +\infty} f(x)/x = +\infty$ we have
  $\sup_{n \in \nset} \int_{\rset^d} \abs{\Psi_n} \rmd \mu <+\infty$. Let
  $\vareps > 0$, there exists $K > 0$ such that for any $x > K$,
  $x \leq  \vareps f(x)/ B$ with
  $B = \sup_{n \in \nset} \int_{\rset^d} f(\abs{\Psi_n}) \rmd \mu
  <+\infty$. Therefore, we have for any $n \in \nset$
  \begin{equation}
    \textstyle \int_{\cball{0}{K}^\complementary} \abs{\Psi_n} \rmd \mu  \leq (\vareps / B) \int_{\cball{0}{K}^\complementary} f(\abs{\Psi_n}) \rmd \mu \leq \vareps ,
  \end{equation}
  which concludes the proof.
\end{proof}

\begin{lemma}\label{lemma:ruschendorf_4}Assume \textup{\Cref{ass:closeness}} and \textup{\Cref{ass:ratio_bounded}}. Then,  $\{\exp[\bigoplus_{i\in \mss}\psi^n_i]\}_{n\in \nset}$ is uniformly integrable w.r.t.~$\bar \pi$.
\end{lemma}
\begin{proof} 
  It is enough to show that the sequence
  $\{f(\exp[\bigoplus_{i\in \mss}\psi^n_i)]\}_{n\in \nset}$ is bounded in
  $\mathrm{L}^1(\bar{\pi})$, where $f:u\mapsto u \log(u)$ is continuous, convex
  and such that $\lim_{u\to \infty}f(u)/u=+\infty$, see
  \Cref{lemma:uniform_integrability}.  Let $n\in \nset$. We have
\begin{align}
   & \textstyle{\int_{(\rset^d)^{\ell+1}} f(\exp[\bigoplus_{i\in \mss}\psi^n_i)]\rmd \bar{\pi}= \KL(\pi^{nK}\mid \bar{\pi})}\\
    &\quad = \textstyle{\int_{\rset^d}\psi^{n}_{i_{K-1}}\rmd \mu_{i_{K-1}} + \sum_{k=0}^{K-2}\int_{\rset^d} \psi^{n}_{i_k} \exp[\psi^{n}_{i_k} - \psi^{n+1}_{i_k}] \rmd \mu_{i_k}} && \text{(\Cref{lemma:ruschendorf_2})}\\
    &\quad = \textstyle{\sum_{k=0}^{K-1} \int_{\rset^d}\psi^{n}_{i_k}\rmd \mu_{i_k} + \sum_{k=0}^{K-2} \int_{\rset^d}\psi^{n}_{i_k}\{\exp[\psi^{n}_{i_k} - \psi^{n+1}_{i_k}]-1\} \rmd \mu_{i_k}}\\
    & \quad \leq \textstyle{\KL(\pi^\star \mid \bar{\pi}) + (\bar{c}+1)\sum_{k=0}^{K-2} \int_{\rset^d}\psi^{n}_{i_k}\rmd \mu_{i_k}} &&\text{(\Cref{lemma:ruschendorf_1}-\ref{lemma1-item1}, \Cref{ass:ratio_bounded})}\\
    & \quad \leq \textstyle{\KL(\pi^\star \mid \bar{\pi}) + (\bar{c}+1)\sum_{k=0}^{K-2}\sup_{n\in \nset}\int_{\rset^d}\abs{\psi^{n}_{i_k}}\rmd \mu_{i_k} <\infty} \eqsp . && \text{(\Cref{lemma:ruschendorf_1}-\ref{lemma1-item3})}
\end{align}
\end{proof}

With the preliminary results stated above, we are now ready to prove \Cref{prop:conv_mipf}.

\begin{proof}[Proof of \Cref{prop:conv_mipf}] Using \Cref{ass:closeness} and \Cref{ass:ratio_bounded}, we have, by \Cref{lemma:ruschendorf_4}, uniform integrability of $\{\exp[\bigoplus_{i\in \mss}\psi^n_i]\}_{n\in \nset}$ in $\mathrm{L}^1(\bar{\pi})$. Therefore, the sequence $\{\pi^{nK}\}_{n\in \nset}$ is relatively compact with respect to the weak topology of $\sigma(\mathrm{L}^1(\bar{\pi}), \mathrm{L}^\infty(\bar{\pi}))$, denoted as the $\tau$-topology. We recall that $\lim_{n\to \infty}\KL(\pi^{nK+1}\mid \pi^{nK})=0$. This implies that $\{\pi^{nK+1}\}_{n\in \nset}$ is also relatively $\tau$-compact. By trivial recursion, we obtain that the sequences $\{\pi^{nK+k}\}_{n\in \nset}$, where $k\in \{2, \hdots,K-1\}$ are also relatively $\tau$-compact. Therefore, $\{\pi^{n}\}_{n\in \nset}$ is relatively $\tau$-compact and $\tau$-sequentially compact.

We consider an increasing function $\Phi:\nset \to \nset$ such that $\{\pi^{m}\}_{m\in \Phi(\nset)}$ is a $\tau$-convergent subsequence, and we denote by $\tilde{\pi}$ its limit for this topology. In particular, $\tilde{\pi}\in \Pmeasure_\mss$ by \Cref{prop:marginal}. We assume without loss of generality that $\Phi(\nset)\subset K \nset$. 

Using the lower semi-continuity of the Kullback-Leibler divergence \cite[Lemma
1.4.3]{dupuis2011weak}, we get
\begin{align}
    \KL(\tilde{\pi}\mid \Bar{\pi}) & \leq \lim \inf \KL(\pi^m\mid \Bar{\pi}) \leq 
     \lim \sup \KL(\pi^m\mid \Bar{\pi}) \eqsp .
\end{align}
Consider $k\in \{0, \hdots, K-2\}$. By \eqref{eq:ratio_pi^n}, we have
\begin{align}
  \textstyle
  \frac{\rmd \mu_{i_k}}{\rmd \pi^{nK+k}_{i_k}}= \frac{\rmd \pi^{nK+k+1}_{i_k}}{\rmd \pi^{nK+k}_{i_k}}= \frac{\rmd \pi^{nK+k+1}}{\rmd \pi^{nK+k}}= \exp[\psi^{n+1}_{i_k} - \psi^{n}_{i_k}] \eqsp ,
\end{align}
and thus, 
\begin{align}
    \textstyle\|\mu_{i_k}-\pi^{nK+k}_{i_k}\|_{\mathrm{TV}} = (1/2)\int_{\rset^d} \abs{\rmd \pi^{nK+k}_{i_k} /\rmd \mu_{i_k}-1} \rmd \mu_{i_k}= (1/2)\int_{\rset^d} \abs{\exp[\psi^n_{i_k} - \psi^{n+1}_{i_k}]-1} \rmd \mu_{i_k} \eqsp .
\end{align}
With \Cref{prop:marginal}, we obtain that
$\{\exp[\psi^n_{i_k} - \psi^{n+1}_{i_k} ]\}_{n\in \nset}$ converges to 1 in
$\mathrm{L}^1(\mu_{i_k})$. In addition using the uniform integrability of
$\{\psi^n_{i_k}\}_{n\in \nset}$ and \Cref{ass:ratio_bounded}, we get
\begin{align}
  \textstyle \lim \sup_{n \to +\infty} \int_{\rset^d} \psi^n_{i_k}\exp[\psi^n_{i_k} - \psi^{n+1}_{i_k} ] \rmd \mu_{i_k} = \lim \sup_{n \to +\infty} \int_{\rset^d} \psi^n_{i_k} \rmd \mu_{i_k} \eqsp .
\end{align}

We denote $m=K\ell$. Since $\KL(\pi^m\mid \Bar{\pi})=\int_{\rset^d}\psi^{\ell}_{i_{K-1}}\rmd \mu_{i_{K-1}} + \sum_{k=0}^{K-2}\int_{\rset^d} \psi^{\ell}_{i_k} \exp[\psi^{\ell}_{i_k} - \psi^{\ell+1}_{i_k}] \rmd \mu_{i_k}$ by \Cref{lemma:ruschendorf_2}, we finally have
\begin{align}
    \textstyle\KL(\tilde{\pi}\mid \Bar{\pi}) \leq \lim \sup \{\sum_{k=0}^{K-1} \int_{\rset^d}\psi^{\ell}_{i_k}\rmd \mu_{i_k}\}\leq \KL(\pi^\star \mid \Bar{\pi})
\end{align}
where the last inequality comes from \Cref{lemma:ruschendorf_1}.

Since $\tilde{\pi}_i=\mu_i$ for any $i \in \mss$, using \Cref{prop:marginal}, we
have $\tilde{\pi}=\pi^\star$ by uniqueness of $\pi^\star$. Hence, $\pi^\star$ is
the only limit point of $\{\pi^n\}_{n \in \nset}$ in the $\tau$-topology. In
particular, $\KL(\pi^n \mid \bar{\pi}) \to \KL(\pi^\star \mid \bar{\pi})$. Since
$\Pmeasure_\mss$ is convex, this last result implies
$\|\pi^\star -\pi^n\|_{\mathrm{TV}} \to 0$, see the proof of Theorem 2.1 in
\cite{csiszar1975divergence}.
\end{proof}

We finish this section by highlighting that \Cref{ass:ratio_bounded} is stronger
than \cite[B1]{ruschendorf1995convergence}. A natural extension of the latter assumption would consist of having a guarantee on the $(K-1)$ first potentials given by \eqref{eq:def_psi_n}, as presented below.
\begin{assumption}\label{ass:ideal_ass}
    There exist $0 < \underline{c} < \bar c$ such that for any $k\in \{0, \hdots, K-2\}$, we have $\underline{c}\leq \exp(-\psi^1_{i_k})\leq \bar{c}$.
\end{assumption}

Under \Cref{ass:ideal_ass}, \cite[Lemma 4.3]{ruschendorf1995convergence} can be adapted as written below.

\begin{lemma}\label{lemma:ruschendorf_3} Assume \Cref{ass:ideal_ass}. Then, for any $n \in \nset^*$
\begin{enumerate}[wide, labelwidth=!, labelindent=0pt, label=(\alph*)]
    \item \label{lemma3-item1}for any $k\in \{0, \hdots, K-2\}$, there exists $\alpha_{n,k}\in \nset$ such that 
    \begin{align}
        \underline{c}\cdot(\underline{c}/\bar c)^{\alpha_{n,k}(K-2)}\leq \exp[\psi^{n-1}_{i_k}-\psi^{n}_{i_k}]\leq \bar{c} \cdot (\bar{c}/\underline{c})^{\alpha_{n,k}(K-2)}
    \end{align}
    \item \label{lemma3-item2} there exists $\alpha_{n,K-1}\in \nset$ such that
    \begin{align}
        1/\bar{c}^{K-1} \cdot (\underline{c}/ \bar c)^{\alpha_{n,K-1}(K-2)} \leq \exp[\psi^{n-1}_{i_{K-1}}-\psi^{n}_{i_{K-1}}]\leq 1/\underline{c}^{K-1}\cdot (\bar c/\underline{c})^{\alpha_{n,K-1}(K-2)} 
    \end{align}
\end{enumerate}
where $\{\alpha_{n,k}\}_{n\in \nset^*, k\in \{0, \hdots, K-1\}}$ is a strictly increasing sequence that can be explicitly defined.
\end{lemma}
\begin{proof}We prove the result by recursion on $n \in \nset^*$. 

Take $n=1$. Let $k\in \{0, \hdots, K-2\}$. We define $\alpha_{1,k}=0$ and directly obtain \ref{lemma3-item1} by \Cref{ass:ratio_bounded} since $\psi_{i_k}^0=0$. Let us prove \ref{lemma3-item2}. We have by \eqref{eq:def_psi_n}
\begin{align}
    \exp[\psi^{0}_{i_{K-1}}-\psi^{1}_{i_{K-1}}]& \textstyle =\frac{\int_{(\rset^d)^\ell} \exp[\bigoplus_{k=0}^{K-2} \psi^1_{i_k}]h \rmd \nu_{-i_{K-1}}}{\int_{(\rset^d)^\ell} \exp[\bigoplus_{k=0}^{K-2} \psi^0_{i_k}]h \rmd \nu_{-i_{K-1}}}\\
    &\textstyle = \frac{\int_{(\rset^d)^\ell} \exp[\bigoplus_{k=0}^{K-2} \{\psi^1_{i_k}-\psi^0_{i_k}\} + \bigoplus_{k=0}^{K-2} \psi^0_{i_k}]h \rmd \nu_{-i_{K-1}}}{\int_{(\rset^d)^\ell} \exp[\bigoplus_{k=0}^{K-2} \psi^0_{i_k}]h \rmd \nu_{-i_{K-1}}} \eqsp .
\end{align}
Using \ref{lemma3-item1} at rank $n=1$, we have
\begin{align}
    \textstyle 1/\bar{c}^{K-1} \leq \exp[\bigoplus_{k=0}^{K-2} \{\psi^1_{i_k}-\psi^0_{i_k}\}] \leq 1/\underline{c}^{K-1} \eqsp ,
\end{align}
and therefore, we obtain \ref{lemma3-item2} by taking $\alpha_{1,K-1}=0$. Let
us assume that the result is verified for some $n\in \nset^*$. We have
\begin{align}
    \exp[\psi^{n}_{i_0}-\psi^{n+1}_{i_0}]&\textstyle =\frac{\int \exp[\bigoplus_{k=1}^{K-1} \psi^{n}_{i_k}]h \rmd \nu_{-i_0}}{\int \exp[\bigoplus_{k=1}^{K-1} \psi^{n-1}_{i_k}]h \rmd \nu_{-i_0}} \\
    &\textstyle= \frac{\int \exp[\bigoplus_{k=1}^{K-2} \{\psi^{n}_{i_k}-\psi^{n-1}_{i_k}\} \oplus \{\psi^{n}_{i_{K-1}}-\psi^{n-1}_{i_{K-1}}\} + \bigoplus_{k=1}^{K-1} \psi^{n-1}_{i_k}]h \rmd \nu_{-i_0}}{\int \exp[\bigoplus_{k=1}^{K-1} \psi^{n-1}_{i_k}]h \rmd \nu_{-i_0}}
\end{align}
Using \ref{lemma3-item1} and \ref{lemma3-item2} at rank $n$, we have
\begin{align}
  1/\bar{c}^{K-2} \cdot (\underline{c}/\bar c)^{(K-2)\sum_{k=1}^{K-2}\alpha_{n,k}} & \leq \textstyle{\exp[\bigoplus_{k=1}^{K-2} \{\psi^{n}_{i_k}-\psi^{n-1}_{i_k}\}]} \\
  & \qquad \qquad \leq 1/ \underline{c}^{K-2} \cdot (\bar c/ \underline{c})^{(K-2)\sum_{k=1}^{K-2}\alpha_{n,k}} \eqsp ,\\
   \underline{c}^{K-1}\cdot (\underline{c}/\bar c)^{\alpha_{n,K-1}(K-2)}& \leq \exp[\psi^{n}_{i_{K-1}}-\psi^{n-1}_{i_{K-1}}]\leq \bar{c}^{K-1} \cdot (\bar c/ \underline{c})^{\alpha_{n,K-1}(K-2)}\eqsp .
\end{align}
Therefore, we obtain
\begin{align}
    \underline{c}\cdot (\underline{c}/\bar c)^{(K-2)\sum_{k=1}^{K-1}\alpha_{n,k}}& \leq\textstyle{\exp[\bigoplus_{k=1}^{K-2} \{\psi^{n}_{i_k}-\psi^{n-1}_{i_k}\} \oplus \{\psi^{n}_{i_{K-1}}-\psi^{n-1}_{i_{K-1}}\}]}\\ & \qquad \qquad \leq \Bar{c} \cdot (\bar c/ \underline{c})^{(K-2)\sum_{k=1}^{K-1}\alpha_{n,k}} \eqsp ,\\
    \underline{c}\cdot (\underline{c}/\bar c)^{(K-2)\sum_{k=1}^{K-1}\alpha_{n,k}}&\leq \exp[\psi^{n}_{i_0}-\psi^{n+1}_{i_0}]  \leq \Bar{c} \cdot (\bar c/ \underline{c})^{(K-2)\sum_{k=1}^{K-1}\alpha_{n,k}} \eqsp .
\end{align}

Now, we define $\alpha_{n+1,0}=\sum_{k=1}^{K-1}\alpha_{n,k}$ to obtain \ref{lemma3-item1} for $k=0$. Consider now $k\in \{1, \hdots, K-2\}$. Following the same steps as above, we recursively define
\begin{align}
  \textstyle 
    \alpha_{n+1,k}=\sum_{j=0}^{k-1}\alpha_{n+1,j}+ \sum_{j'=k+1}^{K-1}\alpha_{n,j'}\eqsp ,
\end{align}
which gives \ref{lemma3-item1} at rank $n+1$. Let us now prove \ref{lemma3-item2} at rank $n+1$. We have
\begin{align}
    \exp[\psi^{n}_{i_{K-1}}-\psi^{n+1}_{i_{K-1}}]& \textstyle =\frac{\int \exp[\bigoplus_{k=0}^{K-2} \psi^{n+1}_{i_k}]h \rmd \nu_{-i_{K-1}}}{\int \exp[\bigoplus_{k=0}^{K-2} \psi^{n}_{i_k}]h \rmd \nu_{-i_{K-1}}}\\
    &\textstyle = \frac{\int \exp[\bigoplus_{k=0}^{K-2} \{\psi^{n+1}_{i_k}-\psi^{n}_{i_{k}}\} + \bigoplus_{k=0}^{K-2} \psi^{n}_{i_k}]h \rmd \nu_{-i_k}}{\int \exp[\bigoplus_{k=0}^{K-2} \psi^{n}_{i_k}]h \rmd \nu_{-i_{K-1}}} \eqsp .
\end{align}
Using \ref{lemma3-item1} at rank $n+1$, we obtain
\begin{align}
  \textstyle 1/\bar{c}^{K-1} \cdot (\underline{c}/ \bar c)^{(K-2)\sum_{k=0}^{K-2}\alpha_{n+1,k}} &\textstyle \leq \exp[\bigoplus_{k=0}^{K-2} \{\psi^{n+1}_{i_k}-\psi^{n}_{i_k}\}] \\
  &\leq 1/\underline{c}^{K-1}\cdot (\bar c/\underline{c})^{(K-2)\sum_{k=0}^{K-2}\alpha_{n+1,k}} \eqsp .
\end{align}
Therefore, by taking $\alpha_{n+1,K-1}=\sum_{k=0}^{K-2}\alpha_{n+1,k}$, we obtain \ref{lemma3-item2}, which concludes the proof.
\end{proof}

Unfortunately, \Cref{lemma:ruschendorf_3} only yields non-vacuous bounds in the case $K=2$. Indeed, when $K>2$, the sequence $\{\alpha_{n,k}\}_{n\in \nset^*, k\in \{0,\hdots, K-1\}}$ leads to increase the bounds on the quantities $\exp[\psi^{n-1}_{i_k}-\psi^{n}_{i_k}]$, which motivates the use of \Cref{ass:ratio_bounded}. 


\subsection{Proof of \Cref{sec:barycenter}}
\label{subsec:proofs-barycenter}

For the rest of this section, we consider the multi-marginal Schr\"odinger bridge problem given by \eqref{eq:tree_sb_static} and establish in \Cref{prop:wasserstein_propagation} the correspondence with the regularized Wasserstein propagation problem presented in \cite{solomon2014wasserstein, solomon2015convolutional}. We first state a technical result.

\begin{lemma} \label{lemma:equiv_KL} Let $\varepsilon>0$. Assume that $\pi^0$ is given by \eqref{eq:pi_0}, where $r\in \msv$ is chosen arbitrarily. Then, for any $\pi \in \Pmeasure_{\mst_r}$, we have
\begin{align}
     \varepsilon\KL(\pi\mid \pi^0) & =\textstyle{ 
     \sum_{(v,v')\in \mse_r} \{w_{v,v'}\mathbb{E}_{\pi_{v,v'}}[\|X_v-X_{v'}\|^2] -\varepsilon \ent(\pi_{v,v'})}\}\\
     & \qquad \textstyle{+\varepsilon \sum_{v \in \msv} \card(\mathrm{C}_v) \ent(\pi_v) + \varepsilon\KL(\pi_r\mid \pi^0_r)} \eqsp, 
\end{align}
where we recall that $\mathrm{C}_v=\{v' \in \msv: (v,v')\in \mse_r\}$.
\end{lemma}

\begin{proof} Since $\pi,\pi^0 \in \Pmeasure_{\mst_r}$, we obtain the following decomposition
\begin{align}\label{eq:KL_multi}
     & \KL(\pi\mid \pi^0)\\
     & \quad= \textstyle{\KL(\pi_r \prod_{(v,v')\in \mse_r}\pi_{v'|v} \mid \pi^0_r \prod_{(v,v')\in \mse_r}\pi^0_{v'|v} )}\\
     &\quad=\textstyle{\KL(\pi_r\mid \pi^0_r) + \sum_{(v,v') \in \mse_r} \int_{\rset^d} \KL(\pi_{v'|v}(\cdot|x_v)\mid \pi^0_{v'|v}(\cdot|x_v))\rmd \pi_{v}(x_v)}\\
     & \quad= \textstyle{\KL(\pi_r\mid \pi^0_r)- \sum_{(v,v') \in \mse_r} \int_{\rset^d \times \rset^d} \log \pi^0_{v'|v} \rmd \pi_{v, v'} - \sum_{(v,v')\in \mse_r}\int_{\rset^d} \ent(\pi_{v'|v}(\cdot|x_v)) \rmd \pi_{v}(x_{v})} \eqsp.
\end{align}
We finally obtain the result by using the definition of $\pi^0$ and noticing that $\int_{\rset^d} \ent(\pi_{v'|v}(\cdot|x_v)) \rmd \pi_{v}(x_{v})= \ent(\pi_{v,v'})-\ent(\pi_v)$ for any $(v,v')\in \mse_r$.
\end{proof}

\begin{proposition}\label{prop:wasserstein_propagation}Let $\vareps>0$ and $\mu_0\in \Pmeasure$ such that $\mu_0 \ll \Leb$. Assume that $\pi^0$ is given by \eqref{eq:pi_0}, where $r\in \msv$ is chosen arbitrarily, and that  $\varphi_r=\rmd \mu_0/\rmd \Leb$. Also assume
\Cref{ass:kl_bounded}. Then, the set of marginals of the solution to \eqref{eq:tree_sb_static} is exactly the solution to the entropic-regularized Wasserstein Propagation problem \citep{solomon2014wasserstein, solomon2015convolutional} defined by
\begin{align}\label{eq:wasserstein_prop}\tag{WP}
\textstyle{\arg \min \{ \sum_{(v,v')\in \mse_r} w_{v,v'} W^2_{2, \varepsilon/w_{v,v'}}(\nu_v, \nu_{v'}) + \varepsilon \sum_{v \in \msv} \card(\mathrm{C}_v) \ent(\nu_v)  + \varepsilon \KL(\nu_r \mid \mu_0):} \\
\textstyle{\{\nu_v\}_{v \in \msv} \in \Pmeasure^{\ell +1}, \eqsp \nu_{i}=\mu_{i}, \forall i \in \mss\}} \eqsp,
\end{align}
where we recall that $\mathrm{C}_v=\{v' \in \msv: (v,v')\in \mse_r\}$.
\end{proposition}

\begin{proof} Assume that $\pi^0$ is given by \eqref{eq:pi_0}, where $r\in \msv$ is chosen arbitrarily, and that $\varphi_r=\rmd \mu_0/\rmd \Leb$. In particular, we have $\pi^0_r=\mu_0$. Moreover, it is clear that $\pi^0$ verifies \Cref{ass:measures}, and \Cref{ass:equivalence} by \Cref{prop:ass_equiv}.

Let $\{\nu_v\}_{v \in \msv} \in \Pmeasure^{\ell+1}$ and $\{\nu^{(v,v')}\}_{(v,v') \in \mse_r}\in (\Pmeasure^{(2)})^{\abs{\mse_r}}$. We define
\begin{align}
    F(\{\nu_v\})& = \textstyle{\sum_{(v,v')\in \mse_r} w_{v,v'} W^2_{2, \varepsilon/w_{v,v'}}(\nu_v, \nu_{v'}) + \varepsilon \sum_{v \in \msv} \card(\mathrm{C}_v) \ent(\nu_v) + \varepsilon \KL(\nu_r \mid \mu_0)} \eqsp, \\
    G(\nu_r,\{\nu^{(v,v')}\})& = \textstyle{\sum_{(v,v')\in \mse_r} \{w_{v,v'}\mathbb{E}_{\nu^{(v,v')}}[\|X_{v}-X_{v'}\|^2] -\varepsilon\ent(\nu^{(v,v')})\}}\\
    & \quad + \textstyle{\vareps \sum_{(v,v')\in \mse_r}\ent(\nu^{(v,v')}_v) + \varepsilon \KL(\nu_r \mid \mu_0)}\eqsp .
\end{align}

By definition of the regularized Wasserstein distance given in \eqref{eq:true_reg_wasserstein}, we have for any $\{\nu_v\}_{v \in \msv} \in \Pmeasure^{\ell+1}$
\begin{align}\label{eq:F_G}
    F(\{\nu_v\})=\min \{G(\nu_r,\{\nu^{(v,v')}\}): \nu^{(v,v')} \in \Pmeasure^{(2)}, \nu^{(v,v')}_v=\nu_v, \nu^{(v,v')}_{v'}=\nu_{v'}, \forall (v,v')\in \mse_r  \} \eqsp .
\end{align}
In particular, we have $F(\{\pi_v\})\leq G(\pi_r,\{\pi_{v,v'} \})$ for any $\pi \in \Pmeasure^{(\ell+1)}$. We now prove the result of \Cref{prop:wasserstein_propagation} in two steps denoted by \textbf{Step 1} and \textbf{Step 2}.

\paragraph{Step 1.} Let us not assume
\Cref{ass:kl_bounded} for now. In this case, we prove in \textbf{Step 1.a} and \textbf{Step 1.b} that solving \eqref{eq:wasserstein_prop} is equivalent to solving a modified version of \eqref{eq:tree_sb_static} given by 
\begin{align}\tag{$\mst_r$-TreeSB}
  \label{eq:multimarginal_sb_static_fac}
     \textstyle{\pi^\star  = \argmin \ensembleLigne{\KL(\pi|\pi^0)}{\pi \in \Pmeasure_{\mst_r}, \eqsp \pi_i = \mu_i\eqsp, \forall i \in \mss}}\eqsp .
\end{align}
Remark that any solution to \eqref{eq:multimarginal_sb_static_fac} is a solution to \eqref{eq:tree_sb_static}, but the converse result may not be true. 

\paragraph{Step 1.a: \eqref{eq:wasserstein_prop}$\implies$ \eqref{eq:multimarginal_sb_static_fac}.} Consider a solution $\{\nu^\star_v\}_{v\in \msv}$ to \eqref{eq:wasserstein_prop}. For any $(v,v')\in \mse_r$, $W^2_{2, \varepsilon/w_{v,v'}}(\nu^\star_v, \nu^\star_{v'})$ is well defined and thus, there exists $\nu^{(v,v')}\in \Pi(\nu^\star_v, \nu^\star_{v'})$ such that 
\begin{align} \label{eq:wasserstein_existence}
    \textstyle\nu^{(v,v')} \in \argmin \{\mathbb{E}_\pi[\|X_{v}-X_{v'}\|^2] -(\varepsilon/w_{v,v'})\ent(\pi): \pi\in \Pi(\nu^\star_v, \nu^\star_{v'})\} \eqsp.
\end{align}
Using the gluying lemma, we build the probability measure
$\pi^\star=\nu^\star_r \prod_{(v,v')\in \mse_r}\nu^{(v,v')}_{v'|v}$ such that (i) $\pi^\star\in \Pmeasure_{\mst_r}$, and (ii) $\pi^\star_{v,v'}$ and $\nu^{(v,v')}$ have the same distribution for any $(v,v')\in \mse_r$. In particular, we have $\pi^\star_i=\mu_i$ for any $i\in \mss$.

Let us show now that $\pi^\star$ is a solution to \eqref{eq:multimarginal_sb_static_fac}. Let $\pi \in  \Pmeasure_{\mst_r}$ such that $\pi_i=\mu_i$ for any $i\in \mss$. We have
\begin{align}
    \epsilon \KL(\pi \mid \pi^0)& =G(\pi_r, \{\pi_{v,v'}\})&& \text{(\Cref{lemma:equiv_KL})}\\
    & \geq F(\{\pi_v\})\\
    & \geq F(\{\nu^\star_v\}) && \text{(definition of $\nu^\star$)}\\
    & = G(\nu^\star_r, \{\nu^{(v,v')}\}) && \text{(see \eqref{eq:wasserstein_existence})}\\
    &= G(\pi^\star_r, \{\pi^\star_{(v,v')}\}) && \text{(definition of $\pi^\star$)}\\
    & = \epsilon \KL(\pi^\star \mid \pi^0) \eqsp . && \text{(\Cref{lemma:equiv_KL})}
\end{align}
Therefore, $\pi^\star$ is a solution to \eqref{eq:multimarginal_sb_static_fac}.

\paragraph{Step 1.b: \eqref{eq:multimarginal_sb_static_fac}$\implies$\eqref{eq:wasserstein_prop}.} Consider now a solution $\pi^\star$ to \eqref{eq:multimarginal_sb_static_fac}. Since $\pi^\star \in \Pmeasure_{\mst_r}$, we have $\pi^\star=\pi^\star_r \prod_{(v,v')\in \mse_r}\pi^\star_{v'|v}$ and $\pi^\star_i=\mu_i$ for any $i \in \mss$. 

Let us show that $\{\pi^\star_v\}_{v\in \msv}$ is a solution to \eqref{eq:wasserstein_prop}. Let $\{\nu_v\}_{v \in \msv} \in \Pmeasure^{\ell+1}$ such that $\nu_i=\mu_i$ for any $i \in \mss$.

Let $\{\nu^{(v,v')}\}_{(v,v') \in \mse_r}$ be a family of probability measures such that $\nu^{(v,v')} \in \Pmeasure^{(2)}, \nu^{(v,v')}_v=\nu_v, \nu^{(v,v')}_{v'}=\nu_{v'}$ for any $(v,v')\in \mse_r$.

Using the gluying lemma, we build the probability measure $\pi=\nu_r \prod_{(v,v')\in \mse_r}\nu^{(v,v')}_{v'|v}$, such that (i) $\pi \in \Pmeasure_{\mst_r}$ and (ii) $\pi_{v,v'}$ and $\nu^{(v,v')}$ have the same distribution for any $(v,v')\in \mse_r$. We have
\begin{align}
   \varepsilon \KL(\pi \mid \pi^0)&=G(\pi_r,\{\pi_{(v,v')}\}) && \text{(\Cref{lemma:equiv_KL})}\\
   & =  G(\nu_r,\{\nu^{(v,v')}\}) && \text{(definition of $\pi$)}\\
   & \geq \varepsilon \KL(\pi^\star \mid \pi^0) && \text{(definition of $\pi^\star$)}\\
   & = G(\pi^\star_r,\{\pi^\star_{(v,v')}\}) \eqsp . && \text{(\Cref{lemma:equiv_KL})}
\end{align}
By taking the infimum in the previous inequality over the families $\{\nu^{(v,v')}\}_{(v,v')\in \mse_r}$, we obtain by \eqref{eq:F_G} that
\begin{align}
    F(\{\nu_v\}) \geq G(\pi^\star_r,\{\pi^\star_{(v,v')}\}) \geq F(\{\pi^\star_v\}),
\end{align}
and therefore, $\{\pi^\star_v\}_{v\in \msv}$ is a solution to \eqref{eq:wasserstein_prop}.

\paragraph{Step 2.} We now assume
\Cref{ass:kl_bounded}. By \Cref{prop:existence}, there exists a unique solution $\pi^\star \in \Pmeasure^{(\ell+1)}$ to \eqref{eq:tree_sb_static} such that we $\pi^0$-a.s. have  $(\rmd \pi^\star /\rmd \pi^0)=\exp[\bigoplus_{i \in \mss}\psi_i^\star]$, where $\{\psi_i^\star\}_{i \in \mss}$ are measurable potentials with $\psi^\star:\rset^d \to \rset$. Since $\pi^0 \in \Pmeasure_{\mst_r}$, we also have $\pi^\star \in \Pmeasure_{\mst_r}$, \ie, the potentials $\{\psi_i^\star\}_{i \in \mss}$ do not modify the Markovian nature of $\pi^0$. Therefore, $\pi^\star$ is also the unique solution to \eqref{eq:multimarginal_sb_static_fac}. Using the equivalence between \eqref{eq:multimarginal_sb_static_fac} and \eqref{eq:wasserstein_prop} established in \textbf{Step 1}, we finally obtain the result of \Cref{prop:wasserstein_propagation}.
\end{proof}
In particular, \Cref{prop:wasserstein_barycenter_equi} directly derives from \Cref{prop:wasserstein_propagation} by taking $r=i_{K-1}$ and $\mu_0=\mu_{i_{K-1}}$.

\subsection{Comparison with \cite{haasler2021multimarginal}}
\label{subsec:proofs-haasler}

In their work, \citet{haasler2021multimarginal} study the \emph{static} and \emph{discrete-state}
counterpart of our approach. Given a state space $\msx$ such that
$\abs{\msx}=n+1$ with $n\in \nset$, they establish a correspondence between
multi-marginal EOT with a general tree-based cost and discrete-time
multi-marginal static Schr\"odinger bridge, and provide an efficient method to solve
these problems. In this section, we provide details on their framework and give
a precise comparison between our theory and their results.

To be coherent with the setting of \cite{haasler2021multimarginal}, we adapt
here some of our notation. Let us define $\msz^{(q)}=\rset^{(n+1)^{q}}_+$. For
any $q \in \nset^*$, the set of probability measures on $\msx^{q}$ is defined as
$\Pmeasure^{(q)}=\{ M\in \msz^{(q)}: \langle M, \mathbf{1}\rangle =1\}$. We
denote $\Pmeasure=\Pmeasure^{(1)}$. For any tensors $M,P \in \msz^{(q)}$, the
Kullback-Leibler divergence between $M$ and $P$ is defined as
$\KL(M \mid P)=\langle M \log(M/P) -M +P, \mathbf{1}\rangle$ and the entropy of
$M$ is defined as $\ent(M)=-\KL(M \mid \mathbf{1})$, where the operations are
meant componentwise. In the rest of the section, we consider an undirected tree
$\mst=(\msv, \mse)$ with $\abs{\msv}=\ell+1$ such that $\msv$ may be identified
with $\{0, \hdots, \ell\}$.

\paragraph{Details on the results of \cite{haasler2021multimarginal}.}In their
paper, the authors consider a cost tensor $C\in \msz^{(\ell+1)}$ that factorizes
along $\mst$, \ie, for any $\{j_0, \dots, j_\ell\}$ with for any
$i\in \{0, \hdots, \ell\}$, $j_i\in \{0,\hdots,n\}$, we have
\begin{align}
  \textstyle 
    C_{j_0, \hdots, j_\ell}=\sum_{(v,v')\in \mse}C_{j_v, j_{v'}}^{\{v,v'\}} ,
\end{align}
where $C^{\{v,v'\}}\in \msz^{(2)}$ is a cost matrix for transportation between the marginals at vertices $v$ and $v'$, see \cite[Eq. (3.1)]{haasler2021multimarginal}. In particular, this cost can be seen as the discrete counterpart of the tree-based cost introduced in \eqref{eq:tree_cost} in the quadratic setting.

Given a subset $\mss\subset \msv$ with $\abs{\mss}=K$ and a set of marginals $\{\mu_i\}_{i \in \mss}\in \Pmeasure^K$, \cite{haasler2021multimarginal} study the EOT problem associated to $\mst$, see \cite[Eq. (2.4)]{haasler2021multimarginal}, which is given by
\begin{align}\tag{discrete-EmOT}\label{eq:discrete_EOT}
    \argmin \{\langle C, M \rangle-\varepsilon\ent(M): M \in \Pmeasure^{(\ell+1)}, \operatorname{proj}_i(M)=\mu_i, \forall i \in \mss\} \eqsp .
\end{align}
This problem may be solved with Sinkhorn algorithm
\citep{cuturi2013sinkhorn,knight2008sinkhorn,sinkhorn1967concerning}, for which
the authors provide an efficient implementation adapted to the tree-based
setting, see \cite[Algorithm 3.1]{haasler2021multimarginal}. Moreover, they
state the convergence of their method in \cite[Theorem
3.5]{haasler2021multimarginal}, as a direct consequence of the results presented
in \cite{luo1992convergence}.

In \cite[Section 4.2]{haasler2021multimarginal}, it is assumed that $\mss$
corresponds to the set of the leaves of $\mst$, as we do, and it is shown an
equivalence between \eqref{eq:discrete_EOT} and the discrete-state static SB problem
stated in \cite[Eq 4.2]{haasler2021multimarginal}, which is given by
\begin{align}
    & \textstyle{\argmin\{\sum_{(v,v')\in \mse_r}\KL(M^{(v,v')}\mid \diag(\nu_v)A^{(v,v')}): } \tag{discrete-TreeSB}\label{eq:discrete_TreeSB}\\
    & \quad M^{(v,v')}\in \Pmeasure^{(2)}, \{\nu_v\}_{v \in \msv} \in \Pmeasure^{\ell+1}, M^{(v,v')} \mathbf{1}=\nu_{v}, {M^{(v,v')}}^\transpose \mathbf{1}=\nu_{v'}, \eqsp \nu_i=\mu_i, \forall i \in \mss \} \eqsp,
\end{align}
where $\mst_r=(\msv, \mse_r)$ is the directed version of $\mst$ rooted in an arbitrary vertex $r\in \mss$, and $A^{(v,v')}=\exp(-C^{(v,v')}/\varepsilon) \in \msz^{(2)}$ for any $(v,v')\in \mse_r$. Remark that $A^{(v,v')}$ may not necessarily be a transition probability matrix.

Finally, \cite{haasler2021multimarginal} provide two main numerical
experiments. In \cite[Section 5.2]{haasler2021multimarginal}, they consider a
tree with 15 vertices, 14 edges and 8 leaves, combined to the state-space
$\msx=\{0,1\}^{50 \times 50}$, and solve the corresponding
\eqref{eq:discrete_EOT} problem for the quadratic cost. In \cite[Section
6]{haasler2021multimarginal}, they apply their methodology to estimate ensemble
flows on a hidden Markov chain. Given $\tau \in \nset^*$, they consider a tree
$\mst$ with $\tau$ internal vertices (modeling the distribution of $N$ agents at
time $t\in \{1, \hdots, \tau\}$), that are linearly linked, and such that each
of these vertices is independently linked to $S$ leaves of $\mst$ (modeling
observations at time $t\in \{1, \hdots, \tau\}$). In this setting, the state
space is given by $\msx=\{1, \hdots, 100\}^N$. They solve the formulation
\eqref{eq:discrete_TreeSB} where the reference measure is chosen as a random
walk.

\paragraph{Comparison with our results.} We now establish remarks on the main differences between our methodology and the work of \cite{haasler2021multimarginal}.

First of all, the  continuous state-space counterpart of \eqref{eq:discrete_TreeSB} is given by 
\begin{align} \label{eq:tree_dsb_cont}
    \argmin\{\KL(\pi \mid \pi^0): \pi \in \Pmeasure_{\mst_r}, \pi_i=\mu_i, \forall i \in \mss\} \eqsp, 
\end{align}
where $\pi^0$ is a reference measure which factorizes along $\mst_r$. In this case, $\pi_{v,v'}$, $\pi_v$ and $\pi^0_{v'|v}$ in \eqref{eq:tree_dsb_cont} respectively correspond to the continuous version of $M^{(v,v')}$, $\nu_v$ and $A^{(v,v')}$ in \eqref{eq:discrete_TreeSB}. In contrast, our formulation of the multi-marginal Tree Schr\"odinger Bridge problem given in \eqref{eq:tree_sb_static} is a minimization problem over all probability measures $\pi \in \Pmeasure^{(\ell+1)}$, and is not restricted to the distributions that admit a Markovian factorization along $\mst$ as in \eqref{eq:tree_dsb_cont}. Hence, our framework may be considered more general. Remark that under \Cref{ass:measures}, \Cref{ass:kl_bounded} and $\Cref{ass:equivalence}$, \Cref{prop:existence} states that \eqref{eq:tree_sb_static} admits a unique solution $\pi^\star\ll \pi^0$ such that $(\rmd \pi^\star /\rmd \pi^0)$ can be written with potentials. Then, $\pi^\star \in \Pmeasure_{\mst_r}$ since $\pi^0 \in \Pmeasure_{\mst_r}$, and \eqref{eq:tree_sb_static} is then equivalent to \eqref{eq:tree_dsb_cont}.

Furthermore, \eqref{eq:mot_eps} is more general than the continuous version of
\eqref{eq:discrete_EOT}, which we can recover by taking any measure $\nu$ of the
form $(\rmd \nu / \rmd \Leb)=\exp[\bigoplus_{i\in \mss} \varphi_i]$ in
\eqref{eq:mot_eps}, where $\{\varphi_i\}_{i\in \mss}$ is a family of potentials
such that $\abs{\int_{\rset^d}\varphi_i\rmd \mu_i}<\infty$ for any $i \in
\mss$. As a consequence, our setting allows us to choose the root
$r\in \msv \backslash\mss$ for the SB problem, whereas
\cite{haasler2021multimarginal} only consider the case where $r\in \mss$. In
the latter case, we establish in \Cref{sec:addit-deta-tree} that $r$ can be
chosen arbitrarily, as stated by \cite[Corollary 4.3]{haasler2021multimarginal}.

Finally, TreeDSB deeply differs from the framework of \cite{haasler2021multimarginal} due its \emph{dynamic} nature. Although we solve the same tree-based static SB problem (up to continuous/discrete state-space consideration), our approach consists in computing dynamic iterates (\ie, path measures) using diffusion-based methods instead of static iterates (\ie, distributions) using Sinkhorn algorithm. This paradigm is at the core of the DSB \citep{debortoli2021diffusion} methodology, and offers an efficient approach to tackle high-dimensional settings, where Sinkhorn algorithm would 
fail. 

Here, we present some advantages of the method proposed by
\cite{haasler2021multimarginal} compared to ours. First,
\cite{haasler2021multimarginal} may choose any kind of tree-based cost in
practice, while our methodology only holds for the quadratic cost.  This
limitation is shared with all approaches based on the DSB
\citep{debortoli2021diffusion} methodology. Indeed, since the cost is determined
by the reference path measure, we often choose quadratic costs associated with
Brownian motions or Ornstein-Uhlenbeck processes. Moreover,
\citet{haasler2021multimarginal} may consider various inhomogeneous (discrete) state spaces
for the vertices of $\mst$, as presented in their numerical experiments. In our
case, this approach is not compatible with our diffusion-based method. Finally,
unlike \cite{haasler2021multimarginal}, our method is not scalable with the
number of vertices or edges in $\mst$ due to computational limits. This limitation
is common to all multi-marginal approaches which rely on neural networks to
parameterize the potential and/or the distributions of the multi-marginal OT
method, see 
\cite{li2020continuous,fan2020scalable,korotin2022wasserstein,korotin2021continuous} for instance.


\section{Further results on TreeSB}
\label{sec:addit-deta-tree}

\paragraph{Choice of the root $r$ in \eqref{eq:tree_sb_static}.} We recall that the reference measure $\pi^0$ considered in \eqref{eq:tree_sb_static}, which is defined in \eqref{eq:pi_0}, verifies $\pi^0 \in \Pmeasure_{\mst_r}$ for some fixed root $r\in \msv$ and $\pi^0_r\ll \Leb$ with density $\varphi_r$. Moreover, we have $\pi^0_{v'|v}(\cdot \mid x_v)=\mathrm{N}(x_v, \vareps/(2w_{v,v'})\Idd)$ for any $(v,v')\in \mse_r$, and thus, $\pi^0$ is entirely determined by the choice of the root $r$ and the density on the corresponding vertex $\varphi_r$. 

As presented in \Cref{ref:subsec:proofs-treeDSB}, we recall that \eqref{eq:tree_sb_static} is equivalent to any multi-marginal Tree-SB problem with a reference measure $\Bar{\pi}^0$ given by \eqref{eq:bar_pi_0}, \ie, $\Bar{\pi}^0$ writes as $(\rmd \bar{\pi}^0/\rmd\pi^0)=\exp[\bigoplus_{i \in \mss}\psi_i^0]$, where $\{\psi_i^0\}_{i \in \mss}$ is a family of measurable potentials with $\psi^0_i:\rset^d \to \rset$ such that $\abs{\int_{\rset^d} \psi_i^0\rmd \mu_i}<\infty$ for any $i\in \mss$. In the case where $r$ is chosen as a leaf of $\mst$, this result implies that \eqref{eq:tree_sb_static} is unchanged if
\begin{enumerate}[wide, labelwidth=!, labelindent=0pt, label=(\alph*)]
    \item $\varphi_r=\rmd \nu/\rmd \Leb$ where $\nu\in \Pmeasure$ is such that $\KL(\mu_{r}|\nu)<\infty$,
    \item $r$ is replaced by $r'\in \mss$, as long as $\ent(\mu_{r})<\infty$ and $\ent(\mu_{r'})<\infty$.
\end{enumerate}
Therefore, under \Cref{ass:mu_i}, the setting chosen in \Cref{sec:tree-based-diffusion} is equivalent to any other setting where $r$ is arbitrarily chosen in $\mss$ and $\varphi_r=\rmd \nu/\rmd \Leb$ where $\KL(\mu_{r}|\nu)<\infty$.

Consider now the case where $r\in \mss^\complementary$, \ie, $r$ is not a leaf of $\mst$. Then, the choice of $\varphi_r$ can not be made arbitrarily anymore, since it determines a further regularization on the $r$-th marginal of the solution to \eqref{eq:tree_sb_static}. In this setting, the sequence defined by \eqref{eq:ipf_multimarginal} is unchanged. Hence, TreeDSB proceeds in the same manner as presented in \Cref{sec:tree-based-diffusion}, except for the first iteration, which we detail now.

Let us define $\msp=\operatorname{path}_{\mst_{i_0}}(i_0,r)$, where $\mst_{i_0}=(\msv, \mse_{i_0})$ is the directed version of $\mst$ rooted in $i_0$. We recall that first iterate of \eqref{eq:ipf_multimarginal} is defined by
\begin{align}
    \pi^1=\argmin \{\KL(\pi\mid \pi^0): \pi \in \Pmeasure^{(\ell+1)}, \pi_{i_0}=\mu_{i_0}\} \eqsp .
\end{align}
Following the proof of \Cref{lemma:mipf-markov}, it is clear that 
\begin{align}
    \pi^1= \mu_{i_0} \bigotimes_{(v,v')\in \msp}\pi^0_{v'|v} \bigotimes_{(v,v')\in \mse_{i_0}\backslash \msp}\pi^0_{v'|v}= \mu_{i_0}\bigotimes_{(v,v')\in \mse_{i_0}}\pi^0_{v'|v} \eqsp,  
\end{align}
where we emphasize that $\msp=\{(v,v')\in \mse_{i_0}: (v',v)\in \mse_{r}\}$. Therefore, \Cref{prop:sinkhorn_continuous} still applies between $r$ and $i_0$, by considering $r$ instead of $i_{K-1}$. In practice, this means that the first iteration of TreeDSB consists in computing the time reversal of the path measures $\mathbb{P}^0_{(v',v)}$ for any $(v,v') \in \msp$.

\paragraph{Extension of the regularized Wasserstein barycenter problem \eqref{eq:wasserstein_barycenter_pb}.} Consider the regularized Wasserstein-2 barycenter problem defined as follows
\begin{align} \tag{$\mu_0$-regWB} \label{eq:wasserstein_barycenter_pb_leaf}
    \textstyle{\mu^\star_\vareps=\arg \min \{ \sum_{i=1}^\ell w_{i}W^2_{2, \vareps/w_i}(\mu, \mu_i) + \ell\vareps \ent(\mu) +\varepsilon \KL(\mu \mid \mu_0) : \mu \in \Pmeasure\} \eqsp, }
\end{align}
where $(w_i)_{i \in \{1, \hdots,\ell\}}\in (0, +\infty)^\ell$ and $\mu_0\in \Pmeasure$ is a reference measure. This formulation admits a further regularization compared to \eqref{eq:wasserstein_barycenter_pb}, which tends to make $\mu^\star_{\varepsilon}$ closer to $\mu_0$. In particular, given a Wasserstein barycenter problem onto a star-shaped tree, the formulation \eqref{eq:wasserstein_barycenter_pb_leaf} may be more adapted than \eqref{eq:wasserstein_barycenter_pb} if we have an \emph{a priori} on the form of the regularized barycenter. In the case where $\mu_0=\mathrm{N}(0, \sigma_0^2 \Idd)$, letting
 $\sigma_0 \to \infty$, we recover the $(\ell \vareps, (\ell-1) \vareps)$ doubly-regularized Wasserstein barycenter problem \eqref{eq:wasserstein_barycenter_pb}. In the same spirit as \Cref{prop:wasserstein_barycenter_equi}, we can derive the following result from \Cref{prop:wasserstein_propagation}, which proves that \eqref{eq:wasserstein_barycenter_pb_leaf} can be solved with TreeDSB.
\begin{proposition} Let $\vareps>0$ and $\mu_0\in \Pmeasure$ such that $\mu_0 \ll \Leb$. Assume \Cref{ass:mu_i}. Also assume that $\mst$ is a star-shaped tree with central node indexed by $0$, and that the reference measure of \eqref{eq:tree_sb_static} defined in \eqref{eq:pi_0} verifies $r=0$ and $\varphi_r=\rmd \mu_0/\rmd \Leb > 0$. Under
  \textup{\Cref{ass:kl_bounded}},
  \eqref{eq:wasserstein_barycenter_pb_leaf} has a unique solution 
  $\pi^\star_0$, where $\pi^\star$ is the solution to
  \eqref{eq:tree_sb_static}.
\end{proposition}

Below, we provide practical guidelines to parameterize $\mu_0$ when it is chosen as a Gaussian distribution.

 \paragraph{Gaussian design of $\mu_0$ in \eqref{eq:wasserstein_barycenter_pb_leaf}.} Consider an undirected star-shaped tree $\mst$ with $K+1$ vertices and leaves $\{1, \hdots, K\}$. In order to incorporate the marginal constraints in the penalization brought by $\mu_0$ when it is a Gaussian distribution, we set its mean to
$\sum_{i=1}^K \mathbb{E}[\mu_i]/ K$ and its diagonal covariance matrix as
$\alpha\times(\sum_{i=1}^K \diag(\Cov[\mu_i])^{-1}/K)^{-1}$, where the inverse operation is
component-wise and $\alpha$ is a positive hyperparameter. This choice of variance helps to correctly explore the state-space at the very first iteration of TreeDSB, which is key to ensure numerical stability. In this setting, \eqref{eq:tree_sb_static} verifies \Cref{ass:kl_bounded} and \Cref{ass:equivalence}, by \Cref{prop:ass_kl} and \Cref{prop:ass_equiv}. In particular, we use this approach for two of our experiments: synthetic Gaussian datasets and Bayesian fusion, see \Cref{sec:details-exps}.


\section{Algorithmic techniques}
\label{sec:details-algo}

\paragraph{Time discretization in TreeDSB.} Denote $k_n=(n-1)\mod(K)$ for any $n\in \nset$. Let $\mst=(\msv, \mse)$ be a weighted undirected tree and consider the multi-marginal Schr\"odinger bridge problem \eqref{eq:tree_sb_static} associated to this tree. We recall that for any $\{v,v'\}\in \mse$, we define $T_{v,v'}=\varepsilon/(2w_{v,v'})$.

Consider the path measures $\{\Pbb^{n}_{(v,v')}\}_{n \in \nset, (v,v')\in \mse_{k_n}}$ recursively defined by \ref{item:a_sinkhorn} and \ref{item:b_sinkhorn}. By combining \Cref{prop:init_tree_dsb}, \Cref{prop:sinkhorn_continuous} and results on time reversal theory \citep{haussmann1986time}, we obtain by recursion that for any $n\in \nset$, any $(v,v')\in \mse_{k_n}$, $\Pbb^{n}_{(v,v')}$ is associated with a Stochastic Differential Equation on $[0, T_{v,v'}]$ given by
\begin{align}\label{eq:sde}
    \rmd \bfX_t = f^n_{t, v, v'}(\bfX_t) \rmd t + \rmd \bfB_t, \quad \bfX_0 \sim \pi_{v}^{n} \eqsp .
\end{align}

Let $N\in \nset^*$. In order to sample from the dynamics \eqref{eq:sde} at iteration $n\in \nset$, we consider its Euler-Maruyama discretization on $(N+1)$ time steps,
\begin{align}\label{eq:sde-discrete}
    X_{m+1}= X_m + \gamma_{m+1} f^n_{t_m, v, v'}(X_m) + \sqrt{\gamma_{m+1}}Z_{m+1},\quad X_0 \sim \pi_{v}^{n} \eqsp ,
\end{align}
where $Z_m \sim \mathrm{N}(0, \Idd)$ for any $m\in \{1, \hdots, N\}$, $t_m= \sum_{i=1}^m \gamma_i$, and $\{\gamma_m\}_{m=1}^N \in (0, \infty)^N$ is a time schedule such that $\sum_{m=1}^N \gamma_m =T_{v,v'}$. This results in approximating the path measure $\Pbb^{n}_{(v,v')}$ by the joint distribution $\pi^{n,N}_{(v,v')}\in \Pmeasure^{(N+1)}$ defined by
\begin{align}
    \textstyle\pi^{n,N}_{(v,v')}= \pi_{v}^{n} \bigotimes_{m=0}^{N-1} \pi^{n,N}_{(v,v'), m+1 | m} \eqsp , 
\end{align}
where $\pi^{n,N}_{(v,v'), m+1 | m}(\cdot|x_m)= \mathrm{N}(x_m + \gamma_{m+1} f^n_{t_m, v, v'}(x_m), \gamma_{m+1}\Idd)$ for any $m\in \{0, \hdots, N-1\}$. If $N$ is chosen large enough, then $\pi^{n,N}_{(v,v'),m}$ and $\Pbb^{n}_{(v,v'), t_m}$ have approximately the same distribution for any $m \in \{0, \hdots, N\}$. Consequently, $(\Pbb^{n}_{(v,v')})^R$ is naturally approximated by the joint distribution $\tilde{\pi}^{n,N}_{(v,v')}\in \Pmeasure^{(N+1)}$ defined by
\begin{align}
\textstyle
    \tilde{\pi}^{n,N}_{(v,v')}&= \textstyle \pi_{v'}^{n} \bigotimes_{m=0}^{N-1} \pi^{n,N}_{(v,v'), N-m-1 | N-m} \eqsp .
\end{align}
If $N$ is chosen large enough, we obtain that
\begin{align}
    &\pi^{n,N}_{(v,v'), N-m-1 | N-m}(\cdot |x_{N-m})\\
    &\qquad \qquad =\mathrm{N}(x_{N-m}-\gamma_{N-m}f^n_{t_{N-m}, v, v'}(x_{N-m})+\gamma_{N-m}\nabla \log p_{v,v', t_{N-m}}(x_{N-m}), \gamma_{N-m}\Idd) \eqsp,
\end{align}
where $p_{v,v', t}$ is the density of $\Pbb^{n}_{(v,v'),t}$ w.r.t. the Lebesgue measure.

Following the construction of our dynamic iterates, we now explain how the sequence $\{\pi^n_{(v,v')}\}_{n\in \nset^*, (v,v')\in \mse_{k_n}}$ is recursively defined. Let $n \in \nset$, $k_n=(n-1) \mod(K)$. Define the path
  $\msp_n = \operatorname{path}_{\mst_{i_{k_n}}}(i_{k_n}, i_{k_n+1})$. Then, for any
  $(v,v') \in \mse_{k_n+1}$,
  \begin{enumerate}[wide, labelwidth=!, labelindent=0pt, label=(\alph*)]
  \item \label{item:ipf-discrete-a} if $(v,v') \in \mse_{k_n} \backslash \msp_n$, then
    $ \pi^{n+1,N}_{(v,v')} = \pi_v^{n+1} \bigotimes_{m=0}^{N-1} \pi^{n,N}_{(v,v'), m+1 | m}$,
  \item \label{item:ipf-discrete-b}if $(v',v) \in \msp_n$, then  $\pi^{n+1,N}_{(v,v')} = \pi_{v}^{n+1} \bigotimes_{m=0}^{N-1} \pi^{n,N}_{(v',v), N-m-1 | N-m}$.
  \end{enumerate}

These computations may be obtained by considering the sequence given by \eqref{eq:ipf_multimarginal} to solve the multi-marginal Tree-SB problem associated to $\mst^{(N)}=(\msv^{(N)}, \mse^{(N)})$, the $N$-discretized version of $\mst$ (see \Cref{sec:recap-trees}) with weights $w_{e_m}^{(N)}=2 \gamma_m/\varepsilon$, which is given by
\begin{equation}
  \textstyle{\pi^\star  = \argmin \ensembleLigne{\KL(\pi|\pi^{0,N})}{\pi \in \Pmeasure^{(\msv^{(N)})}, \eqsp \pi_i = \mu_i\eqsp, \forall i \in \mss}} \eqsp ,
  \end{equation}
with $\pi^{0,N}=\pi^0_r \bigotimes_{(v,v')\in \mse_r}\pi^{0,N}_{(v,v'),1:N|0}$. 

To approximate the IPF recursion given by \ref{item:ipf-discrete-a} and \ref{item:ipf-discrete-b}, we use \textbf{on each edge} of $\mst$ the score-matching approach of \cite{debortoli2021diffusion}, which avoids heavy computations of score approximations. The next proposition is direct adaptation of \cite[Proposition 3]{debortoli2021diffusion}.

\begin{proposition} \label{prop:score_matching}
Assume that for any $n\in \nset$, any $(v,v')\in \mse_{k_n}$ with $k_n=(n-1) \mod(K)$, we have
\begin{align}
    \pi^{n,N}_{(v,v'), m+1|m}(\cdot|x_m) =\mathrm{N}(F^n_{m,v,v'}(x_m), \gamma_m\Idd) \eqsp .
\end{align}

 Let $n \in \nset$. Consider the path
  $\msp_n = \operatorname{path}_{\mst_{i_{k_n}}}(i_{k_n}, i_{k_n+1})$. Let $(v,v') \in \mse_{k_n+1}$. Define $p^n=\pi^{n,N}_{(v,v')}$ and $m_N=N-m-1$. Then, if $(v',v) \in \msp_n$, we have 
\begin{align} \label{eq:score_matching_loss}
    &\textstyle F^{n+1}_{m,v,v'}
     =\argmin_{\mathrm{F} \in \mathrm{L}^2(\mathbb{R}^d, \mathbb{R}^d)} \\ &\quad\mathbb{E}_{p^n_{m_N, m_N+1}}[\|\mathrm{F}(X_{m_N +1}) -(X_{m_N+1}+F^{n}_{m_N,v',v}(X_{m_N})-F^{n}_{m_N,v',v}(X_{m_N+1}))\|^2] , 
\end{align}
otherwise, we have $F^{n+1}_{m,v,v'}= F^n_{m,v,v'}$.
\end{proposition}

In practice, we use two neural networks per edge $\{v,v'\}\in \mse$, one for each possible direction of the edge, such that $F_{v,v'}(\theta^n_{v,v'}, m,x)\approx F^n_{m,v,v'}(x)$ and $F_{v',v}(\theta^n_{v',v}, m,x)\approx F^n_{m,v',v}(x)$. For any $\{v,v'\}\in \mse$, the parameter $\theta^n_{v,v'}$ is updated at iteration $n$ via the score matching loss defined by \eqref{eq:score_matching_loss} in \Cref{prop:score_matching} if $(v,v')\in \operatorname{path}_{T_{i_{k_n}}}(i_{k_n}, i_{k_n +1})$, see \Cref{algo:treedsb}.


\section{Additional experimental results and details}
\label{sec:details-exps}

The numerical experiments presented in \Cref{sec:experiments} are obtained by our own Pytorch implementation, which is inspired from the code\footnote{\url{https://github.com/JTT94/diffusion_schrodinger_bridge}} provided by \cite{debortoli2021diffusion}. We first provide information on the general setting of our experiments in \Cref{sec:details-exps-1}, and then give details on each of them in \Cref{sec:details-exps-2} along with additional results. We recall that a mIPF cycle is defined as a subset of $K$ consecutive iterations of \eqref{eq:ipf_multimarginal} and that the order of the leaves given by $\{i_0, \hdots, i_{K-1}\}$ is randomly shuffled at each new mIPF cycle.

\subsection{General experimental setup}
\label{sec:details-exps-1}

\paragraph{Implementation of \Cref{algo:treedsb} in practice.}Let $n\in \nset$, with $k_n=(n-1)\mod(K)$, $k_n+1=n \mod(K)$. Consider the path $\msp_n=\operatorname{path}_{\mst_{i_{k_n}}}(i_{k_n}, i_{k_n+1})$. Assume that we are provided with a dataset $\msd_{i_{k_n}}$, which contains $M$ samples from $\pi^n_{i_{k_n}}$. Following Lines 7-9 in \Cref{algo:treedsb}, we apply processes \ref{item:process-a} and \ref{item:process-b} recursively on the edges $(v,v')\in \msp_n$. 
\begin{enumerate}[wide, labelwidth=!, labelindent=0pt, label=(\alph*)]
    \item \underline{Sampling step (Line 7).}\label{item:process-a} For any $x_0\in \msd_{v}$, we sample from the diffusion trajectory \eqref{eq:sde-discrete} given by the Euler Maruyama discretization of $\mathbb{P}^n_{v,v'}$ starting from $x_0$. This gives us $M\times N$ trajectory samples. We then store the last iterate of each trajectory in a new dataset $\msd_{v'}$, which thus approximates $\pi^n_{v'}$.
    \item \underline{Training step (Lines 8-9).}\label{item:process-b} In order to avoid heavy computation, we approximate the \emph{mean-matching} loss \eqref{eq:score_matching_loss} by an unbiased estimator obtained by subsampling $b$ elements from the \emph{full} trajectories computed in the sampling process, see \cite[Eq. (97)-(98)]{debortoli2021diffusion}. Here, $b$ refers to the \emph{batch-size} parameter of the neural networks. Then, we perform gradient descent to optimize the parameter $\theta_{v',v}$, which parameterizes the \emph{backward} drift on the edge $(v,v')$.
\end{enumerate}
To avoid any bias issue, the whole trajectories obtained at process \ref{item:process-a} are refreshed at a certain frequency over the training iterations of the neural networks by once again simulating the diffusion \eqref{eq:sde-discrete}. In our experiments, this refresh occurs each 500 iterations. 
 
\paragraph{Setting of the time discretization.}The number of time-steps $N$ in the time discretization of the diffusions is chosen to be even and identical for each of the edges of the tree. Let $\{v,v'\}\in \mse$. We now give details on the design of the time schedule $\{\gamma_k\}_{k=1}^N$ related to the edge $\{v,v'\}$, see \Cref{sec:details-algo}. Following \cite{debortoli2021diffusion}, we choose this sequence to be invariant by time reversal and consider $\gamma_k=\gamma_0 + (2k/N)(\bar{\gamma}-\gamma_0)$ for any $k\in \{0, \hdots, N/2\}$ (the rest of the sequence being obtained by symmetry) where $\gamma_0$ is a free parameter and $\bar{\gamma}$ is determined by $\sum_{k=1}^N \gamma_k=T_{v,v'}$. In our experiments, we set $N=50$ and $\gamma_0=10^{—5}$.

\paragraph{Sampling improvement.} In our code, we implemented the corrector scheme of \cite{song2020score} and the \emph{probability flow}-based sampling approach detailed in \cite[Section H.3]{debortoli2021diffusion}, but did not observe any significant improvement in our experiments using one of these techniques.

\paragraph{Choice of the architectures of the neural networks.} In the case of the experiments related to synthetic datasets (two-dimensional toy datasets, Gaussian distributions) and to the subset posterior aggregation task, we implement the same architecture as presented in \cite[Figure 3]{debortoli2021diffusion}. We refer to this model as ``Basic Model'' and detail it in \Cref{fig:architecture}. In the ``Basic Model'', the PositionalEncoding block applies the sine transform described in \cite{vaswani2017attention}, with output dimension equal to 32, and each MLP Block represents a Multilayer Perceptron Network. In particular, MLPBlock (1a) has shape $(d, 128, \max(256, 2d))$, MLPBlock (1b) has shape $(32, 128, \max(256, 2d))$, and MLPBlock (2) has shape $(2\times \max(256, 2d), \max(256, 2d), \max(128,d), d)$, where $d$ denotes the dimension of input data. We optimize the networks with ADAM \citep{kingma2014adam} with learning rate $10^{-4}$ and momentum $0.9$. For each of the networks, we set the batch size to 4,096 and the number of iterations to 10,000 for the synthetic datasets and 15,000 for the subset posterior aggregation task. Our experiments ran on 1 Intel Xeon CPU Gold 6230 20 cores @ 2.1 Ghz CPU.
\begin{figure}[h!]
    \centering
    \includegraphics[width=\linewidth]{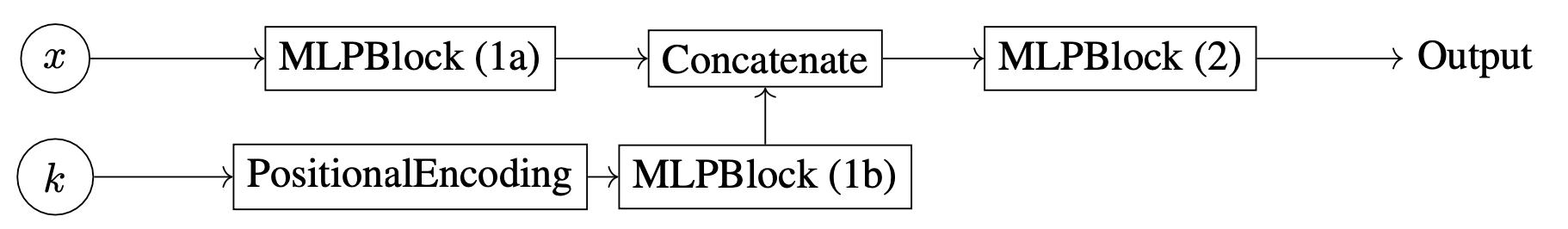}
    \caption{Architecture of the ``Basic Model''. }
    \label{fig:architecture}
\end{figure}

In the case of the experiments related to MNIST dataset, we use a reduced UNET architecture based on \cite{nichol2021improved}, where we set the number of channels to 64 rather than 128. We implement an exponential moving average of network parameters across training iterations, with rate 0.999. We optimize the networks with ADAM \citep{kingma2014adam} with learning rate $10^{-4}$ and momentum $0.9$. Finally, we set the batch size to 256 and the number of training iterations to 30,000. Our experiments ran using 1 Nvidia A100.

\paragraph{Details on regularized state-of the art methods.} We run the fsWB algorithm \citep{cuturi2014fast} with the implementation provided by \cite{flamary2021pot}. For each experiment, we run 100 Sinkhorn iterations with 1500 samples for each dataset (\ie, the maximum number of samples that it can generate) and set the regularization parameter $\vareps$ to its lowest value such that the
algorithm is stable. Finally, for sake of fairness with our method, we initialise the barycenter measure with $\pi^0_r$ when solving the problem \eqref{eq:wasserstein_barycenter_pb_leaf} for synthetic Gaussian datasets and Bayesian fusion. To run the crWB algorithm \citep{li2020continuous}, we use the code provided by the authors. We consider the quadratic regularization, which is shown to be empirically more stable than entropic regularization. Following \cite{fan2020scalable}, we choose the potential networks to be fully connected neural networks with 3 hidden layers of shape $(\max(128,2d), \max(128, 2d), \max(128,2d))$. The activation functions are ReLu. We optimize the networks with ADAM \citep{kingma2014adam} with learning rate $10^{-4}$ for the subset posterior aggregation task and $10^{-3}$ for the Gaussian experiment. Finally, we set the batch size to 4,096 and the number of training iterations to 50,000. We highlight that fsWB and crWB solve a regularized Wasserstein barycenter problem, which does not contain an additional \emph{penalization} term on the entropy of the barycenter, contrary to TreeDSB.

\subsection{Details on the experiments}
\label{sec:details-exps-2}

\paragraph{Synthetic Gaussian datasets.} For each dimension that we consider, we generate three different triplets of random non-diagonal covariance matrices whose condition number is less than 10. We then run the algorithms on each triplet and aggregate the obtained results. The Gaussian datasets contain 1,500 samples for fsWB, and 10,000 samples for crWB and TreeDSB. We run fsWB with the following settings $(d,\varepsilon)\in \{(2, 0.1), (16,0.2), (64, 0.5), (128, 1.0), (256, 2.0)\}$. We run TreeDSB for 10 mIPF cycles with regularization parameter $\vareps=0.1$, starting from the central node initialized to a Gaussian distribution $\mu_0$ chosen as detailed in \Cref{sec:details-algo} with $\alpha=1$. Thus, we solve the regularized Wasserstein barycenter problem \eqref{eq:wasserstein_barycenter_pb_leaf}, which contains an additional regularization with respect to $\mu_0$. This choice is justified, since the non-regularized barycenter is known to be a Gaussian distribution, and $\mu_0$ can be seen as an \textit{a priori} for the regularized barycenter. For each of the three settings, we keep the best result among the $30$ mIPF iterations. In this setting, TreeDSB and crWB have roughly the same training time.

\paragraph{Subset posterior aggregation.} When considering a dataset splitted into several subdatasets, a common paradigm in bayesian inference consists in running Monte Carlo Markov Chain methods separately on these subdatasets, and then merge the obtained posteriors to recover the full posterior. The barycenter of these subdataset posteriors is proved to be close to the full data posterior under mild assumptions \citep{srivastava2018scalable}. In our setting, we consider the posterior aggregation problem for the logistic regression model associated to the \texttt{wine} dataset\footnote{\url{https://archive.ics.uci.edu/ml/datasets/wine}} (d = 42) with 3 subdatasets. We consider here two splitting methods: (i) either, data is uniformly
splitted between 3 subdatasets with respect to the label distribution, denoted by \texttt{wine-homogeneous}, or (ii) data is splitted with some heterogeneity according to a Dirichlet distribution whose parameter is randomly chosen, denoted by \texttt{wine-heterogeneous}. Following \cite{korotin2021continuous}, we use the stochastic approximation trick so that the subset posterior samples do not vary consistently from the full
posterior in covariance \citep{minsker2014scalable}. We implement the Unadjusted Langevin Algorithm (ULA) to sample from each subdataset posterior and from the full posterior. In each case, we run ULA for $5.5\cdot 10^{6}$ iterations with a well chosen step-size, and obtain 9,900 samples after applying a \emph{burn-in} of order $10\%$ and then a \emph{thinning} of size 500. We provide in \Cref{fig:bayesian} some metrics which assess the quality of this sampling process. We recall that the the full posterior samples serve as ground truth in this experiment.

The results presented in \Cref{tab:comparison_bayesian} were computed as follows. For fsWB, we first subsample 1,500 samples out of the 9,900 samples from each posterior, and then run the algorithm with $\varepsilon=0.5$. We repeat three times this procedure and then aggregate the results. In the case of crWB and TreeDSB, we run the algorithms three times with various seeds. Similarly to the Gaussian setting, we run TreeDSB for 10 mIPF cycles with regularization parameter $\vareps=0.1$. We start from the central node with a Gaussian distribution $\mu_0$ chosen as detailed in \Cref{sec:details-algo} with $\alpha=1$, and thus solve the barycenter formulation \eqref{eq:wasserstein_barycenter_pb_leaf}. For each of the three settings, we keep the best result among the $30$ IPF iterations. In this setting, TreeDSB and crWB have roughly the same training time.

\begin{figure}[h!]
  \centering
  \includegraphics[width=.45\linewidth]{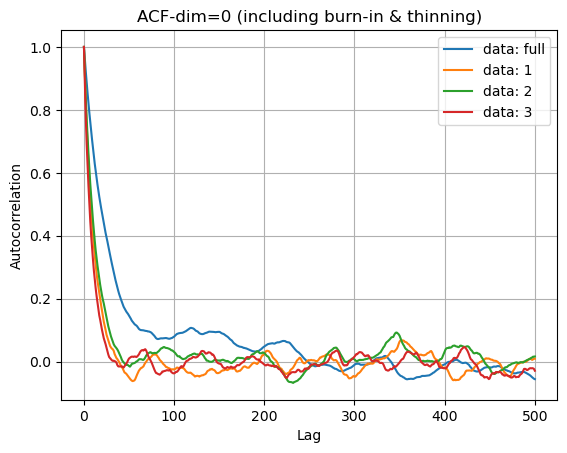} \hfill
  \includegraphics[width=.45\linewidth]{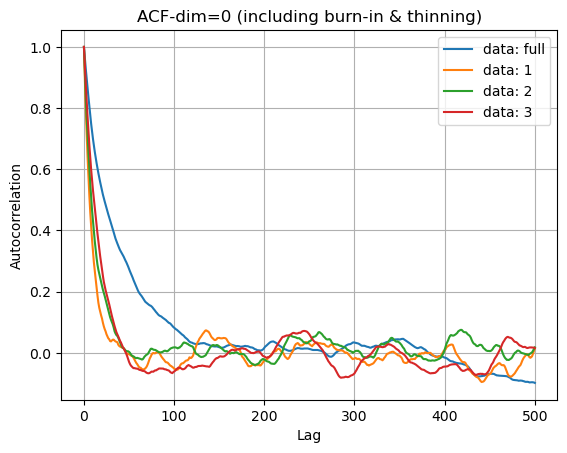}
  \includegraphics[width=.45\linewidth]{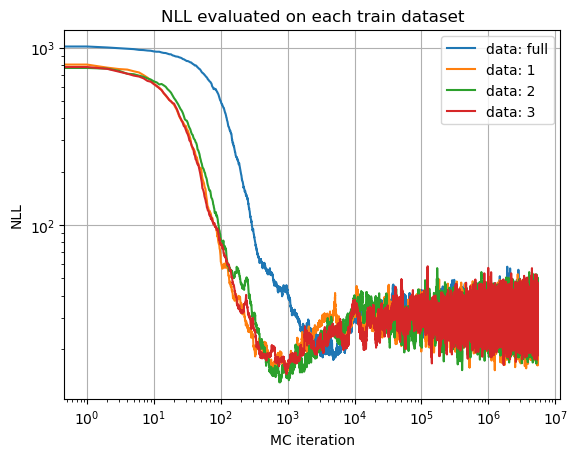} \hfill
  \includegraphics[width=.45\linewidth]{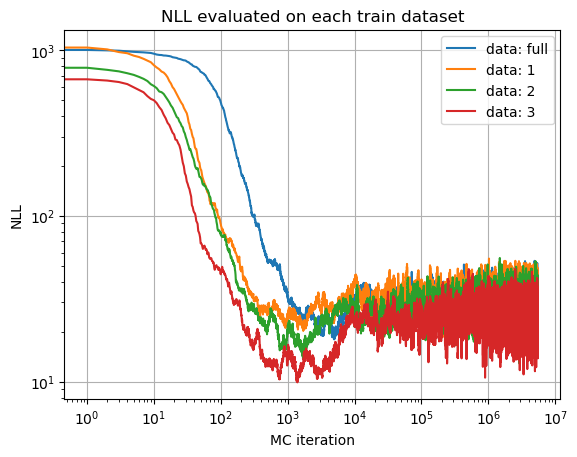}
  \caption{Evaluation of the sampling process for \texttt{wine-homogeneous} (left) and \texttt{wine-heterogeneous} (right). We display the Autocorrelation function on 500 lags (above) and the evolution over the iterations of ULA of the negative log-likelihood (NLL) evaluated on each training dataset  (below). In particular, the samples are decorrelated and the NLL has a satisfying profile.}
  \label{fig:bayesian}
\end{figure} 

\newpage
\paragraph{Synthetic two-dimensional datasets.} In this setting, we consider three different datasets (\emph{Swiss-roll}, \emph{Circle} and \emph{Moons}) that each contain 10,000 samples. Since we do not have an \textit{a priori} on the shape of the barycenter between these datasets, we consider the regularized Wasserstein barycenter problem \eqref{eq:wasserstein_barycenter_pb}, \ie, $r$ is chosen as a leaf and corresponds to one of the input datasets. We emphasize that this experiment is not intended to demonstrate the superiority of TreeDSB to compute 2D Wasserstein barycenters, but is rather meant to illustrate that (a) the marginals of the leaves are well recovered by the algorithm, see \Cref{fig:wasserstein_barycenter_2d}, and that (b) the obtained barycenter is consistent when diffusing from the different leaves, see \Cref{fig:wasserstein_coherence}. In all our experiments on 2D datasets, we observed that (a) was persistently verified without difficulty. In this section, we rather aim at illustrating (b) by providing additional results which assess the quality of the barycenter obtained by TreeDSB with respect to the choice of the starting leaf $r$ and to the choice of the regularization parameter $\varepsilon$. 

To do so, we consider three different choices of regularization in TreeDSB: (i) $\varepsilon=0.2$ (50 mIPF cycles), see \Cref{fig:2d-epsilon=0.2}, (ii) $\varepsilon=0.1$ (50 mIPF cycles), see \Cref{fig:2d-epsilon=0.1} and (iii) $\varepsilon=0.05$ (60 mIPF cycles), see \Cref{fig:2d-epsilon=0.05}. For each of these settings, we run TreeDSB with the starting leaf $r$ chosen as \emph{Swiss-roll} (first row), \emph{Circle} (second row) or \emph{Moons} (third row), and display the final barycenter obtained by diffusing from \emph{Swiss-roll} (first column), \emph{Circle} (second column) and \emph{Moons} (third column). Note that the vertex 0 always corresponds to the starting leaf, the vertex 1 to the barycenter node and that \Cref{fig:wasserstein_coherence} corresponds to the first row of \Cref{fig:2d-epsilon=0.1}.

We can make the following observations. First, the estimated barycenter is always coherent within each row, which assesses the convergence of our method. Then, for each value of $\varepsilon$, the TreeDSB barycenter is rather consistent between the rows, \ie, the choice of the starting leaf does not have a meaningful impact on our method. Finally, as expected, we observe that the support of the barycenter is less and less diffuse as long as $\varepsilon$ decreases.
\newpage

\begin{figure}[h!]
  \centering
  \includegraphics[width=.3\linewidth]{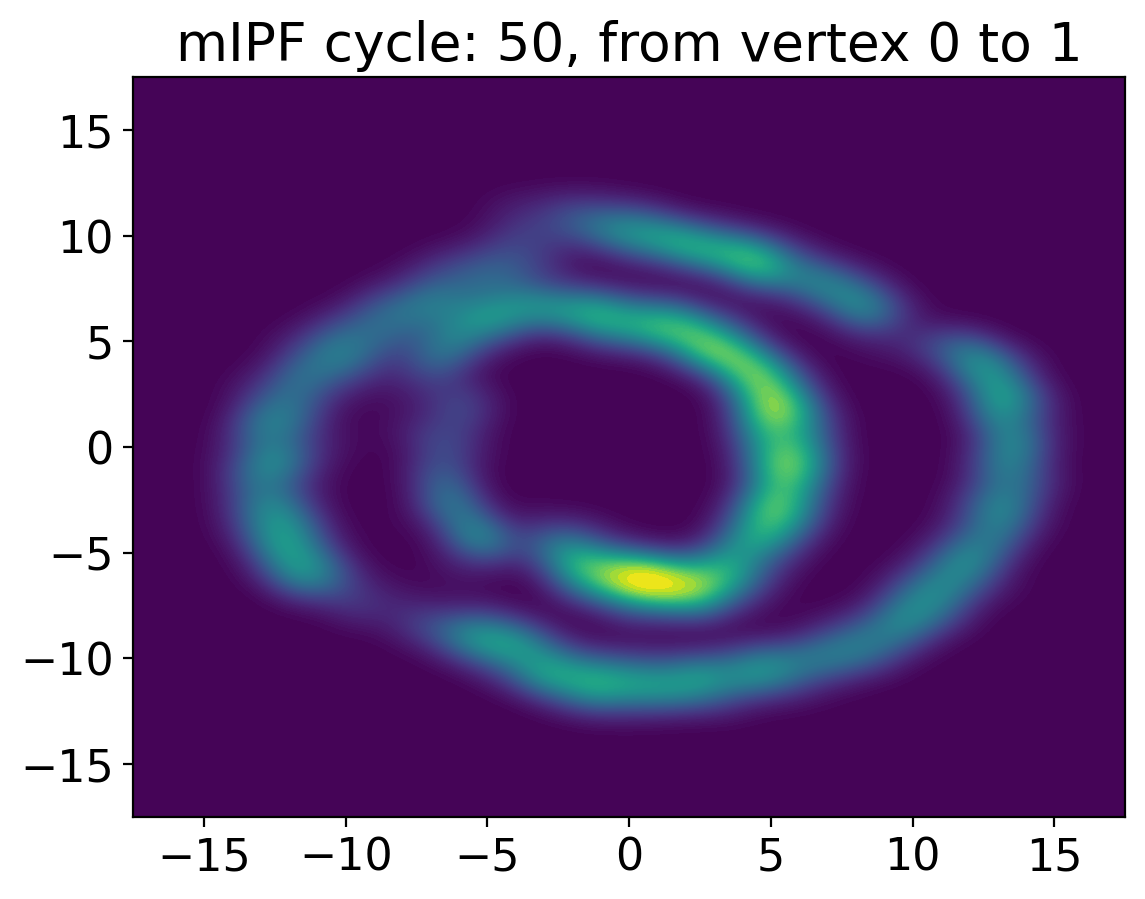} \hfill
  \includegraphics[width=.3\linewidth]{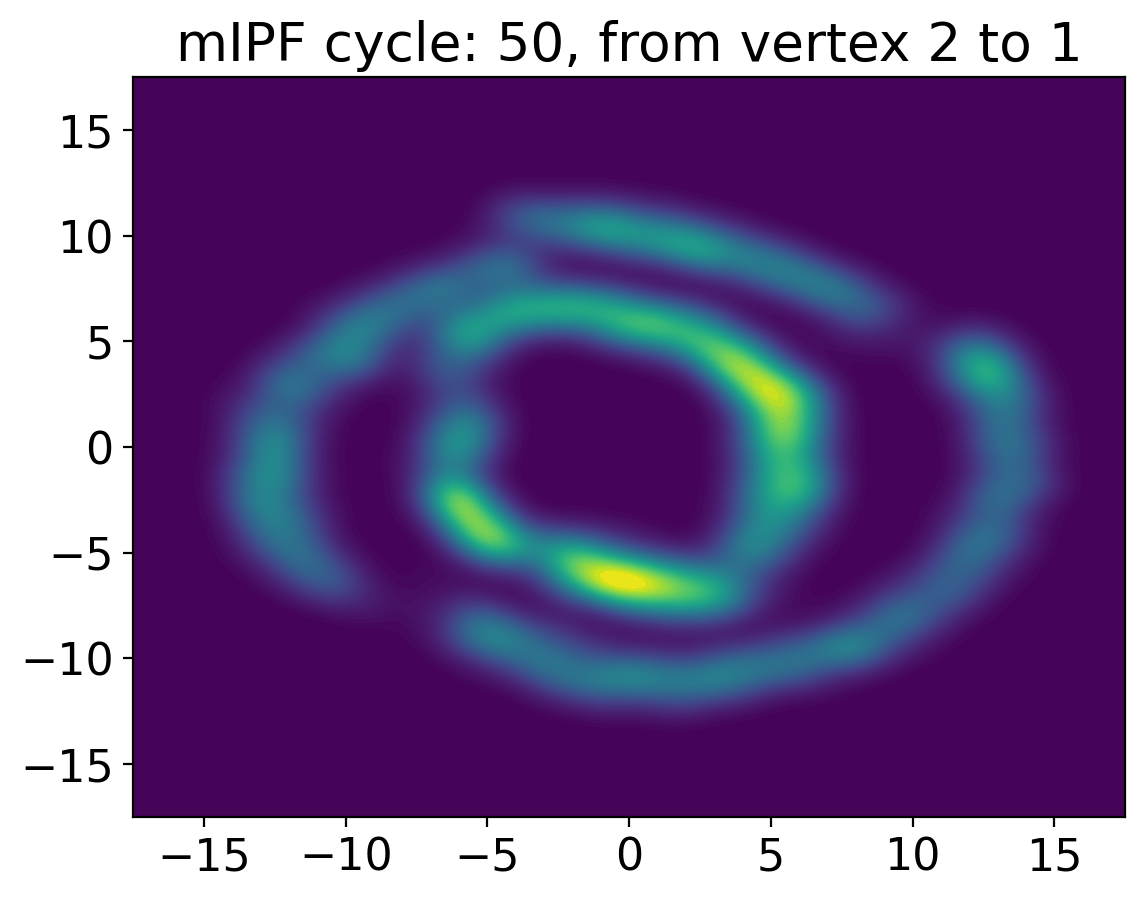} \hfill
  \includegraphics[width=.3\linewidth]{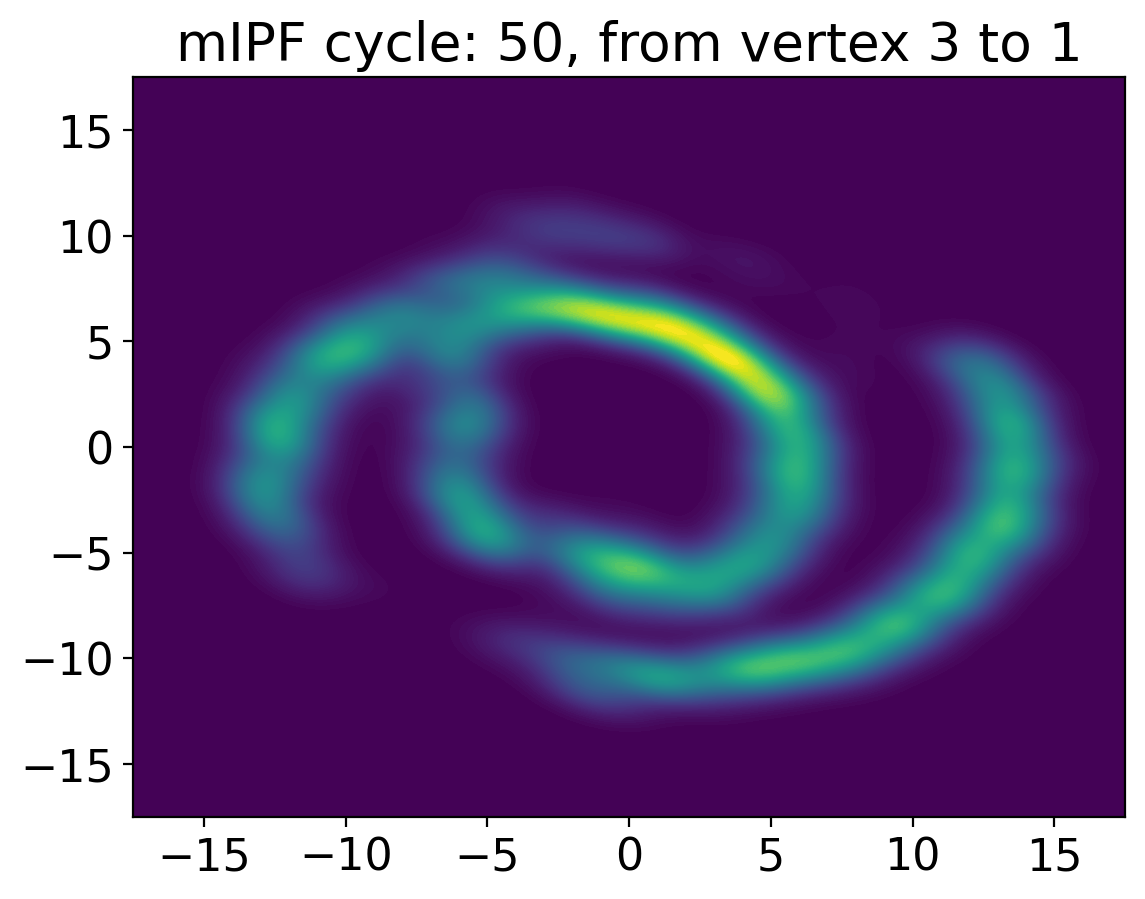}\\
  \includegraphics[width=.3\linewidth]{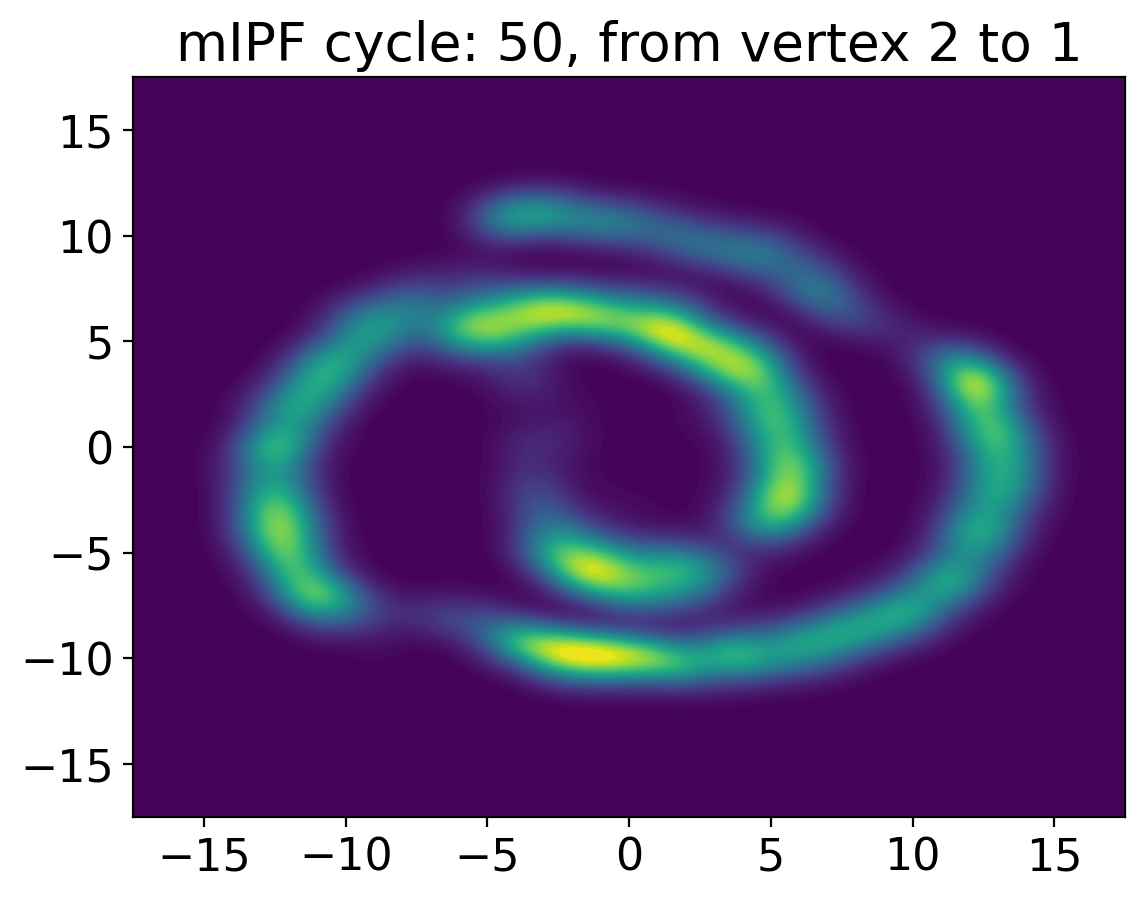} \hfill
  \includegraphics[width=.3\linewidth]{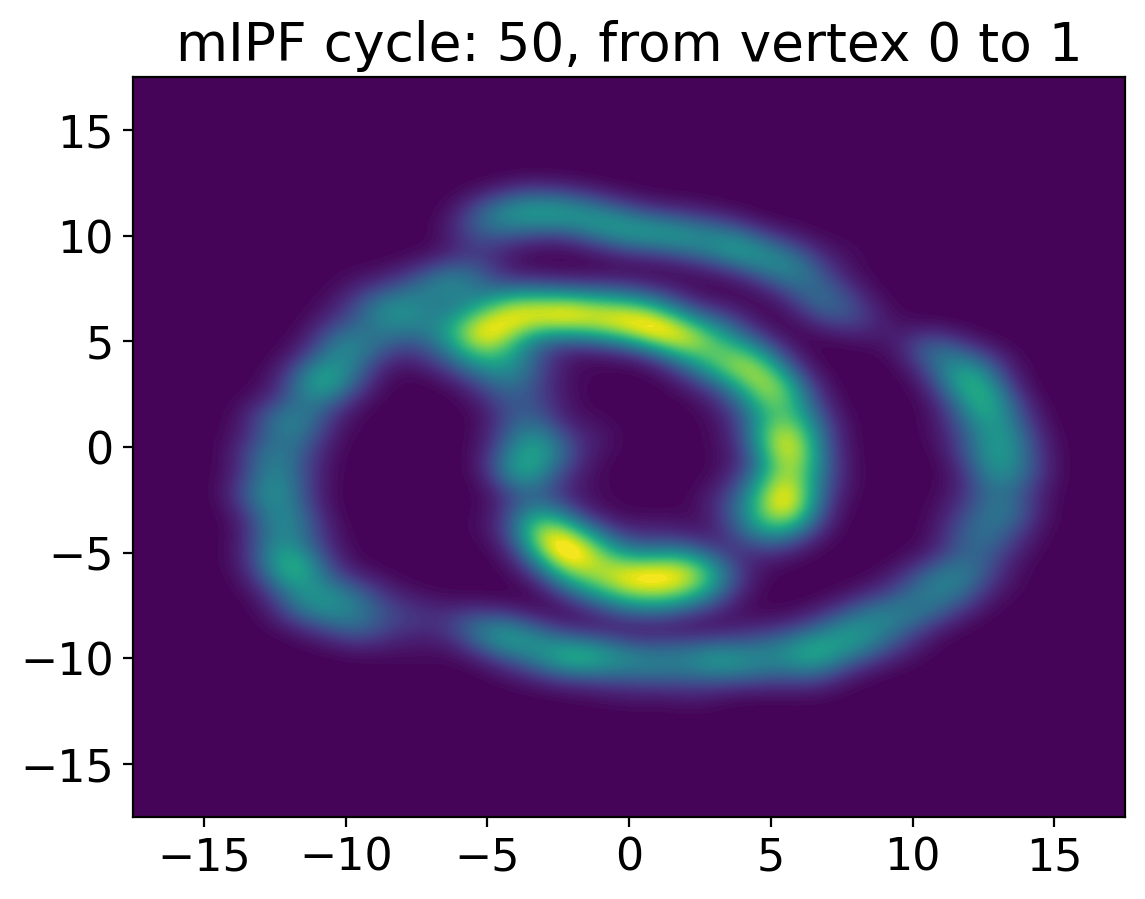} \hfill
  \includegraphics[width=.3\linewidth]{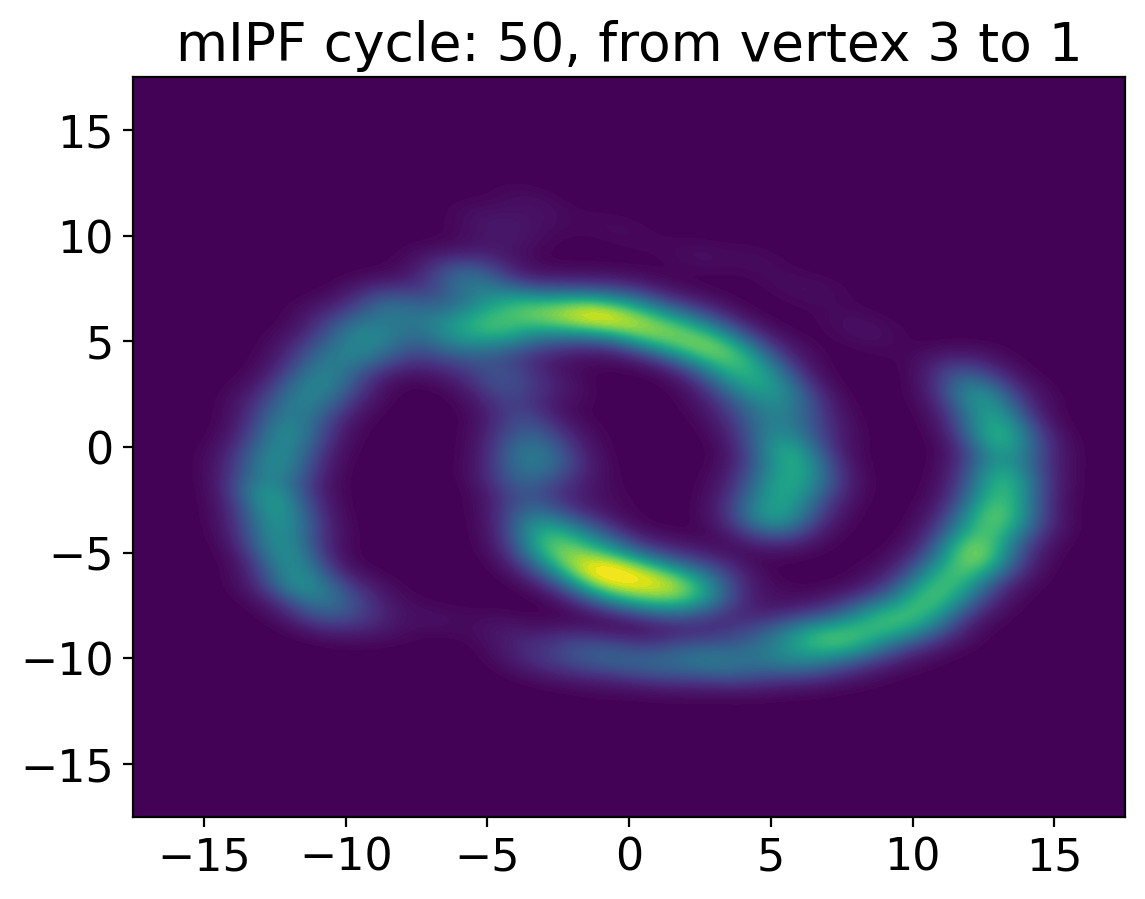}\\
  
  \includegraphics[width=.3\linewidth]{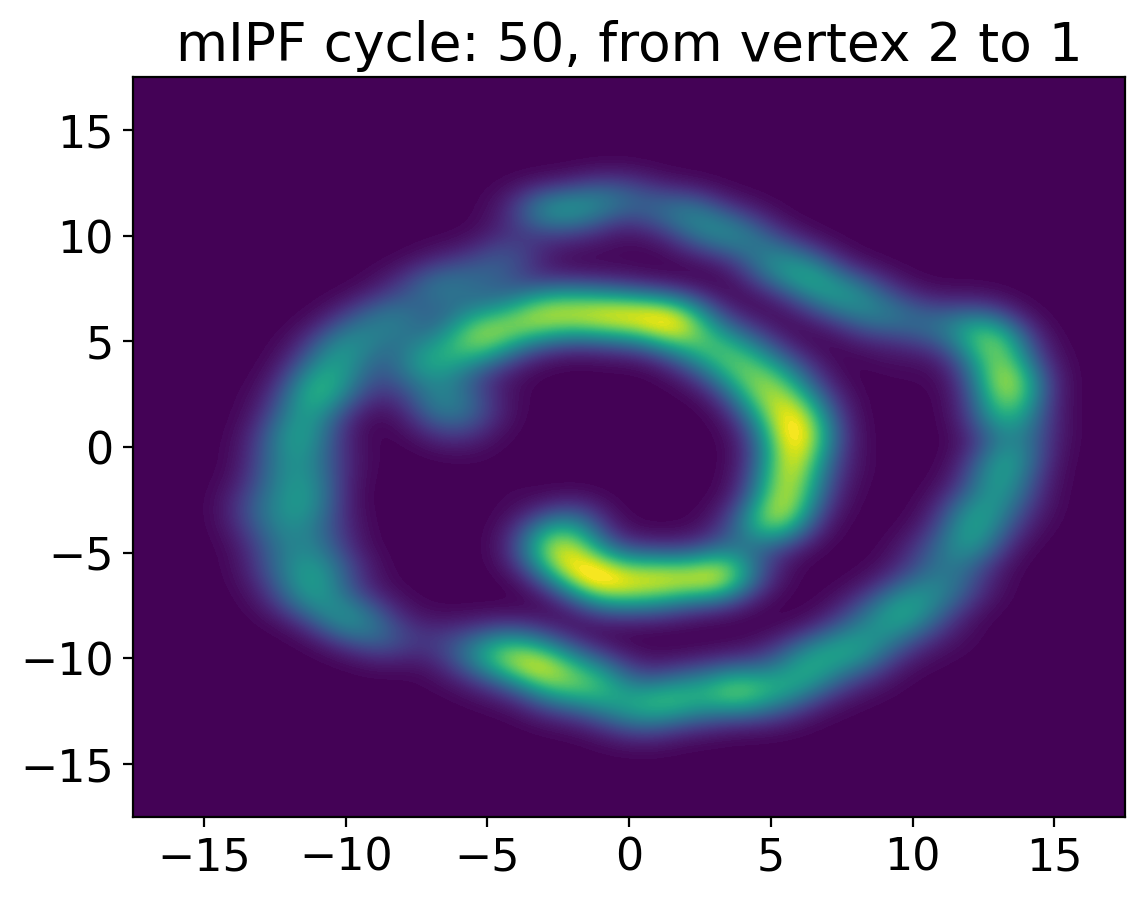} \hfill
  \includegraphics[width=.3\linewidth]{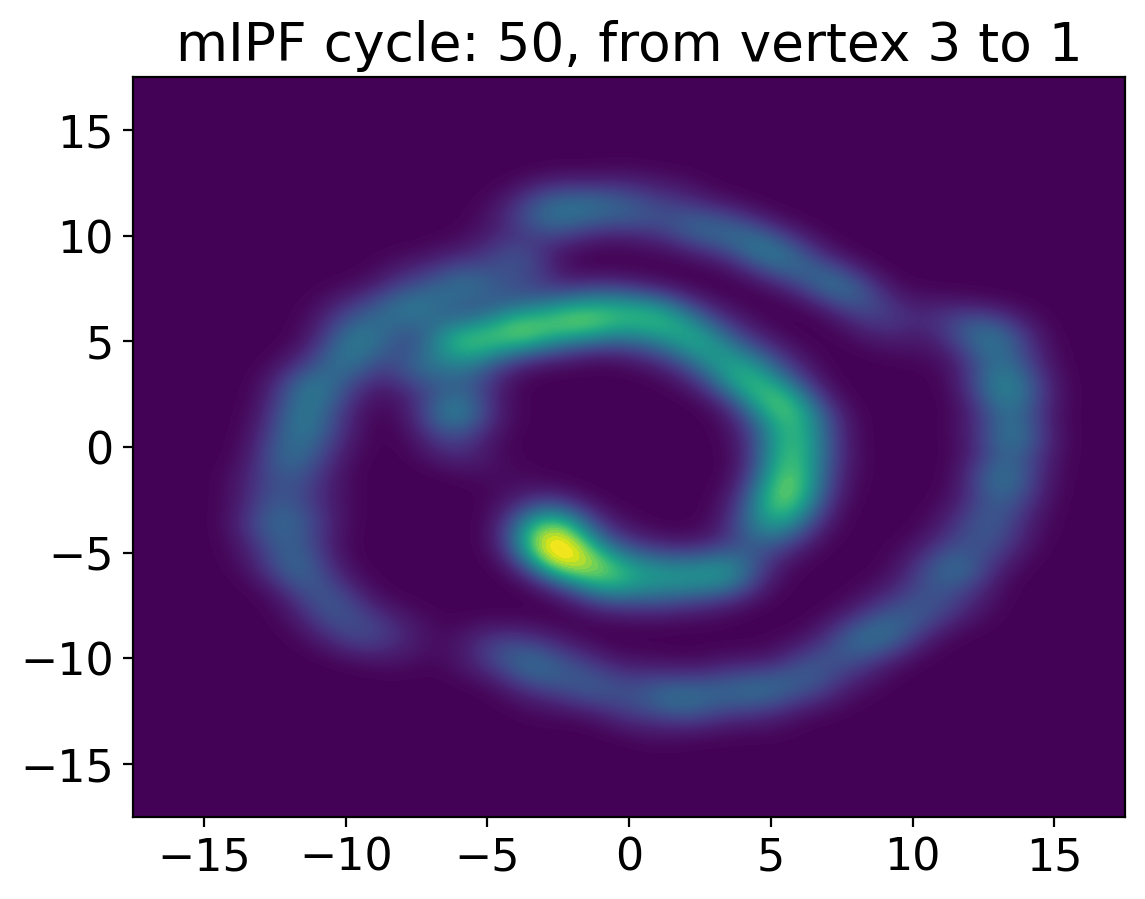} \hfill
  \includegraphics[width=.3\linewidth]{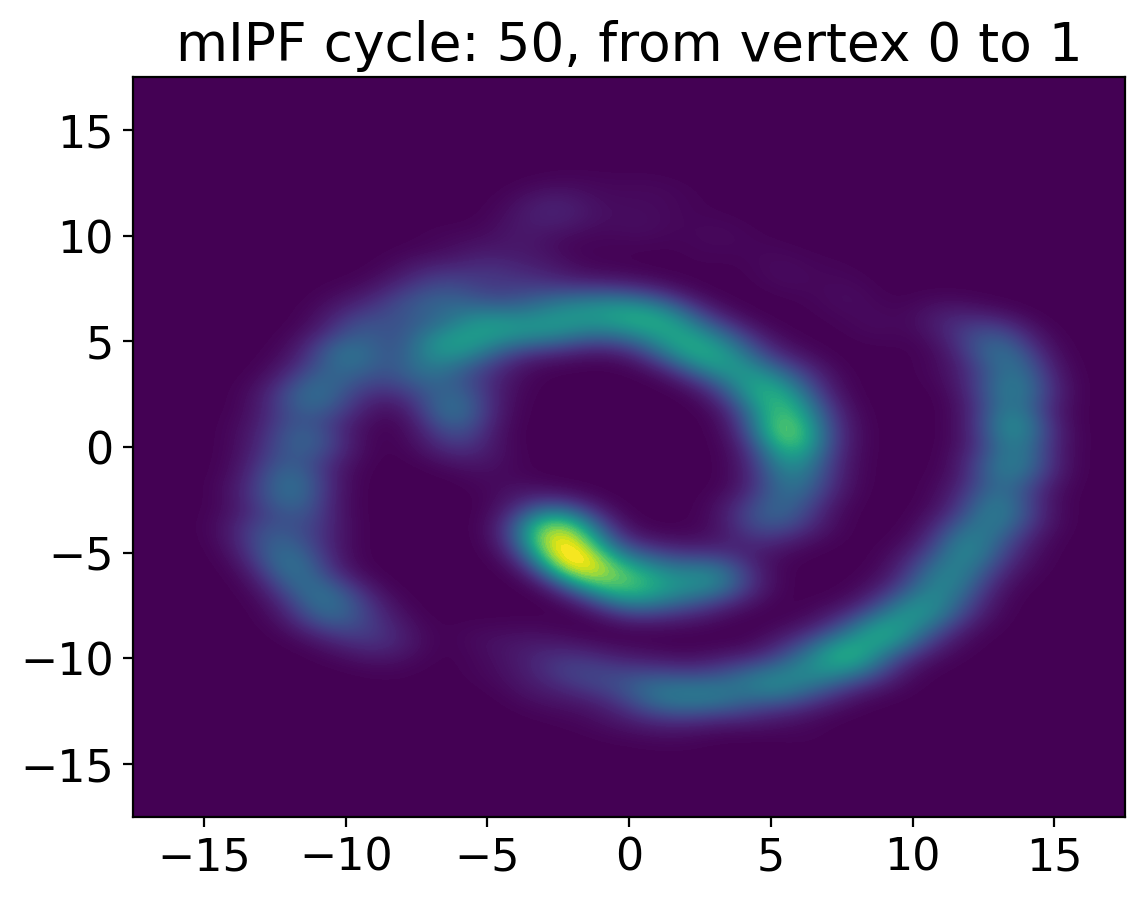}
  \caption{Estimated 2D barycenter obtained by TreeDSB with $\varepsilon=0.2$ (50 mIPF cycles). First row: starting from \emph{Swiss-roll}. Second row: starting from \emph{Circle}. Third row: starting from \emph{Moons}.}
  \label{fig:2d-epsilon=0.2}
\end{figure}
\vfill
\begin{figure}[h!]
  \centering
  \includegraphics[width=.3\linewidth]{img_2d/2D-epsilon=0.1/from-swiss/source=0,dest=1/0_f_50_sde_epsilon=0.100_eps1_density_49_smooth.png} \hfill
  \includegraphics[width=.3\linewidth]{img_2d/2D-epsilon=0.1/from-swiss/source=1,dest=2/0_b_50_sde_epsilon=0.100_eps1_density_49_smooth.png} \hfill
  \includegraphics[width=.3\linewidth]{img_2d/2D-epsilon=0.1/from-swiss/source=1,dest=3/0_b_50_sde_epsilon=0.100_eps1_density_49_smooth.png}\\
  \includegraphics[width=.3\linewidth]{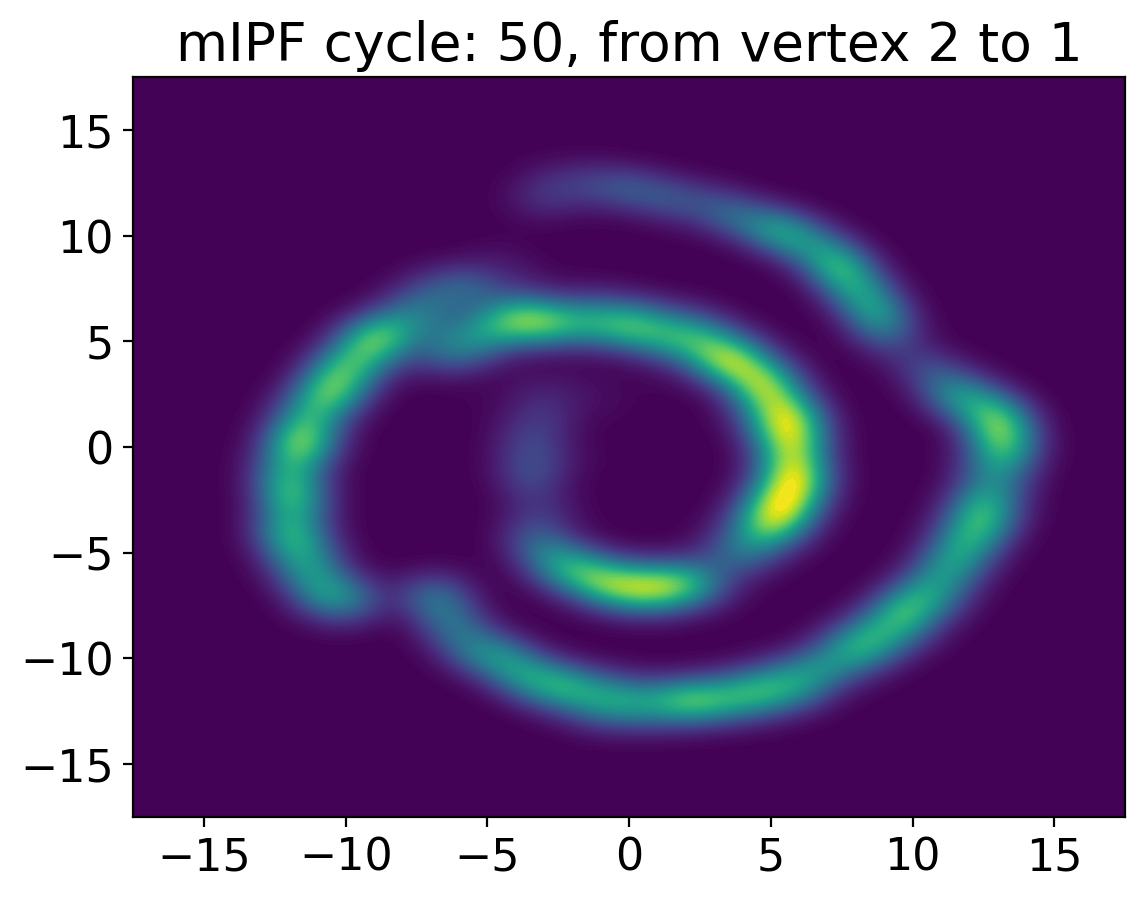} \hfill
  \includegraphics[width=.3\linewidth]{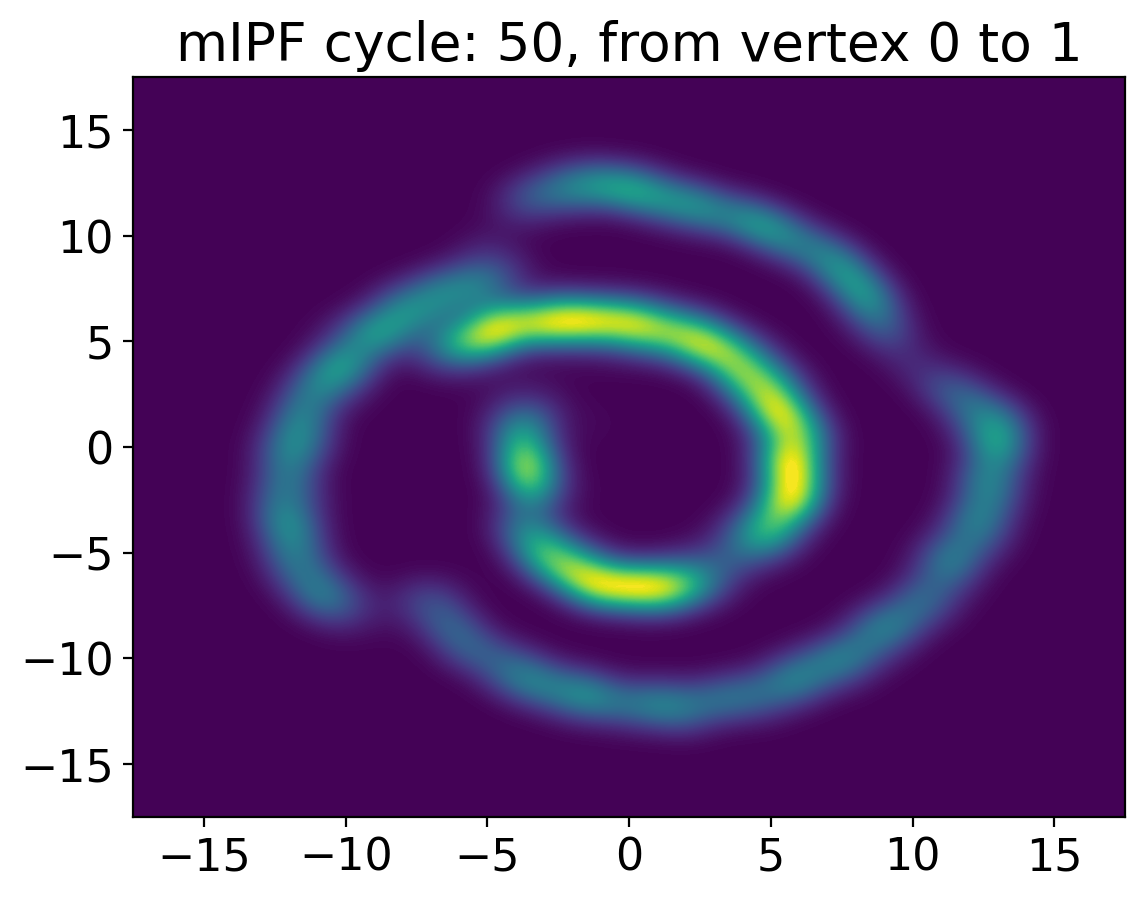} \hfill
  \includegraphics[width=.3\linewidth]{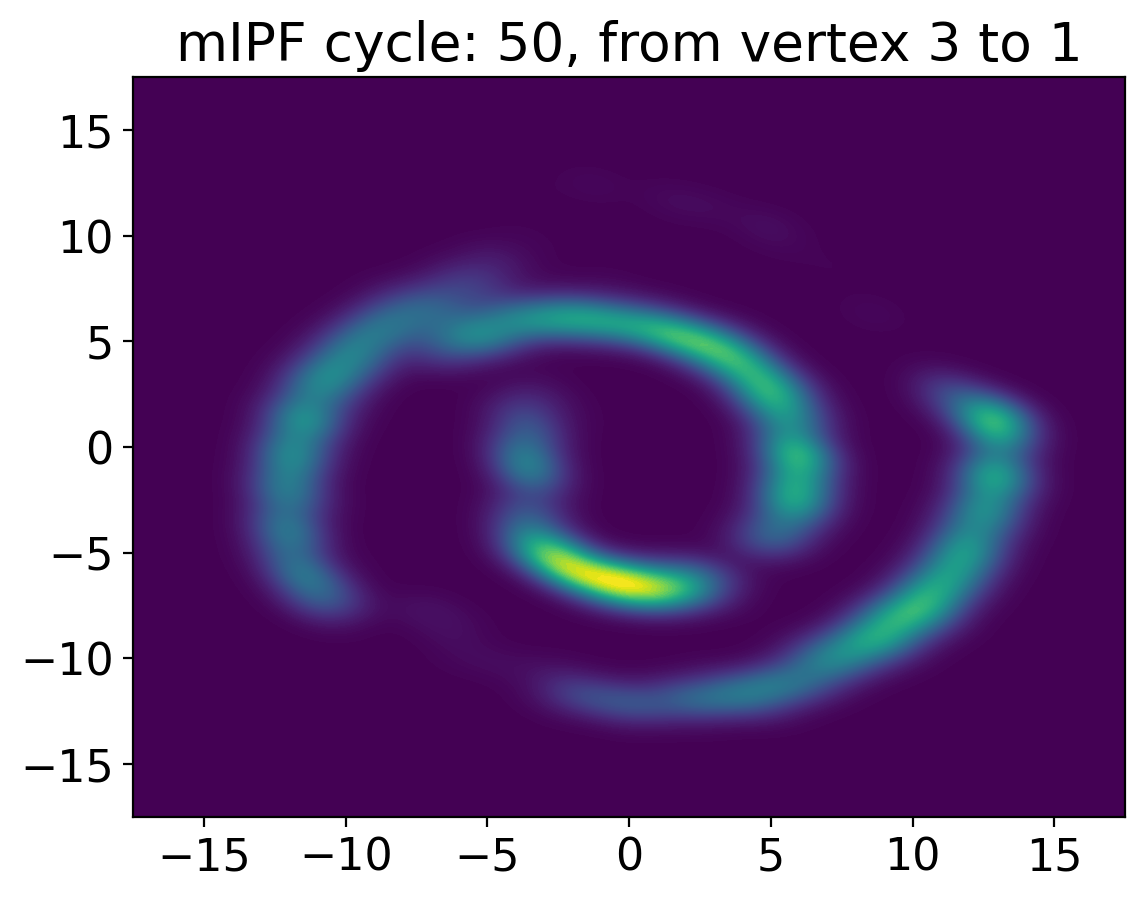}\\
  \includegraphics[width=.3\linewidth]{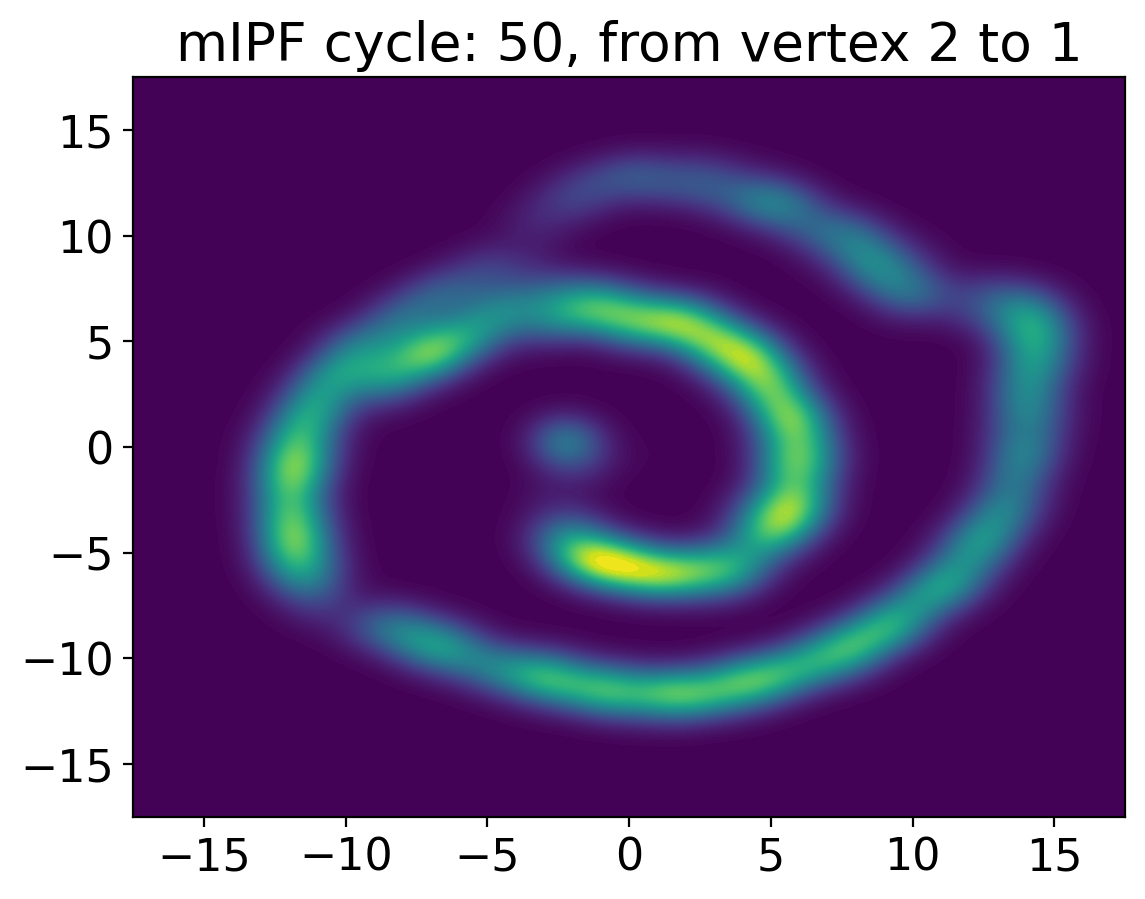} \hfill
  \includegraphics[width=.3\linewidth]{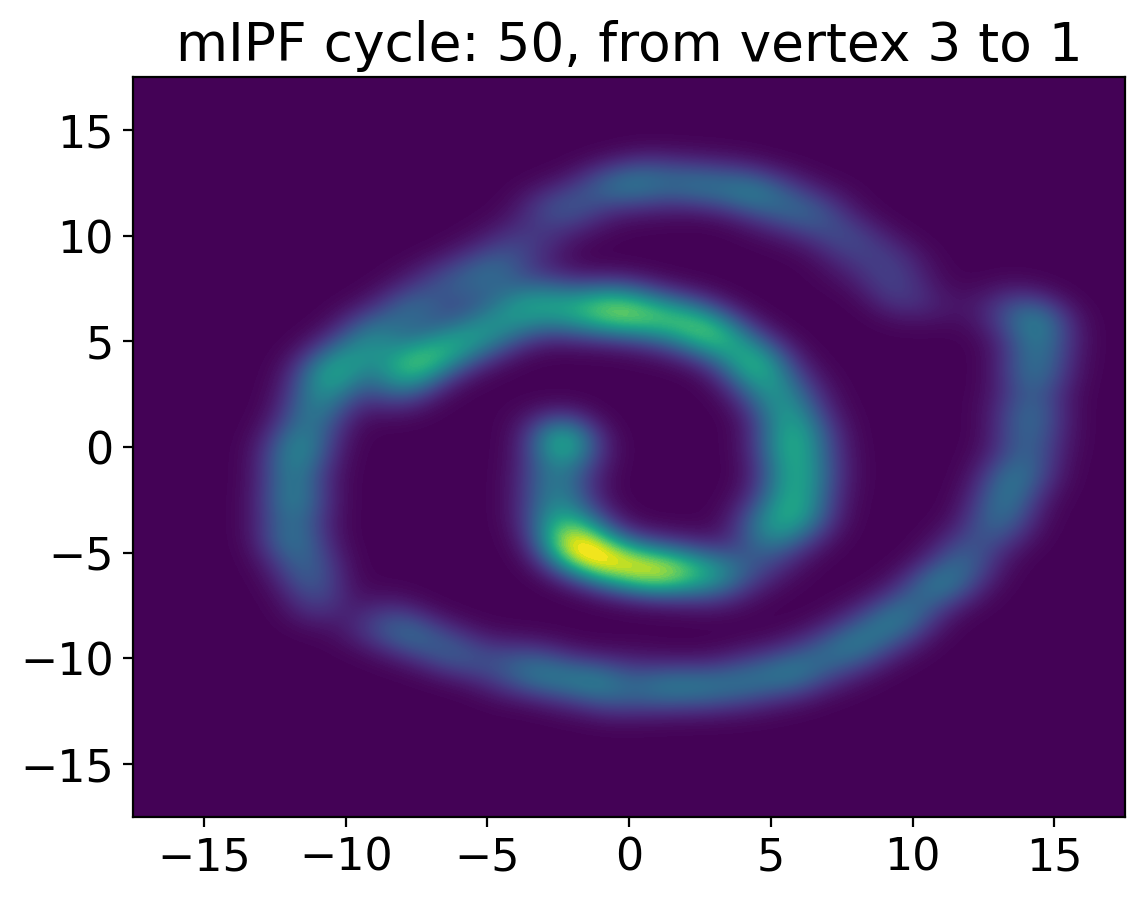} \hfill
  \includegraphics[width=.3\linewidth]{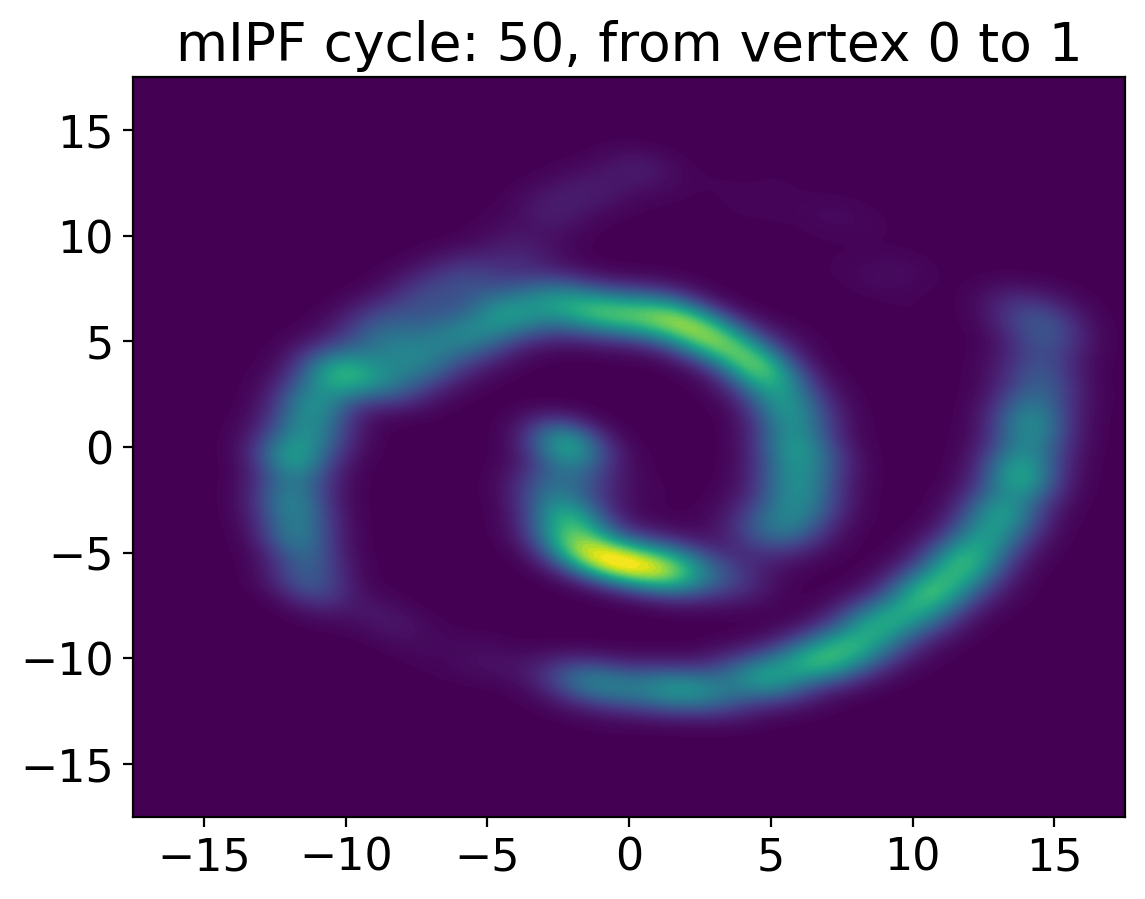}\\
  \vspace{-0.2cm}
  \caption{Estimated 2D barycenter obtained by TreeDSB with $\varepsilon=0.1$ (50 mIPF cycles). First row: starting from \emph{Swiss-roll}. Second row: starting from \emph{Circle}. Third row: starting from \emph{Moons}.}
  \label{fig:2d-epsilon=0.1}
\end{figure}

\newpage

\begin{figure}[h!]
  \centering
  \includegraphics[width=.3\linewidth]{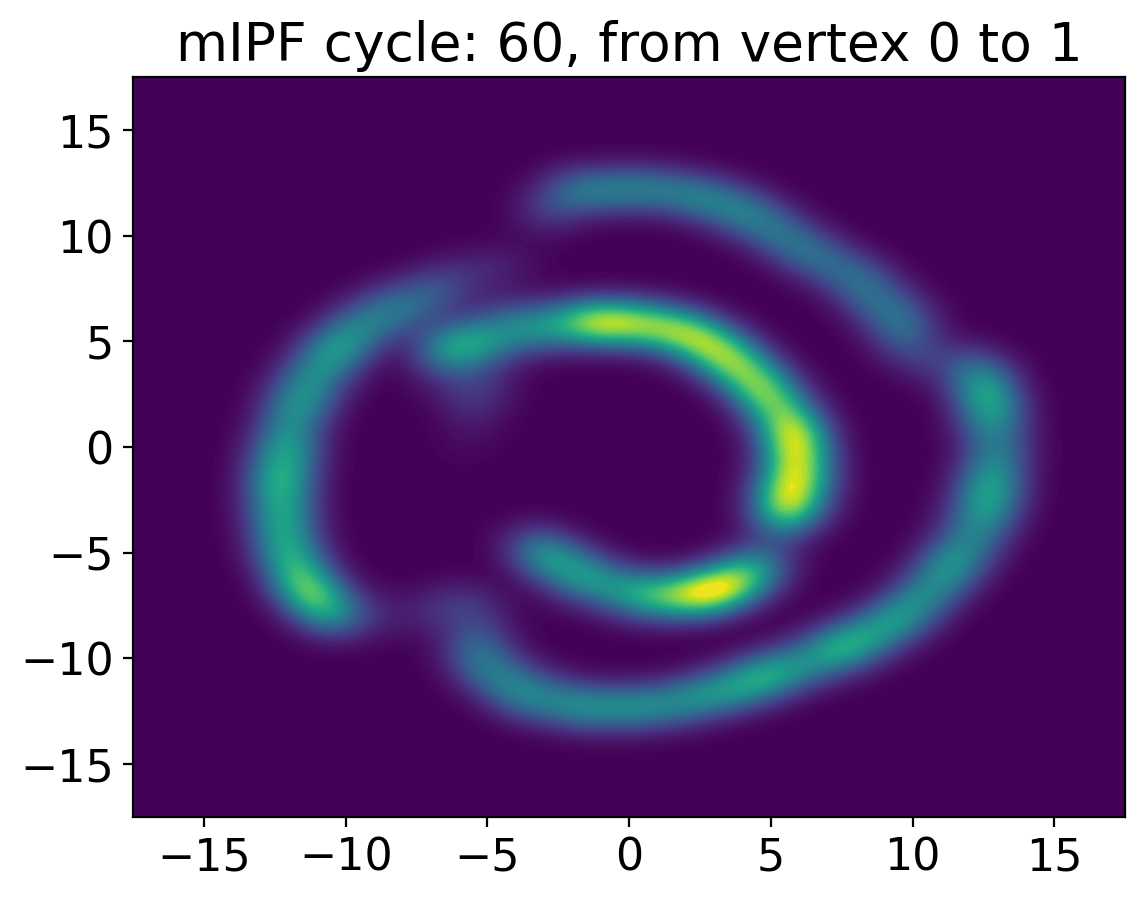} \hfill
  \includegraphics[width=.3\linewidth]{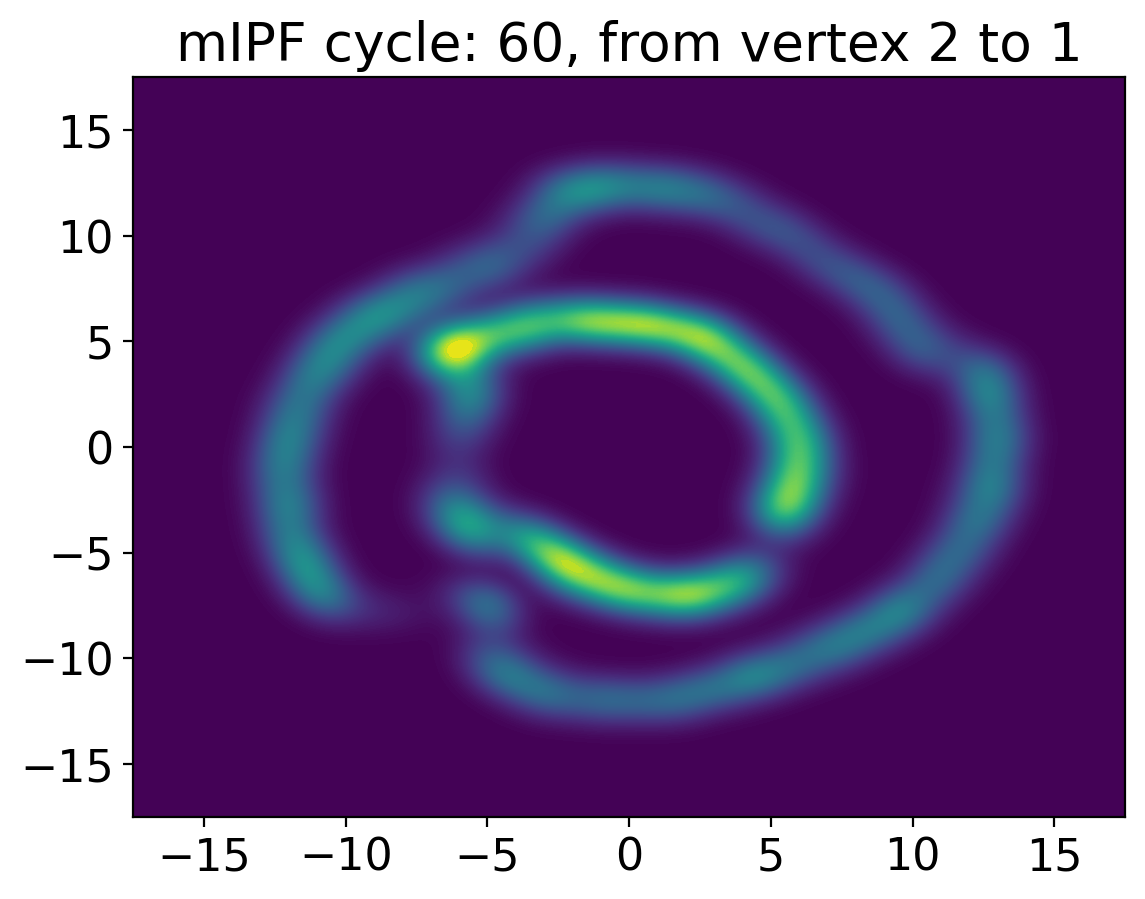} \hfill
  \includegraphics[width=.3\linewidth]{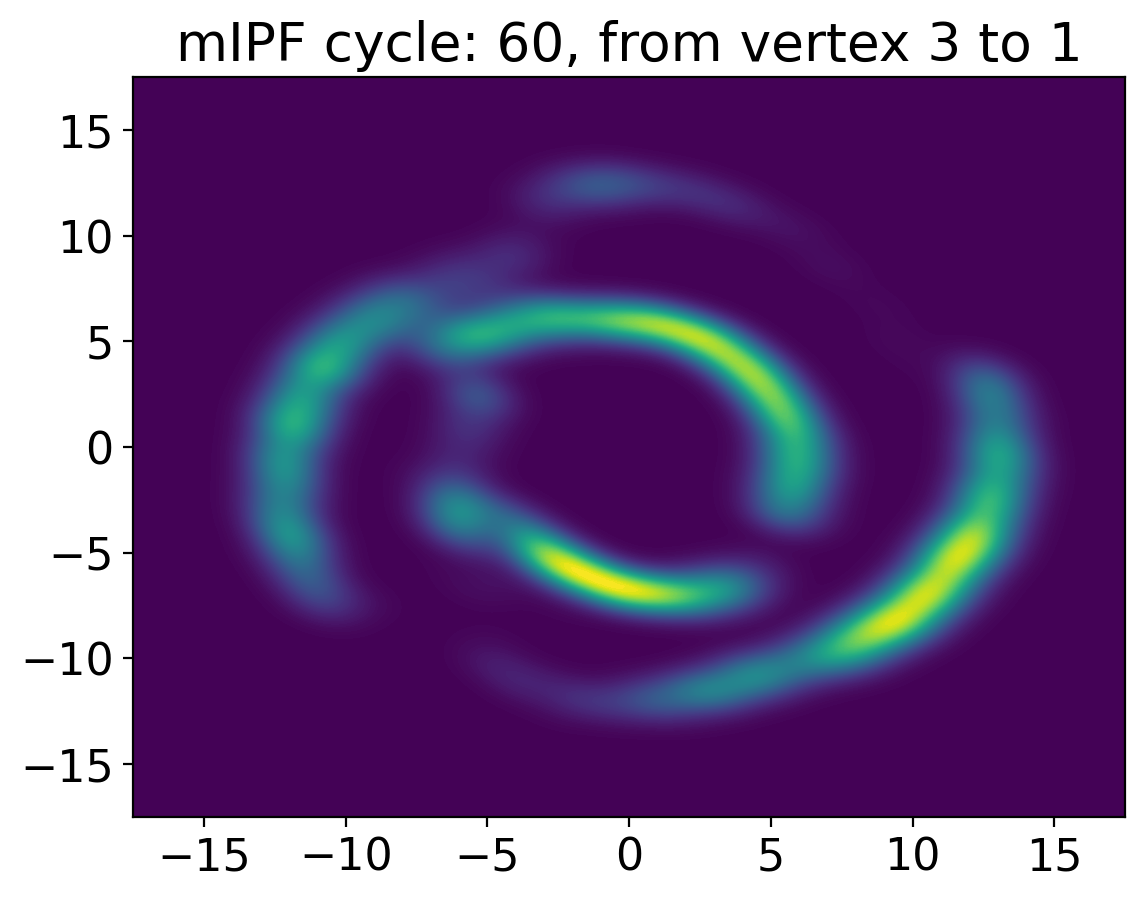}\\
  \includegraphics[width=.3\linewidth]{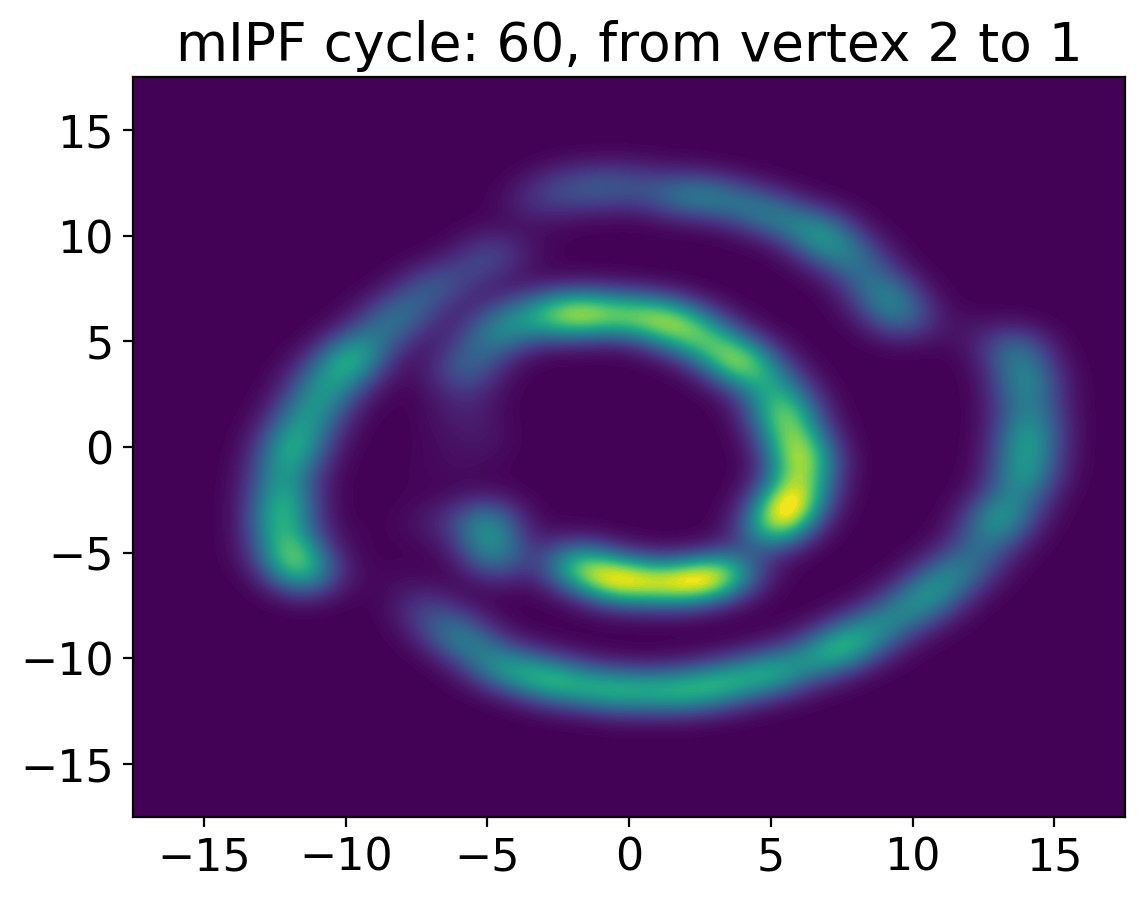} \hfill
  \includegraphics[width=.3\linewidth]{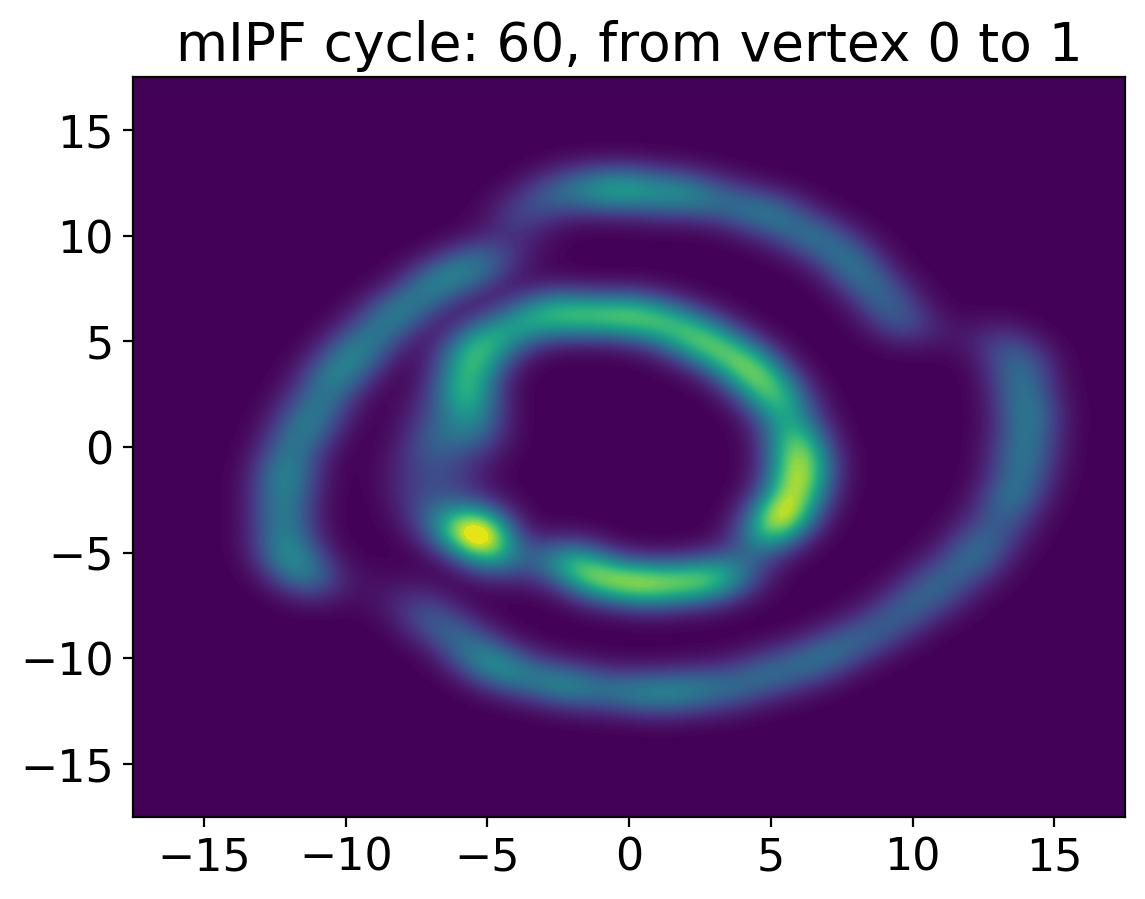} \hfill
  \includegraphics[width=.3\linewidth]{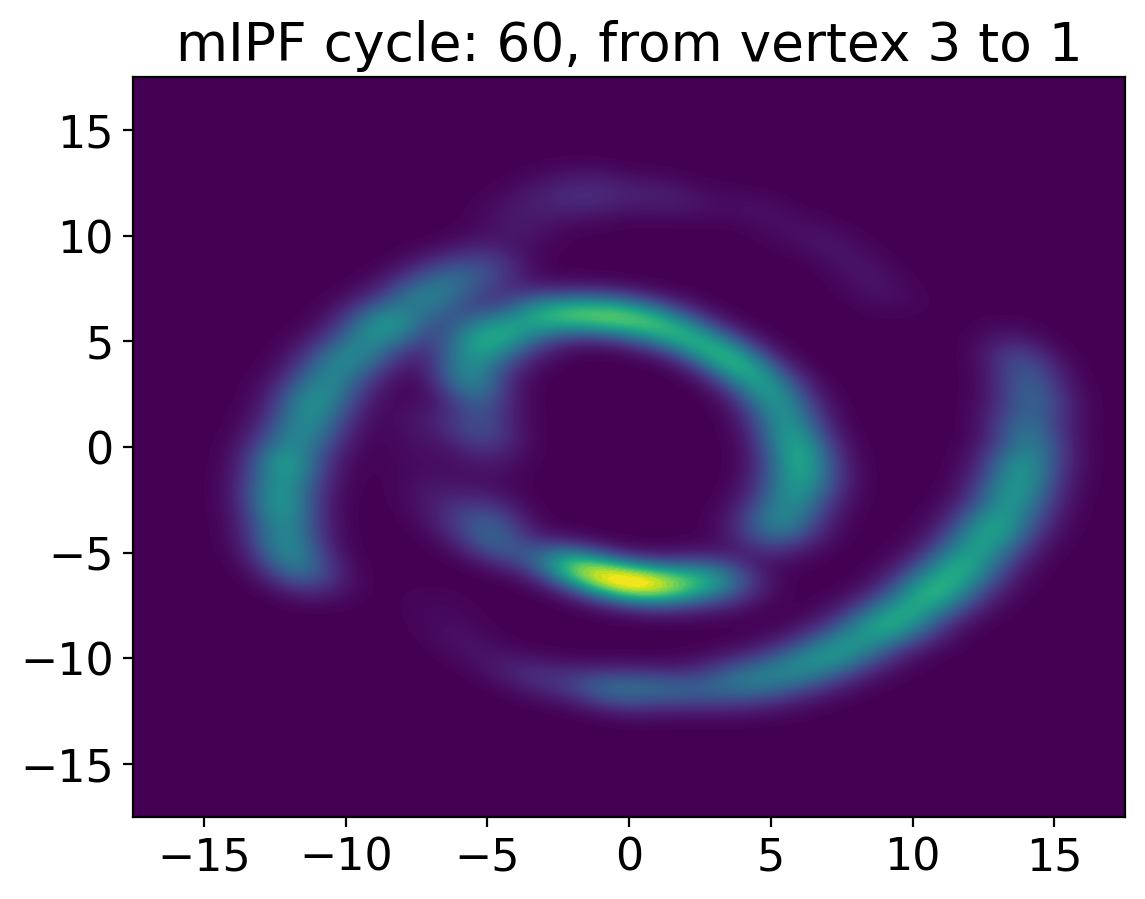}\\
  \includegraphics[width=.3\linewidth]{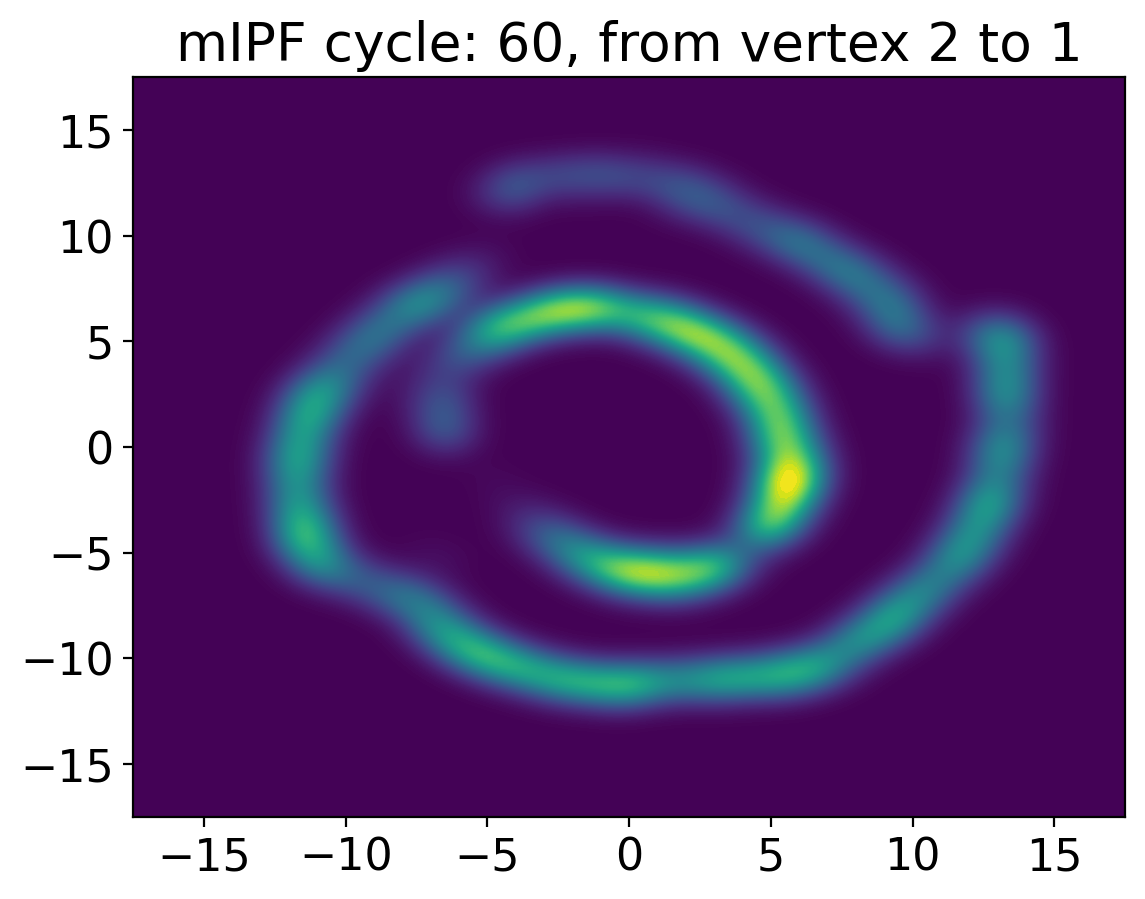} \hfill
  \includegraphics[width=.3\linewidth]{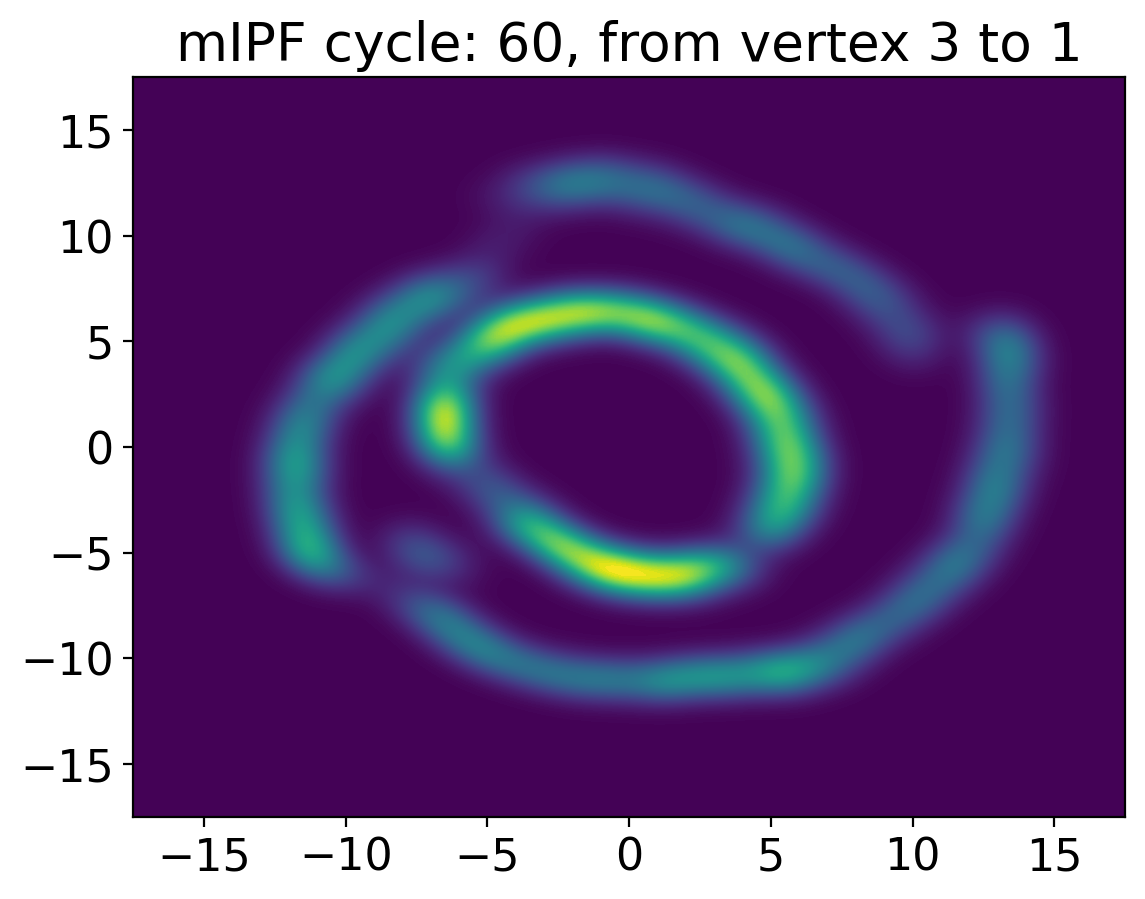} \hfill
  \includegraphics[width=.3\linewidth]{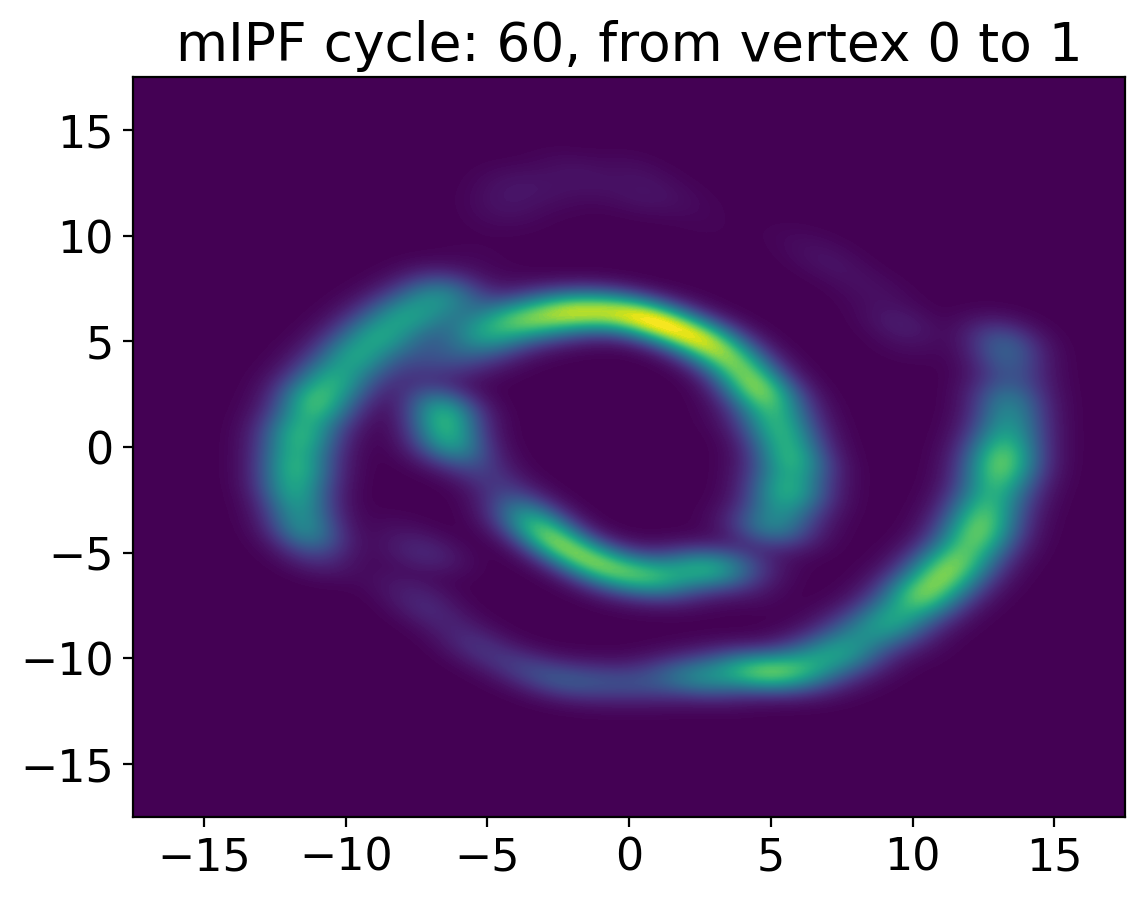}
  \caption{Estimated 2D barycenter obtained by TreeDSB with $\varepsilon=0.05$ (60 mIPF cycles). First row: starting from \emph{Swiss-roll}. Second row: starting from \emph{Circle}. Third row: starting from \emph{Moons}.}
  \label{fig:2d-epsilon=0.05}
\end{figure} 

For purpose of illustration, we provide in \Cref{fig:sota-2D} the barycenter obtained by state-of-the-art two-dimensional \emph{in-sample} methods that are available in POT library \citep{flamary2021pot}: (i) non-regularized free-support Wasserstein barycenter \citep{cuturi2014fast}, (ii) entropic-regularized free-support Wasserstein barycenter (fsWB) with $\varepsilon=0.5$ \citep{cuturi2014fast} and (iii) entropic-regularized convolutional Wasserstein barycenter with $\varepsilon=5. 10^{-4}$ \citep{solomon2015convolutional}, which is specifically designed for images. We notably observe that TreeDSB cannot capture the full complexity of the 2D barycenter compared to these methods. We infer that this gap comes from the \emph{dynamic} nature of TreeDSB, since increasing the number of training iterations per IPF iteration or improving the complexity of the neural networks did not bring any significant change in our results. Finally, we recall that the methods (i), (ii) and (iii) do not scale well with dimension, and have to be completely run again when new data samples are available, contrary to TreeDSB.

\begin{figure}[h!]
    \centering
    \includegraphics[width=.32\linewidth]{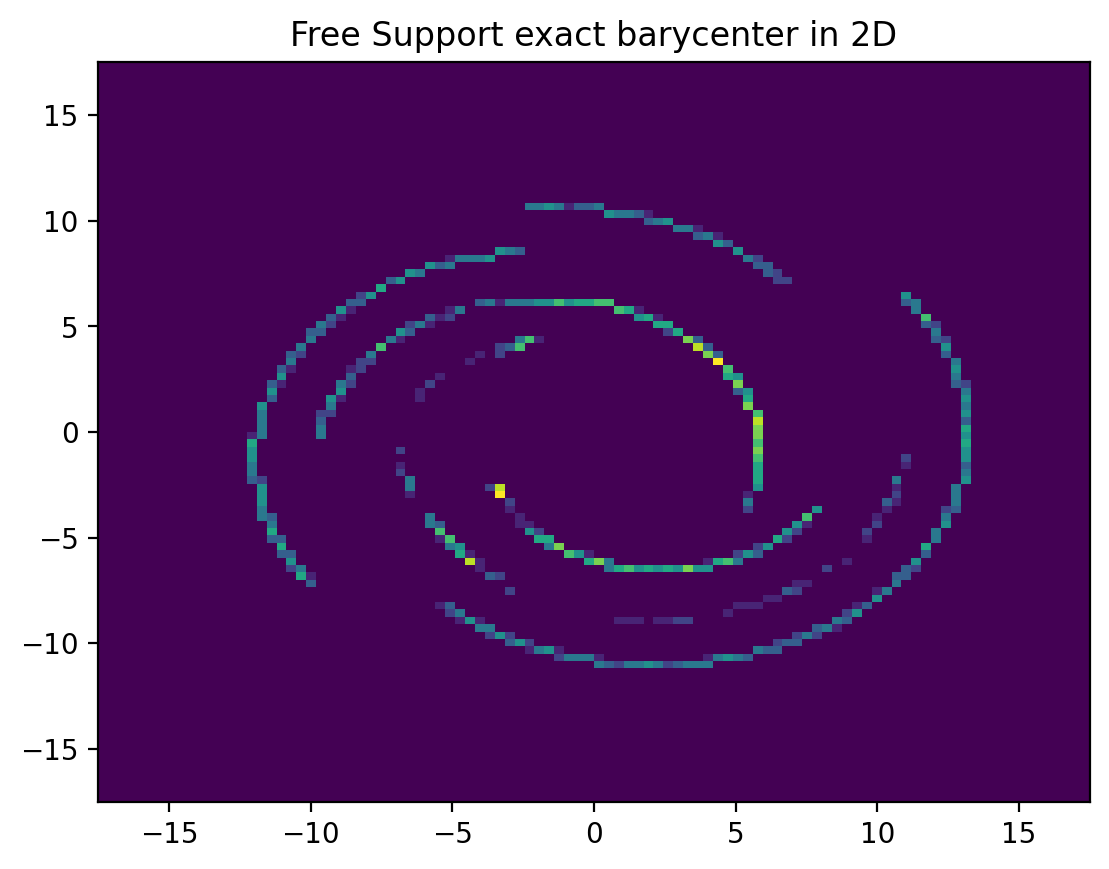}\hfill
    \includegraphics[width=.32\linewidth]{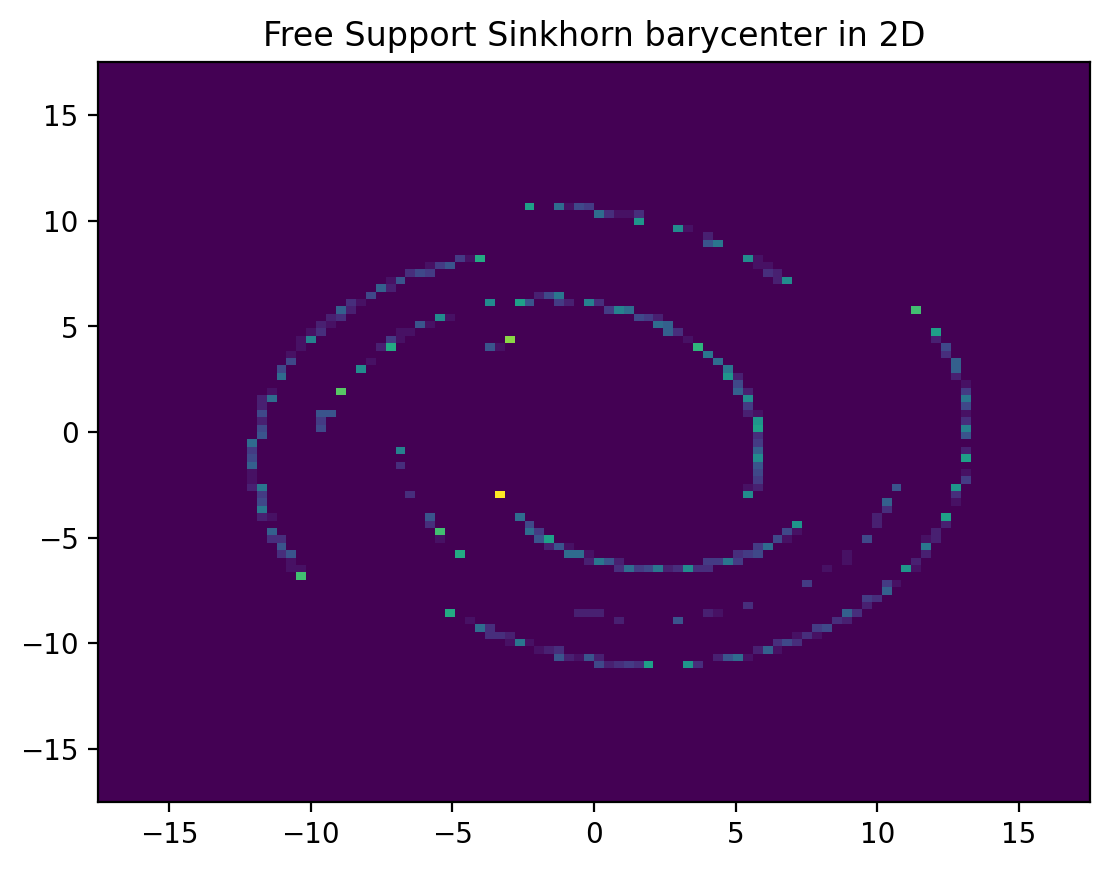}\hfill
    \includegraphics[width=.32\linewidth]{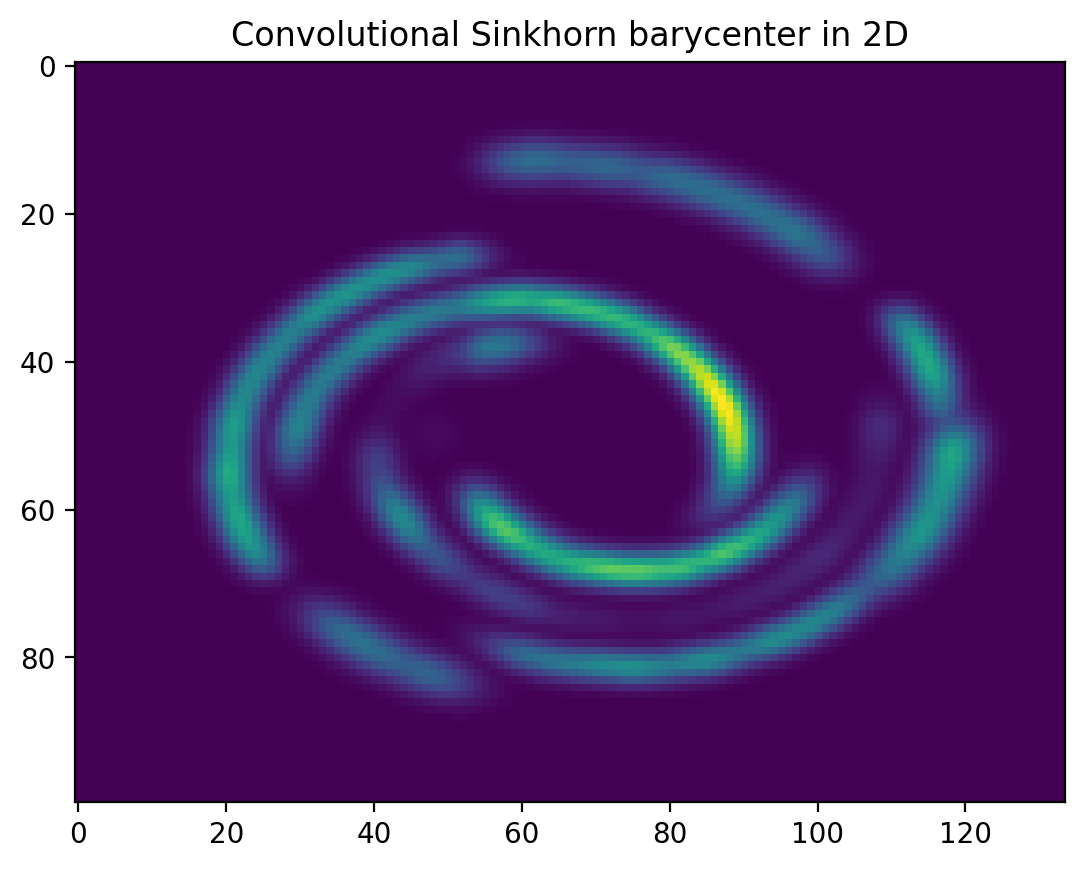}
    \caption{Estimated 2D barycenter obtained by \emph{in-sample} algorithms. From left to right: \cite{cuturi2014fast} (non-regularized), \cite{cuturi2014fast} (regularized), \cite{solomon2015convolutional}.}
    \label{fig:sota-2D}
\end{figure}

\newpage
\paragraph{MNIST Wasserstein barycenter.} This setting can be qualified as \emph{high-dimensional}, since the data dimension is $d=784$. Here, each digit dataset contains 1,000 samples. As in the two-dimensional setting, we do not have an \textit{a priori} on the shape of the barycenter between MNIST digits, and thus consider the formulation \eqref{eq:wasserstein_barycenter_pb}, where the root $r$ is chosen as a leaf. We propose below several experiments to assess the scalability of TreeDSB to this setting.

\vspace{-0.2cm}
 \paragraph{Digits 0 and 1.} In \Cref{fig:mnist 0-1}, we report the results obtained by running TreeDSB on MNIST digits 0 and 1, for 15 mIPF cycles with $\varepsilon=0.5$, starting from the leaf MNIST-0. We display 25 samples from the reconstructed MNIST-0 marginal (first column), from the reconstructed MNIST-1 marginal (fourth column), from the estimated barycenter by diffusing from MNIST-0 (second column) and diffusing from MNIST-1 (third column). We notably observe that the digits are well recovered and that the barycenter samples are consistent. We draw the reader's attention to the fact that TreeDSB showed numerical unstability with a regularization value $\varepsilon$ lower than $0.5$. For purpose of illustration, we display in \Cref{fig:wasserstein_barycenter} the Wasserstein barycenter obtained by \emph{non-regularized} methods from \cite{fan2020scalable} and \cite{korotin2021continuous}, and by the \emph{regularized} approach from \cite{li2020continuous}.
 \begin{figure}[h!]
  \centering
  \includegraphics[width=.2\linewidth]{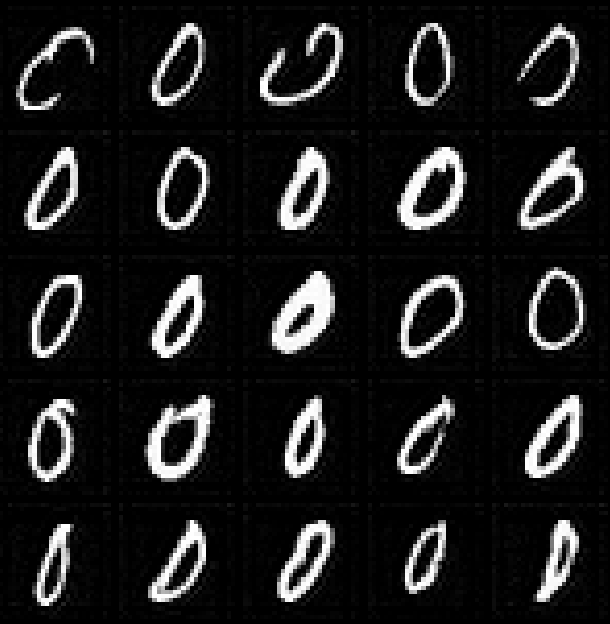} \hfill
  \includegraphics[width=.2\linewidth]{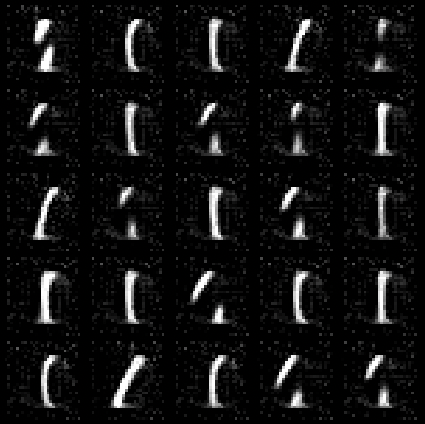} \hfill
  \includegraphics[width=.2\linewidth]{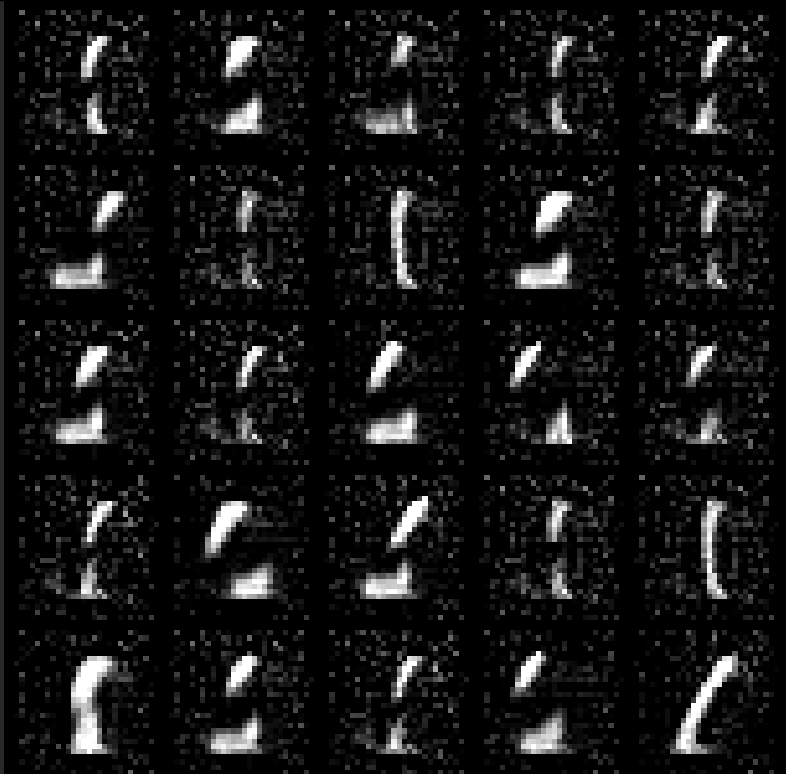} \hfill
  \includegraphics[width=.2\linewidth]{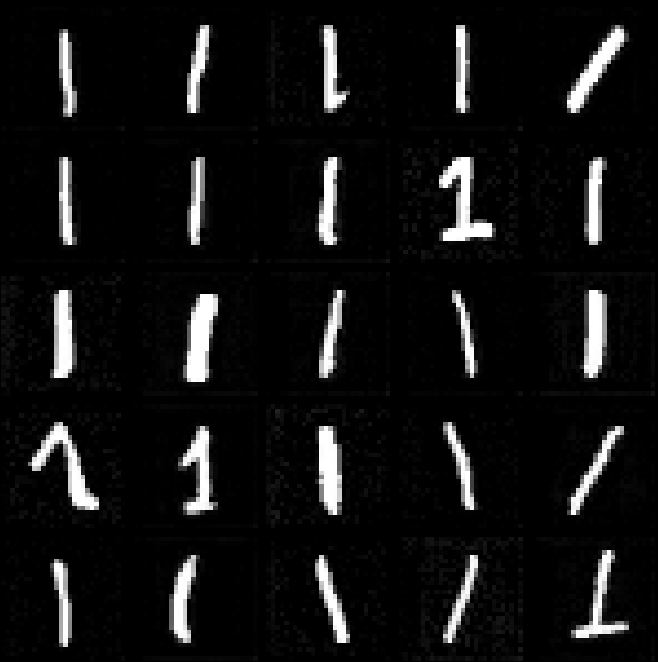}
  \caption{Tree DSB results for MNIST digits 0 and 1, after 15 mIPF cycles with $\varepsilon=0.5$.}
  \label{fig:mnist 0-1}
\end{figure}
\vspace{-0.2cm}
\begin{figure}[h!]
  \centering
  \includegraphics[width=.2\linewidth]{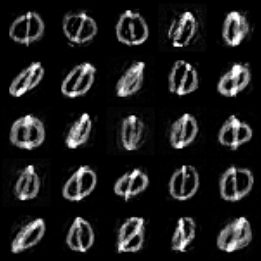} \hfill
  \includegraphics[width=.2\linewidth]{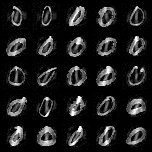} \hfill 
  \includegraphics[width=.2\linewidth]{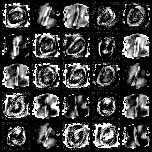} \hfill
  \caption{From left to right: \cite{fan2020scalable}, \cite{korotin2021continuous} and \cite{li2020continuous}.}
  \label{fig:wasserstein_barycenter}
\end{figure}
\vspace{-0.2cm}
\paragraph{Digits 2,4 and 6.}  In \Cref{fig:mnist 2-4-6 a}, we report the results obtained by running TreeDSB on MNIST digits 2,4 and 6, for 10 mIPF cycles with $\varepsilon=0.5$. Here, we consider three settings which differ by the starting leaf $r$ in the algorithm: MNIST-2 (first row), MNIST-4 (second row), or MNIST-6 (third row). For each of these settings, we display 30 samples from the estimated barycenter by diffusing from MNIST-2 (first column), diffusing from MNIST-4 (third column) and diffusing from MNIST-6 (third column). We notably observe a global consistency of the barycenter samples across the various settings. In \Cref{fig:mnist 2-4-6 b}, we report the results obtained by running TreeDSB on MNIST digits 2,4 and 6, for 10 mIPF cycles with $\varepsilon=0.2$, starting from MNIST-6. We display 30 samples from the reconstructed marginals (first row), from the estimated barycenter (second row) by diffusing from MNIST-2 (first column), diffusing from MNIST-4 (second column) and diffusing from MNIST-6 (third column). As expected, we observe less noisy barycenter samples compared to  \Cref{fig:mnist 2-4-6 a}, while still well recovering MNIST digits.

\vspace{-0.2cm}
\paragraph{Digits 0,1 and 4.} In \Cref{fig:mnist 0-1-4 a}, we report the results obtained by running TreeDSB on MNIST digits 0,1 and 4, for 10 mIPF cycles with $\varepsilon=0.5$. We consider two settings which differ by the starting leaf $r$ in the algorithm: MNIST-0 (second row) and MNIST-1 (first/third rows), for which we display samples from the reconstructed measures (first row). In \Cref{fig:mnist 0-1-4 b}, we report the results obtained by running TreeDSB on MNIST digits 0,1 and 4, for 10 mIPF cycles with $\varepsilon=0.2$. We consider two settings which differ by the starting leaf $r$ in the algorithm: MNIST-0 (first/second row), for which display samples from the reconstructed measures (first row), and MNIST-1 (third row). For all of these settings, we display 30 samples from the estimated barycenter by diffusing from MNIST-0 (first column), diffusing from MNIST-1 (third column) and diffusing from MNIST-4 (third column). Similarly to the digits 2-4-6, we observe consistency within the barycenter samples, unconditionally to the starting leaf, and less noise as $\varepsilon$ decreases. Note that the reconstructed MNIST digits are less truthful to the original datasets when $\varepsilon$ is low.

\begin{figure}[h!]
    \centering
    \includegraphics[width=.32\linewidth]{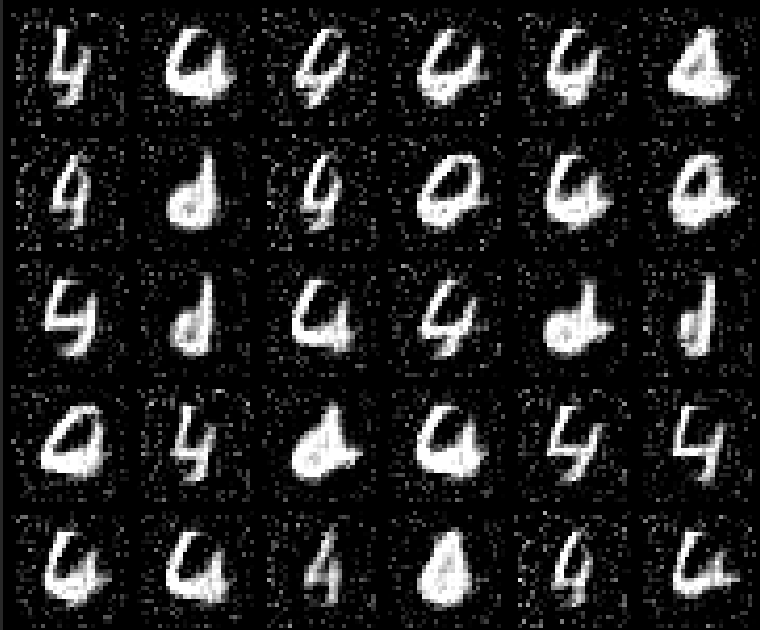}\hfill
    \includegraphics[width=.32\linewidth]{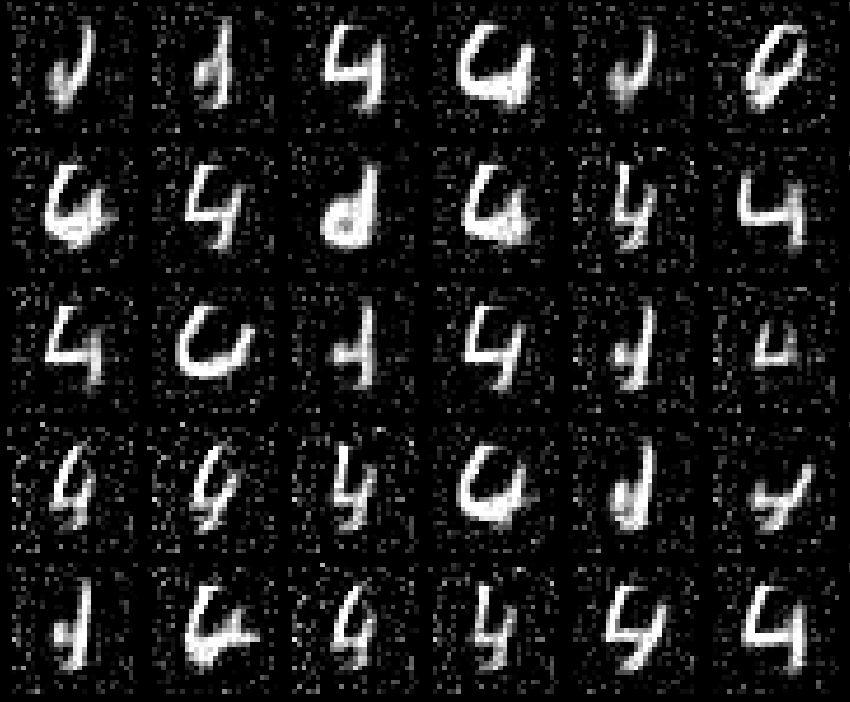}\hfill
    \includegraphics[width=.32\linewidth]{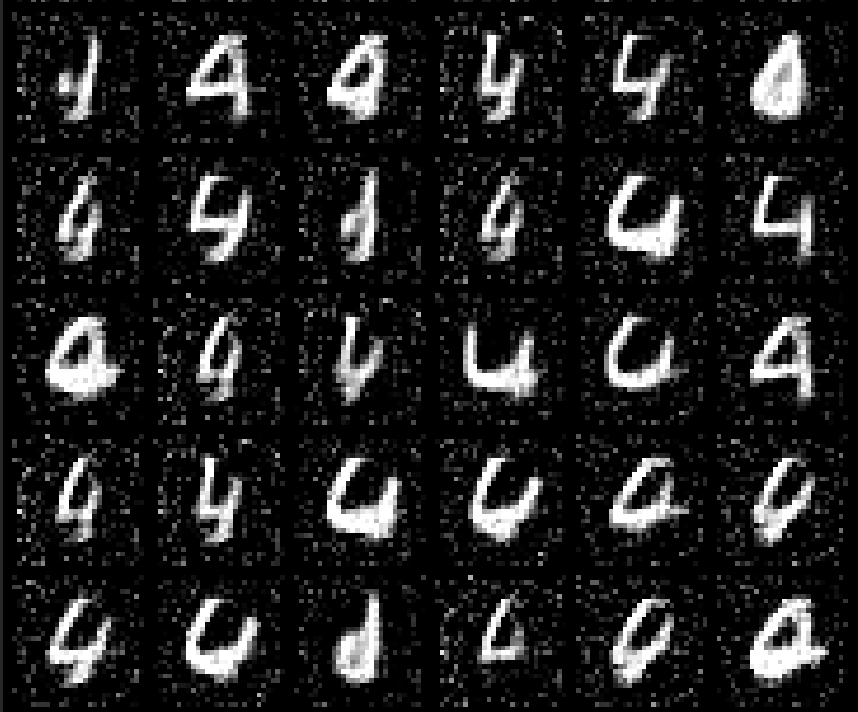}\\
    \vspace{0.2cm}
    \includegraphics[width=.32\linewidth]{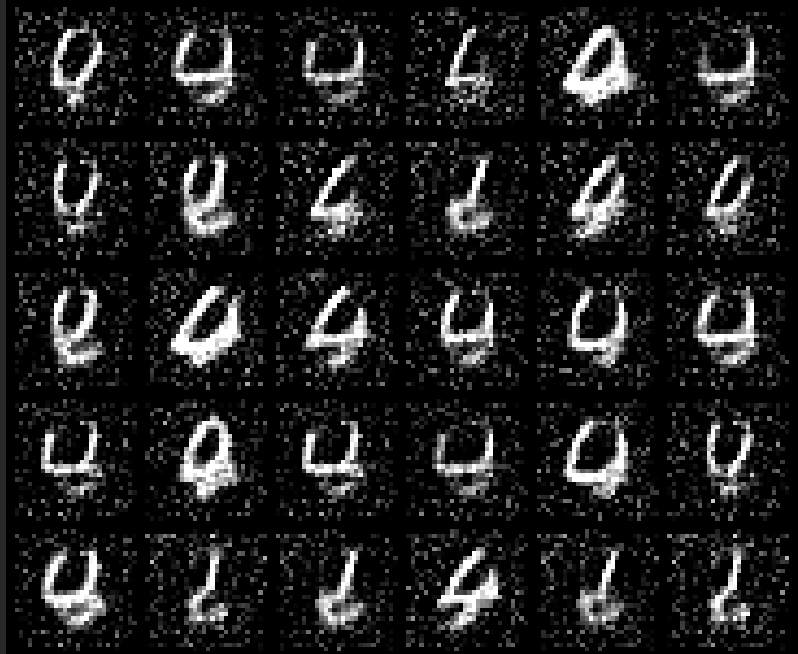}\hfill
    \includegraphics[width=.32\linewidth]{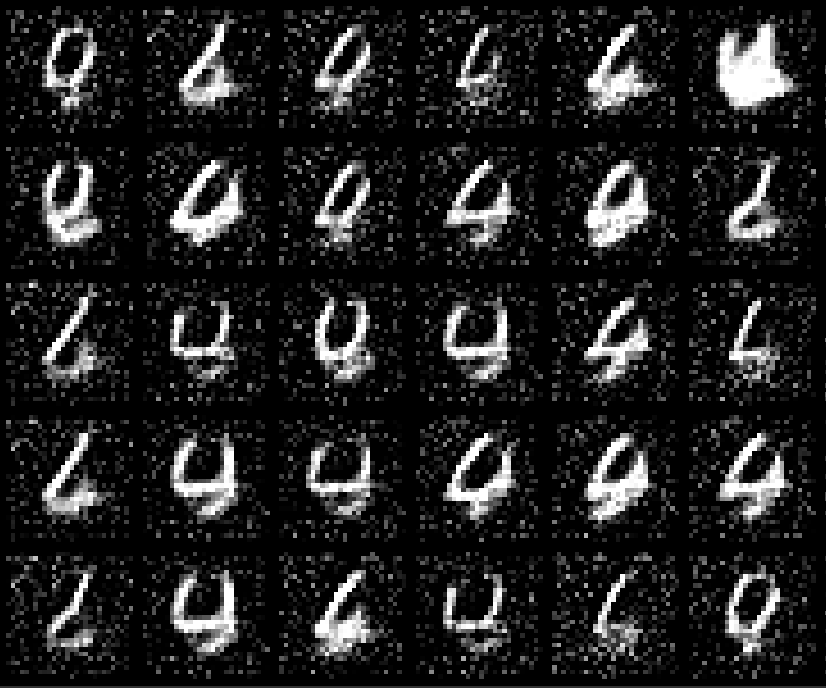}\hfill
    \includegraphics[width=.32\linewidth]{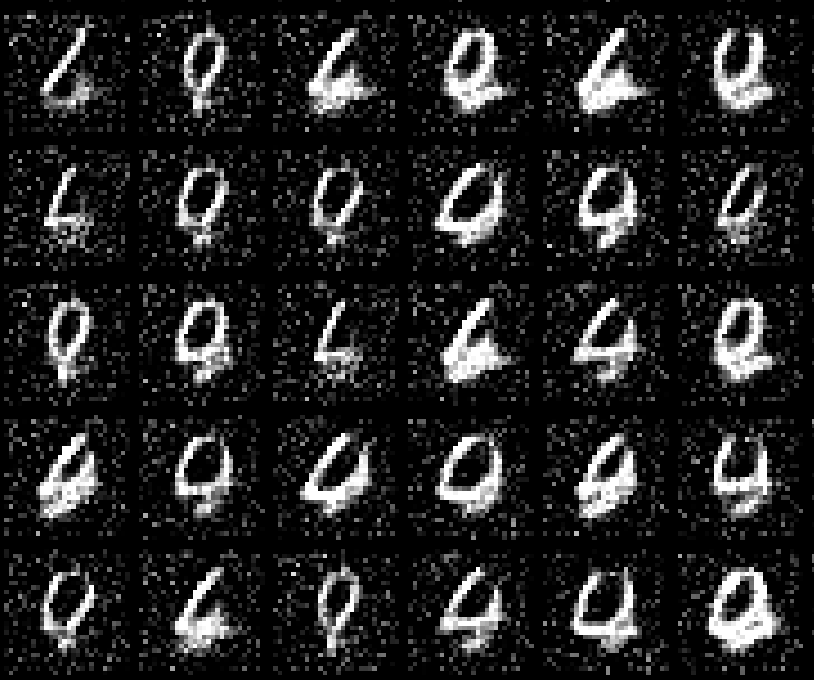}\\
    \vspace{0.2cm}
    \includegraphics[width=.32\linewidth]{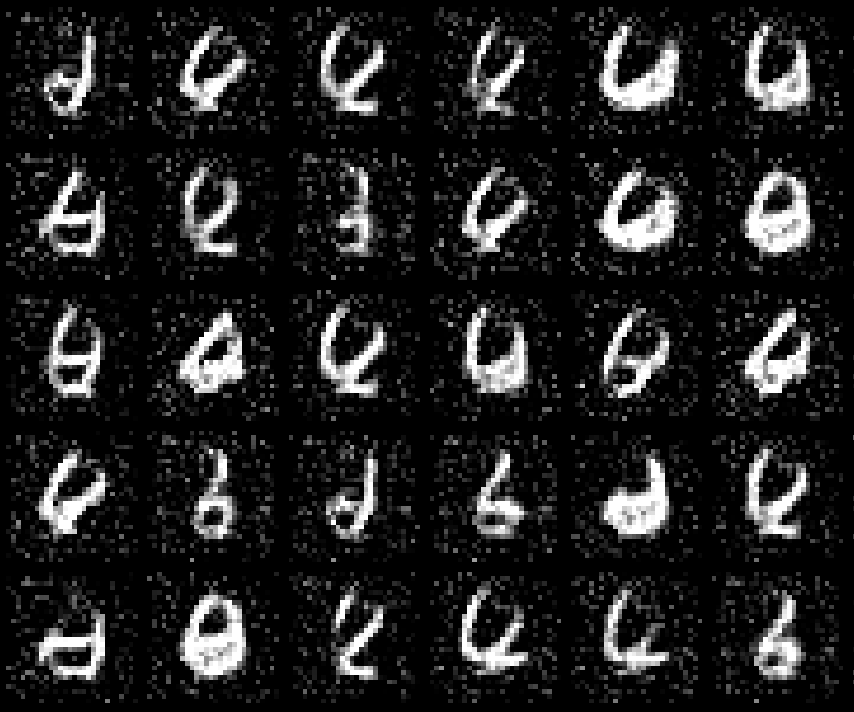}\hfill
    \includegraphics[width=.32\linewidth]{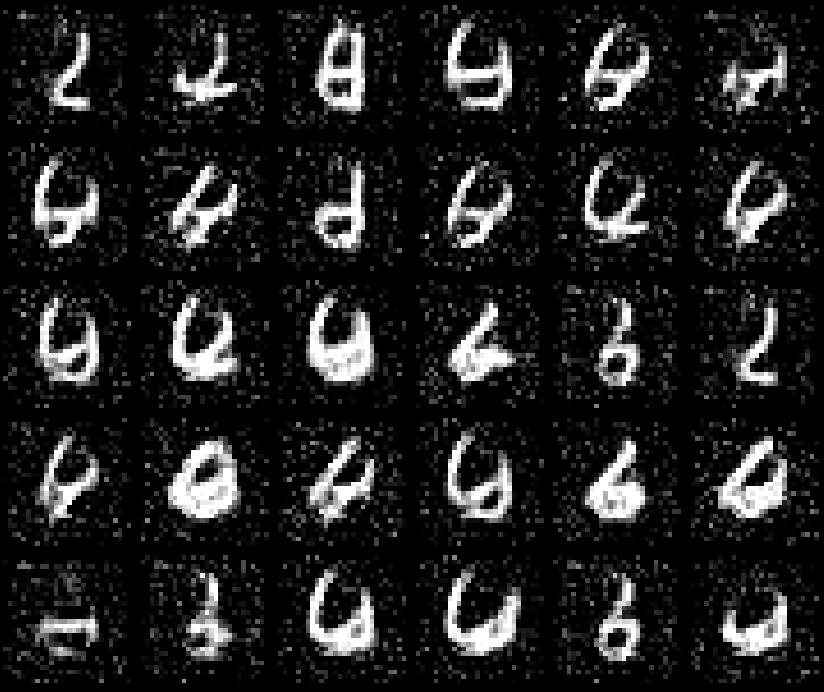}\hfill
    \includegraphics[width=.32\linewidth]{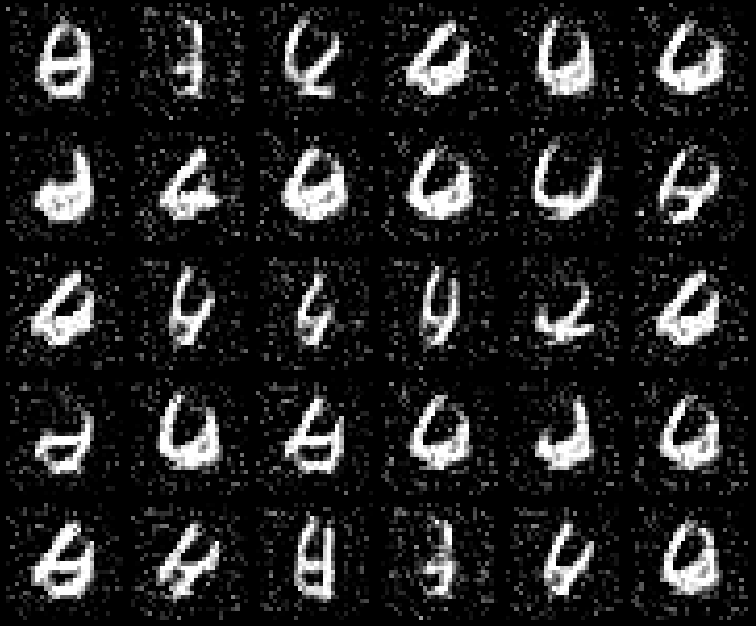}
    \caption{Tree DSB results for MNIST digits 2,4 and 6, after 10 mIPF cycles with $\varepsilon=0.5$. First row: starting from MNIST-2. Second row: starting from MNIST-4. Third row: starting from MNIST-6. }
    \label{fig:mnist 2-4-6 a}
\end{figure}

\begin{figure}[h!]
    \centering
    \includegraphics[width=.32\linewidth]{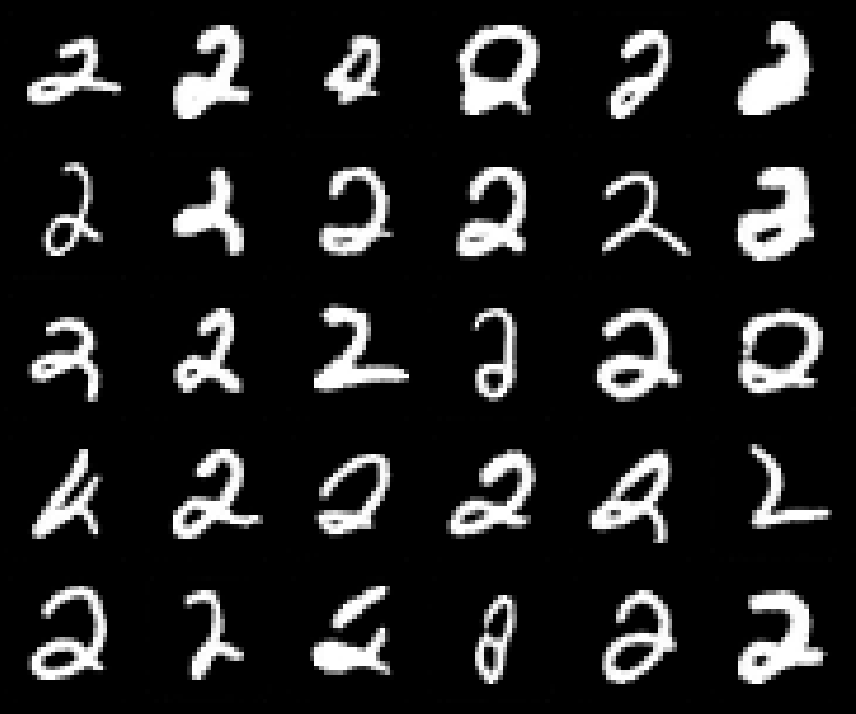}\hfill
    \includegraphics[width=.32\linewidth]{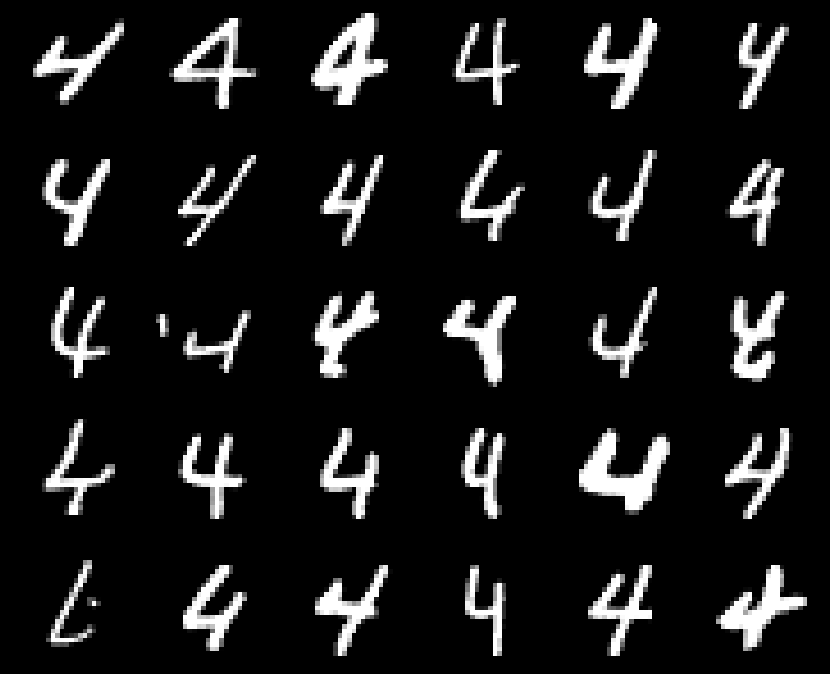}\hfill
    \includegraphics[width=.32\linewidth]{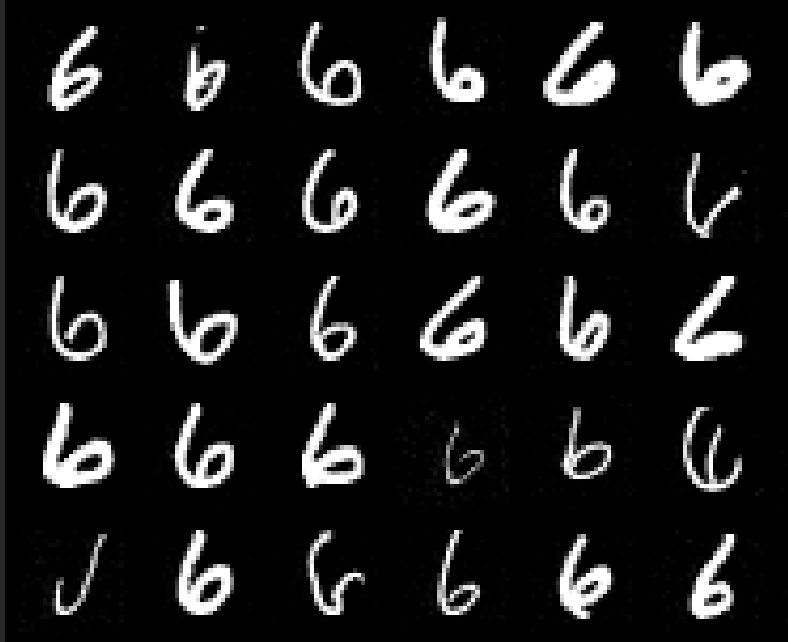}\\
    \vspace{0.2cm}
    \includegraphics[width=.32\linewidth]{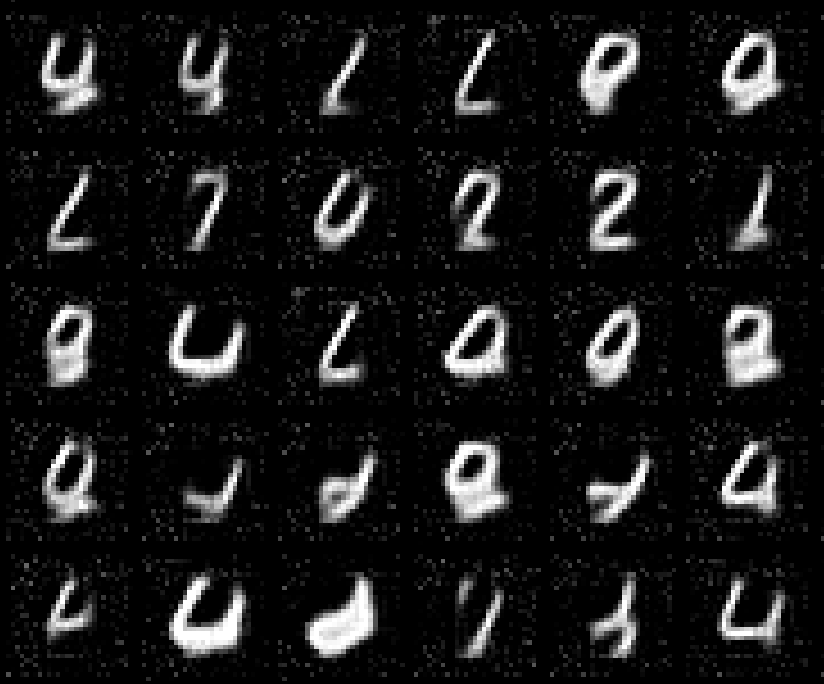}\hfill
    \includegraphics[width=.32\linewidth]{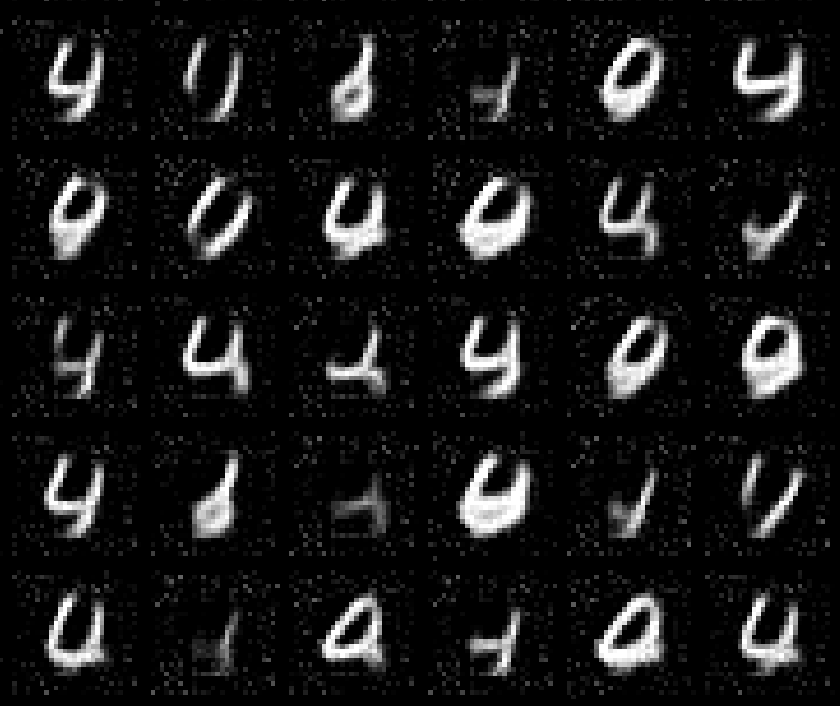}\hfill
    \includegraphics[width=.32\linewidth]{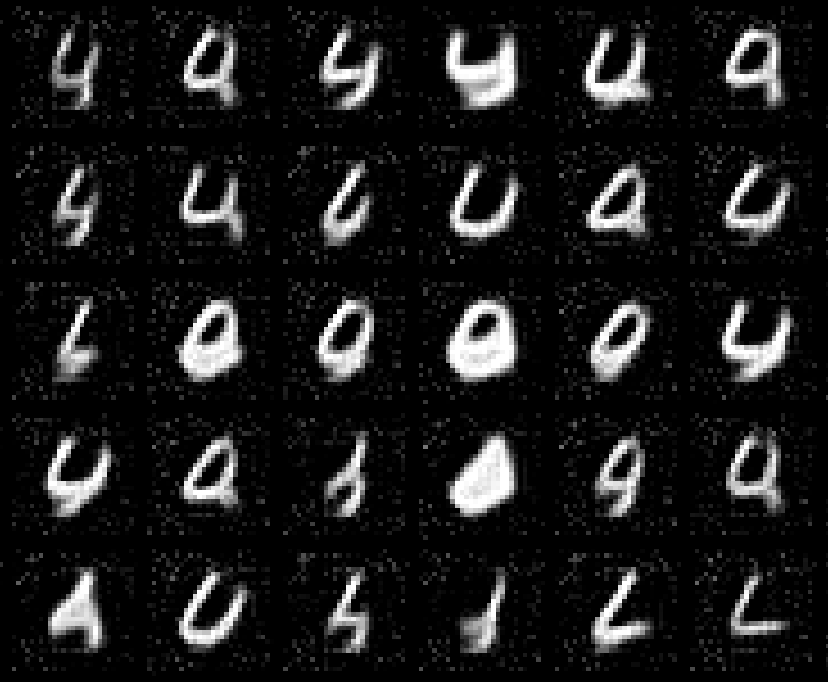}\\
    \caption{Tree DSB results for MNIST digits 2,4 and 6, after 10 mIPF cycles with $\varepsilon=0.2$, starting from MNIST-6. First row: samples from the reconstructed marginals. Second row: samples from the estimated barycenter. }
    \label{fig:mnist 2-4-6 b}
\end{figure}

\begin{figure}[h!]
    \centering
    \includegraphics[width=.32\linewidth]{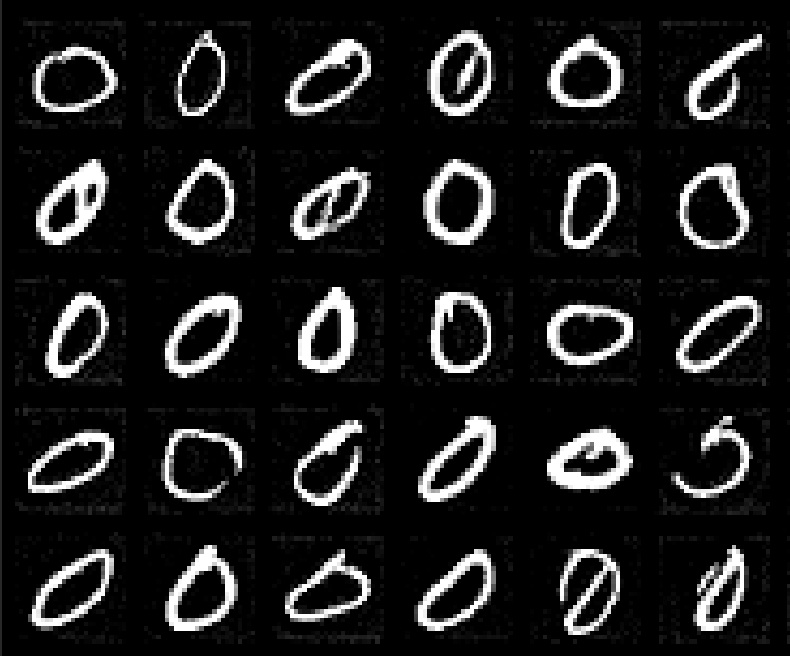}\hfill
    \includegraphics[width=.32\linewidth]{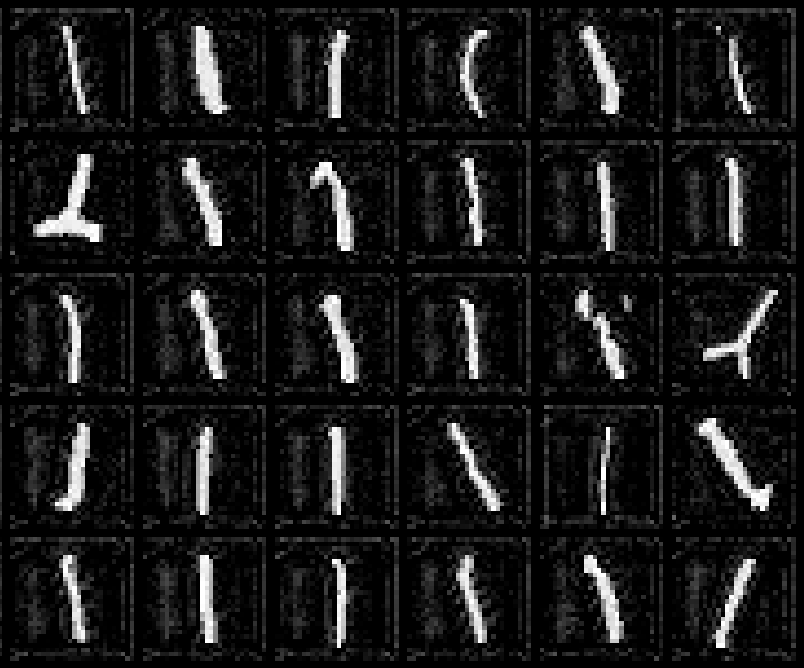}\hfill
    \includegraphics[width=.32\linewidth]{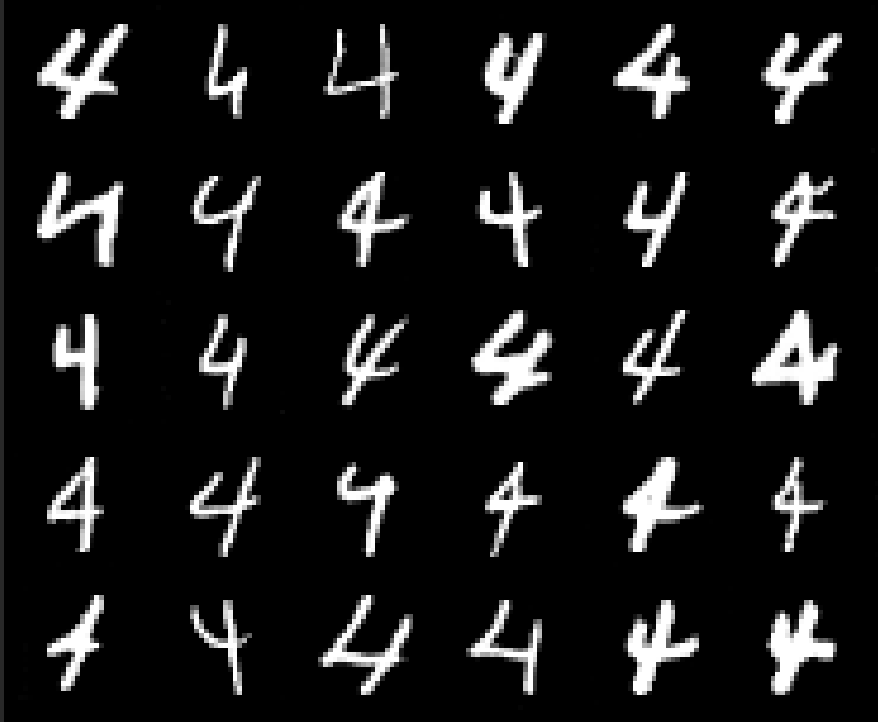}\\
    \vspace{0.2cm}
    \includegraphics[width=.32\linewidth]{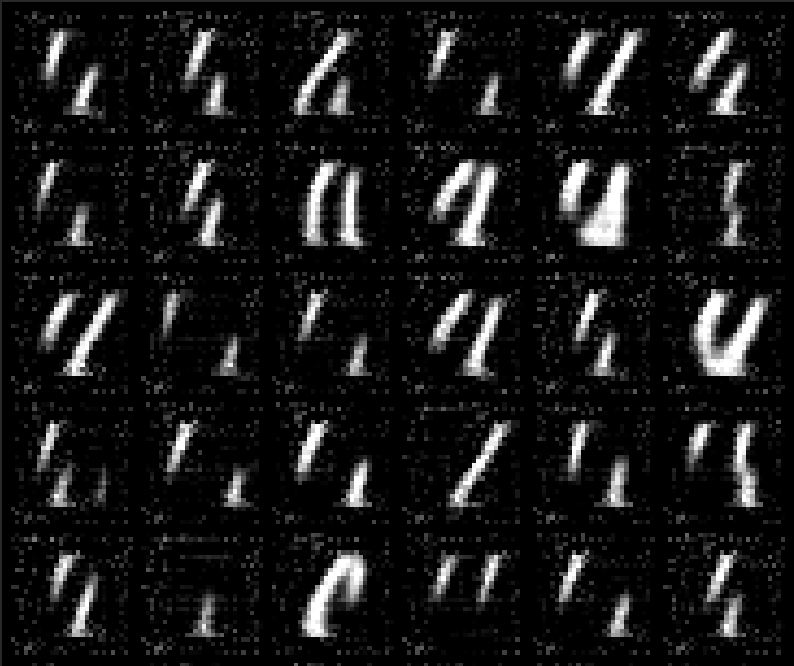}\hfill
    \includegraphics[width=.32\linewidth]{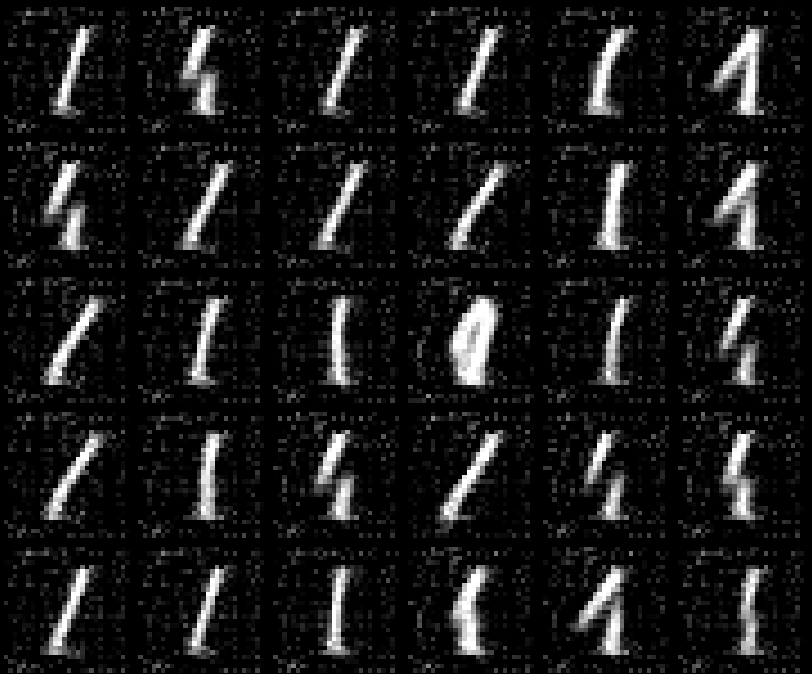}\hfill
    \includegraphics[width=.32\linewidth]{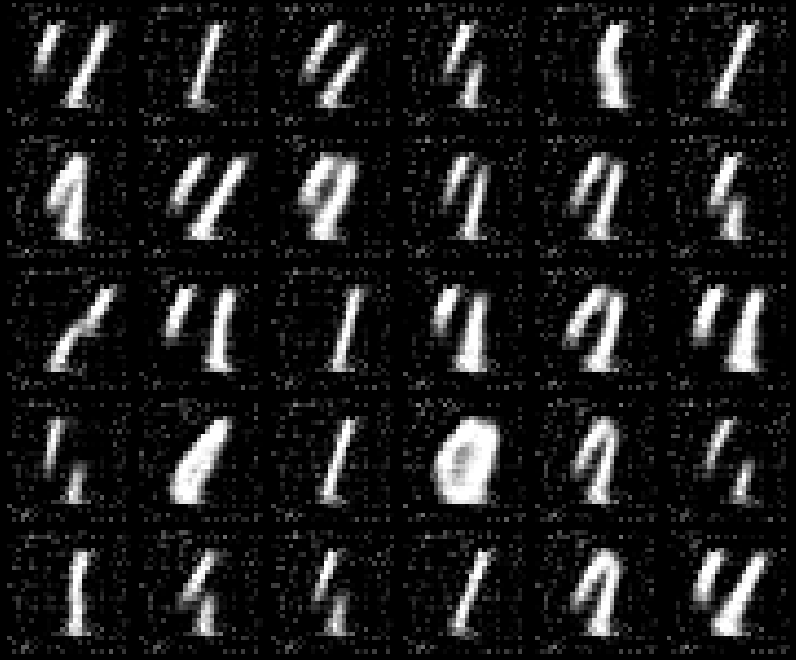}\\
    \vspace{0.2cm}
    \includegraphics[width=.32\linewidth]{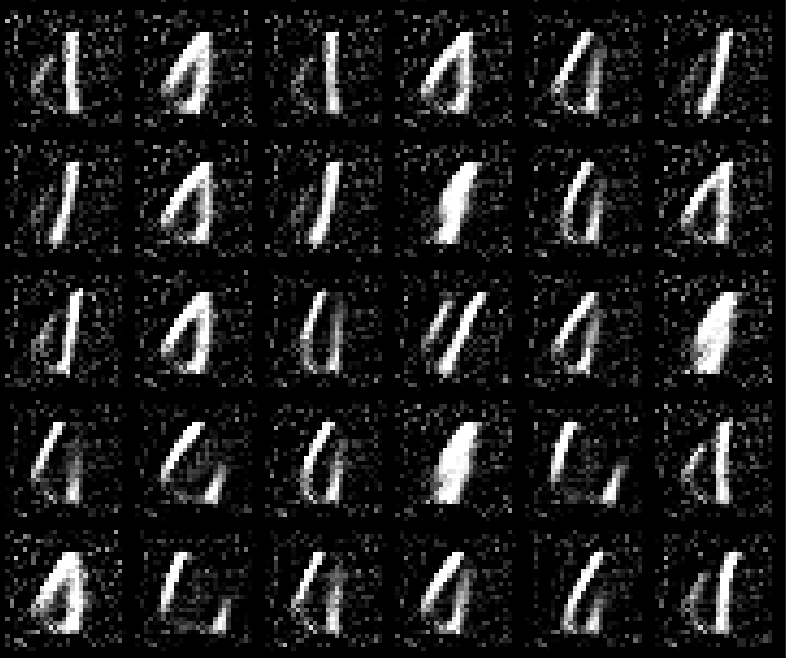}\hfill
    \includegraphics[width=.32\linewidth]{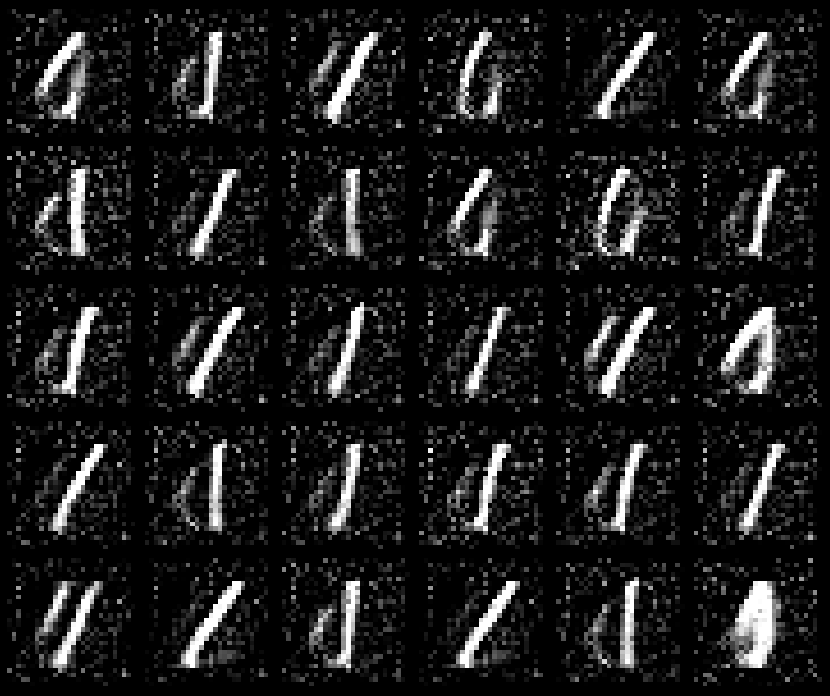}\hfill
    \includegraphics[width=.32\linewidth]{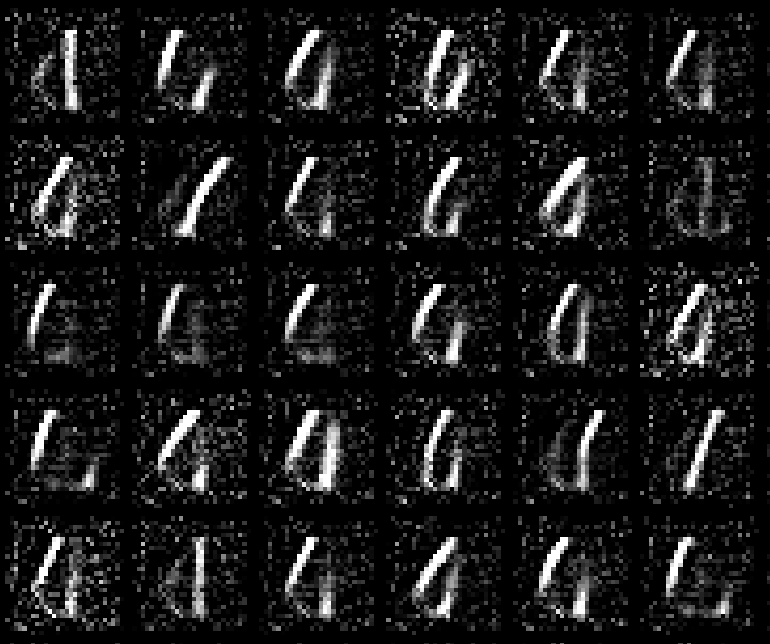}\\
    \caption{Tree DSB results for MNIST digits 0,1 and 4, after 10 mIPF cycles with $\varepsilon=0.5$. First row: samples from the reconstructed marginals, starting from MNIST-1. Second row: samples from the estimated barycenter, starting from MNIST-0. Third row: samples from the estimated barycenter, starting from MNIST-1.}
    \label{fig:mnist 0-1-4 a}
\end{figure}

\begin{figure}[h!]
    \centering
    \includegraphics[width=.32\linewidth]{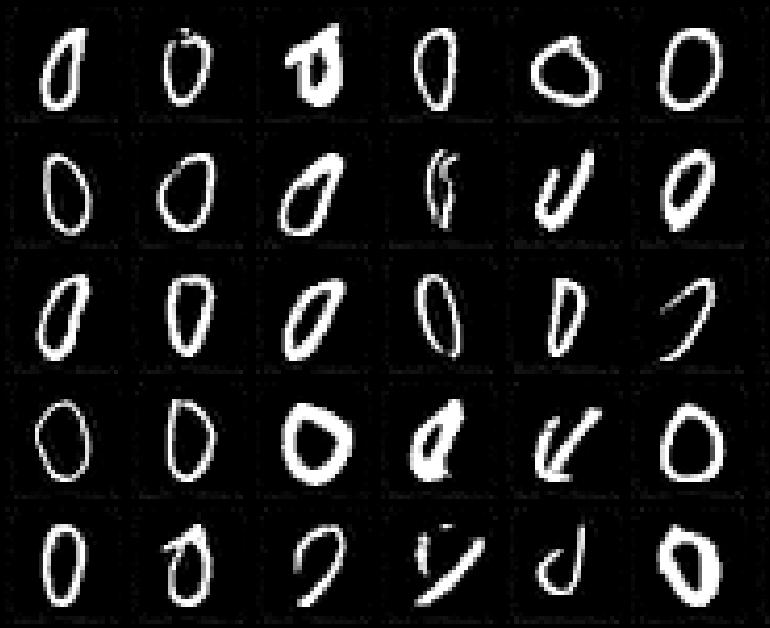}\hfill
    \includegraphics[width=.32\linewidth]{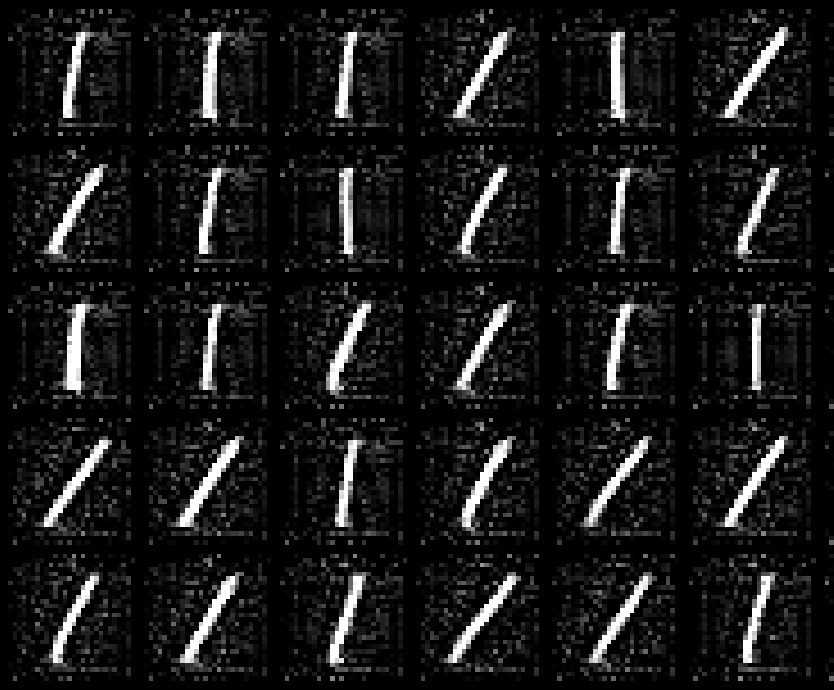}\hfill
    \includegraphics[width=.32\linewidth]{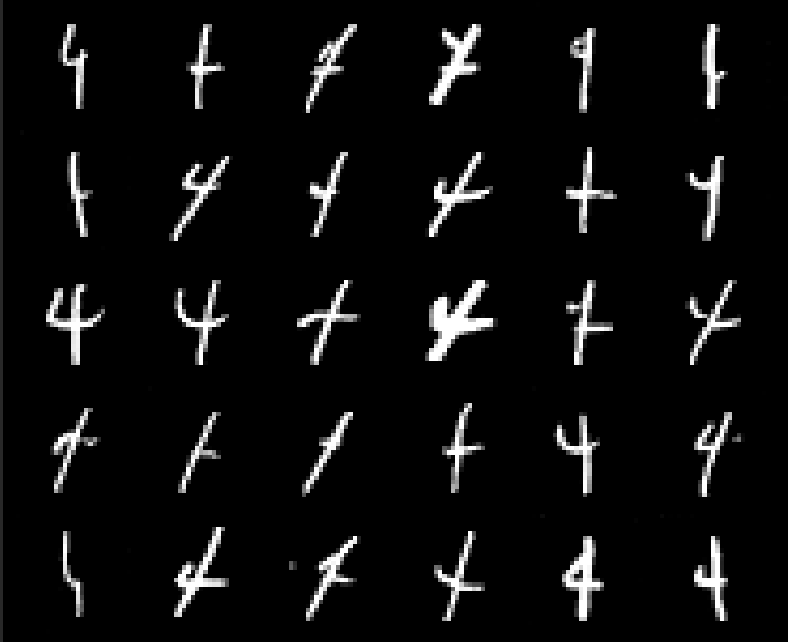}\\
    \vspace{0.2cm}
    \includegraphics[width=.32\linewidth]{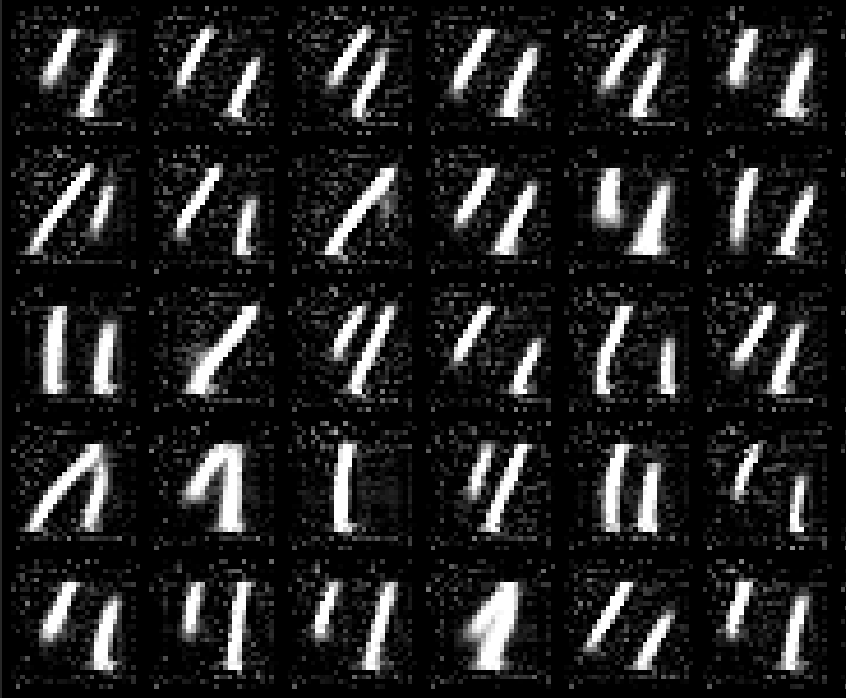}\hfill
    \includegraphics[width=.32\linewidth]{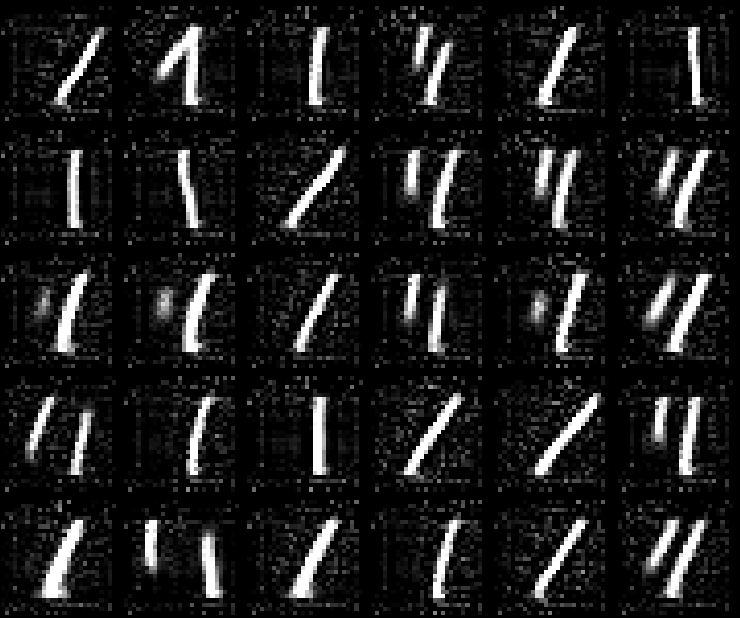}\hfill
    \includegraphics[width=.32\linewidth]{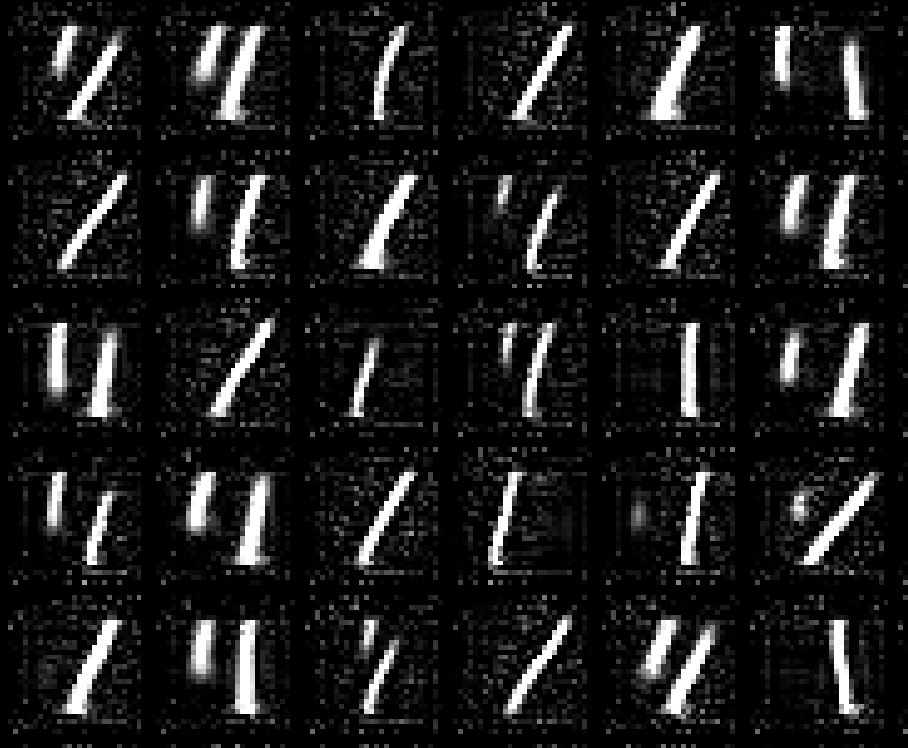}\\
    \vspace{0.2cm}
    \includegraphics[width=.32\linewidth]{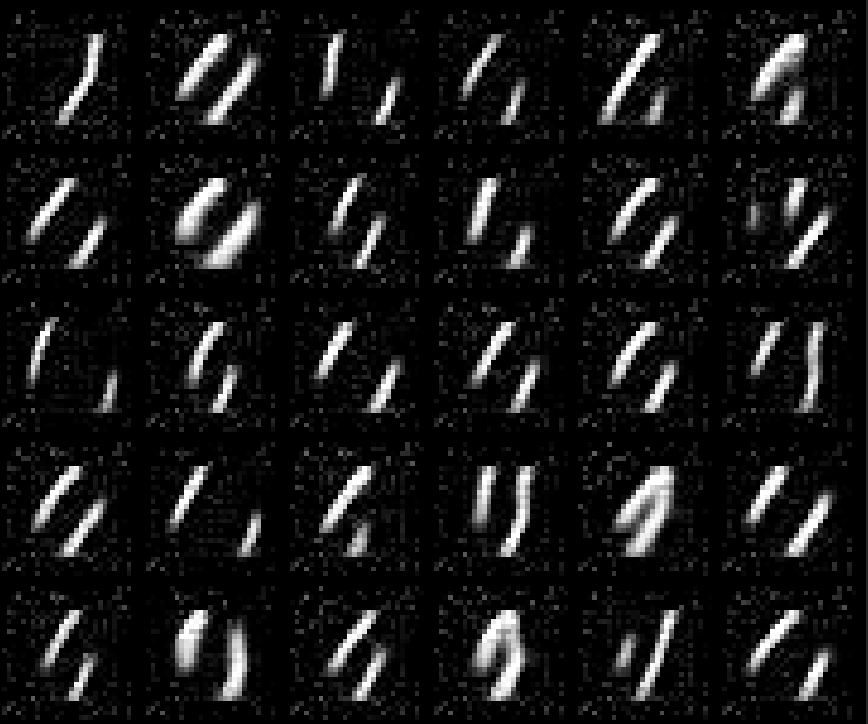}\hfill
    \includegraphics[width=.32\linewidth]{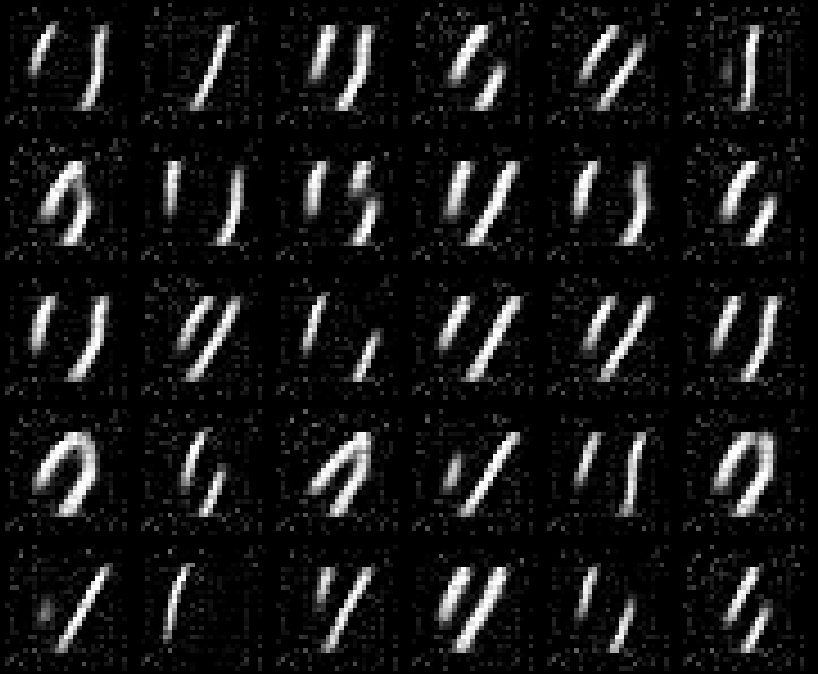}\hfill
    \includegraphics[width=.32\linewidth]{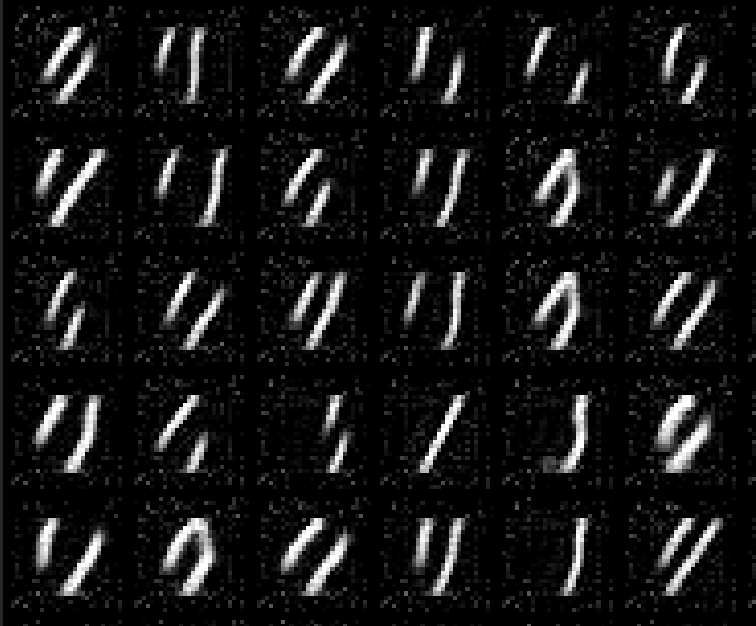}\\
    \caption{Tree DSB results for MNIST digits 0,1 and 4, after 10 mIPF cycles with $\varepsilon=0.2$. First row: samples from the reconstructed marginals, starting from MNIST-0. Second row: samples from the estimated barycenter, starting from MNIST-0. Third row: samples from the estimated barycenter, starting from MNIST-1.}
    \label{fig:mnist 0-1-4 b}
\end{figure}

\onecolumn


\end{document}